\documentclass[12pt]{report}
\pdfoutput=1
\usepackage{graphicx}
\graphicspath{{Images/}}
\usepackage{subcaption}
\usepackage[a4paper,margin=2.5cm]{geometry}
\usepackage{natbib}
\bibpunct{(}{)}{;}{a}{,}{,}
\usepackage[colorlinks=true,
            linkcolor=red,
            urlcolor=blue,
            citecolor=blue]{hyperref}

\usepackage[T1]{fontenc}
\usepackage[cp1250]{inputenc}
\usepackage{amsmath}
\usepackage{amsthm}
\usepackage{amssymb,graphics,pstricks}
\usepackage{multicol}
\usepackage{booktabs}
\usepackage{multirow}
\usepackage{verbatim}
\usepackage{mathrsfs}
\usepackage{algpseudocode}
\usepackage{fancyhdr,graphicx}
\usepackage[ruled,vlined]{algorithm2e}
\include{pythonlisting}

\def\thetitle{Structure Learning and Parameter Estimation \\for Graphical Models via Penalized Maximum\\Likelihood Methods}
\def\theauthor{Maryia Shpak}

\def\themonth{
April}
\def\theyear{
2022}
\def\thesupervisor{
dr~hab.~Mariusz Bieniek, prof.~UMCS}

\def\themonthyear{\themonth{} \theyear}

\author{\theauthor}
\title{\thetitle}

\def\titlepages{\newpage
\thispagestyle{empty}
\begin{centering}
\Large
Maria Curie-Sk\l{}odowska University in Lublin\\
\large
Faculty of Mathematics, Physics and Computer Science\\
\vspace{3.75cm}
\theauthor\\
\vspace{0.85cm} \LARGE
\thetitle\\
\vspace{0.3cm} \normalsize
\textit{PhD dissertation}\\
\vspace{5.6cm}
\begin{flushright}
Supervisor\\
\vspace{0.25cm}
\thesupervisor\\
\vspace{0.75cm}
Institute of Mathematics\\
University of Maria Curie-Sklodowska\\
\end{flushright}
\vfill
\themonthyear\\
\end{centering}
\newpage
}

\theoremstyle{plain}
\newtheorem{theorem}{Theorem}[chapter]
\theoremstyle{remark}
\newtheorem{model}{Model}

\def\be#1\ee{\begin{equation}#1\end{equation}}
\newcommand{\ba}{\begin{eqnarray} }
\newcommand{\ea}{\end{eqnarray} }

\theoremstyle{plain}
\newtheorem{proposition}[theorem]{Proposition} 
\newtheorem{corollary}[theorem]{Corollary}
\newtheorem{definition}[theorem]{Definition}

\newtheorem{lemma}[theorem]{Lemma}

\theoremstyle{remark}
\newtheorem{remark}[theorem]{Remark}
\newtheorem{example}[theorem]{Example}

\def\1c{\mathbb{I}_C(x)}

\def\ccB{{\cal B}}
\def\ccC{{\cal C}}
\def\ccD{{\cal D}}
\def\ccE{{\cal E}}
\def\ccF{{\cal F}}
\def\ccG{{\cal G}}

\def\ccK{{\cal K}}
\def\ccL{{\cal L}}
\def\ccM{{\cal M}}
\def\ccN{{\cal N}}

\def\ccP{{\cal P}}
\def\ccV{{\cal V}}

\def\ccS{{\cal S}}
\def\ccT{{\cal T}}
\def\X{{\cal X}}

\def\C{{\mathbf C}}
\def\S{{\mathbf S}}
\def\E{\mathcal{E}}

\def\II{{\mathbb I}}
\def\N{{\mathbb N}}

\def\R{{\mathbb R}}

\def\U{{\mathbf{U}}}
\def\u{{\mathbf{u}}}

\def\ccX{\mathcal{X}}

\def\For{\mbox{for all}}

\def\VV{\mathbf{V}}
\def\QQ{\mathbf{Q}}
\def\qq{\boldsymbol{q}}
\def\vv{\boldsymbol{v}}

\def\ee{\boldsymbol{e}}
\def\dd{\boldsymbol{d}}
\def\ww{\boldsymbol{w}}
\def\SS{\boldsymbol{S}}

\def\ss{\boldsymbol{s}}

\def\s{\mathbf{s}}
\def\tt{\boldsymbol{t}}
\def\eeta{\boldsymbol{\eta}}
\def\ttheta{\boldsymbol{\theta}}
\def\II{\mathbf{I}}
\def\CC{\mathbf{C}}

\def\WW{\mathbf{W}}
\def\OO{\mathbf{O}}
\def\AA{\mathbf{A}}
\def\HH{\mathbf{H}}
\def\aa{\boldsymbol{a}}
\def\YY{\mathbf{Y}}
\def\yy{\boldsymbol{y}}
\def\ZZ{\mathbf{Z}}
\def\zz{\boldsymbol{z}}
\def\EE{\mathbf{E}}
\def\intq[#1]{q_{#1}(\VV)}

\def\gr{{\mathcal{G}}}
\def\graph{{(\mathcal{V},\mathcal{E})}}
\def\parents{\mathbf{pa}_{\gr}}
\def\ppa{{\mathbf{pa}}}
\def\Pa{{\mathbf{Pa}}}

\def\edge(#1to#2){#1\rightarrow #2}

\def\Val(#1){\mathit{Val}(#1)}
\def\Card(#1){|#1|}

\def\PCond(#1on#2){\mathbb{P}(#1\mid#2)}
\def\Matrix[#1,#2,#3,#4,#5,#6,#7,#8,#9]{
\begin{bmatrix}  
#1 & #2 & \dots & #3 \\
#4 & #5 & \dots & #6 \\
\vdots & \vdots & \ddots & \vdots\\
#7 & #8 & \dots & #9
\end{bmatrix}
}
\newcommand{\pr}{\ensuremath p }

\DeclareMathOperator{\prox}{prox}
\DeclareMathOperator{\dist}{dist}
\DeclareMathOperator*{\argmin}{argmin}
\DeclareMathOperator*{\argmax}{argmax}
\def\pa{{\rm pa}}
\def\pax{{\rm pa}}

\def\tmax{t_{\rm max}}

\def\Ind{\mathbb{1}}

\def\V{\mathcal{V}}
\def\I{S}
\def\XX{\mathbf{X}}
\def\xx{\boldsymbol{x}}
\def\cone{\mathcal{C}}

\def\Pr{\mathbb{P}}
\def\Ex{\mathbb{E}}
\def\Ind{\mathbb{I}}
\def\bssw{\beta_{s,s\p}^{w}}
\def\tssw{\theta_{s,s\p}^{w}}
\def\lssw{\ell _{s,s\p}^{w}}
\def\e{{\rm e}}

\def\p{^\prime}

\def\d{{\rm d}}

\newcounter{commentcounter}

\newcommand{\commentblock}[1]{}

\begin{document}

   \pagenumbering{roman}
\titlepages

\chapter*{Abstract}
\addcontentsline{toc}{chapter}{Abstract}

\bigskip
\normalsize

Probabilistic graphical models (PGMs) provide a compact and flexible framework to model very complex real-life phenomena. They combine the probability theory which deals
with uncertainty and logical structure represented by a graph which allows to cope with the~computational complexity and also interprete and communicate the obtained know\-ledge. In the thesis we consider two different types of PGMs: Bayesian networks (BNs) which are static, and continuous time Bayesian networks which, as the name suggests, have temporal component. We are interested in recovering their true structure, which is the first step in learning any PGM. This is a challenging task, which is interesting in~itself from the causal point of view, for the purposes of interpretation of the model and the~decision making process. All approaches for structure learning in the thesis are united by the same idea of maximum likelihood estimation with LASSO penalty. The~problem of~structure learning is reduced to the~problem of finding non-zero coefficients in~the~LASSO estimator for a generalized linear model. In case of CTBNs we consider the problem both for complete and incomplete data. We support the theoretical results with experiments.

\bigskip

\textbf{Keywords and phrases: Probabilistic graphical models, PGM, Bayesian networks, BN, continuous time Bayesian networks, CTBN, maximum likelihood, LASSO penalty, structure learning, Markov Jump Process, MJP, Markov chain, Markov chain Monte Carlo, MCMC, Stochastic Proximal Gradient Descent, drift condition, incomplete data, Expectation-Maximization, EM.}

\bigskip

\chapter*{Acknowledgements}
\addcontentsline{toc}{chapter}{Acknowledgements}
Throughout the process of writing this thesis I have received a lot of support and assistance and I wish to express my gratitude.

First, I would like to thank my supervisor, Professor Mariusz Bieniek, who was a great support during this challenging process. His curiosity, open-mindedness and extensive knowledge gave me a chance to research things that are outside of his main field of expertise, and his strive for quality and perfection never let me settle for mediocre results.

Next, I want to thank my second advisor Professor B\l{}a\.zej Miasojedow from University of Warsaw, who introduced us to the field of probabilistic graphical models and some other areas of statistics, stochastic processes and numerical approximation. His great expertise and enormous patience allowed me to gain massive knowledge and understanding of these fields, when sometimes I did not believe I could.

I also would like to thank dr.~Wojciech Rejchel from Nicolaus Copernicus University in Toru\'{n} , whose expertise in model selection was key in the analysis of theoretical properties of our novel methods for structure learning. I wish to thank mgr.~Grzegorz Preisbich and mgr.~Tomasz C\k{a}ka\l{}a for making many numerical results for our methods possible.

I want to thank my university, Maria Sk\l{}odowska-Curie University, for an academic leave giving me the opportunity to finish the dissertation and some additional funding.
The part of the research was also supported by the Polish National Science Center grant: NCN contract with the number UMO-2018/31/B/ST1/00253. Also I would like to show my appreciation to other people from my university helping me in various ways, among them are Professor Maria Nowak, Professor Jaros\l{}aw Bylina, Professor Tadeusz Kuczumow, Professor Jurij Kozicki and many others.

Finally, I would like to thank my parents, Pavel and Natallia, who were always there for me to guide me and help me through years of research, without their support this thesis would not be possible.
I also wish to extend my special thanks to my dear friends for their emotional support and helping me to stay disciplined, especially I thank Elvira Tretiakova and Olga Kostina.


\setcounter{tocdepth}{1}

\tableofcontents

   \chapter{Introduction} \label{chapter: introduction}
\pagenumbering{arabic}
\section{Motivation}

 It is a common knowledge that we live in the world where data plays crucial role in many areas and applications of great importance for our society and the importance of data is still growing. The amount of data in the world is now estimated in dozens of zettabytes, and by 2025 the amount of data generated \textbf{daily} is expected to reach hundreds of exa\-bytes. There is a demand for models and algorithms that can deal with these amounts of data effectively finding useful patterns and providing better insights into the data. On top of it, most environments require reasoning under uncertainty. Probabilistic graphical models (PGMs) provide such a framework that allows to deal with these and many other challenges in various situations. The models combine the \textit{probability theory} which deals with uncertainty in a mathematically consistent way, and \textit{logical structure} which is represented by a graph encoding certain independence relationships among variables allowing to cope with the computational complexity.

PGMs encode joint distributions over a set of random variables (often of a significant amount) combining the graph theory and probabilities, which allows to represent many complex real-world phenomena compactly and overcome the complexity of the model which is exponential in the number of variables. There are also some other advantages that these models have. Namely, because of their clear structure, PGMs enable us to visualize, interprete and also communicate the gained knowledge to others as well as make decisions. Some models, for example Bayesian networks, have directed graphs in their core and offer ways to establish causality in various cases. Moreover, graphical models allow us not only to fit the observed data but also elegantly incorporate prior knowledge, e.g.~from experts in the domain, into the model. Besides, certain models take into account a temporal component and consider systems' dynamics in time.

Graphical models are successfully applied to a large number of domains such as image processing and object recognition, medical diagnosis, manufacturing, finance, statistical physics, speech recognition, natural language processing and many others. Let us briefly present here a few examples of various applications.

Bayesian networks, one of the PGMs considered in this thesis, are extensively used in the development of medical decision support systems helping doctors to diagnose patients more accurately. In the work by \cite{Wasyluk} the authors built and described a probabilistic causal model for diagnosis of liver disorders. In the domain of hepatology, inexperienced
clinicians have been found to make a correct diagnosis in jaundiced patients in less than 45\% of the cases. Moreover, the number of cases of liver disorders is on the rise and, especially at early stages of a disease, the correct diagnosis is difficult yet critical, because in many cases damage to the liver caused by an untreated disorder may be irreversible. As we already mentioned and as it is stressed out in the work above, a huge advantage that these models have is that they allow to combine existing frequency data with expert judgement within the framework as well as update themselves when the new data are obtained, for example patients data within a hospital or a clinic. What is also important in the medical diagnosis is that PGMs, Bayesian networks in particular, efficiently model simultaneous presence of multiple disorders, which happens quite often, but in many classification approaches the disorders are considered to be mutually exclusive. The overall model accuracy, as the authors \cite{Wasyluk} claim, seems to be better than that of beginning diagnosticians and reaches almost 80\%, which can be used for the diagnosis itself as well as the way to help new doctors to learn the strategy and optimization of the diagnosis process. A few other examples of the PGMs application in medical field are management of childhood malaria in Malawi (\cite{Malawi}),  estimating risk of coronary artery disease (\cite{CAD}), etc.

The next popular area of graphical models application is computational biology, for example Gene Regulatory Network (GRN) inference. GRN consists of genes or parts of genes, regulatory proteins and interactions between them and plays a key role in mediating cellular functions and signalling pathways in cells. Accurate inference of GRN for a specific disease returns disease-associated regulatory proteins and genes, serving as potential targets for drug treatment. \cite{Xuan20} argued that Bayesian inference is particularly suitable for GRNs as it is very flexible for large-scale data integration, because the main challenge of GRNs is that there exist hundreds of proteins and tens of thousands of genes with one protein possibly regulating hundreds of genes and their regulatory relationship may vary across different cell types, tissues, or diseases. Moreover, the estimation is more robust and easier to compare on multiple datasets. \cite{Xuan20} demonstrated this by applying their model to breast cancer data and identified genes relevant to breast cancer recurrence. As another example in this area, \cite{Sachs} used Bayesian network computational methods for derivation of causal influences in cellular signalling networks. These methods automatically elucidated most of the traditionally reported signalling relationships and predicted novel interpathway network causalities, which were verified experimentally. Reconstruction of such networks might be applied to understanding native-state tissue signalling biology, complex drug actions, and dysfunctional signalling in diseased cells.

The use of probability models is extensive also in computer vision applications. In their work \cite{Frey} advocate for the use of PGMs in the computer vision problems requiring decomposing the data into interacting components, for example, methods for automatic scene analysis. They apply different techniques in a vision model of multiple, occluding objects and compare their performances. Occlusion is a very important effect and one of the biggest challenges in computer vision that needs to be taken into account, and PGMs are considered to be a good tool to handle that effect. PGMs are also used for tracking different moving objects in video sequences, for example long-term tracking of groups of pedestrians on the street (\cite{JorgeLongTerm}), where the main difficulties concern total occlusions of the objects to be tracked, as well as group merging and splitting. Another example is on-line object tracking (\cite{JorgeOnline}) useful in real time applications such as video surveillance, where authors overcame the problem of needing to analyze the whole sequence before labelling trajectories to be able to use the tracker on-line and also the problem of unboundedly growing complexity of the network.

\section{Probabilistic Graphical Models}
In the previous subsection we described the advantages of PGMs and why one might be interested in studying them. In this work we focus on two types of PGMs: Bayesian Networks (BN) and Continuous Time Bayesian Networks (CTBN). The first term has rather long history and tracks back to 1980s (\cite{Pearl85}) whereas the second term is relatively modern (\cite{Nod1}).  The underlying structure for both models is a directed graph, which can be treated either as a representation of a certain set of
independencies or as a skeleton for factorizing a distribution. In some cases the directions of arrows in the graph can suggest causality under certain conditions and allow not only the inference from the data but also intervene into the model and manipulate desired parameters in the future. BNs are static models, i.e.~they do not consider a temporal component, while in CTBNs as the name suggests we study models in the context of continuous time. The framework of CTBNs is based on homogeneous Markov processes, but utilizes ideas from Bayesian networks to provide a graphical representation language for these systems.

A broad and comprehensive tutorial on existing research for learning Bayesian networks and some adjacent models can be found in \cite{daly2011}. The subject of causality is extensively explored in \cite{Spirtes2000} and \cite{Pearl2000}, some references are also given in \cite{daly2011}. Several examples of the use of BNs were presented above.

In contrast to regular Bayesian networks, CTBNs have not been studied that well yet. The most extensive work concerning CTBNs is PhD thesis of \cite{Nod4}. Some related works include for example learning CTBNs in non-stationary domains (\cite{Villa}), in relational domains (\cite{Yang2016}) and continuous time Bayesian network classifiers (\cite{STELLA2012}). As an example, CTBNs have been successfully used to model the presence of people at their computers together with their availability (\cite{horvitzNodelman}), for dynamical systems reliability modeling and analysis (\cite{boudali2006}), for network intrusion detection (\cite{XuShelton2008}), to model social networks (\cite{fan2012}), to model cardiogenic heart failure (\cite{gatti2012}), and for gene network inference (\cite{Stella} or \cite{Stella16}).

\section{Overview of the thesis and its contributions}
There are several problems within both the BN and CTBN frameworks. Both of them have graph structures which need to be discovered and this is considered to be one of the main challenges in the field. This thesis is dedicated exclusively to solving this problem in both frameworks. Another problem is to learn the parameters of the model: in the case of BNs it is a set of conditional probability distributions and in the case of CTBNs it is a set of conditional intensity matrices (for details see Chapter \ref{chapter: preliminaries}). The last problem is the statistical inference based on the obtained network (details are in Chapter \ref{chapter: inference}).

The thesis is constructed as follows. In Chapter \ref{chapter: preliminaries} we provide all the necessary preliminaries for better understanding the frameworks of Bayesian networks and continuous time Bayesian networks. Next, in Chapter \ref{chapter: inference} we overview known results on learning networks' parameters as well as inference to fully cover the concept of interest. Chapter~\ref{chapter: BN_structure} is dedicated to the structure learning problem for BNs, where we provide novel algorithms for both discrete and continuous data. Chapters \ref{chapter: CTBN structure} and Chapter \ref{chapter:CTBNpartial} cover the problems of structure learning for CTBNs in cases of complete and incomplete data, respectively. Finally, Chapter \ref{chapter:conclusions} concludes the thesis with the summary and the discussion of obtained results.

Algorithms in both Chapters \ref{chapter: BN_structure} and \ref{chapter: CTBN structure} lean on feature selection in generalized linear models with the use of LASSO (Least Absolute Shrinkage and Selection Operator) penalty function. It relies on the idea of penalizing the parameters of the model, i.e.~adding or subtracting the sum of absolute values of the parameters of the model with some hyperparameter, in order to better fit the model and perform a variable selection by forcing some parameters to be equal to 0. The term first appeared in \cite{tibshirani}. More on the topic of LASSO can be found for example in \cite{hastie}. In Section \ref{sec:lasso} we provide a short description of the concept.

The main contributions of the thesis are collected in Chapters \ref{chapter: BN_structure}, \ref{chapter: CTBN structure} and \ref{chapter:CTBNpartial} and they are as follows:
\begin{itemize}
    \item we provide the novel algorithm for learning the structure of BNs based on penalized maximum likelihood function both for discrete and continuous data;
    \item we present and prove the consistency results for the algorithm in case of continuous data;
    \item we compare the effectiveness of our method with other most popular methods for structure learning applied to benchmark networks of continuous data of different sizes;
    \item we provide the novel algorithm for learning the structure of CTBNs based on pe\-na\-li\-zed maximum likelihood function for \textbf{complete} data and present two theoretical consistency results with proofs;
    \item we provide the novel algorithm for learning the structure of CTBNs based on penalized maximum likelihood function for \textbf{incomplete} data where the log-likelihood function is replaced by its Markov Chain Monte Carlo (MCMC) approximation due to inability to express it explicitly;
    \item we present and prove the convergence of the proposed MCMC scheme and the consistency of the learning algorithm;
    \item for the mentioned above MCMC approximation we designed the algorithm to produce necessary samples;
    \item in both cases of complete and incomplete data we provide results of the simulations to show the effectiveness of proposed algorithms.
\end{itemize}

Part of the content (Chapter \ref{chapter: CTBN structure}) in its early stages has been published on arXiv:

Shpak, M., Miasojedow, B., and Rejchel, W., \textit{Structure learning for CTBNs via penalized maximum likelihood methods}, arXiv e-prints, 2020, 
\url{https://doi.org/10.48550/arXiv.2006.07648}.

   \chapter{Preliminaries}\label{chapter: preliminaries}

In this chapter we provide theoretical background on Bayesian networks (BNs), Markov processes, conditional Markov processes and continuous time Bayesian networks (CTBNs).
We start with the notation common for BNs and CTBNs which we will use through the~whole thesis. Then we provide a few basic definitions needed to define and understand the concepts of BNs and CTBNs with their interpretation and examples. Most of the~contents of this chapter comes from the \cite{Nod1}, \cite{Nod4}, \cite{koller2009}.


\section{Notation}

First, by upper case letters, for example, $X_i$, $B$, $Y$, we denote random variables. In the~case of CTBNs upper case letters represent the whole  collection of random variables indexed by continuous time, hence in this case $X_i(t)$, $Y(t)$ are random variables for particular time points $t$. 

Values of variables are denoted by lower case letters, sometimes indexed by numbers or otherwise representing different values of the same random variable - e.g.~$x_i$, $s$, $s'$. The~set of possible values for a variable $X$ is denoted by $\Val(X)$ and by $\Card(X)$ we will denote the~number of its elements.

Sets of variables are denoted by bold-face upper case letters - e.g. $\XX$ - and correspon\-ding sets of values are denoted by bold-face lower case letters - e.g. $\xx $ or $\mathbf{x}$. The set of possible values and its size is denoted by $\Val(\XX)$ and $\Card(\XX)$.

A pair $\ccG=\graph$ denotes a directed graph, where $\V$ is the set of nodes and $\E$ is the~set of edges. The notation $\edge(u to w)$ means that there exists an edge from the node~$u$~to the~node~$w$. We will also call them arrows. The set $\mathcal{V}\setminus\{w\}$ is denoted 
by $-w$. Moreover, we define the set of the parents of the node $w$ in the graph $\gr$ by 
\[\parents(w)=\{u\in\mathcal{V}\;:\;u\to w\}.\]
When there is no confusion, for convenience we sometimes write $\ppa(w)$ instead of~$\parents(w)$. Other useful and relevant locally notation we provide in the corresponding sections.

\section{Bayesian networks}\label{sec: Bayesian}
In this section we provide an overview of Bayesian networks (BNs). We start with the~intuition behind BNs followed by the representation of BNs together with its formal definition and notation. The problems of inference and learning for BNs are considered more thoroughly in Section \ref{sec: BN_inference} and Chapter \ref{chapter: BN_structure} respectively.

The goal is to represent a joint distribution \pr  over some set of random variables $\XX =\{X_1,\dotsc, X_n\}.$ Even in the simplest case where these variables are binary-valued, the joint distribution requires the specification of $2^n-1$ numbers - the probabilities of the $2^n$ different assignments of the values $\{x_1,\dotsc, x_n\}$. The explicit representation of the joint distribution is hard to handle from every perspective except for small values of~$n$. Computationally, it is very expensive to manipulate and generally too large to store in computer memory. Cognitively, it is impossible to acquire so many numbers from a human expert; moreover, most of the numbers would be very small and would correspond to events that people cannot reasonably consider. Statistically, if we want to learn the distribution from data, we would need ridiculously large amounts of data to estimate so many parameters robustly (\cite{koller2009}).

Bayesian networks help us specify a high-dimensional joint distribution compactly by exploiting its independence properties. The key notion behind the BN representation is \textit{conditional independence}, which on the one hand allows to reduce amount of estimated parameters significantly and on the other hand, allows to avoid very strong and naive independence assumptions.
\begin{definition}
Two random variables $X$ and $Y$ are \textbf{independent} (denoted by $X\perp Y$) if and only if the equality
\[
\Pr(X\in A, Y \in B) = \Pr(X\in A)\Pr( Y \in B) 
\]
holds for all Borel sets $A,B\subseteq\R$.
\end{definition}
For short, we will write it in the form $\Pr(X,Y) = \Pr(X)\Pr(Y).$ There is also another way to think of independence. If the random variables $X$ and $Y$ are independent, then $\PCond(X\in \cdot on Y) = \Pr(X\in \cdot)$. Intuitively, this says that having evidence about $Y$ does not change the distribution of our beliefs on the occurrence of $X$.

If we wish to model a more complex domain represented by some set of variables, it is unlikely that any of the variables will be independent of each other. Conditional independence is a weaker notion of independence, but it is more common in real-life situations.

\begin{definition}
Two random variables $X$ and $Y$ are \textbf{conditionally independent} given a set of random variables $\CC$ (symbolically $X\perp Y\mid\CC$) if and only if
\begin{equation}\label{eq:cond_ind}
    \PCond(X\in A, Y \in B on \CC) = \PCond(X\in A on \CC)\PCond(Y \in B on \CC)
\end{equation}
holds for all Borel sets $A,B\subseteq\R$.
\end{definition}
Obviously (\ref{eq:cond_ind}) implies
\[
\PCond(X\in A on \CC, Y) = \PCond(X\in A on \CC),
\]
which can be written shortly as 
\[
\PCond(X on \CC, Y) = \PCond(X on \CC).
\]

So intuitively, the influence that $X$ and $Y$ have on each other is mediated through the variables in the set $\CC$. It means that, when we have some evidence about variables from $\CC$, having any additional information about $Y$ does not change our beliefs about $X$. Let us demonstrate this definition on a simplified example. Let $X$ be a random variable representing the case if a person has lung cancer and $Y$ representing the case if the same person has yellow teeth. These variables are not independent as having yellow teeth is one of the secondary symptoms of lung cancer. However, when we know that the person is a smoker knowing that they have yellow teeth does not give us any additional insight on lung cancer, and vice versa, as we consider smoking to be the reason of both symptoms.


It is easier to believe that in a given domain most variables will not directly affect most other variables. Instead, for each variable only a limited set of other variables influence~it. This is the intuition which leads to the notion of a Bayesian network $\ccB$ over a set of random variables $\mathbf{B}$ which is a compact representation of a specific joint probability distribution. The formal definition is as follows.

\begin{definition}\label{BNdef}
A Bayesian network $\ccB$ over a set of random variables $\mathbf{B}$ is formed by
\begin{itemize}
    \item a directed acyclic graph (DAG) $\gr$ whose nodes correspond to the random variables $B_i\in \mathbf{B}$, $i = 1,\dots,n.$
    \item the set of conditional probability distributions (CPDs) for each $B_i$, specifying the conditional distribution $\PCond(B_i on \parents(B_i))$ of $B_i$ as a function of its parent set in $\gr$.
\end{itemize}
\end{definition}

The CPDs form a set of local probability models that can be combined to describe
the full joint distribution over the variables $\mathbf{B}$ via the chain rule:
\begin{equation}\label{eq:cpds}
    \Pr(B_1, B_2,..., B_n) = \prod_{i=1}^n \PCond(B_i on \parents(B_i)).
\end{equation}

The graph $\gr$ of a Bayesian network encodes a set of conditional independence assumptions. In particular, a variable $B\in\mathbf{B}$ is independent of its non-descendants given the set of its parents $\parents(B)$. See for example Figure \ref{fig:student} of an Extended Student network taken from \cite{koller2009}. As it can be seen, each variable is connected only to a small amount of other variables in the network. In this example according to (\ref{eq:cpds}) the joint distribution takes the following form:
\begin{equation*}
\begin{split}
   & \Pr(C, D, I, G, S, L, J, H) =\\ & = \Pr(C)\PCond(D on C)\Pr(I) \PCond(G on D,I) \PCond(S on I)\PCond(L on G)\PCond(J on L, S)\PCond(H on G,J).
\end{split}
\end{equation*}
This example will be considered in more detail further in the thesis.

\begin{figure}[ht]
\begin{center}
\includegraphics[width=0.45\textwidth]{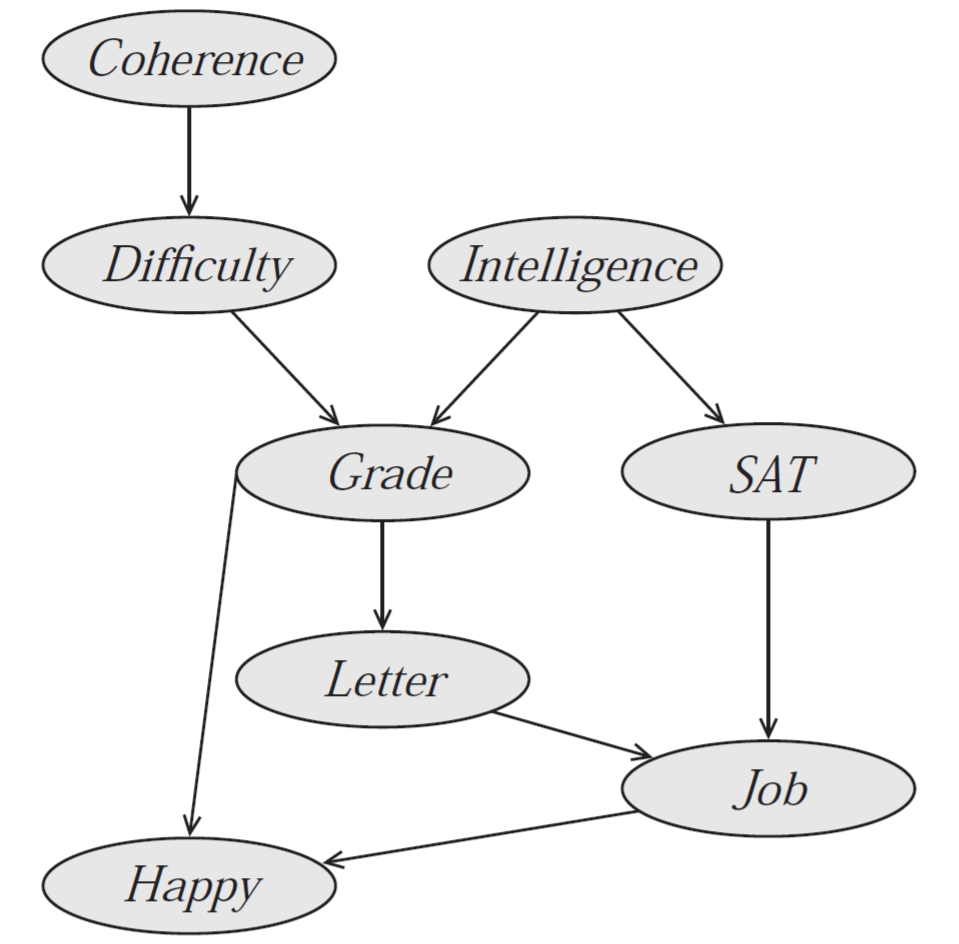}
\caption{The Extended Student network}
\label{fig:student}
\end{center}
\end{figure}

Now we discuss basic structures for BNs including some examples and give the interpretation of the structures. BNs represent probability distributions that can be formed via products of smaller, local conditional probability distributions (one for each variable). If the joint distribution is expressed in this form, it means that the independence assumptions for certain variables are introduced into our model. To understand what types of independencies are described by directed graphs for simplicity let us start from looking at BN $\ccB$ with three nodes: $X,$ $Y,$ and~$Z$. In this case, $\ccB$ essentially has only three possible structures, each of which leads to different independence assumptions.

\begin{itemize}
    \item \textit{Common parent,} also called \textit{common cause.} If $\gr$ is of the form $X\leftarrow Y\rightarrow Z$, and $Y$ is observed, then $X\perp Z\mid Y$. However, if $Y$ is unobserved, then $X\not\perp Z $. Intuitively this stems from the fact that $Y$ contains all the information that determines the outcomes of $X$ and $Z$; once it is observed, there is nothing else that affects these variables' outcomes. The case with smoking and lung cancer described above is such an example of common cause. See the illustration (c) in Figure \ref{fig:4str}.
    \item \textit{Cascade, or indirect connection.} If $\gr$ is of the form $X\rightarrow Y\rightarrow Z$, and $Y$ is observed, then, again $X\perp Z\mid Y$. However, if $Y$ is unobserved, then $X\not\perp Z $. Here, the intuition is again that $Y$ holds all the information that determines the outcome of $Z$; thus, it does not matter what value $X$ takes. In Figure \ref{fig:4str} in (a) and (b) there are shown cases of indirect causal and indirect evidential effects, respectively.
    \item \textit{$V$-structure or common effect}, also known as \textit{explaining away}. If $\gr$ is of the form $X\rightarrow Y\leftarrow Z$, then knowing $Y$ couples $X$ and $Z$. In other words, $X\perp Z$ if $Y$ is unobserved, but $X\not\perp Z\mid Y$ if $Y$ is observed. See the case (d) in Figure \ref{fig:4str}.
\end{itemize}
\begin{figure}[!ht]
\begin{center}
\includegraphics[width=0.7\textwidth]{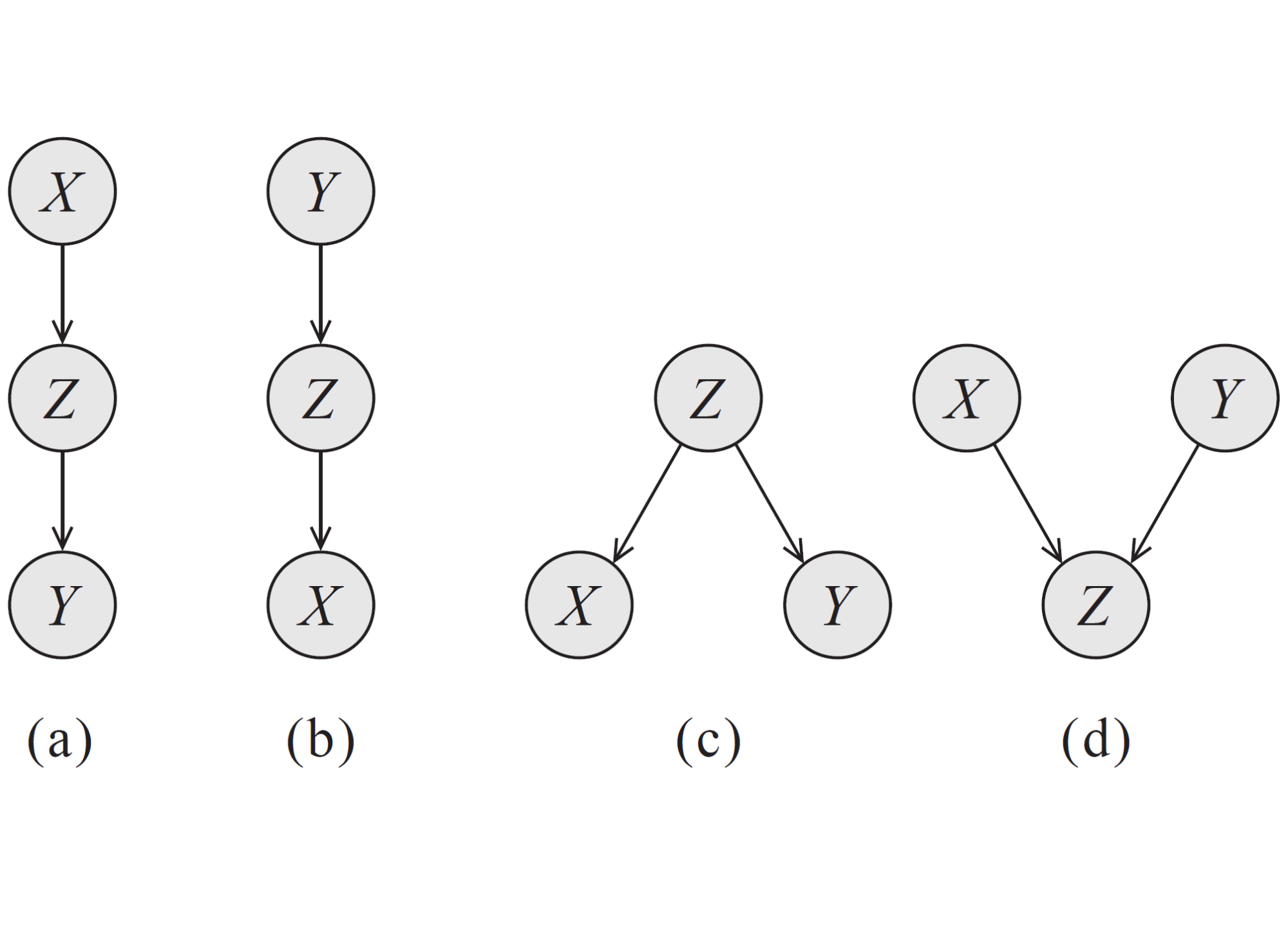}
\caption{The four possible two-edge trails from $X$ to $Y$ via $Z$: (a) An indirect causal effect; (b) An indirect evidential effect; (c) A common cause; (d) A common effect.}
\label{fig:4str}
\end{center}
\end{figure}

The last case requires additional explanation. Suppose that $Y$ is a Boolean variable that indicates whether our lawn is wet one morning; $X$ and $Z$ are two explanations for it being wet: either it rained (indicated by $X$), or the sprinkler turned on (indicated by $Z$). If we know that the grass is wet ($Y$ is true) and the sprinkler did not go on ($Z$ is false), then the probability that $X$ is true must be one, because that is the only other possible explanation. Hence, $X$ and $Z$ are not independent given $Y$.

To generalize this for a case of more variables and demonstrate the power but also the limitations of Bayesian networks we will need the notions of $d$-separation and $I$-maps. Let $\QQ, \WW,$ and $\OO$ be three sets of nodes in a Bayesian network $\ccB$ represented by $\ccG$, where the variables $\OO$ are observed. Let us use the notation $I(p)$ to denote the set of all independencies of the form $(\QQ\perp \WW\mid \OO)$ that hold in a joint distribution $p$. 
To extend structures mentioned above to more general networks we can apply them recursively over any larger graph, which leads to the notion of $d$-separation.

Recall that we say that there exists an undirected path in $\gr$ between the nodes $u$ and~$w$ if there exists the sequence $v_1,\dots,v_n\in\ccV$ such that $v_i\rightarrow v_{i+1}$ or $v_i\leftarrow v_{i+1}$  for each $i=0,1,\dots,n$, where $v_0 = u$ and $v_{n+1} = w$. Moreover, an undirected path in $\ccG$ between $Q\in\QQ$ and $W\in\WW$ is called \textit{active} given observed variables $\OO$ if for every consecutive triple of variables $X, Y, Z$ on the path, one of the following holds:

\begin{itemize}
    \item \textit{common cause:} $X\leftarrow Y\rightarrow Z$ and $Y\notin \OO$ ($Y$ is unobserved);
    \item \textit{causal trail:} $X\rightarrow Y\rightarrow Z$ and $Y\notin \OO$ ($Y$ is unobserved);
    \item \textit{evidential trail:} $X\leftarrow Y\leftarrow Z$ and $Y\notin \OO$ ($Y$ is unobserved);
    \item \textit{common effect:} $X\rightarrow Y\leftarrow Z$ and $Y$ or any of its descendants are observed.
\end{itemize}

\begin{figure}
  \centering
  \begin{tabular}{c c}
       \includegraphics[width=0.5\linewidth]{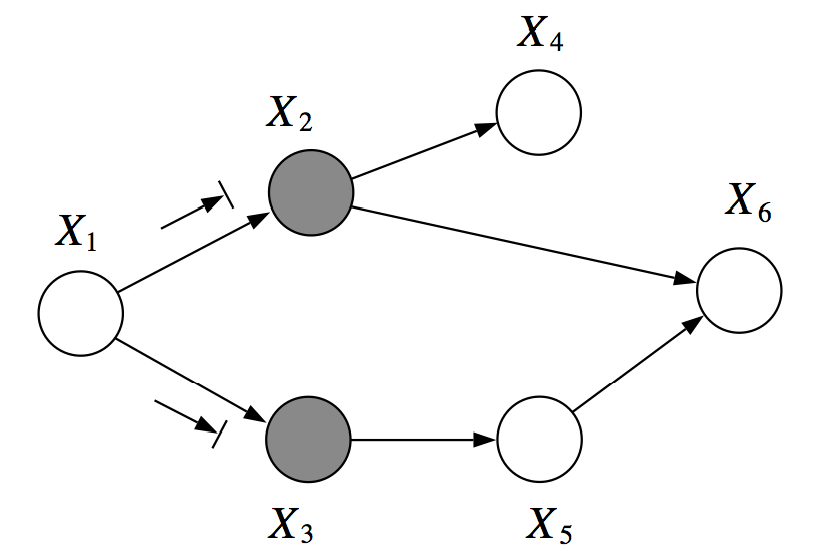} 
       & \includegraphics[width=0.5\linewidth]{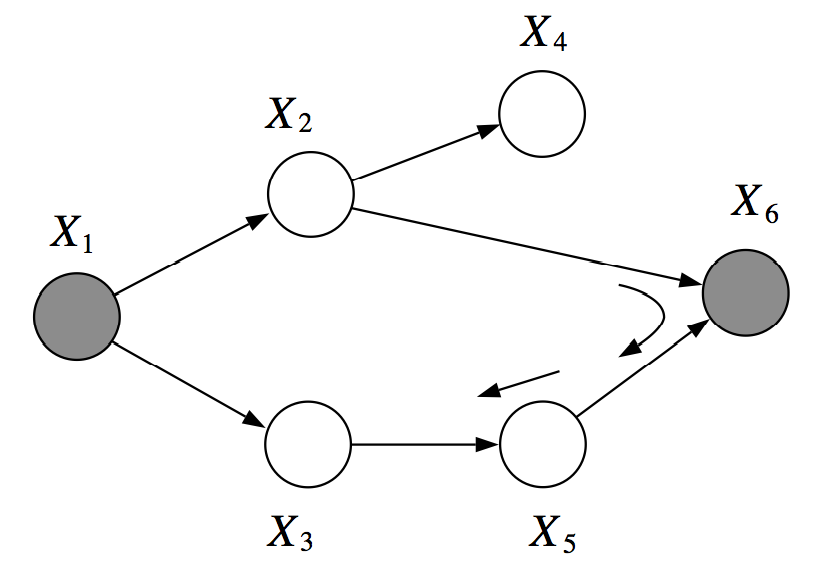}
  \end{tabular}
    \caption{An example for $d$-separation: $X_1$ and $X_6$ are $d$-separated given $X_2$,$X_3$ (left), $X_2, X_3$  are \textbf{not} $d$-separated given $X_1, X_6$ (right).}
    \label{fig:dsep}
\end{figure}
Finally, we say that $\QQ$ and $\WW$ are $d$-separated given $\OO$ if there are no active paths bet\-ween any node $A\in \QQ$ and $B \in \WW$ given $\OO$. See examples for $d$-separation in Figure~\ref{fig:dsep}. In the second example there is no $d$-separation because there is an active path which passes through the $V$-structure created when $X_6$ is observed. The notion of $d$-separation lets us describe a large fraction of the dependencies that hold in our model. It can be shown that if $\QQ$ and $\WW$ are $d$-separated given $\OO$, then $\QQ\perp\WW\mid\OO$.

We will write $I(\ccG) = \{(\QQ\perp \WW\mid \OO): \QQ,\WW \text{ are } d\text{-separated given } \OO  \}$ to denote the set of independencies corresponding to all $d$-separations in $\ccG$. If $p$ factorizes over $\ccG$, then $I(\ccG)\subseteq I(p)$ and $p$ can be constructed easily. In this case, we say that $\ccG$ is an $I$-map for $p$. In other words, all the independencies encoded in $\ccG$ are sound: variables that are $d$-separated in $\ccG$ are conditionally independent with respect to $p$. However, the converse is not true: a distribution may factorize over $\ccG$, yet have independencies that are not captured in $\ccG$.

So an interesting question here is whether for the probability distribution $p$ we can always find a \textit{perfect} map $I(\ccG)$ for which $I(\ccG) = I(p)$ or not. The answer is no (see an~example from \cite{koller2009}). Another related question is whether perfect maps are unique when they exist. This is not the case either, for example, DAGs $X\rightarrow Y$ and $X\leftarrow Y$ encode the same independencies, yet form different graphs. In~a~ge\-ne\-ral case we say that two Bayesian networks $\ccB_1$, $\ccB_2$ are $I$-equivalent if their DAGs encode the same dependencies $I(\ccG_1)=I(\ccG_2).$ For a case of three variables we can notice that graphs~(a), (b) and (c) in Figure \ref{fig:4str} encode the same dependencies, so as long as we do not turn graphs into $V$-structures ((d) is the only structure which encodes the de\-pen\-den\-cy~$X\not\perp Y\mid Z$) we can change directions in them and get $I$-equivalent graphs. This brings us to a fact that if~$\ccG_1, \ccG_2$ have the same skeleton (meaning that if we drop the directionality of the arrows, we obtain the same undirected graph) and the same $V$-structures, then $I(\ccG_1)=I(\ccG_2)$. For the full proof of this statement, other previously made statements and more information about BNs see \cite{koller2009}.

\section{Continuous Time Markov Processes}\label{sec:Markov_Processes}
In this section we collect auxiliary results on Markov processes with continuous time. We can think of a continuous time random process $X$ as a collection of random
variables indexed by time $t\in [0,\infty)$. It is sometimes more convenient to view $X$ across all values of~$t$ as a single variable, whose values are functions of time, also called paths or trajectories.
\begin{definition}
The Markov condition is the assumption that the future of a process is independent of its past given its present. More explicitly, the process $X$
satisfies the~Markov property iff $\PCond(X(t+\Delta t) on X(s),0\leq s\leq t) = \PCond(X(t+\Delta t) on X(t))$ for all $t,\Delta t >0$ (\cite{Chung}).
\end{definition}

In this thesis we focus on Markov processes with finite state space which are basically defined by initial distribution and a matrix of transition intensities. The framework of CTBNs is based on the notion of \textit{homogeneous Markov processes} in which the transition intensities do not depend on time.

\begin{definition}
Let $X$ be a stochastic process with continuous time. Let the state space of~$X$ be $Val(X) = \{x_1,x_2,...,x_N\}$. Then $X$ is a homogeneous
Markov process if and only if its behavior can be specified in terms of an initial distribution $P^X_0$over $Val(X)$ and a~Markovian transition model usually presented as an intensity matrix
\begin{equation}\label{eq:Q}
   \QQ_{X} = \Matrix[-q_1, q_{12}, q_{1N}, q_{21}, -q_2, q_{2N}, q_{N1}, q_{N2},-q_{N}], 
\end{equation}
where $q_i = \sum_{j\neq i}q_{ij}$ and all the entries $q_i$ and $q_{ij}$ are positive.
\end{definition}

Intuitively, the intensity $q_i$ gives the ``instantaneous probability'' of leaving state $x_i$ and the intensity $q_{ij}$ gives the ``instantaneous probability'' of the jump from $x_i$ to $x_j$. More formally, for $i\neq j$
\begin{equation}\label{eq:intensities}
    \lim_{\Delta t\to 0}\PCond(X(t+\Delta t) = x_j on  X(t) = x_i)  = q_{ij}\Delta t + O(\Delta t^2),
 \end{equation}
 and for all $i = 1,\dots,N$
\begin{equation}\label{eq:intensities2}
	\lim_{\Delta t\to 0}\PCond(X(t+\Delta t) = x_i on X(t) = x_i)  = 1-q_{i}\Delta t + O(\Delta t^2).
\end{equation}
Therefore, the matrix $\QQ_X$ describes the instantaneous behavior of the process $X$ and also makes the process satisfy the Markov assumption since it is defined solely in terms of its current state.

The instantaneous specification of the transition model of $X$ induces a probability distribution over the set of its possible trajectories. To see how the distribution is induced, we must first recall the notion of a matrix function.
\begin{definition}
The matrix exponential for a matrix $\QQ$ is defined as
\begin{equation*}
\exp \QQ = \sum_{k=0}^{\infty}\frac{\QQ^k}{k!}.
\end{equation*}
\end{definition}

Now the set of Equations (\ref{eq:intensities}) and \eqref{eq:intensities2} can be written collectively in the form
\begin{equation}\label{eq:exp}
    \lim_{\Delta t\to 0}\PCond(X(t+\Delta t) on X(t)) =  \lim_{\Delta t\to 0}\exp(\QQ_X\Delta t) =  \lim_{\Delta t\to 0} \left(\II + \QQ_X\Delta t+O(\Delta t^2)\right).
\end{equation}
So given the matrix $\QQ_X$ we can describe the transient behavior of $X(t)$ as follows. If $X(0) = x_i$ then the process stays in state $x_i$ for an amount of time exponentially distributed with parameter $q_i$. Hence, the probability density function $f$ and the cor\-res\-ponding distribution function $F$ for the time when $X(t)$ remains equal to $x_i$ are given~by
\begin{equation*}
    \begin{split}
        f(t) & = q_i\exp(-q_it), \quad t\geq 0,\\
        F(t) & = 1-\exp(-q_it), \quad t\geq 0.
    \end{split}
\end{equation*}
The expected time of changing the state is $1/q_i$. Upon transitioning, $X$ jumps to the state~$x_j$ with probability $q_{ij}/q_i$ for $j\neq i$.

\begin{example}
Assume that we want to model the behavior of the barometric pressure $B(t)$ discretized into three states ($b_1 =$ falling, $b_2 =$ steady, and $b_3 =$ rising). Then for instance we could write the intensity matrix as 
	\[ \QQ_B = 
	\begin{bmatrix}
	-0.21&0.2 & 0.01\\
	0.05& -0.1& 0.05\\
	0.01&0.2&-0.21
	\end{bmatrix}.
	\]
If we view units of time as hours, this means that if the pressure is falling, we expect
that it will stop falling in a little less than 5 hours (1/0.21 hours). It will then transition
to being steady with probability $0.2/0.21\approx 0.95$ and to falling with probability $0.01/0.21\approx 0.0476$.
\end{example}

When the transition model is defined solely in terms of an intensity matrix (as above), we refer to it as using \textit{a pure intensity} parameterization. The parameters for an $N$ state process are $\{ q_i, q_{ij}\in\QQ_{X},\ 1\leq i,j\leq N, \ i\neq j\}.$

This is not the only way to parameterize a Markov process. Note that the distribution over transitions of $X$ factors into two pieces: an exponential distribution over \textit{when} the next transition will occur and a multinomial distribution over \textit{where} the process jumps. This is called a \textit{mixed intensity} parameterization.

\begin{definition}\label{def:mixed_intensity}
The mixed intensity parameterization for a homogeneous Markov process~$X$ with $N$ states is given by two sets of parameters
\[
\qq_X = \{ q_i,\ 1\leq i\leq N\}
\]
and
\[
\ttheta_X = \{ \theta_{ij},\ 1\leq i,j\leq N, \ i\neq j\},
\]
where $\qq_X$ is a set of intensities parameterizing the exponential distributions over \textbf{when} the next transition occurs and $\ttheta_X$ is a set of probabilities parameterizing the distribution over \textbf{where} the process jumps.
\end{definition}
To relate these two parametrizations we note the following theorem from \cite{Nod4}.

\begin{theorem}\label{thm:mixed_intensity}
Let $X$ and $Y$ be two Markov processes with the same state space and the same initial distribution. If $X$ is defined by the intensity matrix $\QQ_X$ given by \eqref{eq:Q}, and~$Y$ is the process defined by the mixed intensity parameterization $\qq_Y = \{q'_1,\ldots,q'_N\}$ and $\ttheta_Y = \{\theta'_{ij}, i\neq j\}$, then $X$ and $Y$ are stochastically equivalent, meaning they have the same state space and transition probabilities, if and only if $q'_i = q_i$ for all $i=1,\dots,N$ and
\[
\theta'_{ij}=\frac{q_{ij}}{q_i}
\]
for all $1\leq i,j\leq N, \ i\neq j$.
\end{theorem}

\section{Conditional Markov Processes}

In order to compose Markov processes in a larger network, we need to introduce the notion of a conditional Markov process. This is an inhomogeneous Markov process where the~in\-ten\-sities vary with time, but not as a direct function of time. Rather, the intensities depend on the current values of a set of other variables, which also evolve as Markov processes.

Let $Y$ be a process with a state space $\Val(Y) = \{y_1, y_2,..., y_m\}$. Assume that $Y$ evolves as a Markov process $Y(t)$ whose dynamics are conditioned on a set $\VV$ of variables, each of which can also evolve over time. Then we have a conditional intensity matrix (CIM) which can be written as

\[
\QQ_{Y|\VV} = \Matrix[-\intq[1], \intq[12], \intq[1m], \intq[21], -\intq[2], \intq[2m], \intq[m1], \intq[m2], -q_{m}(\VV)].
\]
Equivalently, we can view CIM as a set of intensity matrices $\QQ_{Y|\vv}$ one for each instantiation of values $\vv$ to the variables $\VV$, see Example \ref{ex: CIM}. Since the framework of CTBNs which we consider in the thesis has a graph at its core, we will refer to the set of variables~$\VV$ as the set of parents of $Y$ and denote it by $\parents(Y)$. Note that if the parent set~$\parents(Y)$ is empty, then CIM is simply a standard intensity matrix. Just as a regular intensity matrix, CIM induces the distribution of the dynamics of $Y$ given the behavior of $\parents(Y) = \VV$. If~$\VV$ takes the value $\vv$ on the interval $[t, t + \varepsilon)$ for some $\varepsilon > 0$, then as in Equation (\ref{eq:exp})
\begin{equation*}
\lim_{\Delta t\to 0}\PCond(Y_{t+\Delta t} on Y_t,\vv) =  \lim_{\Delta t\to 0}\exp(\QQ_{Y|\vv}\Delta t) =  \lim_{\Delta t\to 0} \left(\II + \QQ_{Y|\vv} \Delta t+O(\Delta t^2)\right).
\end{equation*}
If we specify an initial distribution of $Y$, then we have defined a Markov process whose behavior depends on the instantiation  $\vv$ of values of $\parents(Y)$.

\begin{example}\label{ex: CIM}
	Consider a variable $E(t)$ which models whether or not a person is eating ($e_1 =$ not eating, $e_2 =$ eating) conditioned on a variable $H(t)$ which	models whether or not the person is hungry ($h_1 =$ not hungry, $h_2 =$ hungry). Then we	can specify exemplary CIM for $E(t)$ as
\begin{multicols}{2}
$Q_{E|h_1}=
\begin{bmatrix}
-0.01&0.01\\
10&-10
\end{bmatrix}$

$Q_{E|h_2}=
\begin{bmatrix}
-2&2\\
0.01&-0.01
\end{bmatrix}$.
\end{multicols}
For instance, given this model, we expect that a person who is hungry and not eating is going to start eating in half an hour. Also, we expect a person who is not hungry and is eating to stop eating in 6	minutes (1/10 hour).
\end{example}

\section{Continuous time Bayesian networks}\label{sec:CTBN}

In this section we define the notion of CTBN, which in essence is a probabilistic graphical model with the nodes as variables, the state evolving continuously over time, and where the evolution of each variable depends on the state of its parents in the graph.

Before the formal definition we recall an example from \cite{Nod1}. Consider the situation in medical research where some drug has been administered to a patient and we wish to know how much time it takes for the drug to have an effect. The answer to this question will likely depend on various factors, such as how recently the patient ate. We want to model the temporal process for the effect of the drug and how its dynamics depends on other factors. In contrast to previously developed methods of approaching such a problem (e.g.~event history analysis, Markov process models) the notion of CTBN introduced by \cite{Nod1} allows the specification of models with a large structured state space where some variables do not directly depend on others. For example, the distribution of how fast the drug takes effect might be mediated through how fast it reaches the bloodstream, which in turn may be affected by how recently the person ate. Figure \ref{fig:ctbn} shows an exemplary graph structure for CTBN modelling the drug effect. There are nodes for the uptake of the drug and for the resulting concentration of the drug in the bloodstream. The concentration is also affected by how full patient's stomach is. The drug is supposed to alleviate joint pain, which may be aggravated by falling pressure. The drug may also cause drowsiness. The model contains a cycle, indicating that whether the person is hungry depends on how full their stomach is, which depends on whether or not they are eating.

\begin{figure}[!ht]
\begin{center}
\includegraphics[width=0.45\textwidth]{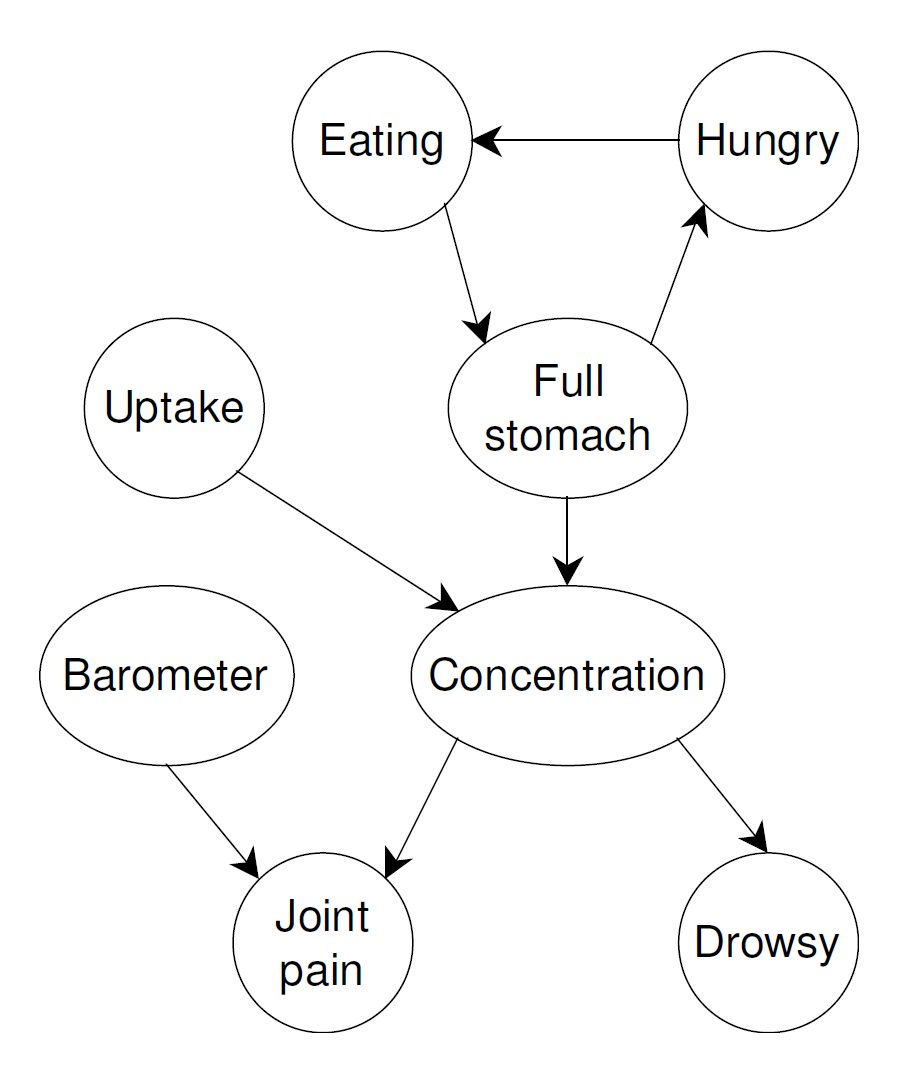}
\caption{(a)}
\label{fig:ctbn}
\end{center}
\end{figure}

Let $\gr = (\mathcal{V},\mathcal{E})$ denote a directed graph \textbf{with possible cycles}, where $\mathcal{V}$ is the set of nodes and~$\mathcal{E}$ is the set of edges. Further in the context of probabilistic graphical models we use the terms ``nodes'' and ``random variables'' interchangeably. For every $w\in\mathcal{V}$ we consider a corresponding space  $\mathcal{X}_w$ of possible states at $w$ and we  
assume that each space $\mathcal{X}_w$ is finite.
We consider a continuous time stochastic process on a product
space $\mathcal{X}=\prod_{w\in\mathcal{V}} \mathcal{X}_w$, so a state $\mathbf{s}\in\mathcal{X}$ is a
configuration $\mathbf{s}=(s_w)_{w\in\mathcal{V}}$, where $s_w\in\mathcal{X}_w$. 
If $\mathcal{W}\subseteq\mathcal{V},$ then we write $s_\mathcal{W}=(s_w)_{w\in\mathcal{W}}$ for the
con\-fi\-gu\-ra\-tion~$\mathbf{s}$ restricted to the~nodes in~$\mathcal{W}$. We also use the notation
$\mathcal{X}_\mathcal{W}=\prod_{w\in\mathcal{W}} \mathcal{X}_w$, so  we can write $s_\mathcal{W}\in\mathcal{X}_\mathcal{W}$. In what follows
we use the~bold symbol $\bf{s}$ to denote configurations belonging to~$\mathcal{X}$ only. All restricted configurations will be denoted with the standard font~$s$.

Now suppose we have a family of functions $Q_w:\X_{\parents(w)}\times(\X_w\times \X_w)\to[0,\infty)$. For a fixed $c\in \X_{\parents(w)}$ we consider $Q_w(c;\cdot,\cdot\;)$ as a conditional intensity matrix (CIM) at the node $w$ (only off-diagonal elements of this matrix have to be specified, the dia\-gonal ones are irrelevant). The state of CTBN at time $t$ is a random element $X(t)$ of the space  $\X$ of all the configurations. Let $X_w(t)$ denote its $w$-th coordinate. The process $\left\{(X_w(t))_{w\in\mathcal{V}}:t\geq 0\right\}$  is assumed to be Markov and its evolution can be described informally as follows: transitions, or jumps, at the node $w$ depend on the current con\-fi\-gu\-ration of its parents. If the state of any parent changes, then the node $w$ switches to other transition probabilities.  If~$s_w\not=s_w\p ,$ where $s_w,s'_w\in \ccX_w$, then  
\begin{equation*}
         \Pr\left(X_w(t+\d t)=s_w\p\mid X_{-w}(t)=s_{-w},X_w(t)=s_w\right)=
              Q_w(s_{\parents(w)},s_w,s_w\p)\,\d t. 
\end{equation*}

\begin{definition}
A continuous time Bayesian network $\ccN$ over a set of random variables $\XX =\{X_1,\dotsc, X_n\}$ is formed by two components. The first one is an initial distribution~$P^0_{\XX}$ specified as a Bayesian network $\ccB$ over $\XX$. The second component is a continuous transition model, specified as 
\begin{itemize}
    \item a directed (possibly cyclic) graph $\gr$ whose nodes correspond to the random variables~$X_i$;
    \item a conditional intensity matrix $\QQ_{X_i\mid\parents(X_i)}$, specifying the continuous dynamic of each variable $X_i$ given its parents' configuration.
\end{itemize}
\end{definition}
Essentially, CTBN is a Markov jump process (MJP) on the state space $\ccX$ with transition intensities given by  
\begin{equation}\label{def: intensity}
    Q(\bf{s,s\p})=
          \begin{cases}
             Q_w(s_{\parents(w)},s_w,s_w\p), & \text{if $s_{-w}=s_{-w}\p$ and $s_{w}\not=s_{w}\p$ for some $w$,} \\        
              0,       &  \text{if $s_{-w}\not=s_{-w}\p$ for all $w$,}
          \end{cases}
\end{equation}
for $\bf{s}\not=\bf{s\p}.$ Obviously, $Q(\bf{s,s})$ is defined ``by subtraction'' to ensure that $\sum\limits_{\bf{s\p}} Q({\bf{s,s\p}})=0$. For convenience, we will often write $Q(\s) = - Q(\s,\s)$ so that $Q(\s)\geq 0$. In particular, $Q_w(c;s_w) = -\sum\limits_{s\neq s'}Q_w(c;s,s')$.

It is important to note that we make a fundamental assumption in the construction of the CTBN model: two variables cannot transition at the same time (a zero in the definition of $Q(\bf{s,s})$). This can be viewed as a formalization of the view that variables must represent distinct aspects of the world. We should not, therefore, model a domain in which we have two variables that functionally and deterministically change simultaneously. For example, in the drug effect network, we should not add a variable describing the type of food, if any, a person is eating. We could, however, change the value space of the ``Eating'' variable from a binary ``yes/no'' to a more descriptive set of possibilities.

Further we omit the symbol $\gr$ in the indices and write $\ppa(w)$ instead of $\parents(w).$ For CTBN the density  of a sample trajectory  $X=X([0,T])$ on a bounded time in\-ter\-val~$[0,T]$ decomposes as follows:
\begin{equation}
 \label{eq:densCTBN}
 p(X)=\nu(X(0))\prod_{w\in\mathcal{V}}p(X_w \mid\mid X_{\ppa(w)})\;,
\end{equation}
where $\nu$ is the initial distribution on $\X$ and $p(X_w \mid\mid X_{\ppa(w)})$ is the density of piecewise 
homogeneous Markov jump process with the intensity matrix equal to $Q_w(c;\cdot,\cdot\;)$ in every time sub-interval such that $X_{\ppa(w)}=c$.   
Below we explicitly write an expression for the~den\-si\-ty $p(X_w \mid\mid X_{\ppa(w)})$ in terms of moments of jumps and the 
skeleton of the process $(X_w,X_{\ppa(w)})$, as in \eqref{eq:densCTBN}, where by skeleton we understand the sequence of states of the~process corresponding to the sequence of moments of time.

Let $T^w=(t_0^w\ldots,t_i^w,\ldots)$ and $T^{\ppa(w)}=(t_0^{\ppa(w)},\ldots,t_j^{\ppa(w)},\ldots)$ denote  moments of jumps at the node $w\in\V$ and at parent nodes, respectively. By convention, put $t_0^w=t_0^{\ppa(w)}=0$ and $t^w_{|T^w|+1}=t^{\ppa(w)}_{|T^{\ppa(w)}|+1}=\tmax$. Analogously, 
$S^w$ and $S^{\ppa(w)}$ denote the corresponding skeletons. Thus we 
divide the time interval $[0,\tmax]$ into disjoint segments $[t^{\ppa(w)}_j,t^{\ppa(w)}_{j+1})$, $j=0,1,\dots |T^{\ppa(w)}|$ such that $X_{\ppa(w)}$ is constant 
and $X_w$ is homogeneous in each segment. Next we define sets 
$I_j=\{i>0 :\ t^{\ppa(w)}_j<t^w_i<t^{\ppa(w)}_{j+1}\}$ with notation $j_{\rm{ beg}} $ and $j_{\rm {end}}$ for the~first and the last element of $I_j$.  
Then we obtain the following formula.
\definecolor{citrine}{rgb}{0.89, 0.82, 0.04}
\definecolor{cyan(process)}{rgb}{0.0, 0.72, 0.92}
\definecolor{darkviolet}{rgb}{0.58, 0.0, 0.83}
\definecolor{darkpastelgreen}{rgb}{0.01, 0.75, 0.24}
\begin{equation*}\label{eq:conddens}
    \begin{split}
        p(X_w&\mid\mid  X_{\ppa(w)}) =p(T^w,S^w \mid\mid S^{\ppa(w)},T^{\ppa(w)}) =\\
        & = \prod_{j=0}^{|T^{\ppa(w)}|}\Bigg\{\Ind(I_j\not=\emptyset)\Bigg[ \prod_{i \in I_j}Q_w({\color{citrine}s_j^{\ppa(w)}};{\color{cyan}s_{i-1}^w,s_i^w})\nonumber \times \\
        &\times \prod_{i \in I_j\setminus\{j_{\rm {beg}}\}}\exp\left(-{\color{darkviolet}(t_i^w-t_{i-1}^w)}Q_w({\color{citrine}s_j^{\ppa(w)}};{\color{cyan}s_{i-1}^w)}\right) \times \\
        &\times \exp\left(-{\color{darkpastelgreen}(t_{j_{\rm {beg}}}^w-t_j^{\ppa(w)})}Q_w({\color{citrine}s_j^{\ppa(w)}};{\color{cyan}s_{j_{\rm {beg}}-1}^w})-{\color{red}(t_{j+1}^{\ppa(w)}-t_{j_{\rm{end}}}^w)}Q_w({\color{citrine}s_j^{\ppa(w)}};{\color{cyan}s_{j_{\rm{end}}}^w})\right) \Bigg]\nonumber + \\
         & + \Ind(I_j=\emptyset)\exp\left(-{\color{blue}(t_j^{\ppa(w)}-t_{j+1}^{\ppa(w)})}Q_w({\color{citrine}s_j^{\ppa(w)}};{\color{cyan}s_{j_{\rm {beg}}-1}^w})\right) \Bigg\}\;.\nonumber
    \end{split}
\end{equation*}
Below in Figure \ref{fig:density} there is an example of a trajectory of the node $w$ with two possible states and of its parent with also two possible states 0 and 1. In this case the sets of indices are $I_0 = \{2,3,4\}$, $I_1 = \{\emptyset\}$ and $I_2 = \{7\}$.
\begin{figure}[!ht]
\begin{center}
 \includegraphics[ width=\textwidth]{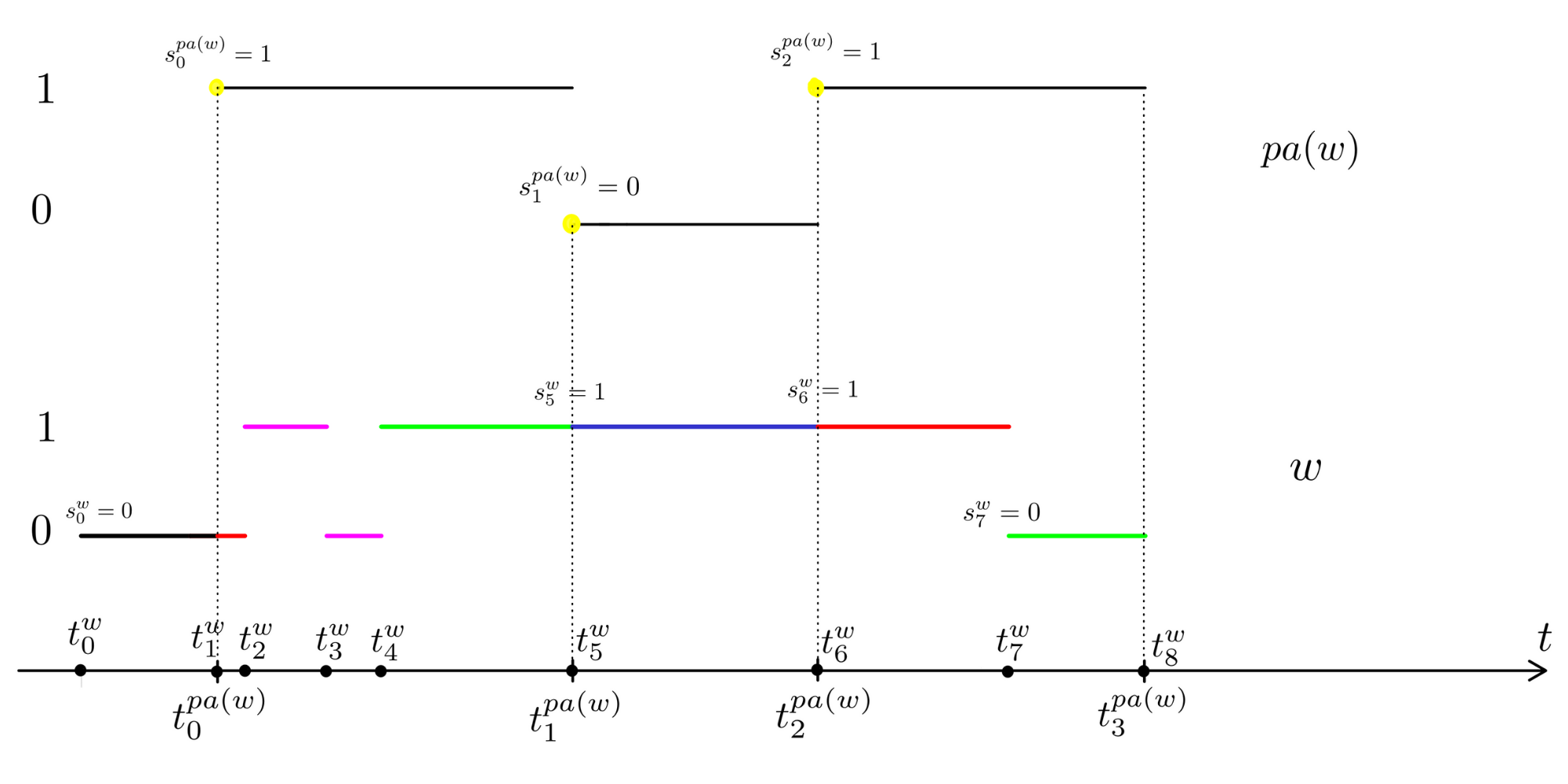}
\end{center}
\caption{An exemplary trajectory of a node $w$ and its parents $pa(w)$.}
\label{fig:density}
\end{figure}

In consequence, using the fundamental property of the exponential function we may write $p(X_w \mid\mid X_{\ppa(w)}) $ in the form
\begin{equation}\label{cbi}
       p(X_w \mid\mid X_{\ppa(w)})=
             \prod_{c\in\X_{\pax(w)}}
                      \prod_{s\in\X_w} \prod_{\substack{s\p\in\X_w \\ s\p\not=s}}
                      Q_w(c;\; s,s\p)^{n_w^T(c;\; s,s\p)} \exp\left[-Q_w(c;\; s,s\p)
                              t_w^T(c;\; s)\right],
\end{equation}
where
\begin{itemize}
    \item  $n_w^T(c; s,s\p)$ denotes the number of jumps from  $s\in\X_w$ to $s\p\in\X_w $ at the node $w$ on the time interval $[0,T]$, which occur when the parent configuration is $c\in\X_{\ppa(w)}$,
    \item  $t_w^T(c; s)$ is the length of time that the node $w$ is in the state $s \in \X_w$  on the time interval $[0,T],$   when the configuration of parents is $c\in\X_{\ppa(w)}$.   
\end{itemize}
To simplify the notation we omit the upper index $T$ in $n_w^T(c; s,s\p)$ and $t_w^T(c; s)$ further in the thesis, except for the part where we consider martingales.

\section{The LASSO penalty}\label{sec:lasso}
In this section we shortly describe the notions of the LASSO penalty and LASSO estimators which constitute the base of the novel algorithms for structure learning in the~thesis. LASSO is the acronym for Least Absolute Shrinkage and Squares Operator. The~term was invented by \cite{tibshirani} though the general concept was introduced even earlier. Most of the contents of this section come from \cite{hastie}.

The underlying idea of the LASSO estimators is the assumption of \textit{sparsity}. A sparse
statistical model is one in which only a relatively small number of parameters
(or predictors) play an important role. Consider a linear regression model with $N$ observations~$y_i$ of a target variable and $x_i = (x_{i1},\dots,x_{ip})^{\top}$ of $p$ associated predictor variables which are also called features. The goal is to predict the target from the predictors for future data and also to discover which predictors are \textit{relevant}. In the linear regression model we assume that 
\begin{equation*}\label{lin_regr}
    y_i = \beta_0 + \sum\limits_{j=1}^{p}\beta_j x_{ij} + \epsilon_i,
\end{equation*}
where $\beta = (\beta_0,\beta_1,\dots,\beta_p)$ is the vector of unknown parameters and $\epsilon_i$ is an error term. The standard way to find $\beta$ is to minimize the least-squares function $$\sum\limits_{i=1}^{N}\Big( y_i - \beta_0 - \sum\limits_{j=1}^{p}\beta_j x_{ij} \Big)^2.$$ Typically all of the estimates appear to be non-zero, which complicates the interpretability of the model especially with a high number of possible predictors. Moreover, since the~data have noise the model will try to fit the training observations too much and the~parameters will most probably take extreme values. In case when $p>N$ the estimates are not even unique, so most of solutions will overfit the data.

The solution is to \textit{regularize} the estimation process, i.e.~add some constraints on the parameters. The LASSO estimator uses $\ell_1$-penalty, which means that we minimize the least-square function with an additional bound on $\ell_1$-norm of $\beta$, namely $\|\beta\|_1 = \sum_{j=1}^p|\beta_j| \leq t$. The value $t$ is the user-specified parameter usually called hyperparameter. The motivation to use $\ell_1$-penalty instead of any other $\ell_q$-penalty comes from the fact that if $t$ is small enough we obtain a sparse solution with only a small amount of non-zero parameters. This does not happen for $\ell_q$-norm if $q>1$, and if $q<1$ the solutions are sparse but the~problem is not convex. Convexity simplifies the computations as well as the theoretical analysis of the properties of the estimator. This allows for scalable algorithms capable of handling problems with even millions of parameters. Before the optimization process we typically standardize the predictors so that each column is centred, i.e.~the mean for each column is~0, and has unit variance, i.e.~the mean of squares is equal to 1. We also centre the target column, so in the result we can omit the intercept term $\beta_0$ in the estimation process.

The LASSO penalty is used not only in linear regression but in a wide variety of models, for example generalized linear models where the target and the linear model are connected through some link function. Hence, in a more general case we can formulate the optimization problem as
\begin{equation*}
    \hat{\beta} = \argmin_{\theta\in\R^{p}}[ \ccL(\theta, \ccD) + \lambda\|\theta\|_1],
\end{equation*}
where $\ccL(\theta, \ccD)$ is the arbitrary loss function for the data $\ccD$ and the parameter vector $\theta$. The tuning hyperparameter $\lambda$ corresponds to the constraining value $t$, there is one-to-one correspondence. This is so-called Lagrangian form for the LASSO problem described above.

In the setting of structure learning for Bayesian networks, both static and continuous, we formulate the problem as an optimization problem for a linear or generalized linear model, where the parameter vectors encode the dependencies between variables in the network. We use the LASSO penalty in all the formulated problems, hence the problem of finding arrows in the graph reduces to recovering certain non-zero parameters in the~LASSO estimator. As the loss functions we use the negative log-likelihood function and the residual sum of squares.
  \chapter{Statistical inference for networks with known structure}\label{chapter: inference}
There are three main classes of problems concerning Bayesian networks (both static and continuous time). The first one is to discover the structure of the network. Namely, we need to specify the underlying graph of the network, which nodes are the variables of interest, and its edges encode the dependencies between the variables. This problem will be covered in subsequent chapters.

The second problem is to learn the parameters of the network. Namely, knowing the~structure of the network we need to specify the behaviour of the network in any specified node given the states of its parents. In the context of static BN this behaviour is encoded by conditional probability distributions (CPD, see \eqref{eq:cpds}). The corresponding parameters in case of CTBNs are conditional intensity matrices (CIM, see \eqref{def: intensity}).

The third type of problems is to make statistical inference using the network with known structure and parameters. For instance, we may want to predict the state of some node of~interest or, knowing states of~some nodes, find which combination of the~remaining nodes explains them the best. Finally, we may be interested in prediction of the~future dynamics (in time) of some nodes of the network.

In this chapter we discuss well known results concerning the problems of learning the~parameters of the network and then the inference based on the fully discovered network. The contents of this chapter are mainly based on \cite{koller2009}, \cite{Nod4} and \cite{heckerman2021tutorial} with more detailed references throughout it.

\section{Learning probabilities in BNs}\label{sec:CPDs}

First we discuss the discrete case. We assume that the Bayesian network with the known underlying graph $\gr$ includes $n$ nodes each corresponding to a variable $X_i\in\XX$ for $i =~1,\dots, n.$ Also, each variable $X_i$ is discrete, having $r_i$ possible values $x_i^1, x_i^2, \dots, x_i^{r_i}.$ We denote an observed value of $X_i$ in $l$-th observation as $X_i[l]$.
If each node is observed $m$ times, then we obtain the sample dataset $\ccD = \{D_1, D_2,\dots, D_m\}$ with the sample $D_l = (X_{1}[l], X_{2}[l],\dots, X_{n}[l])$ indicating the observed values of all the nodes in the $l$-th sampling. We refer to each $D_l$ as \textit{a case}. If all cases are complete, i.e.~no missing values occurred in the dataset $\ccD$, it is considered as \textit{complete data}; otherwise, it is called \textit{incomplete data}. Missing values in data can occur for many different reasons, for instance, people filling out a survey may prefer not to answer some questions or certain measurements might not be available for some patients in a medical setting.

There are mainly two categories of methods for parameter estimation in BN: one is for dealing with the complete data, and the other is for incomplete data. We will provide concise descriptions of two algorithms for the first category such as maximum likelihood estimation and Bayesian method; and we will briefly discuss algorithms for the second category.
 
Assume that as in (\ref{eq:cpds}) we can write the joint distribution of the variables in $\XX$ as follows
\begin{equation*}\label{eq:parametric_cpd}
    \Pr(X_1, X_2,..., X_n\mid \ttheta) = \prod_{i=1}^n \PCond(X_i on \parents(X_i),\ttheta_i)
\end{equation*}
for some vector of parameters $\ttheta = (\ttheta_1,\dots,\ttheta_n)$, where $\ttheta_i$ is the vector of parameters for the local distribution $\PCond(X_i on \parents(X_i),\ttheta_i)$. For shortness, further in this chapter we will write $\ppa(X_i)$ instead of $\parents(X_i)$. In the case of discrete and completely observed data \textit{categorical distribution} is commonly used. We note that in literature concerning learning Bayesian networks this type of distribution is often referred to as multinomial distribution or in some cases as unrestricted multinomial distribution (for example \cite{heckerman2021tutorial}) to differentiate this distribution from multinomial distributions that are low-dimensional functions of $\ppa(X_i)$.

Hence we assume that each local distribution function is a collection of categorical distributions, one distribution for each configuration of its parents, namely
\begin{equation}\label{categorical}
    \PCond(X_i =x_i^k on \ppa_i^j,\ttheta_i) = \theta_{ijk}>0, \text{ for } 1\leq k\leq r_i, 1\leq j \leq q_i,
\end{equation}
where $q_i=\prod\limits_{X_j\in\ppa(X_i)}r_j$ and $\ppa_i^1,\ppa_i^2,\dots,\ppa_i^{q_i}$ denote all possible configurations of $\ppa(X_i)$, and $\ttheta_i=((\theta_{ijk})_{k=2}^{r_i})_{j=1}^{q_i}$ are the parameters. Note that the parameter $\theta_{ij1}$ is given by the difference $1-\sum_{k=2}^{r_i}\theta_{ijk}$. For convenience, let us denote the vector of parameters $\ttheta_{ij}=(\theta_{ij2}, \theta_{ij3},\dots,\theta_{ijr_i})$ for all $1\leq i\leq n$ and $1\leq j\leq q_i$ so that $\ttheta_i = (\ttheta_{ij})_{j=1}^{q_i}$.

As it is well known, the maximum likelihood estimation (MLE) is a method of esti\-ma\-ting the parameters of a probability distribution by maximizing the likelihood function, so that under the assumed statistical model the observed data is the most probable. Basically, if $C_k$ is the result of a random test for an event $C$ with several possible outcomes $C_1, C_2,\dots, C_n$ it will appear in the maximum likelihood for this event. Hence, the estimated value of $\hat{C}$ will be set as parameter $\theta$ if it maximizes the value of the likelihood function $\PCond(C on \theta)$.

For the general Bayesian network with $n$ nodes we denote the likelihood function as
\begin{equation}\label{eq:likelihoodBN}
\begin{split}
    L(\ttheta : \ccD) & = \PCond(\ccD on \ttheta) = \prod_{l=1}^m\PCond(D_l on \ttheta) = \prod_{l=1}^m\PCond(X_{1}[l], X_{2}[l],\dots, X_{n}[l] on \ttheta) =\\ &=\prod_{l=1}^m\prod_{i=1}^n\PCond(X_{i}[l] on \ppa_i[l], \ttheta_i) = 
    \prod_{i=1}^n\prod_{l=1}^m\PCond(X_{i}[l]  on \ppa_i[l], \ttheta_i) = \prod_i L_i(\ttheta_i : \ccD),
    \end{split}
\end{equation}
where by $\ppa_i[l] = \ppa(X_i)[l]$ we denote the $l$-th observation of the parents vector of the~variable $X_i$. This representation shows that the likelihood decomposes as a product of independent factors, one for each CPD in the network. This important property is called \textit{the global decomposition} of the likelihood function. Moreover, this decomposition is an~immediate consequence of the network structure and does not depend on any particular choice of the parameterization for CPDs (see \cite{koller2009}).

If the conditional distribution of $X_i$ given its parents $\parents(X_i)$ is the categorical dist\-ribution, then the local likelihood function can be further decomposed as follows
\begin{equation}\label{eq:likelihoodBN_var}
\begin{split}
    L_i(\ttheta_i : \ccD) & = \prod_{l=1}^m\PCond(X_{i}[l] on \ppa_i[l], \ttheta_i) = \prod_{l=1}^m\prod_{j=1}^{q_i}\prod_{k=1}^{r_i}\PCond(X_i[l]=x^k_i on \ppa_i[l] = \ppa_i^j, \ttheta_i)\\ & = \prod_{j=1}^{q_i}\prod_{k=1}^{r_i}\theta^{N(x^k_i,\ppa_i^j)}_{ilk},
\end{split}
\end{equation}
where $N(x^k_i,\ppa_i^j)$ is the number of cases in $\ccD$ for which $X_i = x^k_i$ and $\ppa(X_i) = \ppa_i^j$.

Considering that the dataset is complete for each possible value $\ppa_i^j$ of the parents $\ppa(X_i)$ of the node $X_i$, the probability $\PCond(X_i on \ppa_i^j)$ is the independent categorical dist\-ribution not related to any other configurations $\ppa_i^l$ of $\ppa(X_i)$ for $ j\neq l$. Therefore, as the~result of the MLE method we obtain the estimated parameter $\hat{\ttheta}$ as follows
\begin{equation*}
    \hat{\theta}_{ijk} = \frac{N(x^k_i,\ppa_i^j)}{N(\ppa_i^j)},
\end{equation*}
where $N(\ppa_i^j)$ denotes the number of cases when the configuration $\ppa_i^j$ appears in the full set of observations for the vector of variables $\ppa(X_i)$.

Note that in general the MLE approach attempts to find the parameter vector $\ttheta$ that is ``the best'' given the data $C$. On the other hand, the Bayesian approach does not attempt to find such a point estimate. Instead, the underlying principle is that we should keep track of our beliefs about  values of $\ttheta$, and use these beliefs for reaching conclusions. In other words, we should quantify the subjective probability we have initially assigned to~different values of $\ttheta$ taking into account new evidence. Note that in representing such subjective probabilities we now treat $\ttheta$ as a random variable. Thus, the Bayesian approach is based on the Bayes rule
\begin{equation}\label{eq:bayes_rule}
p(\ttheta\mid C) = \frac{p(C \mid \ttheta)p(\ttheta)}{p(C)}.
\end{equation}
Hence, the basic idea of the Bayesian method for parameter learning is the following. We treat $\ttheta$ as a random variable with a prior distribution $p(\ttheta)$, and it is very common to~set~$p$ as the uniform distribution, especially in the case when we have no prior knowledge about~$\ttheta$. Given a distribution with unknown parameters and a complete set of observed data~$C$, new beliefs about $\ttheta$, namely $p(\ttheta\mid C)$, can be estimated according to the previous knowledge. The aim is to calculate $p(\ttheta\mid C)$ which is called the posterior probability of the parameter $\ttheta$. For computational efficiency we want to use a conjugate prior, i.e.~when the posterior distribution after conditioning on the data is in the same parametric family as the prior one.

Here we assume that each vector $\ttheta_{ij}$ has the prior Dirichlet distribution, so that
\begin{equation}\label{eq:prior_BN}
    p(\ttheta_{ij}) = Dir(\ttheta_{ij}\mid\alpha_{ij1},\dots,\alpha_{ijr_i}) = \frac{\Gamma(\alpha_{ij})}{\prod_{k=1}^{r_i}\Gamma(\alpha_{ijk})}\prod_{k=1}^{r_i}\theta_{ijk}^{\alpha_{ijk} - 1},
\end{equation}
where $\alpha_{ij} = \sum_{k=1}^{r_i}\alpha_{ijk}$, $\alpha_{ijk} > 0$, $k = 1,\dots,r_i$, $\alpha_{ij1},\dots,\alpha_{ijr_i}$ are hyperparameters and~$\Gamma(\cdot)$ is Gamma function. This is the standard conjugate prior to both categorical and multinomial distributions. Hence, the probability of observed samples is
\begin{equation}\label{eq:data_distribution_BN}
\begin{split}
    p(\ccD) & = \int p(\ttheta_{ij})p(\ccD\mid \ttheta_{ij}) d\ttheta_{ij}= \\ &= \int \frac{\Gamma(\alpha_{ij})}{\prod_{k=1}^{r_i}\Gamma(\alpha_{ijk})}\prod_{k=1}^{r_i}\theta_{ijk}^{\alpha_{ijk} - 1}\times\prod_{k=1}^{r_i}\theta_{ijk}^{N_{ijk}}d\ttheta_{ij}  = \\
    & = \frac{\Gamma(\alpha_{ij})}{\Gamma(\alpha_{ij} + N_{ij})}\prod_{k=1}^{r_i}\frac{\Gamma(\alpha_{ijk} + N_{ijk})}{\Gamma(\alpha_{ijk})},
    \end{split}
\end{equation}
where for shortness $N_{ijk} = N(x^k_i,\ppa_i^j)$  and $N_{ij}= N(\ppa_i^j) =\sum_{k=1}^{r_i} N_{ijk}$. The integral is $(r_i-1)$-dimensional over the set $\{\theta_{ijk}\geq 0,\quad 2\leq k\leq r_i,\quad \sum_{k=2}^{r_i}\theta_{ijk}\leq 1\}$.

As we have already mentioned, in Bayesian method, if we do not have prior distribution we assume it to be uniform, which is consistent with the principle of maximum entropy in information theory, it maximizes the entropy of random variables with bounded support. Thus, if there is no information used for determination of prior distribution, we set hyperparameters $\alpha_1=\dots=\alpha_r = 1.$

Combining (\ref{eq:bayes_rule}), (\ref{eq:prior_BN}) and (\ref{eq:data_distribution_BN}) under the assumptions of parameter independence and complete data finally we obtain the posterior distribution as follows
\begin{equation}\label{eq:posterior_Dir}
    p(\ttheta_{ij}\mid\ccD) = Dir(\ttheta_{ij}\mid\alpha_{ij1} + N_{ij1},\dots,\alpha_{ijr_i}+N_{ijr_i}).
\end{equation}
Therefore, we have an estimate for each parameter $\theta_{ijk}$ from data $\ccD$ as follows
\begin{equation*}
    \hat{\theta}_{ijk} = \dfrac{\alpha_{ijk} + N_{ijk}}{\alpha_{ij} + N_{ij}}, \quad 1\leq k\leq r_i.
\end{equation*}

\commentblock{\begin{enumerate}
    \item The calculation of non-conditional probability $p(X_i\mid\theta)$ (nodes with no parents).
    \begin{equation*}
        p(\theta\mid D) = \frac{p(D \mid \theta)p(\theta)}{p(D)} = \frac{\Gamma(\alpha)}{\prod_{i=1}\Gamma(\alpha_i)}\prod_{k}\theta_k^{\alpha_k - n_k}=  Dir(\alpha_1+n_1,\dots,\alpha_N+n_N)
    \end{equation*}
Therefore,
\[
p(\theta_i\mid D) = \frac{\alpha_i+n_i}{\alpha+N},
\]
where $n_i$ is the number of occurrences of $i$-th possible values for variable $X_i$ in the whole dataset, and $N$ is the number of occurrences of all the possible values for variable $X_i$ in the dataset.
    \item The calculation of conditional probability for each node which has unique parent node.
    We assume that the relationship between nodes $X$ and $Y$ can be denoted as $X\rightarrow Y$, and they both are discrete variables. Hence,
    \[
    p(y\mid x_i,\theta) = Dir(\alpha_{i1}+n_{i1},\dots,\alpha_{ik}+n_{ik})
    \]
    \item The calculation of conditional probability for each node which has multiple parent nodes.
    Firstly, an assumption of parameter independence was given: the parameters, which might have different distribution, are mutually independent. Here, $\theta_{ijk}$ .... ...... ........
    With the assumption of parameter independence, each variable $X_i$ and its parents obey Dirichlet distribution:
\end{enumerate}}

\paragraph{Continuous Variable Networks.} When we were discussing the MLE method for disc\-rete BNs, we mentioned the global decomposition rule which applies to any type of CPD. That is, if the data are complete, the learning problem reduces to a set of local lear\-ning problems, one for each variable. The main difference is in applying the maximum likelihood estimation process to CPD of a different type: how we define the sufficient statistics, and how we compute the maximum likelihood estimate from them. In this paragraph, we briefly discuss how MLE principles can be applied in the setting of linear Gaussian Bayesian networks.

Consider a variable $X$ with parents $\U = \{U_1,\dots,U_k\}$ with linear Gaussian CPD:
\begin{equation*}
    p(X\mid \u) = \ccN(\beta_0+\beta_1u_1+\dots+\beta_ku_k,\sigma^2).
\end{equation*}
Our task is to learn the parameters $\hat{\boldsymbol{\theta}}_{X\mid\U} = (\beta_0,\beta_1,\dots,\beta_k,\sigma^2)$. To find the MLE values of these parameters, we need to differentiate the likelihood function and to solve the~equations that define a stationary point. As usual, it is easier to work with the log-likelihood function. Using the definition of the Gaussian distribution, we have that
\begin{equation*}
    \begin{split}
        \ell(\boldsymbol{\theta}_{X\mid\U}:\ccD) & = \log L_X(\boldsymbol{\theta}_{X\mid\U}:\ccD) = \\ 
        & =
        \sum_l\left[-\frac{1}{2}\log(2\pi\sigma^2)-\frac{1}{2\sigma^2}(\beta_0+\beta_1u_1[l]+\dots+\beta_ku_k[l]-x[l])^2 \right].
    \end{split}
\end{equation*}
We consider the gradients of the log-likelihood with respect to all of the parameters $\beta_0,\dots,\beta_k$ and $\sigma^2$ and as a result we get a number of equations, which describe the~solution to a system of linear equations. From the Theorem 7.3 in \cite{koller2009}, it follows that if $\ccB$ is a linear Gaussian Bayesian network, then it defines a joint distribution that is jointly Gaussian, and the MLE estimate has to match the constraints implied by~it.

Briefly speaking, to estimate $p(X \mid \U)$ we estimate the means of $X$ and $\U$ and the~covariance matrix of $\{X\}\cup \U$ from the data. The vector of means and the covariance matrix define the joint Gaussian distribution over $\{X\}\cup \U$. Then, for example using the formulas provided by Theorem 7.3 in \cite{koller2009}, we find the unique linear Gaussian that matches the joint Gaussian with these parameters.

The sufficient statistics we need to collect to estimate linear Gaussians are the univariate terms of the form $\sum_m x[m]$ and $\sum_m u_i[m]$, and the interaction terms of the form $\sum_m x[m]\cdot u_i[m]$ and $\sum_m u_j[m]\cdot u_i[m]$. From these we can estimate the mean and the~covariance matrix of the joint distribution.

\section{Inference in Bayesian networks}\label{sec: BN_inference}

In this section we assume that the network structure is known, meaning we know all the~existing edges and their directions as well as all the CPDs. The problem of inference for BNs is a challenging task on its own and there is a lot of research done on the subject. We will not go into much of a detail on the inference since our focus is on learning their structure. However, the question of inference is worth mentioning here in order to get a~wholesome picture of such a powerful tool as BNs.

First, we discuss what the notion of inference means in the case of BNs. Typically it refers to:
\begin{itemize}

    \item marginal inference, i.e.~finding \textit{the probability of a variable} being in a certain state, \textit{given} that \textit{other variables} are set to certain values; or
    \item maximum a posteriori (MAP) inference, i.e.~finding \textit{the values of a given set of variables that best explain} (in the sense of the highest MAP
    probability) why a set of other variables have certain values,
\end{itemize}
Let us demonstrate both categories of questions using an example. We will use the BN structure of a well-known ASIA network (see Figure \ref{fig:ASIA}) first introduced in \cite{Lauritzen}. It illustrates the causal structure of a patient having a certain lung disease based on several factors, one being whether or not the patient has recently been to Asia.  In this case, an exemplary question on marginal inference might be what is the~probability of a patient who is a smoker and has dyspnoea having a certain lung disease, e.g.~lung cancer. For the MAP inference, we might want to know  what is the most likely set of conditions (with ``smoking'' and ``dyspnoea'' excluded) that could have caused the symptoms mentioned above.

Now we provide short descriptions of the most popular exact and approximate in\-fe\-rence algorithms for BNs. Among them are variable elimination and belief propagation for the marginal inference, methods for the MAP inference and the sampling-based inference.
For the purposes of transparency of the presentation the inference methods for BNs will be demonstrated for the discrete and finite case.

\begin{figure}[!ht]
\begin{center}
\includegraphics[width=0.65\textwidth]{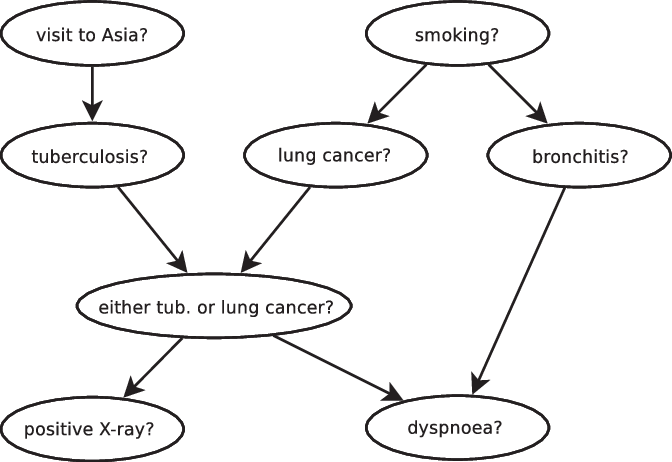}
\caption{The ASIA Bayesian network structure}
\label{fig:ASIA}
\end{center}
\end{figure}

\subsection{Variable Elimination}\label{subsec:VE}

This inference algorithm is defined in terms of so-called factors and is developed to answer questions of marginal inference. Factors generalize the notion of CPDs. \textit{A factor} $\phi$ is a~function of value assignments of a set of random variables $\VV$ with positive real values. The set of variables $\VV$ is called \textit{the scope} of the factor. There are two operations on factors that are repeatedly performed in a variable elimination algorithm (VE) and hence are of great importance.
\begin{itemize}
    \item \textit{The factor product.} If $\VV_1, \VV_2,$ and $\VV_3$ are disjoint sets of variables and we have factors $\phi_1$ and $\phi_2$ with scopes $\VV_1\cup \VV_2 $ and $\VV_2\cup \VV_3 $ respectively, then we define the~factor product $\phi_1\cdot\phi_2$ as a new factor $\psi$ with the scope $\VV_1\cup \VV_2 \cup \VV_3$ by
    \begin{equation*}
        \psi(\VV_1, \VV_2, \VV_3) = \phi_1(\VV_1, \VV_2)\cdot\phi_2(\VV_2, \VV_3).
    \end{equation*}
    This product is the new factor over the union of the variables defined for each instantiation by multiplying the value of $\phi_1$ on the particular instantiation  by the~value of $\phi_2$ on the corresponding instantiation. More precisely,
     \begin{equation*}
        \psi(\vv_1, \vv_2,\vv_3) = \phi_1(\vv_1, \vv_2)\cdot\phi_2(\vv_2, \vv_3)
    \end{equation*}
    for each instantiation, where $\vv_1\in\Val(\VV_1), \vv_2\in\Val(\VV_2)$ and  $\vv_3\in\Val(\VV_3)$.

    \item \textit{The factor marginalization.} This operation ``locally'' eliminates a set of variables from a factor. If we have a factor $\phi(\VV_1,\VV_2)$ over two sets of variables  $\VV_1, \VV_2$, marginalizing $\VV_2$ produces a new factor
    \begin{equation*}
        \tau(\VV_1) = \sum_{\VV_2}\phi(\VV_1, \VV_2),
    \end{equation*}
    where the sum is over all joint assignments for the set of variables $\VV_2$. More precisely,
     \begin{equation*}
        \tau(\vv_1) = \sum_{\vv_2\in\Val(\VV_2)}\phi(\vv_1, \vv_2), \vv_1\in\Val(\VV_1)
    \end{equation*}
    for each instantiation $\vv_1\in\Val(\VV_1)$.
\end{itemize}
Thus, in the context of factors we can write our distribution over all variables as a product of factors, where each factor presents a CPD as in \eqref{eq:cpds}:
\begin{equation}\label{eq:factor_product}
\Pr(X_1, X_2,...,X_n) = \prod_{i=1}^{n}\phi_i(\AA_i),    
\end{equation}
where $\AA_i = (X_i, \parents(X_i))$ represents a set of variables including the $i$-th variable and its parents in the network.

Now we can describe the full VE algorithm. Assume we want to find marginal dist\-ribution of a fixed variable from $X_1,\dots,X_n$. First we need to choose in which order~$O$ to eliminate remaining variables. The choice of an optimal \textit{ elimination ordering} $O$ is an $\mathcal{NP}$-hard problem and it may dramatically affect the running time of the variable elimination algorithm. Some intuitions and techniques on how to choose an adequate ordering are given for example in \cite{koller2009}. For each variable $X_i$ (ordered according to the ordering $O$) we perform the following steps:
\begin{itemize}
    \item multiply all factors containing $X_i$ (on the first round all the $\phi_i$ containing $X_i$);
    \item marginalize out $X_i$ according to the definition of the factor marginalization to obtain a new factor $\tau$ (which does not necessarily correspond to a probability distribution, even though each $\phi$ is CPD);
    \item replace the factors used in the first step with $\tau$.
\end{itemize}
Essentially, we loop over the variables as ordered by $O$ and eliminate them in this order. Performing those steps we use simple properties of product and summation on factors, namely, both operations are commutative and products are associative. The most important rule is that we can exchange summation and product, meaning that if a set of variables $\XX$ is not in the scope of the factor $\phi_1$, then
\begin{equation}\label{property_sum_product}
    \sum_{\XX}\phi_1\cdot\phi_2 = \phi_1\cdot\sum_{\XX}\phi_2.
\end{equation}

So far we saw that the VE algorithm can answer queries of the form $\Pr(\VV)$, where $\VV$ is some subset of variables. However, in addition to this type of questions it can answer marginal queries of the form 
\[
\PCond(\YY on \EE = \ee) = \frac{\Pr(\YY,\EE=\ee)}{\Pr(\EE=\ee)},
\]
where $\Pr(\XX,\YY,\EE)$ is a probability distribution over sets of query variables $\YY$, observed evidence variables $\EE$, and unobserved variables $\XX$. We can compute this probability by performing variable elimination once on $\Pr(\YY,\EE=\ee)$ and then once again on $\Pr(\EE=\ee)$ taking into account only instantiations consistent with $\EE = \ee$.

An exemplary run of the VE algorithm is presented in Table \ref{tableVE}. It corresponds to~Extended Student example first mentioned in Section \ref{sec: Bayesian}.

\begin{table}[h!]
\centering
\begin{tabular}{ | c |c | c | c| c | }
 \hline
\multirow{2}{*}{Step} & Variables & Factors used & Variables & New \\
 & eliminated &  & involved & factor\\
 \hline
 1 & $C$ & $\phi_C(C), \phi_D(D,C)$                 & $C,D$     & $\tau_1(D)$\\
 2 & $D$ & $\phi_G(G,I,D), \tau_1(D)$               & $G,I,D$   & $\tau_2(G,I)$\\
 3 & $I$ & $\phi_I(I), \phi_S(S,I), \tau_2(G,I)$    & $G,S,I$   & $\tau_3(G,S)$\\
 4 & $H$ & $\phi_H(H,G,J)$                          & $H,G,J$   & $\tau_4(G,J)$\\
 5 & $G$ & $\tau_3(G,S), \tau_4(G,J),\phi_L(L,G)$   & $G,J,L,S$ & $\tau_5(J,L,S)$\\
 6 & $S$ & $\tau_5(J,L,S), \phi_J(J,L,S)$           & $J,L,S$   & $\tau_6(J,L)$\\
 7 & $L$ & $\tau_6(J,L)$                            & $J,L$     & $\tau_7(J)$\\
 \hline
\end{tabular}
\caption{A run of variable elimination for the query $P(J)$.}
\label{tableVE}
\end{table}

\subsection{Message Passing Algorithms}\label{subsec:message_passing}
\paragraph{Markov random fields.} In the framework of probabilistic graphical models there exists another technique for compact representation and visualization of a probability distribution which is formulated in the language of undirected graphs. This class of models (known as Markov Random Fields or MRFs) can succinctly represent independence assumptions that directed models cannot represent and the opposite is also true. There are advantages and drawbacks to both of those methods but that is not the focus of this thesis. We will introduce and discuss MRFs only to the extent we need to properly describe and explain notions and methods concerning BNs. Note that the methods provided below for marginal and MAP inference are applicable both to MRFs and BNs.
\begin{definition}
A Markov Random Field (MRF) is a probability distribution over variables $X_1,\dots,X_n$ defined by an undirected graph $\gr$ in which nodes correspond to variables~$X_i$. The probability has the form
\begin{equation*}\label{eq: MRF}
    \Pr(X_1, X_2,...,X_n) = \frac{1}{Z}\prod_{c\in C}\phi_c(X_c),
\end{equation*}
where $C$ denotes the set of cliques (i.e.~fully connected subgraphs) of $\gr$ and each factor $\phi_c$ is a non-negative function over the variables in a clique. The partition function 
\[
Z= \sum\limits_{(x_1,\dots,x_n)}\prod_{c\in C}\phi_c(X_c)
\]
is a normalizing constant that ensures that the distribution sums to one, where the summation is taken over all possible instantiations of all the variables.
\end{definition}
Thus, given a graph $\gr$, our probability distribution may contain factors whose scope is any clique in $\gr$ and the clique can be a single node, an edge, a triangle, etc. Note that we do not need to specify a factor for each clique.

It is not hard to see that Bayesian networks are a special case of MRFs with a nor\-ma\-lizing constant equal to 1 where the clique factors correspond to CPDs. One can notice that if we take a directed graph $\gr$, add side edges to all parents of a given node and remove their directionality, then the CPDs (seen as factors over each variable and its ancestors) factorize over the resulting undirected graph. The resulting process is called \textit{moralization} (see Figure \ref{fig:moralization}). A Bayesian network can always be converted into an undirected network with normalizing constant 1.

 \begin{figure}[!ht]
\begin{center}
\includegraphics[width=0.7\textwidth]{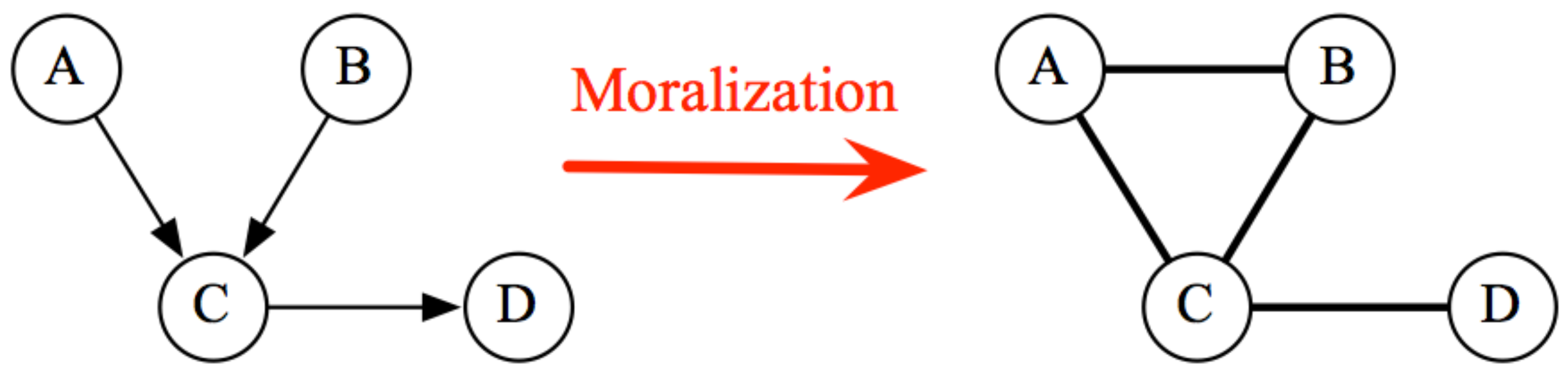}
\caption{Moralization of a Bayesian network.}
\label{fig:moralization}
\end{center}
\end{figure}
\paragraph{Message passing.} As we mentioned above, the VE algorithm can answer marginal queries of the form $\PCond(\YY on \EE=\ee)$. However, if we want to ask the model for another query, e.g.~$\PCond(\YY_2 on \EE_2=\ee_2)$, we need to restart the algorithm from scratch. Fortunately, in the~process of computing marginals, VE algorithm produces many intermediate factors~$\tau$ as a side-product of the main computation, which turn out to be the same as the ones that we need to answer other marginal queries.

Many complicated inference problems can be solved by message-passing algorithms, in which simple messages are passed locally among simple elements of the system. An~il\-lust\-rative example was shown in the book \cite{MacKay2003} for a problem of counting soldiers.
Consider a line of soldiers walking in the mist. The commander, which is in the line, wishes to count the soldiers. The straightforward calculation is impossible because of the mist. However, it can be done in a simple way which does not require any complex operations. The algorithm requires the soldiers' ability to add two integer numbers and add 1 to it. The algorithm consists of the following steps (for example see Figure \ref{fig:soldiers1}): 
\begin{itemize}
    \item the front soldier in the line says the number `one' to the soldier behind him,
    \item the rearmost soldier in the line says the number `one' to the soldier in front of him,
    \item the soldier, which is told a number from the soldier ahead or the soldier behind, adds 1 to it and passes the new number to the next soldier in the line on the other side.
\end{itemize}

\begin{figure}[!ht]
\begin{center}
\includegraphics[width=0.5\textwidth]{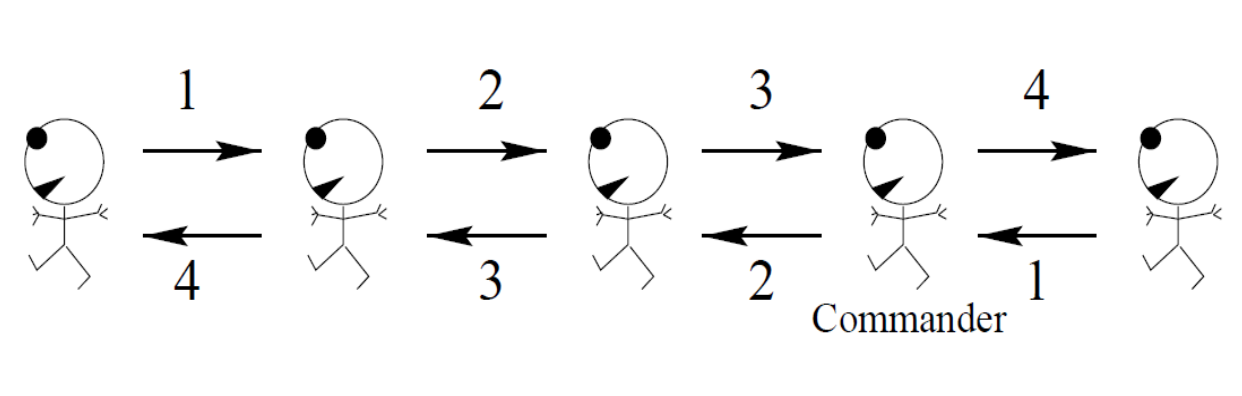}
\caption{A line of soldiers counting themselves using message-passing rule-set.}
\label{fig:soldiers1}
\end{center}
\end{figure}

Hence, the commander can find the global number of soldiers by simply adding together the numbers: heard from the soldier in front of him, from the soldier behind him and 1. This method makes use of a property of the total number of soldiers: the number can be written as the sum of the number of soldiers in front of a point and the number behind that point, two quantities which can be computed separately, because the two groups are separated by the commander. When this requirement is satisfied this message-passing algorithm can be modified for a general graph with no cycles (as an example see Figure~\ref{fig:swarm_nocycle}). When the graph has no cycles (see Figure \ref{fig:swarm_nocycle}) for each soldier we can uniquely separate the~group into two groups, `those in front', and `those behind' and perform the algorithm above. However, it is not always possible for a graph with cycles, for instance for a soldier in a cycle (such as `Jim') in Figure \ref{fig:swarm_cycle} such a separation is not unique.

\begin{figure}[ht]
  \centering
    \begin{subfigure}{0.4\linewidth}
    \includegraphics[width=\linewidth]{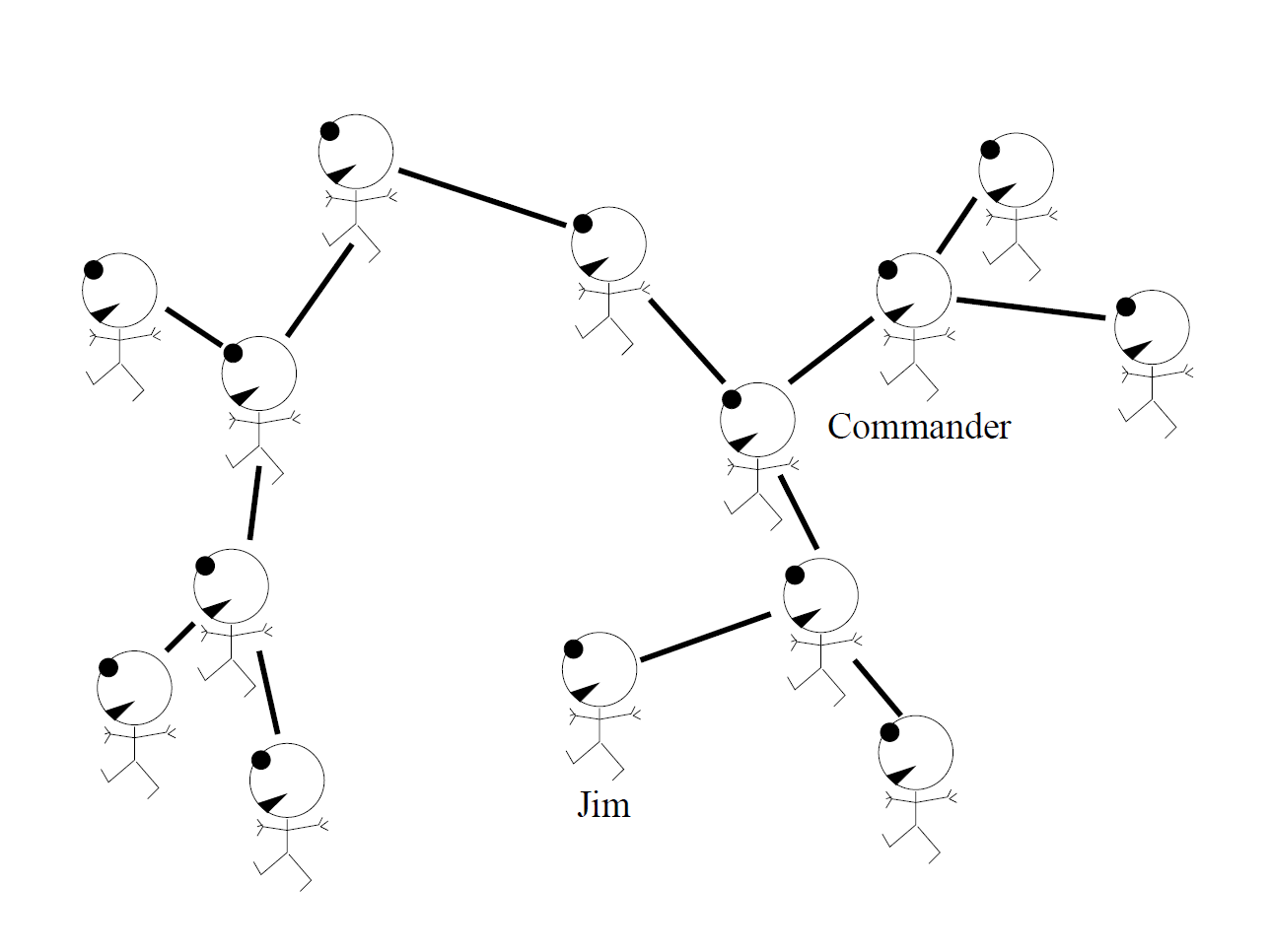}
    \caption{No cycles.}
    \label{fig:swarm_nocycle}
  \end{subfigure}
  \begin{subfigure}{0.4\linewidth}
    \includegraphics[width=\linewidth]{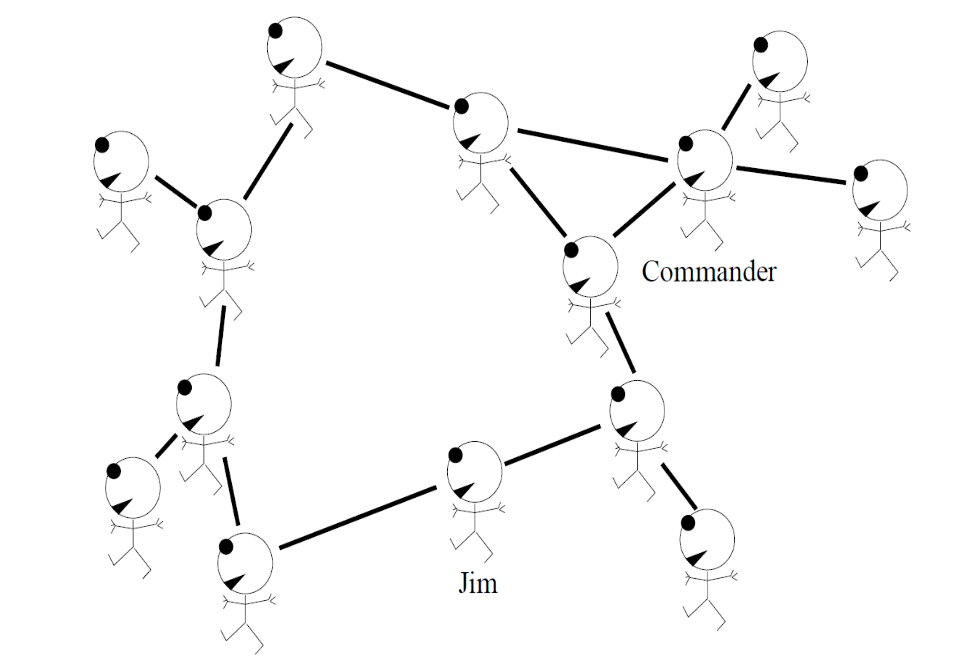}
    \caption{Contains a cycle.}
    \label{fig:swarm_cycle}
  \end{subfigure}
    \caption{A swarm of soldiers.}
\end{figure}

Using the same principle we will now describe the message passing for tree-structured networks (called \textit{belief propagation}, BP for short) and then the modification of the method for general networks (called \textit{clique tree algorithm}).

\paragraph{Belief propagation.} Let us first look at tree-structured graphs. Consider what happens if we run the VE algorithm on a tree in order to compute a marginal distribution $\Pr(X_i)$. We can easily find the optimal ordering for this problem by rooting the tree at the node associated with $X_i$ and iterating through the nodes in post-order (from leaves to the root), just like for a swarm of soldiers with no cycles. At each step, we will eliminate one of the variables, say $X_j$; this will involve computing the factor $\tau_k(x_k)=\sum_{x_j}\phi(x_k, x_j)\tau_j(x_j)$, where $X_k$ is the parent of $X_j$ in the tree. At a later step, the variable $X_k$ will be eliminated in the same manner, i.e.~$\tau_k(x_k)$  will be passed up the tree to the parent $X_l$ of $X_k$ in order to be multiplied by the factor $\phi(x_l, x_k)$ before being marginalized out. As a result we obtain the new factor $\tau_l(x_l)$. The factor $\tau_j(x_j)$ can be thought of as a message that $X_j$ sends to~$X_k$ that summarizes all of the information from the subtree rooted at the node~$X_j$. We can visualize this transfer of information using arrows on the tree, see Figure \ref{fig:tree}. At~the~end of the VE algorithm, the node $X_i$ receives messages from all of its children and the~final marginal $\Pr(X_i)$ is obtained by marginalizing those messages out.

\begin{figure}[!ht]
\begin{center}
\includegraphics[width=0.5\textwidth]{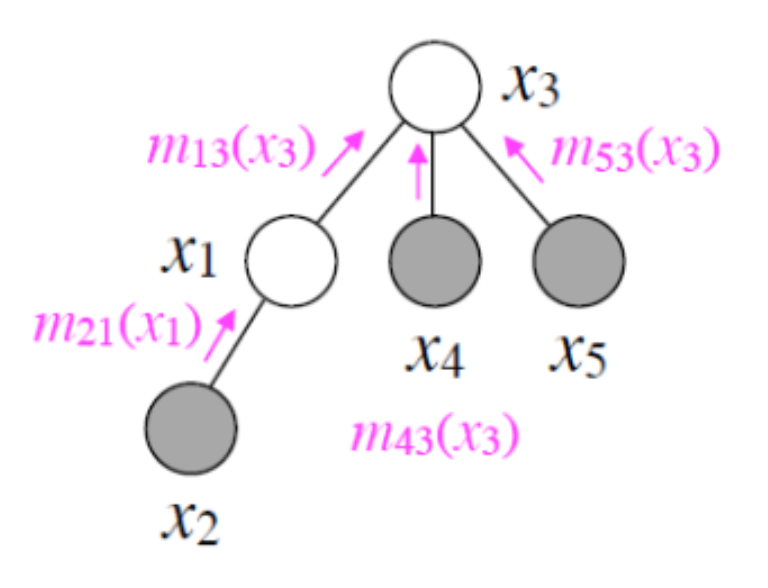}
\caption{Message passing order when using VE to compute $\Pr(X_3)$ on a small tree.}
\label{fig:tree}
\end{center}
\end{figure}

With the same indices as above, suppose that after computing $\Pr(X_i)$ we want to compute $\Pr(X_k)$ as well. We would again run VE for the new tree rooted at the node $X_k$, waiting until it receives all messages from its children. Note that the new tree consists of two parts. The first one is the subtree rooted at $X_k$ with all its descendants from the~original tree (i.e.~rooted at $X_i$). The other part is the subtree rooted at $X_l$ (which was the~parent of $X_k$ in the original tree, but now is the child of $X_k$). Therefore, this part contains the node $X_i$. The key insight is that the messages received by $X_k$ from~$X_j$ now will be the same as those received when $X_i$ was the root. Thus, if we store the~in\-ter\-mediary messages of the VE algorithm, we can quickly compute other marginals as well. Notice for example, that the messages sent to $X_k$ from the subtree containing $X_i$ will need to be recomputed. So, how do we compute all the messages we need? Again, referring to the soldier counting problem, a node is \textit{ready to transmit} a message to its parent after it has received all the messages from all of its children. All the messages will be sent out after precisely $2|\ccE|$ steps, where $|\ccE|$ is the number of edges in the graph, since each edge can receive messages only twice.  

To define belief propagation (BP) algorithm formally let us see what kind of messages can be sent. For the purposes of marginal inference we will use \textit{sum-product message passing}. This algorithm is defined as follows: while there is a node $X_k$ ready to transmit to $X_l$ it sends the message
\begin{equation*}\label{eq:message_tree}
    m_{k \to l}(x_l) = \sum\limits_{x_k}\phi(x_k)\phi(x_k,x_l)\prod\limits_{j\in Nb(k)\setminus \{l\}} m_{j \to k}(x_k),
\end{equation*}
where $ Nb(k)\setminus \{l\}$ means all the neighbours of the $k$-th node, excluding $l$-th node. Note that this message is precisely the factor $\tau$ that $X_k$ would transmit to $X_l$ during a round of variable elimination with the goal of computing $\Pr(X_i)$, and also note that the product on the RHS of this equation naturally equals to $1$ for leaves in the tree.

After having computed all messages, we may answer marginal queries over any variable~$X_j$ in constant time using the equation:
\[
\Pr(X_j) \propto \psi(X_j)\prod\limits_{l\in Nb(j)} m_{l \to j}(x_j),
\]
where $\psi(X_j)$ is a product of all factors $\phi$ whose scope contains $X_j$. In case of BNs we have the equality instead of proportionality.

\paragraph{Clique Tree Algorithm.}

First let us define what is meant by a clique tree. Clique tree is an undirected tree such that its nodes are clusters $\C_i$ of variables, meaning $\C_i$ is a subset of a set of all variables $\{X_1,\dots,X_n\}$. Each edge between clusters $\C_i$ and~$\C_j$ is associated with a \textit{sepset} (separation set) $\S_{i,j} = \C_i\cap \C_j$. See a simple example demonstrating a clique tree for a chain network in Figure \ref{fig:cliquetree1}.

\begin{figure}[!ht]
\begin{center}
\includegraphics[width=0.75\textwidth]{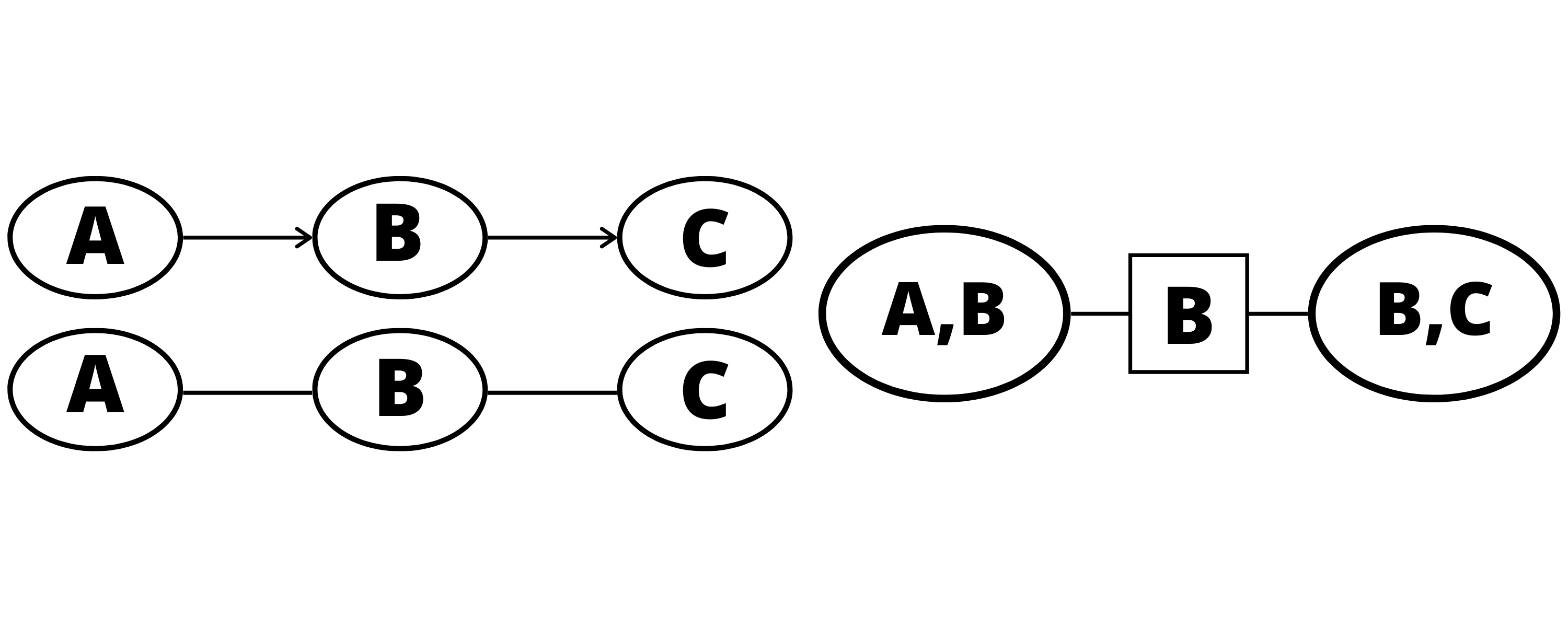}
\caption{An example of a chain network consisting of three variables $A$, $B$ and $C$; corresponding MRF and a clique tree with $\C_1=\{A,B\}$, $\C_2 = \{B,C\}$ and $\S_{1,2}= \{B\}$.}
\label{fig:cliquetree1}
\end{center}
\end{figure}

So far we assumed that the graph is a tree. What if that is not the case? Then the clique tree algorithm (also called the junction tree algorithm in the literature) can be used; it partitions the graph into clusters of variables so that interactions among clusters will have a tree structure, i.e.~a cluster will be only directly influenced by its neighbours in the tree, we denote it $\ccT$. Then we can perform message passing on this tree. This leads to tractable global solutions if the local (cluster-level) problems can be solved exactly.

In addition, clique trees must satisfy two following properties:
\begin{enumerate}
    \item \textit{family preservation}, i.e.~for each factor $\phi$ there is a cluster such that factor's scope is a subset of the cluster;
    \item \textit{running intersection property (RIP)}, i.e.~for each pair of clusters $\C_i$, $\C_j$ and a~variable $X\in \C_i\cap \C_j$ all clusters and sepsets on the unique path between $\C_i$ and $\C_j$ contain the variable $X$.
\end{enumerate}

Note that we may always find a trivial clique tree with one node containing all the~variables in the original graph, but obviously such trees are useless. Optimal trees are the~ones that make the clusters as small and modular as possible; unfortunately, as in case of~VE, the problem of finding the optimal tree is also $\mathcal{NP}$-hard. A special case when we can find it is when we originally have a tree, in this case we can put each connected pair of nodes into a separate cluster, it is easy to check that both conditions are met. One of the~practical ways to find a good clique tree is to use a simulation of VE, i.e.~the elimination order fixed for VE will induce the graph from which we will take maximal cliques and set them as our clusters and form a tree. RIP will be satisfied automatically. Note that we do not need to run VE, just to simulate it for a chosen ordering and get the induced graph. More formally:
\begin{definition}
Let $\Phi$ be a set of factors (CPDs in the case of Bayesian Networks) over $\X = \{X_1,\dots, X_n\}$, and $\prec$ be an elimination ordering for some subset $\XX\subseteq \X$. The~induced graph denoted by $I_{\Phi,\prec}$ is an undirected graph over $\X$ , where $X_i$ and $X_j$ are connected by an edge if they both appear in some intermediate factor $\psi$ generated by the~VE algorithm using $\prec$ as an elimination ordering.
\end{definition}
In Figure \ref{fig:induced} there is an example of an induced graph for the Student example using the elimination ordering of Table \ref{tableVE}, cliques in that graph and a corresponding clique tree. One can see that RIP is satisfied, for a proof that trees corresponding to induced graphs by VE will satisfy RIP see \cite{koller2009}.

Now let us define the full clique tree algorithm. First, we define the potential $\psi_{i}(\C_i)$ of each cluster $\C_i$ as the product of all the factors $\phi$ in $\gr$ that have been assigned to $\C_i$. By the family preservation property, this is well-defined, and we may assume that our distribution is of the form
\begin{equation*}
    \Pr(X_1,\dots, X_n) = \frac{1}{Z}\prod\limits_{i} \psi_i(\C_i).
\end{equation*}

Then, at each step of the algorithm, we choose a pair of adjacent clusters $\C_i$, $\C_j$ in a tree graph $\ccT$ and compute a message whose scope is the sepset $\S_{i,j}$ between the two clusters
\begin{equation}\label{eq:message}
    m_{i \to j}(\S_{i,j}) = \sum\limits_{\C_i\setminus \S_{i,j}}\psi_i(\C_i)\prod\limits_{l\in Nb(i)\setminus \{j\}} m_{l \to i}(\S_{l,i}).
\end{equation}

In the context of clusters, $Nb(i)$ denotes the set of indices of neighboring clusters of~$\C_i$. We choose $\C_i$ and $\C_j$ only if $\C_i$ has received messages from all of its neighbors except $\C_j$. Just as in belief propagation, this procedure will terminate in exactly $2|\ccE_{\ccT}|$ steps because this process is equivalent to making an \textit{upward} pass and a  \textit{downward} pass. In the upward pass, we first pick a root and send all messages towards it starting from leaves. When this process is complete, the root has all the messages. Therefore, it can now send the~appropriate message to all of its children. This algorithm continues until the leaves of the~tree are reached, at which point no more messages need to be sent. This second phase is called the downward pass. After it terminates, we will define \textit{the belief} of each cluster based on all the messages that it receives
\begin{equation}\label{eq:belief}
\beta_i(\C_i) = \psi_i(\C_i)\prod\limits_{l\in Nb(i)} m_{l \to i}(\S_{l,i}).
\end{equation}
These updates are often referred to as \textit{Shafer-Shenoy} updates and the full procedure is also referred as \textit{sum-product belief propagation}. Then each belief is the marginal of the~clique
\begin{equation*}
    \beta_i(\C_i) = \sum_{\ccX\setminus\C_i}\Pr(X_1,\dots,X_n).
\end{equation*}
Now if we need to compute the marginal probability of a particular variable $X$ we can select any clique whose scope contains $X$, and eliminate the redundant variables in the~clique. A key point is that the result of this process does not depend on the clique we selected. That is, if $X$ appears in two cliques, they must agree on its marginal. Two adjacent cliques $\C_i$ and $\C_j$ are said to be \textit{calibrated} if
\begin{equation*}
    \sum_{\C_i\setminus\S_{i,j}}\beta_i(\C_i) = \sum_{\C_j\setminus\S_{i,j}}\beta_j(\C_j).
\end{equation*}
A clique tree $\ccT$ is calibrated if all pairs of adjacent cliques are calibrated. For a calibrated clique tree, we use the term clique beliefs for $\beta_i(\C_i)$ and sepset beliefs for $\mu_{i,j}(\S_{i,j})$ defined as either side of the above equality.

As the end result of sum-product belief propagation procedure we get a calibrated tree, which is more than simply a data structure that stores the results of probabilistic inference for all of the cliques in the tree, i.e.~their beliefs (\ref{eq:belief}). It can also be viewed as an alternative representation of the joint measure over all variables. For sepset beliefs we have that
\begin{equation*}
    \mu_{i,j}(\S_{i,j}) = m_{i \to j}(\S_{i,j}) m_{j \to i}(\S_{i,j}).
\end{equation*}
Using this fact at convergence of the clique tree calibration algorithm, we get the unnormalized joint measure $\tilde{P}$ as 
\begin{equation}\label{eq:reparametrization}
    \tilde{P}(X_1,\dots,X_n) = \prod\limits_{i} \psi_i(\C_i) = \dfrac{\prod_{i}\beta_i(\C_i)}{\prod_{(i,j)}\mu_{i,j}(\S_{i,j})},
\end{equation}
where the product in the numerator is over all cliques and the product in the denominator is over all sepsets in the tree. As a result we get a different set of parameters that captures unnormalized measure that defined our distribution (in case of BNs it is simply the distribution) and there is no information lost in the process. Thus, we can view the~clique tree as an alternative representation of the joint measure, one that directly reveals the clique marginals.

The second approach, mathematically equivalent but using a different intuition, is message passing with division. In sum-product belief propagation messages were passed between two cliques only after one had received messages from all of its neighbors except the other one as in (\ref{eq:message}) and the resulting belief was (\ref{eq:belief}). Nonetheless, a different approach to compute the same expression is to multiply in all of the messages, and then divide the resulting factor by the message from the other clique to avoid double-counting. To make this notion precise, we must define a factor-division operation.

Let $\XX$ and $\YY$ be disjoint sets of variables and let $\phi_1$ and $\phi_2$ be two factors with scopes $\XX\cup\YY$ and $\YY$ respectively. Then we define the division $\dfrac{\phi_1}{\phi_2}$ as a factor-division $\psi$ with the~scope $\XX\cup\YY$ as follows
\begin{equation*}
    \psi(\XX,\YY) = \dfrac{\phi_1(\XX, \YY)}{\phi_2 (\YY)},
\end{equation*}
where we define $\frac{0}{0} = 0$. We now see that we can compute the expression of equation (\ref{eq:message}) by computing the beliefs as in equation (\ref{eq:belief}) and then dividing by the remaining message
\begin{equation*}
    m_{i\to j}(\S_{i,j}) = \dfrac{\sum_{\C_i\setminus\S_{i,j}}\beta_i(\C_i)}{m_{j\to i}(\S_{i,j})}.
\end{equation*}
 The belief of the $j$-th clique is updated by multiplying its previous belief by $m_{i\to j}$ and dividing it by the previous message passed along this edge (regardless of the direction) stored in sepset belief $\mu_{i,j}$ to avoid double counting. This algorithm is called \textit{belief update} message passing and is also known as the \textit{Lauritzen-Spiegelhalter algorithm}.

\begin{figure}[!ht]
\begin{center}
\includegraphics[width=0.75\textwidth]{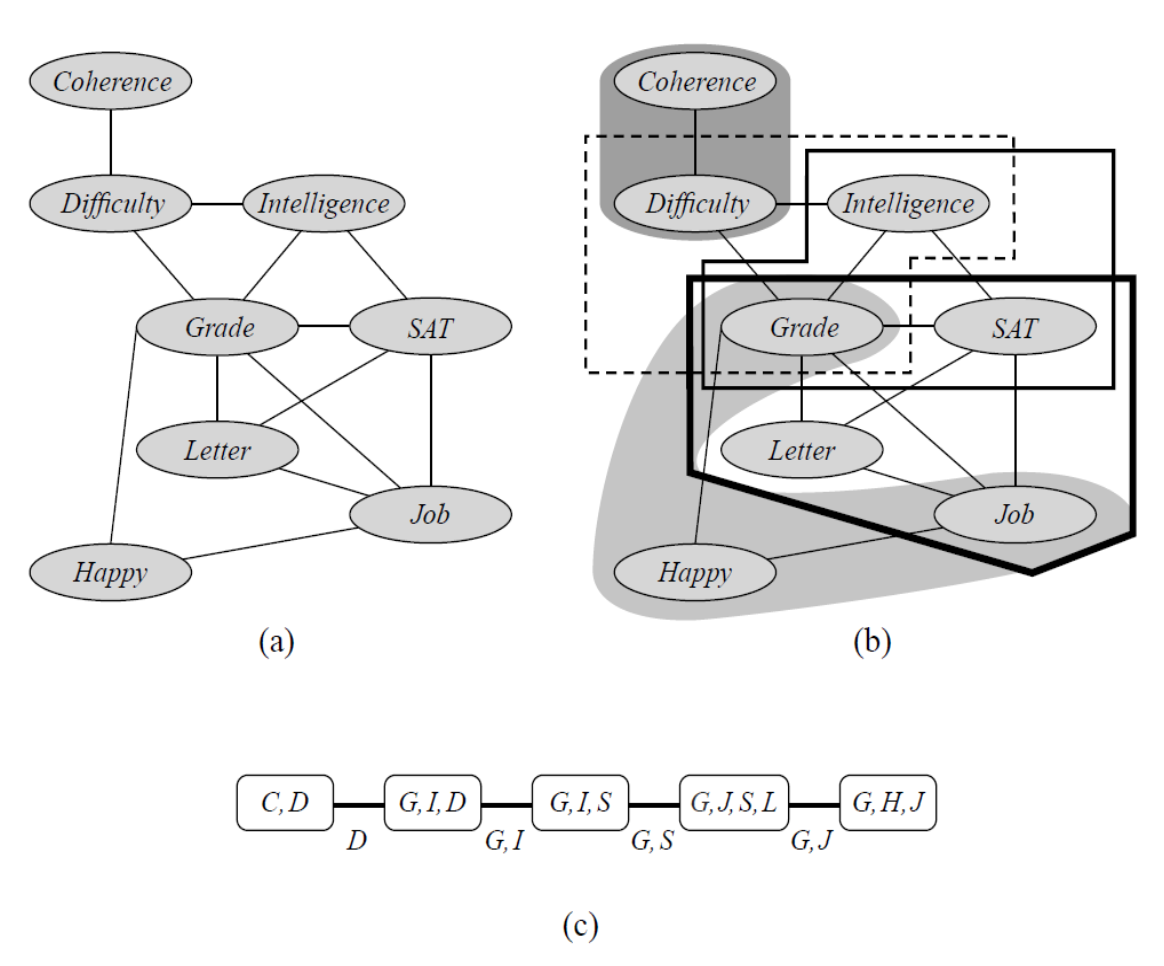}
\caption{(a) Induced graph for VE in the Student example, using the elimination order of Table \ref{tableVE} (b) Cliques in the induced graph:$\{C, D\}$, $\{D, I, G\}$, $\{G, I, S\}$, $\{G,J,S,L\}$ and $\{G, H, J\}$. (c) Clique tree for the induced graph.}
\label{fig:induced}
\end{center}
\end{figure}

\subsection{MAP inference}
The maximum a posteriori (MAP) problem has a broad range of applications, in computer vision, computational biology, speech recognition, and more. By using MAP inference we lose the ability to measure our confidence (or uncertainty) in our conclusions. Nevertheless, there are good reasons for using a single MAP assignment rather than using the~marginal probabilities of the different variables. The first reason is the~preference for obtaining a single coherent joint assignment, whereas a set of individual marginals may not make sense as a whole. The second is that there are inference methods that are applicable to the MAP problem and not to the task of computing probabilities, so that the former may be tractable even when the latter is not. The problem of finding the MAP assignment in the general case is $\ccN\ccP-$hard (\cite{COOPER90}).

There are two types of Maximum a Posteriori (MAP) inference: a MAP query and a~marginal MAP query. Assume first that the set of all variables $\XX = \YY\cup\EE$ consists of two disjoint sets, where $\EE$ is the evidence meaning that we know values of those variables. Then a MAP query aims to find the most likely assignment to all of the non-evidence variables $\YY$
\begin{equation*}
    {MAP}(\YY\mid\EE=\ee) = \argmax_{\yy}\PCond(\YY=\yy on \EE=\ee).
\end{equation*}
Now assume that the set of all variables $\XX = \YY\cup\WW\cup\EE$ consists of three disjoint sets, where $\EE$ is still the evidence. In this case a marginal MAP query aims to find the most likely assignment to the subset $\YY$, marginalizing over the rest of the variables $\WW$
\begin{equation*}
\begin{split}
    {MAP}(\YY\mid\EE=\ee) & = \argmax_{\yy}\PCond(\YY=\yy on \EE=\ee) = \\
    & =\argmax_{\yy}\sum_{\ww}\PCond(\YY=\yy, \WW=\ww on \EE=\ee).
\end{split}
\end{equation*}
Both tasks can be solved within the same variable elimination (VE) and message passing frameworks as marginal inference, where instead of summation we use maximization. The second type of query is much more complicated both in theory and in practice since it involves both maximization and summation. In particular, exact inference methods such as VE can be intractable, even in simple networks. Hence, first we will briefly discuss them and then introduce some more efficient methods.

Recall that while discussing VE we introduced two operations on factors, which were the foundation in performing the algorithm. Now we need to introduce one additional operation called \textit{the factor maximization}. Let $\XX$ be a set of variables, and $Y\not\in \XX$ a~variable not belonging to the set $\XX$. Let $\phi(\XX,Y)$ be a factor over those variables. We define the~factor maximization of $Y$ in $\phi$ to be a factor $\psi$ over $\XX$ such that:
\begin{equation*}
    \psi(\XX) = \max_Y\phi(\XX,Y).
\end{equation*}
More precisely, 
\[
\psi(\xx) = \max_{y\in\Val(Y)}\phi(\xx,y)
\]
for each instantiation $\xx\in\Val(\XX)$. Similarly to the property (\ref{property_sum_product}) we have that if a set of variables $\XX$ is not in the scope of the factor $\phi_1$, then
\begin{equation}\label{property_max_prod}
    \max_{\XX}(\phi_1\cdot\phi_2) = \phi_1\cdot\max_{\XX}\phi_2
\end{equation}
and
\begin{equation}\label{property_max_sum}
    \max_{\XX}(\phi_1+\phi_2) = \phi_1+\max_{\XX}\phi_2 .
\end{equation}
This leads us to \textit{a max-product variable elimination algorithm} for a general MAP query, which is constructed in the same way as a sum-product variable elimination algorithm in Subsection \ref{subsec:VE}, but we replace the marginalizing step (summation) with maximization over corresponding variables.

This way we find the maximum value for the joint probability, though the original and more interesting problem is to find the most probable assignment corresponding to that maximum probability. This process is called $\textit{a traceback procedure}$, which is quite straightforward (details can be found in \cite{koller2009}). In the process of eliminating variables we find their maximizing value given the values of the variables that have not yet been eliminated. When we pick the value of the final variable, we can then go back and pick the values of the remaining variables accordingly.

Recall that the joint distribution $P$ in Bayesian networks is represented by a product of factors, where each factor coincides with a CPD (we introduced this representation in~(\ref{eq:factor_product})). Then we can write the marginal MAP query as
\begin{equation*}
    \argmax_{\yy}\sum_{\WW}P(\yy, \WW) = \argmax_{\yy}\sum_{\WW}\prod_{i}\phi_i,
\end{equation*}
where we skipped the evidence set for the transparency of notation since it does not effect the main point of discussion. First we compute 
\begin{equation*}
    \max_{\yy}\sum_{\WW}\prod_{i}\phi_i.
\end{equation*}
This form immediately suggests an algorithm combining the ideas of sum-product and max-product variable elimination. Specifically, the summations and maximizations outside the product can be viewed as operations on factors. Thus, to compute the value of this expression, we simply have to eliminate the variables in $\WW$ by summing them out, and the variables in $\YY$ by maximizing them out. When eliminating a variable $X$, whether by summation or by maximization, we simply multiply all the factors whose scope involves~$X$, and then eliminate $X$ to produce the resulting factor. The ability to perform this step is justified by the interchangeability of factor summation and maximization with factor product (properties (\ref{property_sum_product}) and (\ref{property_max_prod})). The traceback procedure to find the most probable assignment  can also be found in \cite{koller2009}.

At first glance it seems that algorithms for both queries have the same complexity but that is not the case. It can be shown that even on very simple networks, elimination algorithms can require exponential time to solve a marginal MAP query (see Example 13.7 in \cite{koller2009}). The difficulty comes from the fact that we are not free to choose an arbitrary elimination ordering. When summing out variables, we can utilize the fact that the operations of summing out different
variables commute. Thus, when performing summing-out operations for sum-product variable elimination, we could sum out the variables in any order. Similarly, we could use the same flexibility in the case of max-product elimination. Unfortunately, the max and sum operations do not commute. Thus, in order to maintain the correct semantics of marginal MAP queries, as specified in the equation, we must perform all the variable summations before we can perform any of the variable maximizations.

We can also use the message passing framework, or more general case of clique tree algorithm, to MAP inference. In Subsection \ref{subsec:message_passing} we used clique trees to compute the~sum-marginals over each of the cliques in the tree. Here, we compute a set of max-marginals over each of those cliques. By the max-marginal of a function $f$ defined on the~set $\XX$ relative to a set of variables $\YY\subset\XX$ we denote such a factor that for each $\yy\in\YY$
\[
    MaxMarginal_f(\yy) = \max_{\langle\xx\rangle_{\YY}=\yy}f(\xx)
\]
determines the value of the unnormalized probability of the most likely joint as\-sign\-ment~$\xx\in\XX$ consistent with $\yy$. We compute the whole set for two reasons. First, the set of max-marginals can be a useful indicator for how confident we are in particular components of the MAP assignment. Second, in many cases, an exact solution to the MAP problem via a variable elimination procedure is intractable. In this case, to compute approximate max-marginals we can use message passing procedure in cluster graphs, similar to the clique tree procedure. These pseudo-max-marginals can be used for selecting an assignment; while this assignment is not generally the MAP assignment, we can nevertheless provide some guarantees in certain cases. As before, our task consists of two parts: computing the max-marginals and decoding them to extract a MAP assignment.

As for the first part, in the same way as we modified sum-product VE to sum-product message-passing we modify max-product VE to max-product belief propagation algorithm in clique trees. The resulting algorithm executes precisely the same initialization and overall message scheduling as in the sum-product belief propagation algorithm. The only difference is that we use max-product rather than sum-product message passing. As a~result of running the algorithm we will get a set of max-marginals for every clique of our clique tree.

Each belief is the max-marginal of the clique $\beta_i(\C_i) = MaxMarginal_{p}(\C_i)$ and all pairs of adjacent cliques are \textit{max-calibrated}
\begin{equation*}
    \mu_{i,j}(\S_{i,j}) = \max_{\C_i\setminus\S_{i,j}}\beta_i(\C_i) = \max_{\C_j\setminus\S_{i,j}}\beta_j(\C_j).
\end{equation*}
Similarly to sum-product message passing we get reparameterization of the distribution in the form (\ref{eq:reparametrization}) with corresponding beliefs of the max-product belief propagation algorithm.

Now we need to decode those max-marginals to get a MAP assignment. In the case of variable elimination, we had the max-marginal only for a single last to be eliminated variable and  could identify the assignment for that particular variable. To compute the~assignments to the rest of the variables, we had to perform a traceback procedure. Now the
situation appears different. One obvious solution is to use the max-marginal for each variable to compute its own
optimal assignment, and thereby compose a full joint assignment to all variables. However, this simplistic approach works only in case when there is a unique MAP assignment, equivalently, each max-marginal has a unique maximal value. For generic probability measures this is not a very rigid constraint, thus, we can find the unique MAP assignment by locally optimizing the assignment to each variable separately.

Otherwise, in most cases to break ties we can introduce a slight random perturbation into all of the factors, making all of the elements in the joint distribution have slightly different probabilities. However, there might be cases when we need to preserve the structure in relationships between some variables, for example some variables can share parameters or there might be some deterministic structure that should be preserved. Under these circumstances we find a locally optimal assignment using for example traceback procedure. Afterwards we can verify if this assignment is a MAP assignment (for procedure and verification see \cite{koller2009}).

\paragraph{MAP as Linear Optimization Problem.} In MAP inference we search for assignments which maximize a certain measure, in our case either the joint probability over all non-evidence variables or the probability over some set of variables. Therefore, it is na\-tu\-ral to consider it directly as an optimization problem. There exists extensive literature on optimization algorithms and we can apply some of those ideas and algorithms to our specific case.

The main idea here is to reduce our MAP problem to an Integer Linear Pro\-gram\-ming~(ILP) problem, i.e.~an optimization problem over a set of integer valued variables, where both the objective and the constraints are linear. 
First, to define ILP problem we need to turn the product representation of the joint probability as in (\ref{eq:factor_product}) into a sum, replacing the probability with its logarithm. It is possible because all the factors (CPDs) are positive. Hence, we want to compute
\begin{equation*}
    \argmax_{\xi}\prod_{i=1}^{n}\phi_i(\AA_i) = \argmax_{\xi}\sum_{i=1}^{n}\log(\phi_i(\AA_i)),
\end{equation*}
where $\xi$ is a general assignment for the whole vector of variables in the network, and $\AA_i = (X_i, \parents(X_i))$ represents a set of variables including the $i$-th variable and its parents in the network. Note that the whole discussion in this paragraph is actually identical for MRFs with positive factors, the only difference is the number of factors, but since they are not the focus of this thesis, we formulate everything in the Bayesian networks framework.

For variable indices $r\in\{1,\dots,n\}$ we define the number of corresponding possible vector instantiations $n_r = |\Val(\AA_r)|.$ For any joint assignment $\xi$, if this assignment constrained to the variables from $\AA_r$ takes the value of $\aa_r^j,$ $j=\{1,\dots,n_r\},$ i.e.~$\xi_{\AA_r} = \aa_r^j$, then the factor $\log(\phi_r)$ makes a contribution to the objective of a quantity denoted as $\eta_j^r = \log(\phi_r(\aa_r^j)).$

We introduce optimization variables $q(\xx^j_r)$, where $r$ enumerates the different factors, and $j$ enumerates the different possible assignments to the variables from $\AA_r$. These variables take binary values, so that $q(\xx^j_r)=1$ if and only if $\AA_r=\aa_r^j$ and 0 otherwise. It is important to distinguish the optimization variables from the random variables in our original graphical model; here we have an optimization variable $q(\xx^j_r)$ for each joint assignment $\aa_r^j$ to the model variables $\AA_r$.

Let $\qq$ denote a vector of the optimization variables $\{q(\xx^j_r), \quad 1\leq r\leq n, \quad 1\leq j \leq n_r \}$ and $\eeta$ denote a vector of the coefficients $\eta^j_r$ sorted in the same order. Both of these are vectors of dimension $N=\sum_{r=1}^{n}n_r$. With this interpretation, the MAP objective can be rewritten as:
\begin{equation}\label{eq:max_ilp}
    \max_{\qq}\sum_{r=1}^{n}\sum_{j=1}^{n_r}\eta^j_rq(\xx^j_r)
\end{equation}
or, in shorthand, $\max\limits_{\qq}\eeta^{\top}\qq$.

Now that we have an objective to maximize we need to add some consistency constraints that would guarantee that an assignment $\qq\in\{0,1\}^N$ we get as a solution of optimization problem is legal, meaning it corresponds to some assignment in $\X.$ Namely, first we require that we restrict attention to integer solutions, then we construct two constraints 
to make sure that these integer solutions are consistent.  The first constraint enforces the mutual exclusivity within a factor and the second one implies that factors in our network agree on the variables in the intersection of their scopes. In this way we reformulate the MAP task as an integer linear program, where we optimize the linear objective of equation (\ref{eq:max_ilp}) subject to discussed constraints. We note that the problem of solving integer linear programs is itself $\mathcal{NP}$-hard, so that we do not avoid the basic hardness of the MAP problem.

One of the methods often used to tackle ILP problems is the method of linear program relaxation. In this approach we turn a discrete, combinatorial optimization problem into a continuous problem. This problem is a linear program (LP), which can be solved in polynomial time, and for which a range of very efficient algorithms exists. One can then use the solutions to this LP to obtain approximate solutions to the MAP problem. To~per\-form this relaxation, we substitute the condition that the solutions are integer with a~relaxed constraint that they are non-negative.

This linear program is a relaxation of our original integer program, since every assignment to $\qq$ that satisfies the constraints of the integer problem also satisfies the constraints of the linear program, but not the other way around. Thus, the optimal value of the objective of the relaxed version will be no less than the value of the (same) objective in the~exact version, and it can be greater when the optimal value is achieved at an assignment to $\qq$ that does not correspond to a legal assignment $\xi$. An important special case are tree-structured graphs, in which the relaxation is guaranteed to always return integer solutions, which are in turn optimal (for proof and more detailed discussion see \cite{koller2009}). Otherwise we get approximate solutions, which in order we need to transform into integer (and legal) assignments.

One approach is a greedy assignment process, which assigns values to the variables $X_i$ one at a time. Another approach is to round the LP solution to its nearest integer value. This approach works surprisingly well in practice and has theoretical guarantees for some classes of ILPs (\cite{koller2009}). 

An alternative method for the MAP problem which also comes from the optimization theory is called \textit{dual decomposition}. Dual decomposition uses the principle that our problem can be decomposed into sub-problems, together with linear constraints (the same as in ILP) that enforce some notion of agreement between solutions to the different problems. The sub-problems are chosen such that they can be solved efficiently using exact combinatorial algorithms. The agreement constraints are incorporated using Lagrange multipliers, it is called Lagrangian relaxation, and an iterative algorithm - for example, a subgradient algorithm - is used to minimize the resulting dual. The initial work on dual decomposition in probabilistic graphical models was focused on the MAP problem for MRFs (see \cite{komodakis2007}).

By formulating our problem as a linear program or its dual, we obtain a very flexible framework for solving it; in particular, we also can easily incorporate additional constraints into the LP, which reduce the space of possible assignments of $\qq$, eliminating some solutions that do not correspond to actual distributions over $\X$. The problems are convex and in principle they can be solved directly using standard techniques, but the size of the problems is very large, which makes this approach unfeasible in practice. However, the LP has special structure: when viewed as a matrix, the equality constraints in this LP all have a particular block structure that corresponds to the structure of adjacent clusters. Moreover, when the network is not densely connected, the constraint matrix is also sparse, thus, standard LP solvers may not be fully suited for exploiting this special structure. The theory of convex optimization provides a wide spectrum of tools, and some are already being adapted to take advantage of the structure of the MAP problem (see for example, \cite{Wainwright2005}, \cite{sontag}). The empirical evidence suggests that the more specialized solution methods for the MAP problems are often more effective.

\paragraph{Other methods.} 

Another method for solving a MAP problem is local search algorithms. It is a heuristic-type solution, which starts with an arbitrary assignment and performs ``moves'' on the joint assignment that locally increase the probability. This technique does not offer theoretical justification; however, we can often use prior knowledge to come up with highly effective moves. Therefore, in practice, local search may perform extremely well.

There are also searching methods that are more systematic. They search the space so as to ensure that assignments that are not considered are not optimal, and thereby guarantee an optimal solution. Such methods generally search over the space of partial assignments, starting with the empty assignment and successively assigning variables one at a time. One such method is known as branch-and-bound.

These methods have much greater applicability in the context of marginal MAP problem, where most other methods are not currently applicable.
In the next subsection we discuss sample-based algorithms which can be applied both to marginal and MAP inference.
 
\subsection{Sampling-based methods for inference}

In practice, the probabilistic models that we use can often be quite complex, and simple algorithms like VE may be too slow for them. In addition, many interesting classes of models may not have exact polynomial-time solutions at all, and for this reason, much research effort in machine learning is spent on developing algorithms that yield approximate solutions to the inference problem. In this subsection we consider some sampling methods that can be used to perform both marginal and MAP inference queries; additionally, they can compute various interesting quantities, such as the expectation $\Ex[f(\XX)]$ of a function of the random vector distributed according to a given probabilistic model.

In general, sampling is rather a hard problem. The aim is to generate a random sample of the observations of $\XX$. However, our computers can only generate samples from very simple distributions, such as the uniform distribution over $[0,1]$. All sampling techniques involve calling some kind of simple subroutine multiple times in a properly constructed way. For example, in case of multinomial distribution with parameters $\theta_1,\dots,\theta_k$ instead of directly sampling a multinomial variable we can sample a single uniform variable previously subdividing a unit interval into $k$ regions with region $i$ having size $\theta_i$. Then we sample uniformly from $[0,1]$ and return the value of the region in which our sample falls.

\paragraph{Forward sampling.} Now let us return to the case of Bayesian networks (BN). We can apply the same sampling technique to BNs with multinomial variables. We start from the~nodes which do not have parents, these variables simply have multinomial distribution, and we go down the network to the next generation as arrows point out until we reach the leaves. Therefore, for a particular node we need to wait until all of its parents are sampled. When we know all the values of parents the variable naturally has multinomial distribution. In the Student example to sample student's grade, we would first sample an~exam difficulty $d'$ and an intelligence level $i'$. Then, once we have samples $d'$ and $i'$, we generate a student grade $g'$ from $\PCond(g on d',i')$. There is one problem though, as we cannot perform it in case of having evidence for any variables besides roots.

\paragraph{Monte Carlo and rejection sampling.} Algorithms that construct solutions based on a large number of samples from a given distribution are referred to as Monte Carlo (MC) methods. Sampling from an arbitrary distribution $p$ lets us compute integrals of the form
\begin{equation*}
    \Ex_{\XX\sim p}[f(\XX)] = \sum\limits_{\xx}f(\xx)p(\xx),
\end{equation*}
where the summation extends over all possible values of $\XX$ and $p$ can be thought of as the~density of $\XX$ with respect to counting measure. Below we follow the same interpretation also with regards to joint and conditional distributions.

If $f(\XX)$ does not have special structure that matches the BN structure of $p$, this integral will be impossible to compute analytically; instead, we will approximate it using a large number of samples from $p$. Using Monte Carlo technique we approximate a target expectation with
\begin{equation*}
    \Ex_{\XX\sim p}[f(\XX)] \approx I_T = \frac{1}{T}\sum\limits_{t=1}^{T} f(\xx^t),
\end{equation*}
where $\xx^1,\dots,\xx^T$ are samples drawn according to $p$. It is easy to show that $I_T$ is an~unbiased estimator for $\Ex_{\XX\sim p}[f(\XX)]$ and its variance can be made arbitrarily small with a~sufficiently large number of samples.

Now let us consider rejection sampling as a special case of Monte Carlo integration. For example, suppose we have a Bayesian network over the set of variables $\XX= \ZZ\cup \EE$. We may use rejection sampling to compute marginal probabilities $\Pr(\EE=\ee)$. We can rewrite the probability as
\begin{equation*}
    \Pr(\EE=\ee) = \sum\limits_{\zz}\Pr(\ZZ=\zz, \EE=\ee) = \sum\limits_{\xx}\Pr(\XX=\xx)\Ind(\EE=\ee) = \Ex_{\XX\sim p}[\Ind(\EE=\ee)]
\end{equation*}
and then take the Monte Carlo approximation. In other words, we draw many samples from $p$ and report the fraction of samples that are consistent with the value of the marginal.

\paragraph{Importance sampling.} Unfortunately, rejection sampling can be very wasteful. If $\Pr(\EE=\ee)$ equals, say, 1\%, then we will discard 99\% of all samples. A better way of computing such integrals uses importance sampling. The main idea is to sample from an~auxiliary distribution $q$ (hopefully with $q(\xx)$ roughly proportional to $f(\xx)\cdot p(\xx)$), and then reweigh the samples in a principled way, so that their sum still approximates the~desired integral.

More formally, suppose we are interested in computing $\Ex_{\XX\sim p}[f(\XX)]$. Adopting analogous convention regarding notation for probability distribution we may rewrite this in\-teg\-ral as
\begin{equation*}
\begin{split}
    \Ex_{\XX\sim p}[f(\XX)] & = \sum\limits_{\xx}f(\xx)p(\xx) = \sum\limits_{\xx}f(\xx)\frac{p(\xx)}{q(\xx)}q(\xx) = \\
    & = \Ex_{\XX\sim q}[f(\XX)w(\XX)] \approx \frac{1}{T}\sum\limits_{t=1}^{T} f(\xx^t)w(\xx^t), 
\end{split}
\end{equation*}
where $w(\xx) = \dfrac{p(\xx)}{q(\xx)}$ and the samples $\xx^t$ are drawn from $q$. In other words, instead of sampling from $p$ we may take samples from $q$ and reweigh them with $w(\xx)$; the expected value of this Monte Carlo approximation will be the original integral. By choosing $q(\xx)=\dfrac{|f(\xx)|p(\xx)}{\int |f(\xx)|p(\xx) d\xx}$ we can set the variance of the new estimator to zero. Note that the~denominator is the quantity we are trying to estimate in the first place and sampling from such $q$ is $\ccN\ccP$-hard in general.

In the context of our previous example for computing $\Pr(\EE=\ee)$, we may take $q$ to be the uniform distribution and apply importance sampling as follows:
\begin{equation*}
    \begin{split}
        \Pr(\EE=\ee) & = \Ex_{\zz\sim p}[p(\ee\mid \zz)] = \Ex_{\zz\sim q}\left[ p(\ee\mid \zz)\dfrac{p(\zz)}{q(\zz)}\right] = \\ & = \Ex_{\zz\sim q}\left[ \dfrac{p(\ee,\zz)}{q(\zz)}\right] = \Ex_{\zz\sim q}[w_{\ee}(\zz)] \approx \dfrac{1}{T}\sum\limits_{t=1}^{T} w_{\ee}(\xx^t),
    \end{split}
\end{equation*}
where $w_{\ee}(\zz) = \dfrac{p(\ee,\zz)}{q(\zz)}$. Unlike rejection sampling, this will use all the samples; if $p(\zz\mid \ee)$ is not too far from uniform, this will converge to the true probability after only a very small number of samples.

\paragraph{Markov chain Monte Carlo.} Now let us turn to performing marginal and MAP inference using sampling. We will solve these problems using a very powerful technique called Markov chain Monte Carlo (MCMC).

A key concept in MCMC is that of a Markov chain, which is a sequence of random elements having Markov property (see \ref{sec:Markov_Processes}). A Markov chain $\ccX = (\XX_0, \XX_1, \XX_2,\dots)$ with each random vector $\XX_i$ taking values from the same state space $\Val(\ccX)$ is specified by the initial distribution $\Pr(\XX_0=\xx)$, $\xx\in\Val(\ccX)$, and the set of transition probabilities
\begin{equation*}
    \PCond(\XX_{k+1} = \xx' on \XX_{k} = \xx)
\end{equation*}
for $\xx,\xx'\in \Val(\ccX)$, which do not depend on $k$ (in this case the Markov chain is called homogeneous). Therefore, the transition probabilities at any time in the entire process depend only on the given state and not on the history of the process. In what follows, we consider finite state space only so we may assume $\Val(\ccX) = \{1,\dots,d\}$, unless stated otherwise.

If the initial state $\XX_0$ is drawn from a vector of probabilities $p_0$, we may represent the~probability $p_t$ of ending up in each state after $t$ steps as
\[
p_t = T^tp_0,
\]
where $T$ denotes the transition probability matrix with $T_{ij}= \PCond(\XX_{k+1} = i on \XX_{k} = j)$, $i,j\in \{1,\dots,d\}$, and $T^t$ denotes matrix exponentiation.
If the limit $\lim\limits_{t\to\infty}p_t=\pi$ exists, it is called a stationary distribution of the Markov chain. A sufficient condition for $\pi$ to be a stationary distribution is called detailed balance:
\[
\pi(j)T_{ij} = \pi(i)T_{ji}
\]
for all $i,j\in\Val(\ccX)$.

The high-level idea of MCMC is to construct a Markov chain whose states are joint assignments to the variables in the model and whose stationary distribution is equal to the~model probability $p$. Then, running the chain for a number of times, we obtain the sample from the distribution $p$. In order to construct such a chain, we first recall the~conditions under which stationary distributions exist. This turns out to be true under two sufficient conditions: \textit{irreducibility}, meaning that it is possible to get from any state $\xx$ to any other state $\xx'$ with positive probability in a finite number of steps, and \textit{aperiodicity}, meaning that it is possible to return to any state at any time. In the context of continuous variables, the Markov chain must be \textit{ergodic}, which is a slightly stronger condition than the~above. For the sake of generality, we will require our Markov chains to be ergodic.

At a high level, MCMC algorithms will have the following structure. They take as an~argument a transition operator $T$ specifying a Markov chain whose stationary distribution is $p$, and an initial assignment $\XX_0=\xx_0$ of the chain. An MCMC algorithm then performs the following steps:
\begin{enumerate}
    \item Run the Markov chain from $\xx_0$ for $B$ burn-in steps.
    \item Run the Markov chain for $N$ sampling steps and collect all the states that it visits.
\end{enumerate}
The aim of the burn-in phase is to wait until the state distribution is reasonably close to~$p$. Therefore, we omit the first $B$ states visited by the chain and then we collect a~sample from the chain of the size $N$. A common approach to set the number $B$ is to use a variety of heuristics to try to evaluate the extent to which a sample trajectory has ``mixed'', i.e.~when it is reasonably close to $p$ (see \cite{koller2009}). Also \cite{Geyer} advocates that burn-in is unnecessary and uses other ways of finding good starting points. \cite{Gelman} propose to discard the first half of generated sequences. We may then use these samples for Monte Carlo integration (or in importance sampling). We may also use them to produce Monte Carlo estimates of marginal probabilities. Finally, we may take the sample with the highest probability and use it as an estimate of the~mode (i.e.~perform MAP inference). 

Before we discuss two most important special cases, note that sampling-based methods have theoretical asymptotic justification. Therefore, their application for finite samples of reasonable size may lead to drastically inaccurate results, especially in sophisticated and complex models. Successful implementation heavily depends on how well we understand structure of the model as well as on intensive experimentation. It can also be achieved by combining sampling with other inference methods.

\paragraph{Metropolis-Hastings Algorithm.} The Metropolis-Hastings (MH) algorithm (\cite{Hastings}) is one of the first ways to construct Markov chains within MCMC. The MH method constructs a transition operator $T$  from two components:
\begin{enumerate}
    \item A transition kernel $q$ specified by the user. In practice, the distribution $q(\xx'\mid\xx)$ can take almost any form and very often it is a Gaussian distribution centered at $\xx$.
    \item An acceptance probability for moves proposed by $q$, specified by the algorithm as
    \[
    A(\xx'\mid\xx) = \min\left( 1,\frac{p(\xx)q(\xx'\mid\xx)}{p(\xx')q(\xx\mid\xx')}\right).
    \]
\end{enumerate}

At each step, if the Markov chain is in the state $\xx$, then we choose a new point~$\xx'$ according to the distribution $q$. Then, we either accept this proposed change with the probability $\alpha = A(\xx'\mid\xx)$, or with the probability $1-\alpha$ we remain at our current state. Notice that the acceptance probability encourages the chain to move towards more likely points in the distribution (imagine for example that $q$ is uniform); when $q$ suggests that we move into a low-probability region, we follow that move only a certain fraction of time. Given any $q$ the MH algorithm ensures that $p$ is a stationary distribution of the~resulting Markov Chain. More precisely, $p$ will satisfy the detailed balance condition with respect to the Markov chain generated by MH algorithm. This is a straight consequence of the~definition of $A(\xx'\mid\xx)$.

As the result we wish to build the Markov chain with a small correlation between sub\-se\-quent values, which allows to explore the support of the target distribution rather quickly. This correlation consists of two components. The higher the variance of $q$, the~lower the correlation between the current state and the newly chosen one, and the~lower the variance of $q$, the lower the correlation when we stay at the same state hitting the~low-probability region. To choose a good kernel $q$ we need to find good balance between the two. For multivariate distributions the covariance matrix for the proposal distribution should reflect the covariance structure of the target.

\paragraph{Gibbs sampling.} A widely-used special case of the Metropolis-Hastings methods is Gibbs sampling. It was first described in \cite{geman}. Suppose we have a finite sequence of random variables $X_1,\dots,X_n$. We denote the $i$-th sample as $\xx^{(i)} = (X_1^{(i)},\dots,X_n^{(i)})$. Starting with an arbitrary configuration $\xx^{(0)}$ we perform the procedure below.

Repeat until convergence for $t = 1,2,3,\dots$:
\begin{enumerate}
    \item Set $\xx \leftarrow \xx^{(t-1)}$
    \item For each variable $X_i$
    \begin{itemize}
        \item Sample $X_i' \sim P(X_i\mid X_{-i}^{})$
        \item Update $\xx \leftarrow (X_1^{(t)},\dots, X_{i-1}^{(t)},  X_i', X_{i+1}^{(t-1)},\dots, X_n^{(t-1)})$
    \end{itemize}
    \item Set $\xx^{(t)} \leftarrow \xx$
\end{enumerate}

By $X_{-i}$ we denote all the variables in our set except $X_i$. At each epoch of the step~2 only one site undergoes a possible change, so that successive samples for each iteration can differ in at most one coordinate. Note that at this step we use updated values of the variables for which we have already sampled new values. The sampling step is quite easy to perform because we only condition on variables from $X_i$-th Markov blanket, which consists of its parents, children and other parents of its children.

In \cite{geman} it was stated that the distribution of $\xx^{(t)}$ converges to~$\pi$ as $t\to\infty$ regardless of $\xx^{(0)}$. The only assumption is that we continue to visit each site which is obviously a necessary condition for convergence. As in case of any MCMC algorithm if we choose an arbitrary starting configuration there is a burn-in phase, for the~list of intuitions on how to decide how many samples we want to discard see \cite{gibbssampler}. To avoid the high correlation between successive samples in Gibbs sampler we can also take every $r$-th sample instead of all of them, which is rather a~question of heuristics and experimenting.

\section{Learning probabilities in BNs for incomplete data}
Here we again consider categorical distributions. Suppose we observe a single in\-comp\-lete case in our data, which we denote as $\dd\in\ccD$. Under the assumption of parameter independence, we can compute the posterior distribution of $\ttheta_{ij}$ for our network as follows:
\begin{equation*}
    p(\ttheta_{ij}\mid\dd) = (1-p(\ppa_i^j\mid\dd))\{ p(\ttheta_{ij})\} + \sum_{k=1}^{r_i}p(x_i^k,\ppa_i^j\mid\dd)\{ p(\ttheta_{ij}\mid x_i^k,\ppa_i^j)\}.
\end{equation*}
Each term in curly brackets in this equation is a Dirichlet distribution. Thus, unless both~$X_i$ and all the variables in $\ppa(X_i)$ are observed in case $\dd$, the posterior distribution of~$\ttheta_{ij}$ will be a linear combination of Dirichlet distributions, that is a Dirichlet mixture with mixing coefficients $(1-p(\ppa_i^j\mid\dd))$ and $p(x_i^k,\ppa_i^j\mid\dd)$, $1\leq k\leq r_i$. See \cite{spiegelhalter1990} for the details of derivation.

When we observe a second incomplete case, some or all of the Dirichlet components in the previous equation will again split into Dirichlet mixtures. More precisely, the posterior distribution for $\ttheta_{ij}$ will become a mixture of Dirichlet mixtures. As we continue to observe incomplete cases, where each case has missing values for the same set of variables, the~posterior distribution for $\ttheta_{ij}$ will contain a number of components that is exponential in the number of cases. In general, for any interesting set of local likelihoods and priors, the~exact computation of the posterior distribution for $\ttheta$ will be intractable. Thus, we require an approximation for incomplete data.

One of the possible ways to approximate is Monte-Carlo methods discussed previously, for example the Gibbs sampler, which must be irreducible and each variable must be chosen infinitely often. More specifically for our case, to approximate $p(\ttheta\mid\ccD)$ given an incomplete data set we start with some initial states of the unobserved variables in each case (chosen randomly or otherwise) and as a result, we have a complete random sample~$\ccD_c$. Then we choose some variable $X_{i}[l]$ (variable $X_i$ in case $l$) that is not observed in the original random sample~$\ccD$, and reassign its state according to the probability distribution
\begin{equation*}
    p(x_{il}'\mid\ccD_c\setminus\{x_{il}\}) = \dfrac{p(x_{il}',\ccD_c\setminus\{x_{il}\})}{\sum_{x_{il}''} p(x_{il}'',\ccD_c\setminus\{x_{il}\})},
\end{equation*}
where $\ccD_c\setminus x_{il}$ denotes the data set $\ccD_c$ with observation $x_{il}$ removed, and the sum in the~denominator runs over all states of the variable $X_{i}$. Both the numerator and denominator can be computed efficiently as in (\ref{eq:data_distribution_BN}). In the third step we repeat this reassignment for all unobserved variables in $\ccD$, producing a new complete random sample $\ccD'_c$. The~fourth step is to compute the posterior density $p(\ttheta_{ij}\mid\ccD'_c)$ as in (\ref{eq:posterior_Dir}) and, under the~assumption of parameter independence, the joint posterior $p(\ttheta\mid\ccD'_c)$ will be a product of all densities $p(\ttheta_{ij}\mid\ccD'_c)$. Finally, we iterate through last three steps, and use the average of $p(\ttheta\mid\ccD'_c)$ as our approximation.

Monte-Carlo methods yield accurate results but they are often intractable, for exam\-ple when the sample size is large. Another approximation that is more efficient than Monte-Carlo methods and often accurate for relatively large samples is the Gaussian approximation. The idea is that for large amounts of data we can approximate the~distribution $p(\ttheta\mid\ccD)\propto p(\ccD\mid\ttheta)p(\ttheta)$  as a multivariate-Gaussian distribution, namely
\begin{equation*}
    p(\ttheta\mid\ccD)\approx p(\ccD\mid\tilde{\ttheta})p(\tilde{\ttheta})\exp\left(-\frac{1}{2}(\ttheta-\tilde{\ttheta})H(\ttheta-\tilde{\ttheta})^{\top}\right),
\end{equation*}
where $\tilde{\ttheta}$ is the configuration of $\ttheta$ that maximizes $g(\ttheta) = \ln(p(\ccD\mid\ttheta)p(\ttheta))$ and $H$ is a negative Hessian of $g(\ttheta)$. The vector $\tilde{\ttheta}$ is also called the maximum a posteriori (MAP) configuration of $\ttheta$. There are various methods to compute the second derivatives proposed in literature (\cite{meng}, \cite{raftery}, \cite{thiesson}).

One more way to learn probabilities from incomplete data is the Expectation-Ma\-xi\-mi\-za\-tion (EM) algorithm. It is an iterative algorithm consisting of two alternating steps - Expectation and Ma\-xi\-mi\-za\-tion. When the~data is incomplete we cannot calculate the~likelihood function as in (\ref{eq:likelihoodBN}) and (\ref{eq:likelihoodBN_var}). Now instead of maximizing likelihood or log-likelihood function we will be maximizing \textit{the expected log-likelihood} of the complete data set with respect to the joint distribution for $\XX$ conditioned on the assigned configuration of the parameter vector $\ttheta'$ and the known data $\ccD$. The calculation of the expected log-likelihood (Expectation step) amounts to computing \textit{expected sufficient statistics}. For incomplete data the expected log-likelihood takes the following form
\begin{equation*}
    \Ex[\ell(\ttheta)\mid \ccD,\ttheta'] = \sum_{i=1}^n\sum_{l=1}^{q_i}\sum_{k=1}^{r_i}\hat{N}_{ilk}\log(\theta_{ilk}),
\end{equation*}
where
\begin{equation*}
    \hat{N}_{ilk} = \Ex[\Ind(X_i = x^k_i, \ppa(X_i) = \ppa_i^l)\mid\ccD,\ttheta'] = \sum_{j=1}^{m}\Pr(X_i = x^k_i, \ppa(X_i) = \ppa_i^l\mid\dd_j,\ttheta').
\end{equation*}
Here $\dd_j$ is possibly incomplete $j$-th case in $\ccD$. When $X_i$ and all the variables in $\ppa(X_i)$ are observed, the term for this case requires a trivial computation: it is either zero or one. Otherwise, we can use any Bayesian network inference algorithm discussed above to evaluate the term.

Having performed the Expectation step we want to find the new parameter vector, which is obtained by maximization of the expected log-likelihood (Maximization step). In our case we have new parameters on the $r$-th iteration
\begin{equation*}
    \theta_{ilk}^{r} = \dfrac{\hat{N}_{ilk}}{\sum_{k=1}^{r_i}\hat{N}_{ilk}}.
\end{equation*}
We start algorithm with an arbitrary (for example, random) parameter configuration $\ttheta^0$ and iteratively perform two steps described above until the convergence. \cite{dempster} showed that, under certain regularity conditions, iterations of the expectation and maximization steps will converge to a local maximum.

\section{Learning parameters for CTBNs}
The new method we propose in next chapters for learning CTBNs is capable of performing both tasks of parameter learning and structure learning simultaneously, although naturally these tasks can be performed separately. In this section we review selected methods focused only on parameter learning.

\subsection{Data}\label{subsec:data}
In this thesis we discuss both complete and incomplete data. In essence, CTBN models the joint trajectories of its variables, hence having complete, or fully observed, data means that for each point in time of each trajectory, we know the full instantiation to all variables.

By $\ccD = \{ \sigma[1],\dots,\sigma[m]\}$ we denote a data set of trajectories. In case of complete data each $\sigma[i]$ is a complete set of state transitions and the times at which they occurred. Another way to specify each trajectory is to assign a sequence of states $\xx_i\in\Val(\XX)$, each with an associated duration. 

In contrast to the definition of complete data, an incomplete data set can be represented by a set of one or more partial trajectories. A partially observed trajectory $\sigma \in\ccD$ can be specified as a sequence of \textit{subsystems} $S_i$ of $X$, each with an associated duration. A \textit{subsystem} $S$ describes the behaviour of the process over a subset of the full state space, i.e.~$\Val(S)\subset\Val(\XX)$. It is simply a nonempty subset of states of $\XX$, in which we know the system stayed for the duration of the observation. Some transitions are partially observed, i.e.~we know only that they take us from one subsystem to another. Transitions from one state to another within the subsystem are fully unobserved, hence, we do not know how many transitions there are inside of a particular subsystem nor when they do occur.

\subsection{Learning parameters for complete data}
Recall, that CTBN $\ccN$ consists of two parts. The first is an initial distribution $P_0^{\XX}$, specified as a Bayesian network over $\XX$. The second is a continuous transition model, specified as a directed (and possibly cyclic) graph and a set of conditional intensity matrices (CIM), one for each variable $X_i$ in the network.  For the purposes of this section we abbreviate $\parents(X_i)$ as $\ppa(X_i)$ and we denote CIMs as $\QQ_{X_i\mid\ppa(X_i)}$. Recall that each $\QQ_{X_i\mid\ppa(X_i)}$ consists of intensity matrices $\QQ_{X_i\mid \ppa_i}$, where $\ppa_i$ is a single configuration of $\ppa(X_i)$. Strictly speaking, $\ppa_i$ is one of the possible parent configurations $\ppa_i^1,\dots,\ppa_i^{q_i}$ similar to \eqref{categorical}. In terms of pure intensity parameterization we denote elements of these matrices as $q_{xx'\mid\ppa_i}$ and $q_{x\mid\ppa_i}$. Note, that by Theorem \ref{thm:mixed_intensity} we can divide the set of parameters in terms of mixed intensity into two sets. Then for each variable $X_i$ and each instantiation $\ppa_i$ of its set of parents $\ppa(X_i)$ the parameters of $\QQ_{X_i\mid \ppa(X_i)}$ will be $\qq_{X_i} = \{q_{x\mid \ppa_i}:x\in\Val(X_i)\}$ and $\ttheta_{X_i} = \{\theta_{xx'\mid \ppa_i}:x,x'\in\Val(X_i),\ x\neq x'\}$. More precisely, for each $X_i$ and every $x\in\Val(X_i)$ we have
\begin{equation*}
    \theta_{xx'\mid \ppa_i} = \frac{q_{xx'\mid \ppa_i}}{\sum_{x'}q_{xx'\mid \ppa_i}},\quad x'\in\Val(X_i),\quad x\neq x'.
\end{equation*}

The learning problem for the initial distribution is a Bayesian network learning task, which was discussed previously in this chapter. Therefore, it remains to learn the vector of parameters $(\qq,\ttheta)$.

\paragraph{Likelihood estimation.} Let us start from a fully observed case and a single homogeneous Markov process $X(t).$ As all the transitions are observed, the likelihood of $\ccD$ can be decomposed as a product of the likelihoods for individual transitions $d$. Let $d = \langle x_d,t_d,x'_d \rangle\in\ccD$ be the transition where $X$ transitions to state $x'_d$ after spending the~amount of time $t_d$ in state $x_d$. Using the mixed intensity parameterization, we can write the likelihood for the single transition $d$ as
\begin{equation*}
    L_X(\qq,\ttheta : d) = L_X(\qq : d)L_X(\ttheta : d) = q_{x_d}\exp(-q_{x_d} t_d)\cdot\theta_{x_d x'_{d}}.
\end{equation*}
Then multiplying the likelihoods for each transition $d$ in our data $\ccD$ we can summarize it in terms of sufficient statistics $T[x]$ which describes the amount of time spent in each state $x\in\Val(X)$ and $M[x, x']$ which encodes the number of transitions from $x$ to $x'$, where $x \neq x'$ as follows:
\begin{equation}\label{eq:q_theta_likelihood}
\begin{split}
    L_X(\qq,\ttheta : \ccD) & = \left(\prod_{d\in\ccD}L_X(\qq : d)\right)\left(\prod_{d\in\ccD}L_X(\ttheta : d)\right)\\
    & = \left(\prod_x q_{x}^{M[x]}\exp(-q_{x} T[x])\right)\left(\prod_x\prod_{x'\neq x}\theta_{x x'}^{M[x, x']}\right),
\end{split}
\end{equation}
where $M[x] = \sum\limits_{x'}M[x,x']$.

Now in case of CTBNs, each variable $X$ of the network $\ccN$ is conditioned on its parent set $\Pa = \parents(X)$, and each transition of $X$ must be considered in the context of the~instantiation $\ppa$ of $\Pa$. With complete data, we know the value of $\Pa$ during the~entire trajectory, so at each point in time we know precisely which homogeneous intensity matrix~$\QQ_{X\mid\ppa}$ governed the dynamics of $X$.

Thus, the likelihood decomposes into the product of likelihoods, each corresponding to the variable in the network, as
\begin{equation*}
    L_{\ccN}(\qq,\ttheta : \ccD) = \prod_{X_i\in\XX}L_{X_i}(\qq_{X_i\mid \U_i}, \ttheta_{X_i\mid \U_i} : \ccD) =  \prod_{X_i\in\XX} L_{X_i}(\qq_{X_i\mid \U_i} : \ccD)L_{X_i}(\ttheta_{X_i\mid \U_i} : \ccD).
\end{equation*}
The term $L_{X}(\ttheta_{X\mid \Pa} : \ccD)$ is the probability of the sequence of state transitions, disregarding the times between transitions. These state changes depend only on the value of the~parents at the moment of the transition. For each variable $X\in\XX$ let $M[x,x'\mid \ppa]$ denote  the~number of transitions from $X =x$ to $X=x'$ while $\Pa=\ppa$. Then, with this set of sufficient statistics $M[x,x'\mid \ppa]$, we have
\begin{equation*}
    L_{X}(\ttheta_{X\mid \Pa} : \ccD)=\prod_{\ppa}\prod_{x}\prod_{x'\neq x} \theta_{x x'\mid\ppa}^{M[x, x'\mid\ppa]}.
\end{equation*}
The computation of $L_{X}(\qq_{X\mid \Pa} : \ccD)$ is more subtle since the duration in the state can be terminated not only due to a transition of $X$, but also due to a transition of one of its parents. The total amount of time where $X = x$ and $\Pa = \ppa$ can be decomposed into two different kinds of durations $T[x\mid\ppa] = T_r[x\mid\ppa]+T_c[x\mid\ppa]$, where $T_r[x\mid\ppa]$ is the total length of the time intervals that terminate with $X$ remaining equal to $x$, and $T_c[x\mid\ppa]$ is the total length of the time intervals that terminate with a change in the~value of $X$. However, it is easy to show that we do not need to maintain the distinction between the two of them and we can use the set of $T[x\mid\ppa]$ as sufficient statistics.

Finally, we can write the log-likelihood as a sum of local variable likelihoods of the~form
\begin{equation}\label{eq:loglik_variable}
\begin{split}
    &\ell_X(\qq,\ttheta : \ccD) = \ell_X(\qq : \ccD)+\ell_X(\ttheta : \ccD) = \\ &= \left[\sum_{\ppa}\sum_x M[x\mid\ppa]\log q_{x\mid\ppa} -q_{x\mid\ppa}T[x\mid\ppa]\right] + \left[\sum_{\ppa}\sum_x\sum_{x'\neq x} M[x,x'\mid\ppa]\log\theta_{xx'\mid\ppa}\right].
    \end{split}
\end{equation}
Now we can write the maximum-likelihood (MLE) parameters as functions of the~sufficient statistics as follows (for the proof see \cite{Nod4}):
\begin{equation*}
        \hat{q}_{x\mid\ppa} = \frac{M[x\mid\ppa]}{T[x\mid\ppa]}, \qquad
        \hat{\theta}_{xx'\mid\ppa} = \frac{M[x,x'\mid\ppa]}{M[x\mid\ppa]}.
\end{equation*}

\paragraph{The Bayesian approach.} The other way to estimate parameters in case of fully observed data is the Bayesian approach. To perform Bayesian parameter estimation, si\-mi\-lar\-ly to the case of Bayesian networks, for computational efficiency we use a conjugate prior (one where the posterior after conditioning on the data is in the same parametric family as the prior) over the parameters of our CTBN.

For a single Markov process we have two types of parameters, a vector of parameters $\ttheta$ for categorical distribution and $q$ for exponential distribution. An appropriate conjugate prior for the exponential parameter $q$ is the Gamma distribution $P(q) = Gamma(\alpha, \tau)$, and as we mentioned in Section \ref{sec:CPDs}, the standard conjugate prior to categorical distribution is a Dirichlet distribution $P(\ttheta) = Dir(\alpha_{xx_1}, \dots,\alpha_{xx_k})$. The posterior distributions~$P(\ttheta\mid\ccD)$ and $P(q\mid\ccD)$ given data are Dirichlet and Gamma distributions, respectively.

In order to apply this idea to an entire CTBN we need to make two standard assumptions for parameter priors in Bayesian networks, \textit{global parameter independence}:
\begin{equation*}
        P(\qq, \ttheta) =  \prod_{X\in\XX}P(\qq_{X\mid\parents(X)}, \ttheta_{X\mid\parents(X)})
\end{equation*}
and \textit{local parameter independence} for each variable $X$ in the network:
\begin{equation*}
        P(q_{X\mid\Pa}, \ttheta_{X\mid\Pa}) =  \left(\prod_{x}\prod_{\ppa}P(q_{x\mid\ppa})\right)\left(\prod_{x}\prod_{\ppa}P(\ttheta_{x\mid\ppa})\right).
\end{equation*}
If our parameter prior satisfies these assumptions, so does our posterior, as it belongs to the same parametric family. Thus, we can maintain our parameter distribution in the~closed form, and update it using the obvious sufficient statistics $M[x,x'\mid\ppa]$ for $\ttheta_{x\mid\ppa}$ and $M[x\mid\ppa], T[x\mid\ppa]$ for $q_{x\mid\ppa}$.

Given a parameter distribution, we can use it to predict the next event, averaging out the event probability over the possible values of the parameters. As usual, this prediction is equivalent to using ``expected'' parameter values, which have the same form as the MLE parameters, but account for the ``imaginary counts'' of the hyperparameters:
\begin{equation*}
        \hat{q}_{x\mid\ppa} = \frac{\alpha_{x\mid\ppa} + M[x\mid\ppa]}{\tau_{x\mid\ppa} +T[x\mid\ppa]}, \quad \qquad
        \hat{\theta}_{xx'\mid\ppa} = \frac{\alpha_{xx'\mid\ppa} + M[x,x'\mid\ppa]}{\alpha_{x\mid\ppa} +M[x\mid\ppa]}.
\end{equation*}
Note that, in principle, this choice of parameters is only valid for predicting a single transition, after which we should update our parameter distribution accordingly. However, as is often done in other settings, we can approximate the exact Bayesian computation by ``freezing'' the parameters to these expected values, and use them for predicting an entire trajectory.

\subsection{Learning parameters for incomplete data}
Recall, that in case of Bayesian networks one of the methods to deal with missing data was Expectation-Maximization (EM) algorithm. Here we provide a concise description of the algorithm based on EM for CTBNs presented in detail in \cite{nodelman2012em}. We start again with reviewing the EM scheme for a single Markov process $X$, which is the~basis of the algorithm for CTBNs. Let $\ccD = \{\sigma[1],\dots,\sigma[m]\}$ denote the set of all partially observed trajectories of $X$.

For each partial trajectory $\sigma[i]\in\ccD$ we can consider the space $\HH[i]$ of possible completions of this trajectory. For every transition of $\sigma[i]$ each completion $h[i]\in \HH[i]$ specifies which underlying transition of $X$ occurred. Also it specifies all the entirely unobserved transitions of $X$. Combining $\sigma[i]$ and $h[i]$ gives us a complete trajectory $\sigma^{+}[i]$ over $X$. Note that, in a partially observed trajectory, the number of possible unobserved transitions is unknown. Moreover, there are uncountably many times at which each transition can take place. Nevertheless, we can define the set $\ccD^{+}= \{\sigma^{+}[1],\dots,\sigma^{+}[m]\}$ of completions of all of the partial trajectories in $\ccD$. For examples of completions see \cite{nodelman2012em}.

As we mentioned in the previous subsection, the sufficient statistics of the set of complete trajectories $\ccD^{+}$ for a Markov process are $T[x]$, the total amount of time that~$X$ stays in $x$, and $M[x, x']$, the number of times in which $X$ transitions from $x$ to $x'$. Applying logarithm to ($\ref{eq:q_theta_likelihood}$) we can write the log-likelihood $\ell_X(\qq,\ttheta : \ccD^{+})$  for $X$ as an expression of these sufficient statistics.

Let $r$ be a probability density over each completion in $\HH[i]$ which, in turn, yields a~density over possible completions of the data $\ccD^{+}$. We can write the expectations of the~sufficient statistics with respect to the probability density over possible completions of the data as $\overline{T}[x]$, $\overline{M}[x, x']$ and $\overline{M}[x]$. These expected sufficient statistics allow us to write the expected log-likelihood for $X$ as
\begin{equation*}
    \begin{split}
        \Ex_r[\ell_X(\qq,\ttheta : \ccD^{+})] & = \Ex_r[\ell_X(\qq : \ccD^{+})]+\Ex_r[\ell_X(\ttheta : \ccD^{+})] = \\
        & = \sum_x\left(\overline{M}[x]\ln(q_x) -q_x \overline{T}[x] +\sum_{x'\neq x}\overline{M}[x,x']\ln(\theta_{xx'}) \right).
    \end{split}
\end{equation*}
Now we can use the EM algorithm to find maximum-likelihood parameters $\qq, \ttheta$ of $X$. The EM algorithm begins with an arbitrary initial parameter assignment, $\qq^0, \ttheta^0$. It then repeats the two steps, Expectation and Maximization, updating the parameter set, until convergence. After the $k$-th iteration we start with parameters $\qq^k, \ttheta^k$. The Expectation step goes as following: using the current set of parameters,
we define for each $\sigma[i]\in\ccD$, the~probability density $r^k(h[i]) = p(h[i]\mid\sigma[i],\qq^k,\ttheta^k)$.
We then compute expected sufficient statistics $\overline{T}[x]$, $\overline{M}[x, x']$ and $\overline{M}[x]$ according to this posterior density
over completions of the data given the data and the model. Using the expected sufficient statistics we just have computed as if they came from a complete data set, we set $\qq^{k+1}$ and $\ttheta^{k+1}$ to be the~new maximum likelihood parameters for our model as follows
\begin{equation}\label{eq:maximization_single}
    q_{x}^{k+1} = \frac{\overline{M}[x]}{\overline{T}[x]}, \quad\qquad \theta_{xx'}^{k+1} = \frac{\overline{M}[x,x']}{\overline{M}[x]}.
\end{equation}
The difficult part in this algorithm is the Expectation Step. The space over which we are integrating is highly complex, and it is not clear how to compute the expected sufficient statistics in a tractable way.

In \cite{nodelman2012em} and \cite{Nod4} authors provided in detail the algorithm on how to compute expected sufficient statistics for an $n$-state homogeneous Markov process $X_t$ with intensity matrix $\QQ_X$ with respect to the posterior probability density over completions of the data given the observations and the current model. The statistics are computed for each partially observed trajectory $\sigma \in{\ccD}$ separately and then the results are combined.

A partially observed trajectory $\sigma$ is given as a sequence of $N$ subsystems so that the state is restricted to subsystem $S_i$ during the interval $[t_i, t_{i+1})$ for $0\leq i\leq N-1$. 
To conduct all the necessary computations, for each time $t$, the forward and backward probability vectors $\alpha_t$ and $\beta_t$ are defined, which include evidence of any transition at time~$t$, and also vectors  $\alpha^{-}_t$ and $\beta^{+}_t$, neither of which include evidence of a transition at time~$t$. The total expected time $\Ex[T[j]]$ is obtained by summing the integrals over all intervals of constant evidence $[v,w)$ with the subsystem $S$ to which the state is restricted on that interval. Each integrand is an expression containing $\alpha_v$, $\beta_w$ and $\QQ_S$. The computations for each integral are performed via the Runge-Kutta method of fourth order with an~adaptive step size.

Regarding the expected number of transitions $\Ex[M[x,x']]$ from the state $x$ to $x'$ discrete time approximations of $M[x, x']$ are considered which in the limit as the size of the discretization goes to zero yields an exact equation. As a result we get the sum of expressions where each summand is associated with a time interval. The overall expression for the expected number of transitions consists of two parts: the sum of products corresponding to intervals with partially observed transitions and containing $\alpha^{-}_t$ and~$\beta^{+}_t$ for different time points $t$ and the sum of integrals of practically identical form to those obtained for total expected time.

In order to compute $\alpha_t$ and $\beta_t$ a forward-backward style algorithm (\cite{rabiner}) over the entire trajectory is used to incorporate evidence and get distributions over the state of the system at every time $t_i$. If needed it is possible to exclude incorporation of the evidence of the transition from either forward or backward vector and also obtain~$\alpha^{-}_t$ and $\beta^{+}_t$. We can then write the distribution over the state of the system at time $t$ given all the evidence.

Continuous time Bayesian networks are a factored representation for homogeneous Markov processes, hence, extending the EM algorithm to them involves making it sensitive to a factored state space. As mentioned previously, the log-likelihood decomposes as the~sum of local log-like\-li\-hoods for each variable. With the sufficient statistics $T[x\mid \ppa]$, $M[x,x'\mid\ppa]$ and $M[x\mid\ppa]$ of the set of complete trajectories $\ccD^{+}$ for each variable $X$ in CTBN $\ccN$ the likelihood for each variable $X$ further decomposes as in (\ref{eq:loglik_variable}). By linearity of expectation, the expected log-likelihood function also decomposes in the same way. So we can write the expected log-likelihood $\Ex_r[\ell(\qq,\ttheta : \ccD^{+})]$ as a sum of terms, one for each variable $X$, in a similar form as (\ref{eq:loglik_variable}), but using the expected sufficient statistics $\overline{T}[x\mid \ppa]$, $\overline{M}[x,x'\mid\ppa]$ and $\overline{M}[x\mid\ppa]$.

The EM algorithm for CTBNs is essentially the same as for homogeneous Markov processes. We need only specify how evidence in the network induces evidence on the~induced Markov process, and how expected sufficient statistics in the Markov process give us the necessary sufficient statistics for CTBN.

The Maximization step is practically the same as in (\ref{eq:maximization_single}), we just use proper expected sufficient statistics for the CTBN case:
\begin{equation*}
    q_{x\mid\ppa}^{k+1} = \frac{\overline{M}[x\mid\ppa]}{\overline{T}[x\mid\ppa]}, \quad \theta_{xx'\mid\ppa}^{k+1} = \frac{\overline{M}[x,x'\mid\ppa]}{\overline{M}[x\mid\ppa]}.
\end{equation*}
The Expectation step is again more difficult and could be done by flattening the CTBN into a single homogeneous Markov process with a size of the state space exponential in the number of variables. Then we follow the method described above. However, as the~number of variables in the CTBN grows the process becomes intractable, so we are forced to use approximate inference.

We want this approximate algorithm to be able to compute approximate versions of the forward and backward messages $\alpha_t$ and $\beta_s$ and extract the relevant sufficient statistics from these messages efficiently. In the next subsection we review a cluster graph in\-fe\-rence algorithm which can be used to perform this type of approximate inference. Using obtained cluster beliefs (see below) we can compute $\alpha_{t_{i+1}}$ and $\beta_{t_i}$ and use them in the forward-backward message passing procedure. The cluster distributions are represented as local intensity matrices from which we can compute the expected sufficient statistics over families $X_i, \parents(X_i)$ as described above.

\section{Inference for CTBNs}

To gain the perspective on the whole concept of continuous time Bayesian networks and their power, similarly to Bayesian networks, we discuss the questions of inference although it is not the key subject of this thesis. We start with a discussion of the types of queries we might wish to answer and the difficulties of the exact inference.


Inference for CTBNs can take a number of forms. The common types of queries are:
\begin{itemize}
    \item querying the marginal distribution of a variable at a particular time or also the time at which a variable first takes a particular value,
    \item querying the expected number of transitions for a variable during a fixed time interval,
    \item querying the expected amount of time a variable stayed in a particular state during an interval.
\end{itemize}

 Previously we showed that we can view CTBN as a compact representation of a~joint intensity matrix for a homogeneous Markov process. Thus, at least in principle, we can use CTBN to answer any query that we can answer using an explicit representation of a~Markov process: we can form the joint intensity matrix and then answer queries just as we would do for any homogeneous Markov process.
 
The obvious flaw is that this approach for answering these queries requires us to generate the full joint intensity matrix for the system as a whole. The size of the matrix is exponential in the number of variables, making this approach generally intractable. The graphical structure of the CTBN immediately suggests that we perform the inference in a~decomposed way, as in Bayesian networks. Unfortunately, the problems are significantly more complex in this setting.

In \cite{Nod1} the authors describe an approximate inference algorithm based on ideas from clique tree inference, but without any formal justification for the~algorithm. More importantly, the algorithm covers only point evidence, meaning observations of the value of a variable at a point in time, but in many applications we observe a variable for an interval or even for its entire trajectory. Therefore, we shortly describe an approximate inference algorithm called Expectation Propagation (EP) presented in \cite{Nod3} that allows both point and interval evidence. The algorithm uses message passing in a cluster graph (with clique tree algorithms as a special case), where the clusters do not contain distributions over the cluster variables at individual time points, but over trajectories of the variables through a duration.

As we discussed in this chapter, in cluster graph algorithms we construct a graph whose nodes correspond to clusters of variables and then pass messages between these clusters to produce an alternative parameterization, in which the marginal distribution of the variables in each cluster can be read directly from the cluster. In discrete graphical models, when the cluster graph is a clique tree, two passes of the message passing algorithm produce the exact marginals. In generalized belief propagation (\cite{Yedidia}), message passing is applied to a graph which is not a clique tree, in which case the algorithm may not converge, and produces only approximate solutions. There are several forms of message passing algorithm as we have discussed in Subsection \ref{subsec:message_passing}. The algorithm of \cite{Nod3} is based on multiply-marginalize-divide scheme of \cite{Lauritzen}, which we now briefly review.

A cluster graph is defined in terms of a set of clusters $\ccC_i$, whose scope is some subset of the variables $\XX$. Clusters are connected to each other by edges, along which messages are passed. The edges are annotated with a set of variables called a sepset $S_{i,j}$, which is the set of variables in $\ccC_i\cap \ccC_j$. The messages passed over an edge between $\ccC_i$ and $\ccC_j$ are factors over the scope $S_{i,j}$. Each cluster $\ccC_i$ maintains a potential $\beta_i$, which is a factor reflecting its current beliefs over the variables in its scope. Each edge similarly maintains a message $\mu_{i,j}$ which encodes the last message sent over the edge. The potentials are initialized with a product of some subset of factors parameterizing the model (CIMs in our setting). Messages are initialized to be uninformative. Clusters then send messages to each other, and use incoming messages to update their beliefs over the variables in their scope. The message $m_{i\rightarrow j}$ from $\ccC_i$ to $\ccC_j$ is the marginal distribution $S_{i,j}$ according to $\beta_i$. The neighbouring cluster $\ccC_j$ assimilates this message by multiplying it into $\beta_i$, but avoids double-counting by first dividing by the stored message $\mu_{i,j}$. Thus, the message update takes the form $\beta_j\longleftarrow\beta_j\cdot\frac{m_{i\rightarrow j}}{\mu_{i,j}}$.

In the algorithm the cluster beliefs represent not the factors over values of random variables themselves, but rather cluster potentials and messages both encode measures over entire trajectories of the variables in their scope. The number of parameters grows exponentially with the size of the network, and thus we cannot pass messages exactly without giving up the computational efficiency of the algorithm. To address this issue \cite{Nod3} used the \textit{expectation propagation (EP)} approach of \cite{Minka}, which performs approximate message passing in cluster graphs. In order to get an approximate message each message $m_{i\rightarrow j}$ is projected into a compactly representable space so as to minimize the KL-divergence between the message and its approximation. To encode the cluster potentials CIMs are used. In order to apply the EP algorithm to clusters of this form some basic operations over CIMs need to be defined. They include CIM product and division, approximate CIM marginalization, as well as incorporating the evidence into CIM.


The message propagation algorithm is first considered for one segment of the trajectory with constant continuous evidence. Exactly the same as for Bayesian networks, this process starts with constructing the cluster tree for the graph $\ccG$. Note that cycles do not introduce new issues. We can simply moralize the graph connecting all parents of a~node with undirected edges and then make all the remaining edges undirected. If there~is a~cycle, it simply turns into a loop in the resulting undirected graph. Next we select a set of clusters $\ccC_i$. These clusters can be selected so as to produce a clique tree for the~graph, using any standard method for constructing such trees. We can also construct a~loopy cluster graph and use generalized belief propagation. We did not discuss this topic in the~thesis (for more details see \cite{koller2009}). The message passing scheme described in this section is the same in both cases. 

The algorithm iteratively selects an edge connecting the clusters $\ccC_i$  and~$\ccC_j$ in the~cluster graph and passes the message from  the former to the latter. In clique tree propagation the~order in which we chose edges was basically fixed, meaning that we started from leaves to roots performing an upward pass and then going in the opposite direction. In ge\-ne\-ra\-lized belief propagation, we might use a variety of message passing schemes. Convergence occurs when messages cease to affect the potentials which means that neighboring clusters~$\ccC_i$ and~$\ccC_j$ agree on the approximate marginals over the variables from~$S_{i,j}$.

Now we can generalize the algorithm for a single segment to trajectories containing multiple segments of continuous evidence. \cite{Nod3} applied this algorithm separately to every segment, passing information from one segment to the next one in the~form of distributions. More precisely, consider a trajectory defining a sequence of time points $t_1,\dots, t_n$, with constant continuous evidence on every interval $[t_i, t_{i+1})$ and possible point evidence or observed transition at each $t_i$. Then a sequence of cluster graphs over each segment is constructed. Starting from the initial segment EP inference is run on each cluster graph using the algorithm for a single segment described above, and the~distribution at the end time point of the interval is computed. The resulting distribution is then conditioned on any point evidence or the observed transition, and next used as the initial distribution for the next interval.

However, there is one subtle difficulty relating to the propagation of messages from one interval to another. If a variable $X$ appears in two clusters $\ccC_i$ and $\ccC_j$ in a cluster graph, the distribution over its values in these two clusters is not generally the same, even if the EP computation converges. The reason is that even calibrated clusters only agree on the projected marginals over their sepset, not the true marginals. To address this issue and to obtain a coherent distribution which can be transmitted to the next cluster graph the individual cluster marginals and sepsets for the state variables at the end time point of the previous interval are recalibrated to form a coherent distribution (the conditioning on point evidence can be done at the same time if needed). Then we can extract the new distribution as a set of calibrated cluster and sepset factors, and introduce each factor into the appropriate cluster or sepset in the cluster graph for the next time interval.

The above algorithm performs the propagation of beliefs forward in time. It is also possible to do a similar propagation backwards and pass messages in reverse, where the~cluster graph for one time interval passes a message to the cluster graph for the previous one. Also to achieve more accurate beliefs we can repeat the forward-backward propagation until the entire network is calibrated, essentially treating the entire network as a~single cluster graph. Note that since one cluster graph is used for each segment of fixed continuous evidence, then each cluster will approximate the trajectory of all the variables it contains as a homogeneous Markov process for the duration of the entire segment. Therefore, the choice of segments and the resulting subsets of variables, over which we compute the distribution, determine the quality of the approximation.
    \chapter{Structure learning for Bayesian networks}\label{chapter: BN_structure}
Recall the Definition \ref{BNdef} of Bayesian Networks (BN), the notion of which combines the~structure given by a Directed Acyclic Graph (DAG) and the probability distribution encoded by Conditional Probability Distributions (CPD). By far, in Chapter \ref{chapter: inference} we discussed the problem of finding CPDs and making the inference given the structure. In this chapter we will discuss the problem of learning the structure of Bayesian networks. In Section~\ref{sec:problemBN} we briefly review known approaches to the problem. In Section~\ref{sec:partition} we recall partition MCMC algorithm for learning the structure of the network, whose part concerning the division of the graph into layers will be the first step of our new method. In Sections~\ref{sec:GaussianBN} and \ref{discreteBN} we present a novel approach to structure learning with the use of the above algorithm and LASSO approach for continuous and discrete data, respectively. Section~\ref{sec:BNnumerical} is dedicated to numerical results.
\section{Problem of learning structure of Bayesian Networks}\label{sec:problemBN}
Structure learning is known to be a hard problem, especially due to the superexponential
growth of the DAG space when the number of nodes is increasing. Generally speaking the
literature on the structure learning can be divided into three classes: constraint-based methods, score-and-search algorithms and the dynamic programming approach (as discussed for example in \cite{koller2009}), even though this division is not that strict. The contents of this section come mostly from \cite{kuipers2017} and \cite{daly2011}.

Constraint-based methods use conditional independence tests to obtain information about the underlying causal structure. They start from the full undirected graph and then make decisions about removing the edge in the network based on tests of conditional independence. The widely used algorithm of this nature, PC algorithm (\cite{Spirtes2000}), and constraint-based methods in general are sensitive to the order in which they are run. However \cite{Colombo} proposed some modifications for PC algorithm to remove either partially or altogether this dependence. These methods scale well with the dimension but are sensitive to local errors of the independence tests which are used.

One of the most widely studied ways of learning a Bayesian network structure has been the use of
so-called 'score-and-search' techniques. These algorithms comprise of:
\begin{itemize}
    \item a search space consisting of the various allowable states, each of which
represents a~Bayesian network structure;
    \item a mechanism to encode each of the states;
    \item a mechanism to move from state to state in the search space;
    \item a scoring function assigning some score to a state in the search space which describes the goodness of fit with the sample data.
\end{itemize}
Also some hybrid methods combining ideas from both techniques were proposed, for example the max-min-hill-climbing of \cite{Tsamardinos}.

Within the family of search and score methods we can distinguish a separate class of MCMC methods for the graph space exploration. Their main and huge advantage is that they
can provide a collection of samples from the posterior distribution of the graph given the
data. This means that rather than making the inference based on a single graphical model, we can account for model uncertainty by averaging over all the models in the obtained class. In particular, we can estimate the expectation of any given network feature, such as the posterior probability of an individual edge, by averaging the posterior distributions under each of the models, weighted by their posterior model probabilities (\cite{MadiganYork}, \cite{kuipers2017}). This is especially important in high dimensional domains with sparse data where the single best model cannot be clearly identified, so the inference relying on the best scoring model is not justified.

The first MCMC algorithm over graph structures is due to \cite{MadiganYork}, later refined by \cite{Giudici}. To improve on the mixing and convergence, \cite{OrderMCMC} instead suggested to build a Markov chain on the space of node orders, at the price of introducing a bias in the sampling. For smaller systems with smaller space and time complexity one of the  efficient approaches is the dynamic prog\-ramming (\cite{Koivisto}), which can be further used to extend the proposals of standard structure MCMC approach in a hybrid method (\cite{Eaton}). Within the MCMC approach, to avoid the bias while keeping reasonable convergence rate, \cite{Grzegorczyk2008} more recently proposed a new edge reversal move method combining ideas both of standard structure and order based MCMC. Recently \cite{kuipers2017} presented another MCMC algorithm designed on the
combinatorial structure of DAGs, with the advantage of improving convergence with respect
to structure MCMC, while still providing an unbiased sample since it acts directly on the~space of DAGs. Moreover, it can also be combined with the algorithm of \cite{Grzegorczyk2008} to improve the convergence rate even further.


\section{Partition MCMC method}\label{sec:partition}
In this section we describe the Partition MCMC algorithm of \cite{kuipers2017}, which will be the base of our novel method for learning the structure of BNs. This algorithm considers combinatorial representation of DAGs to build an efficient MCMC scheme directly on the space of DAGs. Its convergence is better than that of the structure MCMC and does not introduce bias as the order based MCMC. As we mentioned, the~authors also proposed a way to combine their method with the new edge reversal move approach of \cite{Grzegorczyk2008} and improve upon their MCMC sampler.

First we need to introduce the notion of \textit{layers} and \textit{partitions} for DAG. Given DAG $\gr = (\ccV,\ccE)$ we define layers $\ell_i$ of the nodes (called interchangeably variables) in the network as follows:
\begin{itemize}
    \item $\ell_0 = \{v\in\ccV:\parents(v)=\emptyset\}$ is the layer of the nodes which do not have any parents;
    \item having defined the layer $\ell_i$ for $i=0, 1,\dots,k-1$ we define the next layer as
    \begin{equation*}
        \ell_k = \{v\in\ccV:\exists w\in\ell_{k-1} \text{\; such that \; } w\in \parents(v) \text{ and } \parents(v)\subseteq L_{k-1} \},
    \end{equation*}
    where $L_{k-1}=\bigcup\limits_{i\leq k-1}\ell_i$.
\end{itemize}
Note that variables from the same layer do not have arrows between them, and that each variable (except for the layer $\ell_0$) has at least one arrow directed towards it from any variable from the adjacent previous layer. For instance, the graph in Figure \ref{fig:partition} has three layers: $\ell_0 = \{1, 3, 5\}$, $\ell_1 = \{4\}$ and $\ell_2=\{2\}$.

Suppose that for some arbitrary graph we have $q+1$ layers. Each layer $\ell_i$ has a certain amount $k_i$ of nodes, which in sum gives the total number of nodes $d$, i.e.~$\sum\limits_{i=0}^{q}k_i= d$. In addition, with each layer representation there is associated a \textit{permutation} of nodes, where we list nodes in the layer order. More precisely, first we write nodes from the first layer, then from the second one, etc. For the graph in Figure \ref{fig:partition} we have the partition $\lambda = [3, 1, 1]$ and the permutation $\pi_{\lambda} = [1, 3, 5, 4, 2]$. Together a pair $(\lambda, \pi_{\lambda})$ is called \textit{a labelled partition}.

\begin{figure}[!ht]
\begin{center}
\includegraphics[height=0.35\textwidth, width=0.7\textwidth]{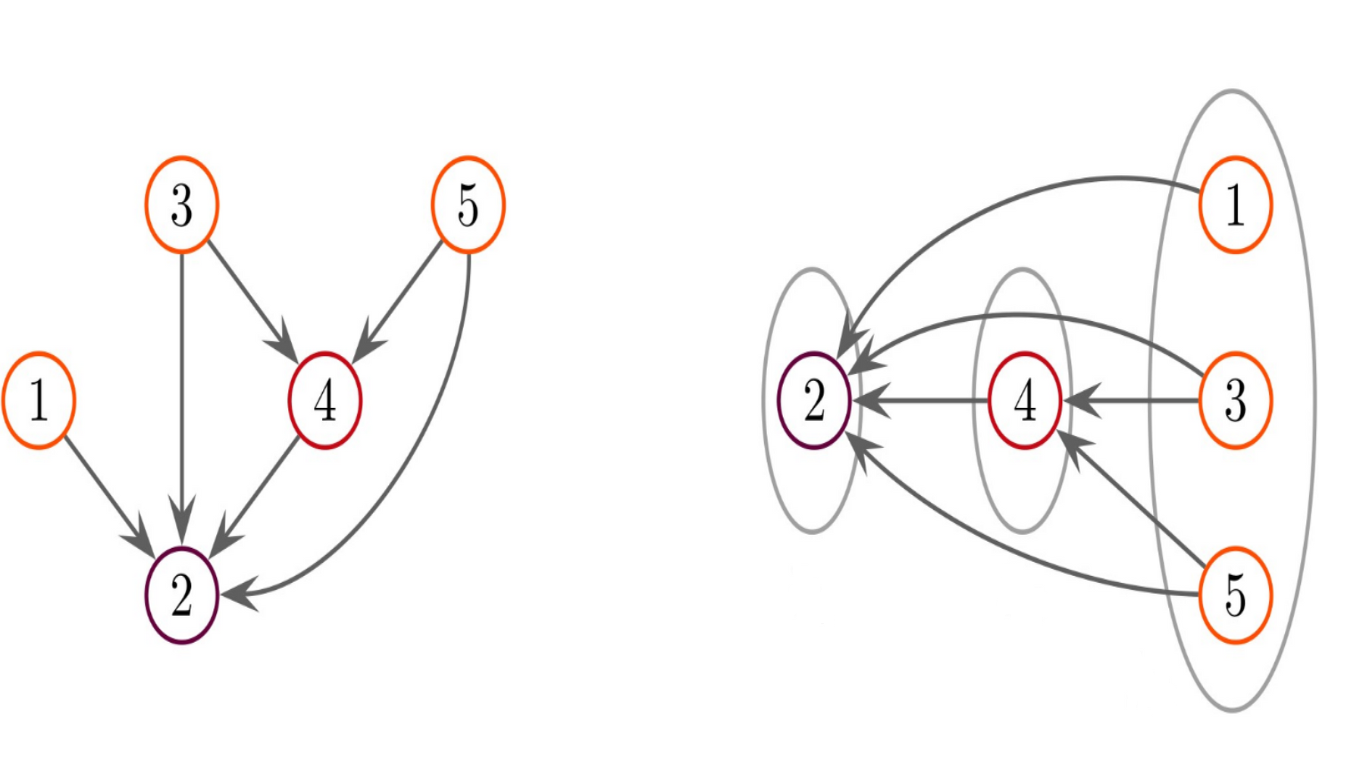}
\caption{An example of partition representation of the DAG.}
\label{fig:partition}
\end{center}
\end{figure}

\cite{kuipers2017} proposed an efficient MCMC algorithm for exp\-loring the~spa\-ce of partitions to find the most probable layer representation given the observed data. Although the full algorithm is suited for structure learning, we want to improve on this algorithm and replace the second part of it with the LASSO estimator. The~authors define an MCMC algorithm on the space of node partitions avoiding in this way over-representation of certain DAGs. Compared to other MCMC methods mentioned above partition MCMC is faster than structure MCMC of \cite{MadiganYork}. It is slower than order MCMC of \cite{OrderMCMC} but does not introduce any bias. The~basic move consists of splitting one element of the partition (i.e.~layer) into two parts or joining two adjacent elements (the authors also propose an additional move consisting of swapping two nodes in adjacent layers). All the partitions reachable from a given partition in one basic move are called the neighbourhood of the partition. So the MCMC scheme consists of sampling a partition from the neighbourhood of the previous partition with a~small probability to stay still defined by the user. The obtained partition is scored and the~score coincides with the posterior probability of the labelled partition. After sampling the partition we sample a single DAG weighted according to its posterior. Then we can average the acquired DAGs in the MCMC chain and choose the model. However, we propose to change the step where we sample DAG from the posterior distribution and average DAGs from the MCMC chain. It is well suited for inference and estimation of network parameters but we believe that we can improve the Bayesian averaging approach in the case of structure learning. We propose to use partition MCMC for finding the~best scoring partition and next to use it for recovering arrows with the LASSO estimator where each parameter corresponds to a certain arrow in the network.

\section{The novel approach to structure learning} \label{sec:GaussianBN}
We want to combine advantages of partition MCMC and LASSO for linear models. First we find the best layer representation using partition MCMC algorithm. Next we obtain the final DAG solving $d$ LASSO problems, where $d$ is the number of variables (nodes). Having found the most probable layer representation for a DAG we consider two models: one for continuous data and one for discrete data.

\subsection{Gaussian Bayesian Networks}
For the continuous case we consider Gaussian Bayesian Networks (GBN) introduced in Section \ref{sec:CPDs}. We denote as  $X_i^{m}$ the $m$-th random variable in the $i$-th layer, where $m\in\{1,\dots,k_i\}$. We assume that each $\epsilon_i^{m}$ has the normal distribution $\ccN(0,\sigma_i^{m})$. We also assume that each $\epsilon_i^{m}$ is independent of all $X_i^m$. Now given the partition $[k_0,  k_1, \dots, k_q]$ we can write the problem of finding the DAG structure as a set of the following $d$ linear model problems:
\begin{equation}\label{GBN linear problems}
    \begin{split}
        & X_0^{1} = \beta^{1}_{0,0} + \epsilon_0^{1} \\
        &\quad\quad \vdots\\
        & X_0^{k_0} = \beta^{k_0}_{0,0} + \epsilon_0^{k_0} \\
        & X_1^{1} = \beta^{1}_{1,0} + \beta^{1,1}_{1,0}X_0^{1} + \dots + \beta^{1,k_0}_{1,0}X_0^{k_0} + \epsilon_1^{1} \\
        & \quad\quad\vdots\\
        & X_1^{k_1} = \beta^{k_1}_{1,0} + \beta^{k_1,1}_{1,0}X_0^{1} + \dots + \beta^{k_1,k_0}_{1,0}X_0^{k_0} + \epsilon_1^{k_1} \\
        & X_2^{1} = \beta^{1}_{2,0} + \beta^{1,1}_{2,0}X_0^{1} + \dots + \beta^{1,k_0}_{2,0}X_0^{k_0} + \beta^{1,1}_{2,1}X_1^{1} + \dots + \beta^{1,k_1}_{2,1}X_1^{k_1} + \epsilon_2^{1} \\
        & \quad\quad\vdots\\
        & X_q^{k_q} = \beta^{k_q}_{q,0} + \sum\limits_{\substack{j<q,\\ 1\leq m_j\leq k_j}}\beta^{k_q,m_j}_{q,j}X_j^{m_j} + \epsilon_q^{k_q}.
    \end{split}
\end{equation}
Then the problem of finding DAG's structure is equivalent to the problem of finding non-zero parameters $\beta^{m_l,m_i}_{l,i}$. This corresponds to starting from the full possible graph and removing non-existing edges by shrinking the parameters to 0. It is possible due to the~fact that we have a partition, where we know which nodes can be parents for which nodes. This would not be possible otherwise because the graph has to be acyclic and we would have to introduce other constraints to the optimization problem. To solve this problem we will apply LASSO regression to each linear model, which tends to shrink the~coefficients to 0 by penalizing those coefficients with $\ell_1$-norm.

Let $m$ be the number of observations. Let $(X_{ik})$ be the matrix of observations, where $i \in \{1,\dots, m\}$ and $k \in \{1,\dots, d\}.$ Then by $X^j$ we denote the matrix formed by the~columns corresponding to the $j$-th layer. Similarly, by $X^{0:(j-1)}$ we denote the matrix which consists of the columns of the matrix of observations corresponding to the~first $j$~layers. By $X^{j}[i]$ we denote the column of the matrix $X^j$ corresponding to the $i$-th variable from the $j$-th layer and by $X_i$ we denote the row of the matrix $(X_{ik})$ corresponding to the~$i$-th observation. Moreover, let $\beta_{j}^{i} = [\beta_{j,1}^{i,1}, \dots, \beta_{j,1}^{i,k_0}, \beta_{j,2}^{i,1}, \dots, \beta_{j,j-1}^{i,k_{j-1}}]$, where $j=\{1,\dots,q\}$ and $i=\{1,\dots,k_j\}$.
To find the required vectors $\beta_j^i$ we solve the following $d$ optimization problems
\begin{equation}\label{eq:optimazationBN}
    \hat{\beta}_j^{i} = \argmin_{\theta\in\R^{k_0+\dots+k_{j-1}}}[RSS_{j,i}(\theta) + \lambda_{j,i}\|\theta\|_1],
\end{equation}
where $RSS_{j,i}(\theta) = 1/2\|X^{j}[i] - \theta^{\top}X^{0:(j-1)}\|_2^2$ is a residual sum of squares for the $i$-th variable in the $j$-th layer and  $\|\theta\|_1 = \sum\limits_{l=0}^{j-1} \sum\limits_{m_l=1}^{k_l}|\theta_{l}^{m_l}|$ is the $l_1$-norm of $\theta$. Note, that $\theta$ depends both on $j$ and $i$. The tuning parameters $\lambda_{j,i} >0$ balance the minimization of the cost function and the penalty function. The form of the penalty is crucial, because its singularity at the origin implies that some coordinates of the minimizer $\hat{\beta}_j^{i}$ are exactly equal to 0 if $\lambda_{j,i}$ is sufficiently large. Thus, starting from the graph with all possible arrows for the given layer representation (i.e.~there are arrows from variables on each layer towards all the variables in next layers) we remove irrelevant edges. The functions $RSS_{j,i}(\theta)$ and the penalty are convex, so \eqref{eq:optimazationBN} is a convex minimization problem. This is an important fact from both practical and theoretical points of view.

\subsection{Theoretical results for GBNs}
By $S_{j,i}$ we denote the support of the true vectors of parameters $\beta_j^i,$ i.e.~the sets of non-zero coordinates of each $\beta_j^i$, and by $S = \{S_{1,1},\dots, S_{1,k_1}, S_{2,1}, \dots, S_{q,k_q}\}.$ Moreover, $\beta^i_{j,\min}$ is the smallest in the absolute value element of $\beta_j^i$ restricted to $S_{j,i}$. The set $S_{j,i}^c$ denotes the complement of $S_{j,i}$, that is the set of zero coordinates of $\beta_j^i.$ For any vector $a$ we denote its $l_\infty$-norm by $\|a\|_\infty = \max_k |a_k|.$ For a vector $a$ and a subset of indices $I$ by~$a_{I}$ we denote the vector $a$ restricted to its coordinates from the set $I$, i.e.~$(a_{I})_i=a_i$ for $i\in I$ and $(a_{I})_i=0$ otherwise. Moreover, $|I|$ denotes the number of elements of $I$. For a vector $a = (a_1,\dots,a_n)$ by $Cov(a)$ we denote the matrix $(c_{ij})$, where $c_{ii} = Var(a_i)$ and $c_{ij} = Cov(a_i,a_j)$.

Before we state the main results of this chapter we introduce the cone invertibility factor (CIF), which plays an important role in the theoretical analysis of the properties of LASSO estimators. In literature there are three related notions which are commonly used in said analysis and help to provide some constraints on the optimized function so that the estimator is ``good'' in certain sense. These notions are the cone invertibility factor, compatibility factor and restricted eigenvalue (see \cite{Cox13} and references therein). For any $\xi>1$ we define the cones $\cone (\xi,S_{j,i}) = \left\{\theta: \|\theta_{S_{j,i}^c}\|_1 \leq \xi \|\theta_{S_{j,i}}\|_1\right\}$. Then CIF is defined as
\begin{equation}
\label{Fij_bar}
\bar{F}_{j,i}(\xi)= \inf _{0 \neq \theta \in \cone (\xi,S_{j,i})} \frac{\|\Sigma^j\theta\|_{\infty}}{\|\theta\|_\infty},
\end{equation}
where $\Sigma^j$ is the covariance matrix for a random vector $(X_1,\dots, X_{j-1})$ of variables from the first $j$ layers. More precisely,
\begin{equation*}
    \Sigma^j = \dfrac{1}{m}\left(X^{0:(j-1)}\right)^{\top}X^{0:(j-1)}.
\end{equation*}
Our goal will be to show that the estimators $\hat{\beta}_j^i$ are close to the true vectors $\beta_j^i$ in a certain sense. However, if the curvature of the function in \eqref{eq:optimazationBN}  around $\beta_j^i$ is relatively small, then the closeness between its values at $\hat{\beta}_j^i$ and $\beta_j^i$ does not necessarily imply the closeness between the arguments $\hat{\beta}_j^i$ and $\beta_j^i$. Hence, we require some additional conditions, for instance, strong convexity of $RSS_{j,i}$ at $\beta_j^i$, i.e.~that the smallest eigenvalue of its Hessian is positive. In the high-dimensional case it is too strong of an assumption, therefore one usually considers restricted strong convexity or restricted smallest eigenvalues, where ``restricted'' means that we take infimum over $\cone (\xi,S_{j,i})$ instead of the whole space. CIF \eqref{Fij_bar} is an example of such reasoning. We also introduce a non-random version $F_{j,i}(\xi)$ of CIF for each $j\in\{1,\dots,q\}$ as follows. First we define
\begin{equation*}
    H^j = Cov(X_1[0:(j-1)]),
\end{equation*}
where $X_1[0:(j-1)]$ denotes the restriction of $X_1$ to variables from the first $j$ layers. We assume that each $H_j$ is positive definite and elements on the diagonal are equal to 1, i.e.~$H^j_{ii} = 1$, where $i\in\{1, 2,\dots,k_1+\dots+k_{j-1}\}$. Then
\begin{equation}
\label{Fij}
{F}_{j,i}(\xi)= \inf _{0 \neq \theta \in \cone (\xi,S_{j,i})} \frac{\|H^j\theta\|_{\infty}}{\|\theta\|_\infty}.
\end{equation}
Since in a Gaussian Bayesian Network the joint probability of all variables is assumed to be Gaussian, then each marginal is Gaussian as well. Hence, for simplicity we can bound the variance for each variable by the same constant $\tau^2$. Also we denote
\begin{equation}\label{mji}
    m_{j,i} = \dfrac{|S|^2 \tau^4 (1+\xi)^2 \log(|L_{j-1}|^2q k_j /{\varepsilon})}{F^2_{j,i}(\xi)}
\end{equation}
for each $j\in\{1,\dots,q\}$ and $i\in\{1,\dots,k_j\}$.
\begin{theorem}\label{th:BNconsistency}
Fix arbitrary $\varepsilon \in (0,1)$ and $\xi >1$. Assume that ${F}_{j,i}(\xi)$ defined in \eqref{Fij} is positive for each $j\in\{1,\dots,q\}$ and $i\in\{1,\dots,k_j\}$. In addition suppose that 
\begin{equation}
\label{m_cond}
m \geq K_1 \max\limits_{j,i} m_{j,i}
\end{equation}
and for each $i$ and $j$ we have
\begin{equation*}\label{eq:lambda}
    \lambda_{j,i} \geq K_2\dfrac{\xi+1}{\xi - 1}\tau\sigma_j^i\sqrt{\dfrac{\log(|L_{j-1}|q k_j/\varepsilon)}{m_{j,i}}}
\end{equation*}
for some universal constants $K_1$ and $K_2$. Then
\begin{equation*}
    \Pr\left( \|\hat{\beta} - \beta\|_{\infty}\leq \frac{4\xi}{\xi+1}\max_{j,i}\dfrac{\lambda_{j,i}}{{F}_{j,i}} \right)\geq 1 - \varepsilon.
\end{equation*}
\end{theorem}
The second main result is about thresholded version of LASSO estimator. It will be proved after the proof of Theorem \ref{th:BNconsistency}. Consider the Thresholded LASSO estimator with the sets of nonzero coordinates $\hat{S}_{j,i}$. The set $\hat{S}_{j,i}$ contains only those coefficients of the LASSO estimator \eqref{eq:optimazationBN}, which are larger in the absolute value than some pre-specified threshold $\delta_{j,i}$ for each $j\in\{1,\dots, q\}$ and $i\in\{1,\ldots,k_j\}$. We denote $\{\hat{S}_{1,1}, \ldots, \hat{S}_{1,k_1}, \hat{S}_{2,1}, \ldots, \hat{S}_{q,k_q}\}$ as~$\hat{S}_{\delta}.$
\begin{corollary}
\label{thm:consistency2BN}
Suppose that assumptions of Theorem \ref{th:BNconsistency} are satisfied. If for each $j$, $i$ and arbitrary $\xi>1$ we have $ \beta^i_{j,min}/2> \delta_{j,i} \geq  \dfrac{4\xi \lambda_{j,i}}{(\xi+1){F}_{j,i}},$ then Thresholded LASSO with $\delta = [\delta_{1,1},\dots,\delta_{q,k_q}]$ is consistent, i.e.
\[
P\left( \hat{S}_{\delta} = S \right) \geq 1 - \varepsilon\,.
\]
\end{corollary}
Before the proof of Theorem \ref{th:BNconsistency} we state and prove an auxiliary result Proposition \ref{BN_main_lemma}, which is interesting on its own. It describes a slightly more general case and it will be used multiple times for different numbers and sets of predictors and targets in order to prove Theorem \ref{th:BNconsistency}. Moreover, to avoid any confusion with indices and notation introduced before for convenience we use more general notation in subsequent proofs.

Hence, let $(Y_1,Z_1), \ldots, (Y_m,Z_m)$ be  i.i.d.~random vectors such that $Y_i \in \mathbb{R}^p$ and $Z_i \in \mathbb{R}$. The coordinates of $Y_i$ will be denoted by $Y_{ij}$ for each $j=\{1,\ldots,p\}$ and by $Y$ we denote the full $(m\times p)-$matrix of predictors $Y=(Y_1,\ldots,Y_m)^{\top}$. Moreover, let  $H=Cov(Y_1)$ is a~positive definite matrix with diagonal elements $H_{jj}=1$ for $j=1,\ldots,p$. We assume that
\begin{equation}\label{model}
Z_i = \beta^{\top}Y_i +\varepsilon _i, \quad i =1,\ldots, m,
\end{equation}
where $\varepsilon_1, \ldots, \varepsilon_m$ are i.i.d.~random variables with $\Ex\varepsilon_i = 0$, which are subgaussian with the~parameter $\sigma^2$,  and are independent of $p$ predictors $Y_i,\ldots,Y_p$. Subgaussianity means that for each $i$ and $a \in \mathbb{R}$
$$
\Ex \exp(a \varepsilon_i) \leq \exp(a^2 \sigma ^2/2).
$$
We also assume that predictors are subgaussian with the parameter $\tau^2,$ i.e. $\Ex \exp(a Y_{1j}) \leq \exp(a^2 \tau ^2/2)$ for each $j=1,\ldots,p.$

The goal is to find the set of indices of the relevant predictors
\begin{equation}
\label{Sbeta}
S=\{j\in\{1,\dots, p\}: \beta_j\neq 0\}.
\end{equation}
The set $S^c$ denotes the complement of $S$, that is the set of zero coordinates of $\beta.$ Now consider the LASSO estimator  
\begin{equation}
\label{Zlasso}
 \hat \beta =\argmin_{\theta \in \mathbb{R}^p} [RSS(\theta)+\lambda \|\theta \| _{1}],
\end{equation}
where 
\begin{equation*}
\label{RSS_rand}
RSS(\theta) = \frac{1}{2m}\sum\limits_{i=1}^m \left(Z_i  - \theta^{\top} Y_i \right)^2.
\end{equation*}
For any $\xi>1$ we define the cone $\cone (\xi,S) = \left\{\theta: \|\theta_{S^c}\|_1 \leq \xi \|\theta_{S}\|_1\right\}$. Then CIF is defined as 
\begin{equation*}
\label{F_bar}
\bar{F}(\xi)= \inf _{0 \neq \theta \in \cone (\xi,S)} \frac{\|Y^{\top}Y\theta/m\|_{\infty}}{\|\theta\|_\infty},
\end{equation*}
and its non-random version is given by
\begin{equation*}
\label{F}
{F}(\xi)= \inf _{0 \neq \theta \in \cone (\xi,S)} \frac{\|H\theta\|_{\infty}}{\|\theta\|_\infty}.
\end{equation*}

\begin{proposition}\label{BN_main_lemma}
Fix arbitrary $a \in (0,1)$ and $\xi >1$. Suppose that $F(\xi)$ is positive and
\begin{equation}
\label{m_cond1}
m \geq \frac{K_1 |S|^2 \tau^4 (1+\xi)^2 \log(p^2/{a})}{{F}^2(\xi)}
\end{equation}
and
\begin{equation}\label{lambda_lemma}
    \lambda \geq K_2\dfrac{\xi+1}{\xi-1}\tau\sigma\sqrt{\dfrac{\log(p/{a})}{m}},
\end{equation}
where $K_1$, $K_2$ are some universal constants.  Then 
\begin{equation}\label{eq:probability_diff}
    \Pr\left( \|\hat{\beta} - \beta\|_{\infty}\leq \dfrac{4\xi\lambda}{(\xi+1){F}(\xi)}\right) > 1- 2a.
\end{equation}
\end{proposition}
The proof of Proposition \ref{BN_main_lemma} relies on Lemma \ref{termE} and \ref{infty} below.

\begin{lemma}
\label{termE} 
In the context of previously defined random variables $Y_{ij}$ and $\varepsilon_i$, where $i=\{1,\ldots, m\}$, for arbitrary $j = 1, \ldots,p$ and $u>0$ we have
\begin{equation*}
\label{pro_termE}
\Pr \left(\frac{1}{n}\sum_{i=1}^n Y_{ij} \varepsilon_i > 2
\tau \sigma \left(
 2 \sqrt{\frac{2u}{n}} + \frac{u}{n}
\right)
\right)\leq \exp(-u).
\end{equation*}
\end{lemma}
The proof of Lemma \ref{termE} uses the following Corollary 8.2 of \cite{Geer2016}.
\begin{lemma}
\label{geer_lemma}
Suppose that $Z_1, \ldots, Z_n$ are i.i.d.~random variables and there exists $L>0$ such that $C^2= \Ex \exp\left(|Z_1|/L\right)$ is finite. Then for arbitrary $u>0$
$$
\Pr\left( \frac{1}{n} \sum _{i=1}^n (Z_i - \Ex Z_i) > 2L \left(
 C \sqrt{\frac{2u}{n}} + \frac{u}{n}
\right)
\right) \leq \exp(-u).
$$
\end{lemma}

\begin{proof}[Proof of Lemma \ref{termE}]
Fix $j =1, \ldots,p$ and $u>0$. We consider an average of i.i.d.~centered random variables $Z_j = Y_{ij}\varepsilon_i$ with $\Ex Z_j = 0$, so we can use Lemma \ref{geer_lemma}. We need to find $L, C>0$  such  that $ \Ex \exp \left(|Y_{1j}\varepsilon_1|/L
\right) \leq C^2.$ Note that
\begin{equation}
\label{lem_form1}
 \Ex \exp \left(|Y_{1j}\varepsilon_1|/L \right)\leq \Ex \exp \left(Y_{1j}\varepsilon_1/L \right) + \Ex \exp \left(-Y_{1j}\varepsilon_1/L \right) .
\end{equation}
For the first term on the right-hand side of \eqref{lem_form1} we have
$$
\Ex \exp \left(Y_{1j}\varepsilon_1/L \right) = \Ex\left[ \Ex \left(\exp (Y_{1j}\varepsilon_1/L ) \mid Y_{1j}\right)\right].
$$
Using independence of $Y_{1j}$ and $\varepsilon_1,$ and subgaussianity of $\varepsilon_1$ for each $y\in\mathbb{R}$  we obtain $$
\Ex \left[\exp \left(Y_{1j}\varepsilon_1/L \right) \mid Y_{1j}=y \right]=
\Ex \exp \left(y\varepsilon_1/L \right) \leq 
\exp \left(y^2\sigma ^2/(2L^2) \right).
$$
Therefore we have
$$
\Ex\left[ \Ex \left(\exp (Y_{1j}\varepsilon_1/L ) | Y_{1j}\right)\right] \leq 
\Ex \exp \left(Y^2_{1j}\sigma ^2/(2L^2) \right),
$$
which, using subgaussianity of $Y_{1j}$ and Lemma 7.4 of \cite{Baraniuk}, we can bound from above by
$$
 \dfrac{1}{\sqrt{(1-\tau ^2 \sigma ^2/L^2)}},
$$
provided that $L>\tau \sigma.$ The second expectation on the right-hand side of \eqref{lem_form1} can be bounded analogously, hence, we obtain
$$
\Ex \exp \left(|Y_{1j}\varepsilon_1|/L \right) \leq  \dfrac{2}{\sqrt{(1-\tau ^2 \sigma ^2/L^2)}},
$$
provided that $L>\tau \sigma.$ We can take $L=2 \tau \sigma$ and obtain $C \geq \frac{2}{\sqrt[4]{3}},$ which finishes the~proof.
\end{proof}

\begin{lemma}
\label{infty}
Suppose that assumptions of Proposition \ref{BN_main_lemma} are satisfied. Then for arbitrary $\varepsilon \in (0,1)$ and $\xi >1$ with probability at least $1-\varepsilon$ we have $\bar{F}(\xi) \geq F(\xi)/2$.
\end{lemma}
\begin{proof}
Fix $\varepsilon \in (0,1)$ and $\xi >1$. We start with considering the $l_\infty$-norm of the matrix
$$
\left\|\frac{1}{m} Y^{\top} Y - \Ex Y_1^{\top} Y_1 \right\| _\infty = \max_{j,k=1,\ldots,p} 
\left|\frac{1}{m} \sum_{i=1}^m Y_{ij} Y_{ik} - \Ex Y_{1j} Y_{1k}  \right|  .
$$
Fix $j,k \in \{1,\ldots, p\}$. Notice that for any two numbers $a$ and  $b$ we have the inequality $ab \leq \frac{a^2}{2} + \frac{b^2}{2}.$ 
Hence, we can write
$$
|Y_{1j}Y_{1k}| \leq \frac{Y_{1j}^2}{2} +  \frac{Y _{1k}^2}{2}.
$$
Therefore, first using the previous inequality and Cauchy-Schwarz inequality afterwards for any positive constant $L$ we obtain
\begin{equation}
\label{lem_form3}
\begin{split}
    \Ex \exp \left(|Y_{1j} Y_{1k}|/L\right) & \leq \Ex \exp \left(Y_{1j}^2/(2L)
\right) \exp \left(Y_{1k}^2/(2L)\right) \\ 
& \leq \sqrt{\Ex \exp \left(Y_{1j}^2/L\right) \Ex \exp \left(Y_{1k}^2/L\right)}.
\end{split}
\end{equation}
The variable $Y_{1j}$ is subgaussian, so using Lemma 7.4 of \cite{Baraniuk} we can bound the first expectation under the square root in \eqref{lem_form3} from above by $\left( 1- \frac{ 2 \tau ^2}{L}\right)^{-1/2}$, provided that $2 \tau^2 <L$. The second expectation under the square root in \eqref{lem_form3} can be bounded by the same value when we use the subgaussianity of $Y_{1k}$. Therefore,
$$
\Ex \exp \left(
|Y_{1j} Y_{1k}|/L\right) \leq \left( 1- \frac{ 2 \tau ^2}{L}\right)^{-1/2} ,
$$
provided that $2 \tau^2 <L.$ Applying Lemma \ref{geer_lemma} with $L=3 \tau ^2$ and  $C = 2$ and $u = \log(p^2/\varepsilon)$ we obtain 
\begin{equation*}
\Pr\left(\left|\frac{1}{m} \sum_{i=1}^m Y_{ij} Y_{ik} - \Ex Y_{1j} Y_{1k}  \right|
> K \tau^2 \sqrt{\frac{\log(p^2/\varepsilon)}{m}}
\right) \leq \frac{\varepsilon}{p^2}\:,
\end{equation*}
where $K$ is an universal constant. 
Therefore,
\begin{equation}\label{lem_form4}
    \begin{split}
        & \Pr\left(\left|\frac{1}{m} Y^{\top} Y - \Ex Y_1^{\top} Y_1 \right| _\infty > K \tau^2 \sqrt{\frac{\log(p^2/\varepsilon)}{m}} \right) =\\
        & = \Pr\left(\max_{j,k=1,\ldots,p} \left|\frac{1}{m} \sum_{i=1}^m Y_{ij} Y_{ik} - \Ex Y_{1j} Y_{1k}  \right| > K \tau^2 \sqrt{\frac{\log(p^2/\varepsilon)}{m}} \right) \leq \\
        & \leq \sum\limits_{j,k} \Pr\left(\left|\frac{1}{m} \sum_{i=1}^m Y_{ij} Y_{ik} - \Ex Y_{1j} Y_{1k}  \right| > K \tau^2 \sqrt{\frac{\log(p^2/\varepsilon)}{m}} \right) \leq \varepsilon.
    \end{split}
\end{equation}
Proceeding similarly to the proof of Lemma 4.1 of \cite{Cox13} we first obtain
\begin{equation*}
    \begin{split}
        &\left|\left\|\frac{1}{m}Y^{\top}Y\theta\right\|_{\infty} - \|H\theta\|_{\infty} \right| \leq \left\|\frac{1}{m}Y^{\top}Y\theta - H\theta\right\|_{\infty} \leq \left\|\frac{1}{m}Y^{\top}Y - H\right\|_{\infty} \|\theta\|_1 = \\
        & = \left\|\frac{1}{m}Y^{\top}Y - H\right\|_{\infty} \left(\|\theta_S\|_1 + \|\theta_{S^c}\|_1\right)\leq (1+\xi)|S|\cdot\|\theta\|_{\infty}\left\|\frac{1}{m} Y^{\top} Y - \Ex Y_1^{\top} Y_1 \right\|_\infty.
    \end{split}
\end{equation*}
This implies that
\begin{equation*}
        \left\|\frac{1}{m}Y^{\top}Y\theta\right\|_{\infty} \geq \|H\theta\|_{\infty} - (1+\xi)|S|\cdot\|\theta\|_{\infty}\left\|\frac{1}{m} Y^{\top} Y - \Ex Y_1^{\top} Y_1 \right\|_\infty.
\end{equation*}
Then by dividing both sides by $\|\theta\|_{\infty}$, taking infimum with respect to $\theta$ over the cone~$\cone(\xi,S)$ and using \eqref{lem_form4} we derive that
$$
\bar{F}(\xi) \geq {F}(\xi) - K (1+\xi) |S|  \tau^2 \sqrt{\frac{\log(p^2/\varepsilon)}{m}}
$$
with probability higher than $1 - \varepsilon$. Finally, using \eqref{m_cond1} we have
$$
\bar{F}(\xi) \geq {F}(\xi) - \frac{K}{\sqrt{K_1}}F(\xi) = \left(1 - \frac{K}{\sqrt{K_1}}\right)F(\xi).
$$
We finish the proof by taking sufficiently large $K_1$.
\end{proof}

\begin{proof}[Proof of Proposition \ref{BN_main_lemma}.]
The central part of the proof is to show that
\begin{equation}\label{mainpart}
    \Pr\left( \|\hat{\beta} - \beta\|_{\infty}\leq \dfrac{2\xi\lambda}{(\xi+1)\bar{F}(\xi)}\right) > 1-a.
\end{equation}
 Let us denote $\Omega = \{\|\nabla RSS(\beta)\|_\infty\leq \frac{\xi-1}{\xi+1} \lambda\}.$ Now we want to bound from below the~pro\-ba\-bi\-lity of $\Omega$. For each $j = 1,\dots,p$ we can calculate $j$-th partial derivative of $RSS(\theta)$ at true $\beta$ 
\begin{equation*}
    \nabla_j RSS(\beta) = \dfrac{\partial RSS}{\partial \theta_j}(\beta) = \dfrac{1}{m}\sum\limits_{i=1}^{m}Y_{ij}\epsilon_i
\end{equation*}
and we bound it from above with high probability using Lemma \ref{termE}. Therefore, taking \eqref{lambda_lemma} into account we have
\begin{equation*}
\begin{split}  
        \Pr(\Omega) & = \Pr\left(\max\limits_{j}\left|\nabla_j RSS(\beta)\right| \leq \frac{\xi-1}{\xi+1} \lambda\right) = \Pr\left({\displaystyle\bigcap\limits_{j=1}^p} \left\{|\nabla_j RSS(\beta)| \leq \frac{\xi-1}{\xi+1} \lambda\right\}\right) = \\
        & = 1 - \Pr\left(\bigcup\limits_{j=1}^p \left\{|\nabla_j RSS(\beta)| > \frac{\xi-1}{\xi+1} \lambda\right\}\right) \geq 1 - \sum\limits_{j=1}^p \Pr\left( |\nabla_j RSS(\beta)| > \frac{\xi-1}{\xi+1} \lambda\right) \geq \\
        & \geq 1 - \sum\limits_{j=1}^p\Pr\left( |\nabla_j RSS(\beta)| > K_2 \tau\sigma\sqrt{\frac{\log(p/a)}{m}} \right).
\end{split}
\end{equation*}
Now applying Lemma \ref{termE} with $u=\log(p/2a)$ and appropriately chosen $K_2$ we bound from below this probability by $1-a$.

In further argumentation we consider only the event $\Omega.$ Besides, we denote $\tilde{\beta}= \hat{\beta} - \beta$ where $\hat{\beta}$ is a minimizer of a convex function given in \eqref{Zlasso}, which is equivalent to   
\begin{equation}
\label{min_equi}
\begin{cases}
    \dfrac{\partial RSS}{\partial \theta_j}(\hat{\beta}) = -\lambda \text{ sgn} (\hat{\beta}_j), & \quad\text{if } \hat{\beta}_j \neq 0,\\
    \left|\dfrac{\partial RSS}{\partial \theta_j}(\hat{\beta}) \right| \leq \lambda,  & \quad\text{if } \hat{\beta}_j = 0,
    \end{cases}
\end{equation}
where $j=1,\ldots,p$. Next we show that $\tilde{\beta}\in \cone(\xi,S)$. Our argumentation is analogous to \cite{YeZhang10}. From conditions in \eqref{min_equi} and the fact that $\|\tilde{\beta}\|_1 = \|\tilde{\beta}_S\|_1 + \|\tilde{\beta}_{S^c}\|_1$ we obtain
\begin{equation*}
    \begin{split}
        0 & \leq \tilde{\beta}^{\top}Y^{\top}Y\tilde{\beta}/m = \tilde{\beta}^{\top}\left[\nabla RSS(\hat{\beta}) - \nabla RSS(\beta) \right] = \\
        & = \sum\limits_{j\in S}\tilde{\beta}_j\nabla_j RSS(\hat{\beta}) + \sum\limits_{j\in S^c}\hat{\beta}_j\nabla_j RSS(\hat{\beta}) - \tilde{\beta}^{\top}\nabla RSS(\beta) \leq \\
        & \leq \lambda\sum\limits_{j\in S}|\tilde{\beta}_j| - \lambda \sum\limits_{j\in S^c}|\hat{\beta}_j| + \|\tilde{\beta}\|_1\|\nabla RSS(\beta)\|_{\infty} = \\
        & = [\lambda + \|\nabla RSS(\beta)\|_{\infty}]\|\tilde{\beta}_S\|_1 + [\|\nabla RSS(\beta)\|_{\infty} - \lambda] \|\tilde{\beta}_{S^c}\|_1 .
    \end{split}
\end{equation*}
Since we exclusively consider the event $\Omega$, we obtain the following inequality
\begin{equation*}
   \|\tilde{\beta}_{S^c}\|_1\leq \dfrac{\lambda + \|\nabla RSS(\beta)\|_{\infty}}{\lambda - \|\nabla RSS(\beta)\|_{\infty}}\|\tilde{\beta}_{S}\|_1 \leq \xi\|\tilde{\beta}_{S}\|_1.
\end{equation*}
Hence, we have just proved that $\tilde{\beta}$ belongs to the cone $\cone(\xi,S)$. Therefore from the~de\-fi\-ni\-tion of $\bar{F}(\xi)$ we have
\begin{equation*}
    \|\hat{\beta} - \beta\|_{\infty} \leq \dfrac{\|Y^{\top}Y(\hat{\beta} - \beta)/m\|_{\infty}}{\bar{F}(\xi)} \leq \dfrac{\|\nabla RSS(\hat{\beta})\|_{\infty} + \|\nabla RSS(\beta)\|_{\infty}}{\bar{F}(\xi)}.
\end{equation*}
Using the second condition in \eqref{min_equi} and the definition of the event $\Omega$ we then obtain \eqref{mainpart}. Finally, having shown \eqref{mainpart}, we apply Lemma \ref{infty} and obtain \eqref{eq:probability_diff} which finishes the proof.
\end{proof}
\begin{proof}[Proof of Theorem \ref{th:BNconsistency}]
In order to show that our estimator is close to the true parameter vector $\beta$ we first use union bounds. So here we have
\begin{equation*}
    \begin{split}
        \Pr\left( \|\hat{\beta} - \beta\|_{\infty}\leq \frac{4\xi}{(\xi+1)}\max_{j,i}\dfrac{\lambda_{j,i}}{{F}_{j,i}} \right) & \geq 
        \Pr\left(\bigcap\limits_{j,i}\left\{ \|\hat{\beta}_{j}^{i} - \beta_{j}^{i}\|_{\infty}\leq \dfrac{4\xi\lambda_{j,i}}{(\xi+1){F}_{j,i}}\right\} \right) = \\ 
        & = 1- \Pr\left(\bigcup\limits_{j,i}\left\{ \|\hat{\beta}_{j}^{i} - \beta_{j}^{i}\|_{\infty}> \dfrac{4\xi\lambda_{j,i}}{(\xi+1){F}_{j,i}}\right\} \right) \geq \\
        & \geq 1 - \sum\limits_{j=1}^{q}\sum\limits_{i=1}^{k_j}\Pr\left(\|\hat{\beta}_{j}^{i} - \beta_{j}^{i}\|_{\infty}> \dfrac{4\xi\lambda_{j,i}}{(\xi+1){F}_{j,i}} \right).
    \end{split}
\end{equation*}
Then, using Proposition \ref{BN_main_lemma} separately for each variable $X^i_j$ in each layer $\ell_j$ for $j=1.\ldots,q$  with $\lambda = \lambda_{j,i}$, with the number of predictors equal to $p = |L_{j-1}|$ and $a = \dfrac{\varepsilon}{q k_j}$ we obtain that the expression above can be bounded from below by $1-\varepsilon.$
The bound on the number of observations $m$ is chosen according to \eqref{mji} and \eqref{m_cond}.
\end{proof}

To prove Corollary \ref{thm:consistency2BN} we apply the same methodology. Namely, we prove an auxi\-liary lemma concerning the model described by \eqref{model}, so the set $S$ is defined by \eqref{Sbeta}. Additionally, by $\beta_{\min}$ we denote the smallest in the absolute value non-zero coordinate of the true parameter vector $\beta$. By $\hat{S}$ we denote the set of non-zero coordinates of the~Thresholded LASSO estimator with the level $\delta$, i.e.~the coordinates of the vector $\hat{\beta}$, which are greater than $\delta$.
\begin{lemma} 
\label{cor}
Fix $a\in(0,1)$ and $\xi>1$. Then under the assumptions of Proposition \ref{BN_main_lemma} and 
\begin{equation}
\label{delta}
 \frac{4 \xi  \lambda}{(\xi +1) F (\xi)} \leq \delta \leq \beta_{\min}/2
\end{equation}
we have
$$
\Pr\left( \hat{S} = S\right) \geq 1- a\,.
$$
\end{lemma}
\begin{proof}
Take any $j \notin S$. Then from Proposition \ref{BN_main_lemma} and \eqref{delta}  with the probability greater than $1- a$ we have
$$
|\hat{\beta}_j| =|\hat{\beta}_j - \beta_j|\leq \|\hat{\beta} - \beta \|_\infty \leq \delta.  
$$
Therefore, the $j$-th coordinate of Thresholded LASSO $\hat{\beta}_j^{\mathrm{TH}}=0$.
Next, we take $j \in S$ and obtain, also from Proposition \ref{BN_main_lemma} and \eqref{delta}, that with probability greater than $1-a$
 $$
|\hat{\beta}_j| \geq |\beta_j| -|\hat{\beta}_j - \beta_j|\geq \beta_{\min} - \|\hat{\beta} - \beta \|_\infty \geq \delta.
$$
Hence, $\hat{\beta}_j^{\mathrm{TH}}\neq 0$.
\end{proof}

\begin{proof}[Proof of Corollary \ref{thm:consistency2BN}]
From Lemma \ref{cor} for each $j\in\{1,\ldots,q\}$ and $i = \{1,\ldots,k_j\}$ under the assumptions of Theorem \ref{th:BNconsistency} we have that for arbitrary $a_{j,i}\in(0,1)$
\begin{equation*}
    \Pr\left( \hat{S}_{j,i} \neq S_{j,i}\right) < a_{j,i}.
\end{equation*}
Now we obtain
\begin{equation*}
    \Pr\left( \hat{S_{\delta}} \neq S\right) = \Pr\left(\bigcup\limits_{j,i}\{ \hat{S}_{j,i} \neq S_{j,i}\}\right)\leq\sum\limits_{j=1}^{q}\sum\limits_{i=1}^{k_j}\Pr\left( \hat{S}_{j,i} \neq S_{j,i}\right).
\end{equation*}
By taking $a_{j,i}=\dfrac{\varepsilon}{qk_j}$ we obtain the bound $\Pr( \hat{S}_{\delta} \neq S) <  \varepsilon$ and finish the proof.
\end{proof}

\section{Discrete case}\label{discreteBN}
As we discussed in Section \ref{sec:CPDs} in the discrete case as the distribution of the model we take a collection of categorical distributions for each variable. First we assume a binary case so that each $X_i\in\{0,1\}$ and we consider the \textit{logistic regression} model. Let us denote the \textit{sigmoid} function as $\sigma(x) = \dfrac{1}{1+e^{-x}}$. In this setting we can write probabilities for each variable in each layer similar to \eqref{GBN linear problems} as follows
\begin{equation}\label{Logistic model}
        \begin{split}
        \Pr&(X_1^{1}=1) = \sigma(\beta^{1}_{1,0} + \beta^{1,1}_{1,0}X_0^{1} + \dots + \beta^{1,k_0}_{1,0}X_0^{k_0}) \\
        &\vdots\\
        \Pr&(X_1^{k_1}=1) = \sigma(\beta^{k_1}_{1,0} + \beta^{k_1,1}_{1,0}X_0^{1} + \dots + \beta^{k_1,k_0}_{1,0}X_0^{k_0}) \\
        \Pr&(X_2^{1}=1) = \sigma(\beta^{1}_{2,0} + \beta^{1,1}_{2,0}X_0^{1} + \dots + \beta^{1,k_0}_{2,0}X_0^{k_0} + \beta^{1,1}_{2,1}X_1^{1} + \dots + \beta^{1,k_1}_{2,1}X_1^{k_1}) \\
        &\vdots\\
        \Pr&(X_q^{k_q}=1) = \sigma(\beta^{k_q}_{q,0} + \sum\beta^{k_q,m_j}_{q,j}X_j^{m_j}).
    \end{split}
\end{equation}
Using the same notation as for the continuous case we need to solve the fol\-lo\-wing~$d$~op\-ti\-mi\-zation problems
\begin{equation*}
    \hat{\beta}_j^i = \argmin_{\theta\in\R^{k_0+\dots+k_{j-1}}}[\ell_{j,i}(\theta) + \lambda_{j,i}\|\theta\|_1], \quad j= 1,\dots, q, \quad i= 1,\dots,k_j,
\end{equation*}
where $\ell_{j,i}$ is the negative log-likelihood for the $i$-th variable in the $j$-th layer and has the~following form
\begin{equation*}
    \ell_{j,i}(\theta) = - \sum\limits_{l=1}^{m} X_{{(p+i)}l} \log\left[\sigma\Big({\theta_j^i}^{\top}X^{0:(j-1)}\Big)\right] + (1-X_{(p+i)l})\log\left[1-\sigma\Big({\theta_j^i}^{\top}X^{0:(j-1)}\Big)\right].
\end{equation*}
Here we denote by $p=p(j) = k_0+\dots +k_{j-1}$ the number of variables in the pre\-vious~$j-1$~layers.

 We can also generalize the above case to the case where each variable has a discrete and finite state space, namely each $X_j^i\in\{1,\dots, N_j^i\}.$ Now instead of the sigmoid function we use the so-called \textit{softmax} function. For any vector $\boldsymbol{a} = (a_1,\dots,a_n)$ we define the softmax function $\sigma(\boldsymbol{a})$ as the vector $\sigma(\boldsymbol{a}) = (\sigma(\boldsymbol{a})[1],\dots, \sigma(\boldsymbol{a})[n])$, where $\sigma(\boldsymbol{a})[i] = \dfrac{\exp(a_i)}{\sum_{j=1}^{n}\exp(a_j)}$. 
We denote as $\XX_j = (X_0^1, X_0^2,\dots, X_0^{k_0},X_1^1,\dots, X_{j-1}^{k_{j-1}})^{\top}$ for $j=1,\dots,q$. Also we denote the vectors of parameters corresponding to the $l$-th class of the $i$-th variable in the $j$-th layer as $$\beta_j^i[l] = (\beta_{j,0}^{i,1}[l],\dots,\beta_{j,0}^{i,k_0}[l], \beta_{j,1}^{i,1}[l],\dots, \beta_{j,j-1}^{i,k_{j-1}}[l])$$ for $j=0,\dots,q-1$, $i=1,\dots,k_j$ and $l=1,\dots, N_{j}^i$.
Then the model analogous to the~logistic model in \eqref{Logistic model} takes the form
\begin{equation*}\label{Multinomial model}
        \begin{split}
        \Pr&(X_1^{1}=1)  = \sigma(\beta^{1}_{1,0}[1] + \beta_1^1[1]\XX_1,\dots,\beta^{1}_{1,0}[N_1^1] + \beta_1^1[N_1^1]\XX_1)[1] \\
        &\vdots\\
        \Pr&(X_1^{1}=N_1^1) = \sigma(\beta^{1}_{1,0}[1] + \beta_1^1[1]\XX_1,\dots,\beta^{1}_{1,0}[N_1^1] + \beta_1^1[N_1^1]\XX_1)[N_1^1] \\
        &\vdots\\
        \Pr&(X_1^{k_1}=1) = \sigma(\beta^{k_1}_{1,0}[1] + \beta^{k_1}_{1}[1]\XX_1, \dots, \beta^{k_1}_{1,0}[N_1^{k_1}] + \beta^{k_1}_{1}[N_1^{k_1}]\XX_1)[1] \\
        &\vdots\\
        \Pr&(X_j^{i}=l) = \sigma(\beta^{i}_{j,0}[1] + \beta^{i}_{j}[1]\XX_{j}, \dots, \beta^{i}_{j,0}[N_j^i] + \beta^{i}_{j}[N_j^i]\XX_{j})[l] \\
        &\vdots\\
        \Pr&(X_q^{k_q}=N_q^{k_q}) = \sigma(\beta^{k_q}_{q,0}[1] + \beta^{k_q}_{q}[1]\XX_{q}, \dots, \beta^{k_q}_{q,0}[N_q^{k_q}] + \beta^{k_q}_{q}[N_q^{k_q}]\XX_{q})[N_q^{k_q}].
    \end{split}
\end{equation*}
This is called \textit{multinomial logistic regression}. It is not difficult to notice that logistic regression is a particular case of multinomial logistic regression with two possible classes. For each variable $X_j^i$ we denote the full vector of parameters $\beta_j^i = (\beta_j^i[1],\dots, \beta_j^i[N_j^i])$. Then we need to solve $d$ optimization problems analogous to the case of logistic regression
\begin{equation*}
    \hat{\beta}_j^i = \argmin_{\theta\in\R^{(k_0+\dots+k_{j-1})N_j^i}}[\ell_{j,i}(\theta) + \lambda_{j,i}\|\theta\|_1], \quad j= 1,\dots, q, \quad i= 1,\dots,k_j,
\end{equation*}
where $\ell_{j,i}$ is also the negative log-likelihood for the $i$-th variable in the $j$-th layer and in this case has the following form
\begin{equation*}
    \ell_{j,i}(\theta) = - \sum\limits_{l=1}^{m}\sum\limits_{k=1}^{N_j^i} \Ind(X_{(p+i)l} = k) \left[\theta_j^i[l]X^{0:(j-1)} - \log\left(\sum\limits_{l=1}^{N_j^i}\theta_j^i[l]X^{0:(j-1)}\right)\right],
\end{equation*}
where we again denoted by $p=p(j) = k_0+\dots +k_{j-1}$ the number of variables in the~previous $j-1$ layers.
\section{Numerical results}\label{sec:BNnumerical}
In this section we describe the details of algorithm implementation as well as the results of experimental studies comparing our algorithm to others.

\subsection{Details of implementation}\label{BNimplementation}
We provide in details practical implementation of the proposed algorithm. The solution of~\eqref{eq:optimazationBN} depends on the choice of $\lambda_{j,i}$. Finding the ,,optimal'' parameters $\lambda_{j,i}$ and the~thresholds~$\delta_{j,i}$ in practice is difficult. We solve it using the information criteria \citep{Xueetal12, pokmiel:15, Rejchel18}. 

 First, recall the function which is being minimized in \eqref{eq:optimazationBN}
 \[
   RSS_{j,i}(\theta) + \lambda_{j,i}\|\theta\|_1  = \dfrac{1}{2}\left\|X^{j}[i] - \theta^{\top}X^{0:(j-1)}\right\|_2^2 + \sum\limits_{l=0}^{j-1} \sum\limits_{m_l=1}^{k_l}|\theta_{l}^{m_l}|,
 \]
with $X^{j}[i]$ being the vector of the length $m$ of observations for the $i$-th variable in the~$j$-th layer. We perform the optimization separately for each variable and the vector $\theta$ is from $\R^{k_0+\ldots+k_{j-1}}$ for $j=1,\dots,q$ and $i=1,\dots,k_{j}$. In our implementation we use the~following scheme. 
We start with computing a sequence of minimizers on the grid, i.e.~for any~$j$ and~$i$ we create a finite sequence $\{\lambda_k\}_{k=1}^{N}$ uniformly spaced on the log scale, starting from the~largest $\lambda_k$, which 
corresponds to the empty model. Next, for each value $\lambda_k$ we compute the estimator $\hat{\beta}^{i}_{j}[k]$ of the vector ${\beta}^{i}_{j}$
 \begin{equation}
\label{lassoijk}
\hat{\beta}^{i}_{j}[k] = \argmin_{\theta\in\R^{k_0+\ldots,k_{j-1}}} \left\{RSS_{j,i}(\theta)+\lambda_k\|\theta\|_1\right\}.
\end{equation}
 To solve \eqref{lassoijk} numerically for a given  $\lambda_k$ we use the FISTA algorithm with backtracking from \cite{FISTA}.
 The final LASSO estimator $\hat{\beta}_{j}^{i}:=\hat{\beta}_{j}^{i} [{k^*}]$ is chosen using  the Bayesian Information Criterion (BIC), which is a popular method of choosing $\lambda_{j,i}$ in the literature  \citep{Xueetal12, Rejchel18}, i.e.
\[
 k^*=\argmin_{1 \leq k \leq N} \left \{m \log(RSS(\hat{\beta}_{j}^{i}[k]))+\log(m)\Vert \hat{\beta}_{j}^{i}[k] \Vert_0\right\}.
\]
Here $\Vert \hat{\beta}_{j}^{i}[k]\Vert_0$ denotes the  number of non-zero elements of  $\hat{\beta}_{j}^{i}[k]$ and $m$ is the number of observations of the network. In our simulations we use $N=100$.

 Finally, the threshold $\delta$ is obtained using the Generalized Information Criterion (GIC). A similar way of choosing a threshold was used previously in \cite{pokmiel:15, Rejchel18}.
For a prespecified sequence of thresholds $\mathscr{D}$ we calculate  
\[
\delta^*_{j,i} =\argmin_{\delta \in \mathscr{D}} \left \{m \log(RSS(\hat{\beta}_{j,\delta}^{i}))+\log(k_0+\dots+k_{j-1})\Vert \hat{\beta}_{j,\delta}^{i} \Vert_0\right\},
\]
where $\hat{\beta}_{j,\delta}^{i}$ is the LASSO estimator $\hat{\beta}_{j}^{i}$ after thresholding with the level $\delta.$

\subsection{Experiments}

In this subsection we compare our algorithm to other algorithms developed for this problem applying them to benchmark networks. We use the \texttt{bnlearn} package in \texttt{R} (\cite{bnlearn}), in which many algorithms for learning Bayesian networks including structure learning are implemented. Algorithms of different types discussed in the beginning of this chapter such as constraint-based algorithms, score-and-search algorithms and hybrid algorithms can be found there. The choice of specific algorithms was made empirically, i.e.~we selected the best performing ones on the chosen networks. We took the networks with continuous data of a medium, large and very large amount of nodes and arcs. We refer to medium, large and very large sizes as 20-50 nodes, 50-100 nodes or 100-1000 nodes, respectively, adopting this classification from the authors of \texttt{bnlearn} package.


We chose a medium-size network $ECOLI70$ with 46 nodes and 70 arcs (\cite{ecoli}), a large network $MAGIC$-$IRRI$ * with 64 nodes and 102 arcs and a~very large network $ARTH150$ with 107 nodes and 150 arcs (\cite{arth150}). The algorithms chosen for comparison are hill-climbing (\texttt{hc}) algorithm, tabu search (\texttt{tabu}), max-min hill-climbing (\texttt{mmhc}) and Hybrid HPC (\texttt{h2pc}) algorithm. Hill-climbing (\cite{Scutari}) is a greedy search algorithm that explores the space of the~directed acyclic graphs (DAGs) by an addition, removal or reversal of a single edge and uses random restarts to avoid local optima. Tabu search (\cite{russel2010}) is a~modified hill-climbing method, which is able to escape local optima by selecting a~network that minimally decreases the score function. Both methods above use search-and-score approach. Max-min hill-climbing algorithm (\cite{Tsamardinos}) is a~hybrid method combining a constraint-based algorithm called \textit{max-min parents and children} and hill-climbing. H2PC (Hybrid HPC, \cite{h2pc}) algorithm is a hybrid algorithm combining an ensemble of weak PC learners (\cite{spirtes}) and hill-climbing. For more different comparisons of methods in \texttt{bnlearn} package see \cite{Scutari}.

\bigskip
*The model $MAGIC$-$IRRI$ was developed as an example of multiple trait modelling in plant genetics for the invited talk ``Bayesian Networks, MAGIC Populations and Multiple Trait Prediction'' delivered by Marco Scutari, the author of \texttt{bnlearn} package, at the 5th International Conference on Quantitative Genetics (ICQG 2016).
\newpage
For each network we used two sizes of the data set with $m=300$ and $m=1000$ observations. In the tables with results we denoted them with 2-3 first letters of the~name of the network followed by the number of observations so it does not create any confusion. For each algorithm we ran the experiment for 100 times, each time with a new set of~$m$~ob\-ser\-vations, and averaged the results in terms of three performance measures:
\begin{itemize}
    \item \textbf{power,} i.e.~the proportion of correctly discovered edges.
    
    \item \textbf{false discovery rate (FDR),} i.e.~the fraction of incorrectly selected edges among all selected edges.
    
    \item \textbf{structural Hamming distance (SHD),} i.e.~the smallest number of operations (such as adding or removing the edge and changing the direction of the arrow) required to match the true DAG and a learned one.
\end{itemize}
In Tables \ref{tab:resultspowerG}-\ref{tab:resultsshdG} we provide the results of experiments for mentioned above data sets and methods including ours. In terms of power our algorithm performs right in the middle of score-and-search and hybrid methods for $ECOLI70$ data set, similarly to the hybrid methods in case of small $MAGIC$-$IRRI$ data sets. We note that for data sets of 1000 observations it performs worse than other methods, however, with the number of observations growing the algorithm's power grows as well. With the number of observations of 10000 it grows up to 0.5 performing as good as other score methods without any increase in FDR. The same situation we observe with $ARTH150$ data set.

In terms of FDR our algorithm performs the best consistently giving very low numbers for false discoveries. This is especially important when the cost of a false discovery is high and makes obtained discoveries more certain. With the numbers of observations of 10000 we constantly get numbers in the~range 0.2-0.4\% for all data sets. When it comes to structural Hamming distance (SHD) our algorithm performs the best or close to the best numbers as well. With the growing number of observations it decreases due to increasing power and consistently low FDR. For the number of observations of 10000 it outperforms other algorithms or has close SHDs to hybrid methods and reaches around 28, 62 and 86 for $ECOLI70$, $MAGIC$-$IRRI$ and $ARTH150$ data sets, respectively.

We also checked our method on a discrete binary network called $ASIA$, introduced in Chapter \ref{chapter: inference}. It is a small network of 8 nodes and 8 edges. Our algorithm recognizes 6 arrows and makes 2 false discoveries, discovering 8 arrows in total. However, after a~closer look we noticed that it could not recognize 2 arrows due to incorrect assignment of layers. Finally, one false discovery was an arrow of an opposite direction to the true one, and the~other one was an arrow from the start to the end of a causal trail (we obtained an additional arrow $X\rightarrow W$ for the trail $X\rightarrow Y\rightarrow Z\rightarrow W$), hence still recovering dependencies in both cases.

\newpage

\begin{table}[h]
\centering
\begin{tabular}{lll c rrr }
  \toprule
 Method  & $EC300$ & $EC1000$ & $MAG300$ & $MAG1000$ & $AR300$ & $AR1000$\\
  \midrule
\texttt{hc} & 0.57   & 0.65 & 0.28 & 0.45 & 0.58 & 0.67\\       
\midrule
\texttt{tabu} & \textbf{0.6}   & \textbf{0.7} & \textbf{0.31}& \textbf{0.5} &  \textbf{0.59} &  \textbf{0.68}\\ 
\midrule
\texttt{mmhc} & 0.39   & 0.45 & 0.25 & 0.46 &  0.48 &  0.58\\
\midrule
\texttt{h2pc} & 0.4   & 0.49 & 0.23 & 0.45 &  0.5 & 0.61\\
\midrule
\text{MCMC + LASSO} & 0.49   & 0.55 & 0.24 & 0.38&  0.38&  0.46\\
 \bottomrule
\end{tabular}
\caption{Average power for $ECOLI70$, $MAGIC$-$IRRI$ and $ARTH150$ networks for 300 and 1000 observations.}
\label{tab:resultspowerG}
\end{table}

\begin{table}[h]
\centering
\begin{tabular}{lll c rrr }
  \toprule
 Method  & $EC300$ & $EC1000$ & $MAG300$ & $MAG1000$ & $AR300$ & $AR1000$\\
  \midrule
\texttt{hc} & 0.049  & 0.044 & 0.04 & 0.037 & 0.028 & 0.19\\    
\midrule
\texttt{tabu} & 0.047   & 0.036 &0.04& 0.036&  0.028 &  0.018\\ 
\midrule
\texttt{mmhc} & 0.021   & 0.024 & 0.022 & 0.022 &  0.01 & 0.008\\
\midrule
\texttt{h2pc} & 0.02   & 0.023 & 0.14 & 0.019 & 0.006 & 0.006\\
\midrule
\text{MCMC + LASSO} & \textbf{0.004}   & \textbf{0.004} &\textbf{0.004}& \textbf{0.006}&  \textbf{0.004} &  \textbf{0.003}\\
 \bottomrule
\end{tabular}
\caption{Average FDR for $ECOLI70$, $MAGIC$-$IRRI$ and $ARTH150$ networks for 300 and 1000 observations.}
\label{tab:resultsfdrG}
\end{table}

\begin{table}[hb!]
\centering
\begin{tabular}{lll c rrr }
  \toprule
 Method  & $EC300$ & $EC1000$ & $MAG300$ & $MAG1000$ & $AR300$ & $AR1000$\\
  \midrule
\texttt{hc} & 65.6   & 51.8 & 129.5 & 96.5 & 214.2 & 151.1\\    
\midrule
\texttt{tabu} & 64.8   & 46.9 & 127.1 & 93.1 &  215 &  150.9\\ 
\midrule
\texttt{mmhc} & 48.1   & 39.2 & 101.7 & \textbf{72.2} & 127.3 &  104.1\\
\midrule
\texttt{h2pc} & 45.9   & 37.9 & 91.6 & 67.21 & 103.7 &  \textbf{90.3}\\
\midrule
\text{MCMC + LASSO} & \textbf{39.2}  & \textbf{35.1} & \textbf{85.9}& 79.3 &  \textbf{103}&  98.7\\
 \bottomrule
\end{tabular}
\caption{Average SHD for $ECOLI70$, $MAGIC$-$IRRI$ and $ARTH150$ networks for 300 and 1000 observations.}
\label{tab:resultsshdG}
\end{table}
  \chapter{Structure learning for CTBNs for complete data}\label{chapter: CTBN structure}
In this chapter we consider continuous time Bayesian networks (CTBNs) introduced and defined in Section \ref{sec:CTBN}. First we consider the fully observed case where we observe the~behaviour of the network at each moment of time.
\section{Notation and preliminaries}
\label{sec:structure}
In this section we describe the proposed method, using the notation introduced in Section~\ref{sec:CTBN}. First, we consider the full graph $\gr = \graph$, namely we assume that $\parents(w) = \ppa (w) = -w$ for each $w \in \V$. Then we remove unnecessary edges using the penalized likelihood technique. 
We start by introducing the new parametrization of the model.
For simplicity, in the main part of this chapter we consider the binary graph, i.e. $\X _w =\{0,1\}$ for each $w \in \V. $ The extension of our results to more general case is described in Section~\ref{sec:discussion}.

Let $d$ be the number of nodes in the graph.
Consider a fixed order $(w_1, w_2,\ldots, w_d)$ of nodes of the graph. Using this order we define a $ (2d) \times d$-dimensional matrix 
\begin{equation}
\label{beta}
\beta=\left(\beta_{0,1}^{w_1}, \beta_{1,0}^{w_1},
\beta_{0,1}^{w_2}, \beta_{1,0}^{w_2},
\ldots, \beta_{0,1}^{w_d}, \beta_{1,0}^{w_d}\right)^\top ,
\end{equation}
whose rows  are vectors $\bssw \in \mathbb{R}^d$
for all $w \in \V$ and $s,s' \in \{0,1\}$ such that $s\neq s'.$ Obviously, the matrix $\beta$ can be easily transformed to $2d^2$-dimensional vector in a standard way. 
In this chapter we assume  that for all $w \in \V$, $c \in \X_{-w}$, $s,s' \in \{0,1\}$, $s\neq s'$ 
the~conditional intensity matrices satisfy
\begin{equation}
\label{def: beta}
\log(Q_w(c,s,s\p))= {\beta_{s,s\p}^{w}}^{\top} Z_w(c),
\end{equation}
where $Z_w\colon \X_{-w}\to \{0,1\}^{d}$ is a binary deterministic function described below. With the slight abuse of notation, by $Z$ we will denote the set of all functions $Z_{w_1},\dotsc,Z_{w_d}$. In~\eqref{def: beta} the conditional intensity matrix $Q_w(\cdot,s,s')$ is modeled in the analogous way to the regression function in generalized linear models (GLM) and the functions $Z_w(\cdot)$ play the~role of explanatory variables (covariates).
In our setting the link function is logarithmic. The analogous approach can be found in \cite{andersen1982, Cox13}, where the Cox model is considered. 
The relation between the intensity and covariates in those papers is similar to \eqref{def: beta}.
Since the considered CTBNs do not contain explanatory variables, we introduce them artificially 
as \textit{any possible representations} of parents' states. Thus, for every $w \in \mathcal{V}$ these
explanatory variables are {\it dummy variables} encoding all possible configurations in $\ppa (w) = -w.$ 
To make it more transparent we consider the following example.

\begin{example}
\label{example_ctbn}
We consider CTBN with three nodes $A,B$ and $C.$ For the node $A$ we define the function $Z_A$ as 
$$
Z_A(b,c)=[1,\Ind(b=1),\Ind(c=1)]^{\top}
$$
for each $b,c \in \{0,1 \}$, where $\Ind (\cdot)$ is the indicator function. Therefore, for 
each con\-fi\-gu\-ration of parents' states (i.e.~values in the nodes $B$ and $C$) 
the value of the~function $Z_A(\cdot,\cdot)$ is a three-dimensional binary vector, whose coordinates correspond to the intercept, 
the value in the parent $B$ and the value in the parent $C,$ respectively. Analogously, we define representations for remaining nodes: 
$Z_B(a,c)=[1,\Ind(a=1),\Ind(c=1)]^{\top}$ and
$Z_C(a,b)= [1,\Ind(a=1),\Ind(b=1)]^{\top}$
for each $a,b,c \in \{0,1 \}.$  In this example the parameter vector~\eqref{beta} is defined as
$\beta=\left(\beta_{0,1}^{A}, \beta_{1,0}^{A},
\beta_{0,1}^{B}, \beta_{1,0}^{B},
\beta_{0,1}^{C}, \beta_{1,0}^{C}\right)^\top\!$.
With slight abuse of notation, the vector $\beta ^A_{0,1}$ is given as 
$\beta^A_{0,1}=\left[\beta^A_{0,1} (1), \beta^A_{0,1} (B), \beta^A_{0,1} (C) \right]^ \top $
and we interpret \eqref{def: beta} as follows:  $\beta^A_{0,1} (B) = 0$ means that the intensity of the change from the state $0$ to $1$ at 
the~node~$A$ does not depend on the state at the node $B.$ Similarly, $\beta^A_{0,1} (C) $ describes the dependence between the above intensity and the state at 
the node $C,$ and  $\beta^A_{0,1} (1)$ corresponds to the~intercept. For the node $B$ the coordinates of the vector 
$\beta^B_{0,1}=\left[\beta^B_{0,1} (1), \beta^B_{0,1} (A), \beta^B_{0,1} (C) \right]$ desc\-ribe the relation between the intensity of the jump 
from the state $0$ to $1$ at the~node~$B$ to the intercept, states at nodes $A$ and $C,$ respectively.

Now what if $Z=\{Z_A,Z_B,Z_C\}$ was defined differently? The new function $\overline{Z}_A$ can be defined in 3 more different ways, for example $\overline{Z}_A(b,c) = [1,\Ind(b=0),\Ind(c=1)]^{\top}$. The same applies to the functions  $\overline{Z}_B$ and  $\overline{Z}_C$. Having defined $\overline{Z}=\{\overline{Z}_A,\overline{Z}_C,\overline{Z}_C\}$ we obtain the new vector of the parameters $\overline{\beta}$. Then for instance we have   $\overline{\beta}^A_{0,1}=\left[\overline{\beta}^A_{0,1} (1), \overline{\beta}^A_{0,1} (B), \overline{\beta}^A_{0,1} (C) \right]$ and so on. Note that both sets $Z$ and $\overline{Z}$ fully describe the state configuration of the~network and  both $\beta^A_{0,1}$ and $\overline{\beta}^A_{0,1}$ correspond to the same dependencies as above. In~particular, it is easy to check that for instance ${\beta}^A_{0,1} (B)={\beta}^A_{1,0} (B)=0$ if and only if $\overline{\beta}^A_{0,1} (B)=\overline{\beta}^A_{1,0} (B)=0$.
\end{example}

Analogously as in Example~\ref{example_ctbn}, for $w\in\V$, $u\not=w$, and $s,s'\in\{0,1\}$, $s\not=s'$ we define  the coordinate of the function $Z_w$ corresponding to the~node~$u$ as an indicator of its state equal to either 0 or 1. Moreover, we denote the coordinate of $\beta_{s,s'}^{w}$ corresponding to the node $u$ by $\beta_{s,s'}^{w}(u)$. We interpret $\beta_{s,s'}^{w}(u)$ as the parameter describing dependence of the~intensity of the jump from the state $s$ to $s'$ at the node $w$ 
on the state at the~node~$u$.

Our goal is to find edges in a directed graph $(\V,\E). $ 
We define the relation between parameters and edges in $(\V,\E) $ in the following way 
\begin{equation}
\label{relation}
  \beta^w_{0,1} (u) \neq 0 \;  {\rm or}\; \beta^w_{1,0} (u) \neq 0 \;
\Leftrightarrow \; {\rm the \; edge} \; u\to w \; {\rm exists},
\end{equation}
which makes parameters compatible with the considered CTBNs.
 Roughly speaking, {\it  the~node~$u$ is a parent of $w$} means that 
 the intensity of switching a state at $w$ depends on  the~state at~$u$. 
 Therefore, the problem of finding edges in the graph is reformulated as the problem of the~estimation of the parameter~$\beta.$
 
 \begin{remark}
 As we mentioned previously in the example, the set $Z$ fully describes the pa\-rents state configuration and the relation above does not depend on the choice of~$Z$. More precisely, assume we have two different properly defined  $Z$ and $\overline{Z}$ and the corresponding vectors of parameters 
\begin{equation*}
    \begin{split}
         & \beta=\left(\beta_{0,1}^{w_1 \top}, \beta_{1,0}^{w_1 \top},
        \beta_{0,1}^{w_2 \top}, \beta_{1,0}^{w_2 \top}, \ldots, \beta_{0,1}^{w_d \top}, \beta_{1,0}^{w_d \top}\right)^\top\!,\\
        & \overline{\beta}=\left(\overline{\beta}_{0,1}^{w_1 \top}, \overline{\beta}_{1,0}^{w_1 \top}, \overline{\beta}_{0,1}^{w_2 \top}, \overline{\beta}_{1,0}^{w_2 \top}, \ldots, \overline{\beta}_{0,1}^{w_d \top}, \overline{\beta}_{1,0}^{w_d \top}\right)^\top\!.
    \end{split}
\end{equation*}
 Then the following is true
 \[
 \beta^w_{0,1} (u) = 0 \wedge \beta^w_{1,0}(u) = 0 \Leftrightarrow \overline{\beta}^w_{0,1} (u) = 0 \wedge \overline{\beta}^w_{1,0}(u) = 0.
 \]
This means that no matter how we define our explanatory functions $Z_w$, we will get the~same arrows in the underlying CTBN.
\end{remark}

\begin{remark}
\label{remark_intercept}
For simplicity, in the rest of the thesis, we omit the first coordinate 
$\bssw (1)$ in the vector $\bssw$ for all $w,$ $s\neq s',$ because it corresponds to the intercept and is not involved in the recognition of the edges in the graph. 
The first coordinates of representations $Z_w(c)$ are discarded as well. 
\end{remark}

\begin{remark}
\label{identify}
The Markov equivalence/identifiability/non-uniqueness problem is chal\-len\-ging for directed graphical models. However, this problem does not appear here for CTBNs. It is a consequence of our Assumption \eqref{def: beta}, which states that we restrict to models having a conditional intensity in the GLM form. Moreover, under this assumption $\beta$ is uniquely determined. Moreover, this uniquely defined $\beta$ determines uniquely the structure of a~graph by \eqref{relation}. In fact, our main result (Theorem \ref{thm:consistency} below) shows consistency of the~estimator of $\beta$, which is a much stronger property than identifiability. Finally, in  Assumption \eqref{def: beta} we require that a conditional intensity of a variable is a~linear function of the~states of its parents. This condition can be easily extended to a~polynomial dependence, so it can cover quite general dependence structure. 
\end{remark}

Our method is based on estimating the parameter $\beta$ using the penalized likelihood method. In the rest of the thesis the term $\beta$ is reserved for the true value of the parameter. Other quantities are denoted by $\theta.$ First, we consider a function 
\begin{eqnarray}\label{def: loglik_beta}
\ell(\theta)=\frac{1}{T} \sum_{w\in\V}\sum_{c\in\X_{-w}} \sum_{ s\not=s'}\left[-n_w(c; s,s\p){\theta_{s,s\p}^{w}}^{\top} Z_w(c)  +t_w(c; s)\exp\left({\theta_{s,s\p}^{w}}^{\top} 
Z_w(c)\right)\right],
\end{eqnarray}
where the third sum in \eqref{def: loglik_beta} is over all $s,s' \in \X _w$ such that $s \neq s'.$ Recall that $n_w(c; s,s\p)$ and $t_w(c; s)$ were introduced in Section \ref{sec:CTBN} to denote the number of jumps from a state~$s$ to~$s'$ and the total time in the state $s$ for the node $w$, respectively, while the parents configuration equals to $c$. Notice that the function \eqref{def: loglik_beta} is the \textit{negative log-likelihood}. 
Indeed, we just apply the negative logarithm to the density 
\eqref{eq:densCTBN} combined with \eqref{cbi} and \eqref{def: beta}, where $\ppa (w) = -w$ for each $w \in \V.$ Then we divide it by $T$ and omit the term corresponding 
to the initial distribution $\nu,$ because $\nu$ does not depend on $\beta .$ 
We define an~estimator of~$\beta$~as
\begin{equation}\label{minimizer}
	 \hat \beta =\argmin_{\theta\in \mathbb{R} ^{2d(d-1)}} \left\{ \ell(\theta)+\lambda \|\theta\|_1\right\},
\end{equation}
where  $\|\theta\|_1 = \sum\limits_{w\in\V} \sum\limits_{ s\not=s'} \sum\limits_{u \in -w} 
|\theta_{s,s\p}^w (u)|$ is the $l_1$-norm of $\theta.$ The tuning parameter $\lambda >0$ characterizes a balance between minimizing the negative log-likelihood and the penalty function. 
As we have mentioned, the form of the penalty is crucial, because
its singularity at the origin implies that some coordinates of the minimizer $\hat \beta$ are exactly equal to 0, if~$\lambda$~is sufficiently large. 
Thus, starting from the full graph we remove irrelevant edges and estimate parameters for existing ones simultaneously.
The function $\ell (\theta)$ and the penalty are convex functions, so \eqref{minimizer} is a convex minimization problem, 
which is an important fact from both practical and theoretical point of views. 

At first glance, computing \eqref{minimizer} seems to be computationally complex, because 
the~number of summands in \eqref{def: loglik_beta} is $d2^d.$ However, the number of nonzero terms of the form
 $n_w(c; s,s\p)$ and $t_w(c; s)$ is bounded by the total number of jumps, which grows linearly with time $T$.
Hence, most of summands in \eqref{def: loglik_beta} are also zeroes and the minimizer \eqref{minimizer} can be calculated efficiently. 

Before we state and prove main results of this chapter we introduce some additional notation. First, for each $w \in \V$ we denote its parents indicated by the true pa\-ra\-meter~$\beta$~as
\begin{equation*}
\label{Sw}
S_w=\left\{u \in -w: \beta^w_{0,1} (u) \neq 0 \quad \text{ or } \quad \beta^w_{1,0} (u) \neq 0\right\}.
\end{equation*}
By $S$ we denote the support of $\beta,$ i.e.~the set of nonzero coordinates of $\beta.$
Moreover,~$\beta _{\min}$~is the smallest in the absolute value element of $\beta$ restricted to $S$. 
The set $S^c$ denotes the~complement of $S$, that is the set of zero coordinates of $\beta.$ Besides, for each $w \in \V$ we define $-S _ w = \V \setminus \{S _w\cup w\}$ and denote $\Delta = \max\limits_{\bf{s,s'}\in\ccX,\;\bf{s \neq s\p}}Q(\bf{s,s\p}).$ 

Recall that for a vector $a$ we denote its $l_\infty$-norm by $\|a\|_\infty = \max_k |a_k|.$ For a subset~$I$ the vector $a_{I}$ denotes a vector such that
$(a_{I})_i=a_i$ for $i\in I$ and $(a_{I})_i=0$ otherwise. Moreover, $|I|$ denotes the number of elements of $I.$

Let $\pi$ be the stationary distribution of the Markov jump process (MJP), which is defined by an intensity matrix $Q.$ 
The initial distribution of this process is denoted by $\nu$ and we define
$
\Vert\nu \Vert_2^2 = \sum\limits_{\s\in \X}\nu^2 (\s)/\pi(\s)
   .  $
Moreover, $\rho_1$ denotes the smallest positive eigenvalue 
of the matrix $-1/2(Q+Q^*)$, where $Q^*$ is an adjoint matrix of $Q_.$

\section{Main results}
\label{subsec:main}

In this subsection, we state two key results on the structure learning for CTBNs for complete data. In the first one (Theorem~\ref{thm:consistency}) we show that the estimation error of the~mi\-ni\-mizer $\hat{\beta}$ given by \eqref{minimizer} 
can be controlled with probability close to 1. In~the~second main result (Corollary \ref{thm:consistency2}) we state that the thresholded version of \eqref{minimizer} is able to recognize the structure of the graph with high probability.

First, we introduce the cone invertibility factor (CIF), which plays an important role in the theoretical analysis of the properties of LASSO estimators.
 Our goal is to show that the estimator $\hat{\beta}$ is close to the true vector $\beta$. To accomplish this goal we show in Lemma~\ref{lem:lambda} that the gradient of the likelihood (\ref{def: loglik_beta}) evaluated at $\beta$ is close to 0. However, this is not sufficient since the likelihood function cannot be too ``flat''. Namely, its curvature around the local optimum needs to be relatively high, because we want to avoid the situation when the loss difference can be small whereas the error is large. In the high-dimensional scenario this is often provided by imposing the restricted strong convexity condition (RSC) on (\ref{def: loglik_beta}), as in \cite{negahban12}. CIF defined below in \eqref{Fbar} plays a similar role to RSC, but gives sharper consistency results \citep{YeZhang10}. Therefore, it is used  here.
 CIF is defined analogously to \cite{YeZhang10, HuangGLM12, Cox13} and is closely related to the compatibility factor \citep{Geer2008} or the restricted eigenvalue condition \citep{Bickel09}. Recall that in the previous chapter we also used a version of CIF for Bayesian networks.
Thus, for any $\xi>1$ we define the cone $\cone (\xi,S) = \left\{\theta: \|\theta_{S^c}\|_1 \leq \xi \|\theta_{S}\|_1\right\}$, where the set $S$ denotes the support of~$\beta$ as mentioned above. Then
CIF is defined as 
\begin{equation}
\label{Fbar}
\bar{F} (\xi)= \inf _{0 \neq \theta \in \cone (\xi,S)} \frac{ \theta^{\top} \nabla ^2 \ell(\beta) \theta}{\|\theta _S\|_1 \|\theta\|_\infty} .
\end{equation}
Notice that only the value of the Hessian $\nabla ^2 \ell(\theta)$ at the true parameter $\beta$ is taken into consideration in \eqref{Fbar}. 
  The main difficulty with CIF in our case is that it is a minimum of the sum of random terms, which number grows exponentially in $d$. To be able to control this quantity, we bound CIF from below by its deterministic counterpart with much fewer summands. Namely, in Lemma~\ref{lem:cif} we prove that $\bar{F} (\xi)$ is bounded from below by the product of $\zeta_0$ given in Theorem \ref{thm:consistency} and
\begin{eqnarray}\label{ass:cif}
F(\xi) =\inf_{0\not = \theta \in C(\xi,S)}\sum_{w\in\V}\sum_{s\p\not=s}\sum_{c_{\I _w}\in\X_{\I _w}}
\frac{\exp\left(\beta_{s,s\p}^{w\top} Z_w(c_{S_w},0)\right) 
  \left[\theta_{s,s\p}^{w\top}  Z_w(c_{S_w},0)\right]^2 }{\|\theta_S\|_1 \|\theta\|_\infty}
\end{eqnarray}
with probability close to 1. Here we divided each parent configuration~$c = (c_{S_w},c_{-S_w})$ into two parts: the first one corresponds to the true parents nodes and the second part corresponds to the remaining nodes. Below we will also use a similar notation for any state of the network $\s\in\ccX$ defining it as a triple  $\s = (c_{S_w},c_{-S_w}; s)$ of the state of true parents of the node $w$, the configuration for the nodes from $S_{-w}$ and the state in the node~$w$. Note, that we restricted the summation in \eqref{ass:cif} only to $c_{S_w}\in\X_{S_w}$ by taking $c_{-S_w}=0$. This allows us to derive the lower bound on $\bar{F}(\xi)$ without considering exponentially many random summands. Our argumentation will be also valid in the case, when we choose some nonzero values in $c_{-S_w}$, unless this values depend on $w$ and $c_{S_w}.$ Next, we  state two main results of this chapter.
\begin{theorem}
\label{thm:consistency}
Fix arbitrary $\varepsilon \in (0,1)$ and $\xi >1$. Suppose that
\begin{equation}
\label{Tform2} 
T>
\frac{36 \left[ \left(\max\limits_{w \in \V} |\I _w| +1\right) \log 2 + \log\left(
d  \|\nu\|_2 /\varepsilon
\right) 
\right]}{\rho_1\min\limits_{ \substack{w \in \V,\,s \in \X _w\\ c_{\I _w}\in\X_{\I _w}}}
\pi^2(c_{\I _w},0;s)} .
\end{equation}
We also assume that $T \Delta \geq 2$ and we choose $\lambda$ such that
\begin{equation}
\label{lambda_form}
2\frac{\xi+1}{\xi-1}\log(K/\varepsilon)\sqrt{\frac{\Delta}{{T}}}
\leq \lambda \leq \frac{2 \zeta_0 F(\xi)}{e(\xi+1)|S|},
\end{equation}
where
 $
K=2(2+e^2)d(d-1)$ and  
\begin{equation*}
  \zeta_0=\min\limits_{ \substack{w \in \V,\,s \in \X _w\\ c_{\I _w}\in\X_{\I _w}}} \pi(c_{\I _w},0;s)/2.
\end{equation*} 
Then with probability at least $1-2\varepsilon$ we have
\begin{equation}
\label{estim_formula}
 \|\hat\beta -\beta\|_\infty\leq \frac{2e \xi \lambda}{(\xi+1)\zeta_0 F(\xi)}.
\end{equation}
\end{theorem}
Now consider the Thresholded LASSO estimator with the set of nonzero coordinates~$\hat{S}.$ The set $\hat{S}$ 
contains only those coefficients of the LASSO estimator \eqref{minimizer}, which are  larger in the absolute value than a pre-specified threshold $\delta.$ 
\begin{corollary}
\label{thm:consistency2}
Suppose that assumptions of Theorem \ref{thm:consistency} are satisfied and let $R$ denote the right-hand side of the inequality~\eqref{estim_formula}.
If $R< \beta _{\min}/2,$ then  for $\delta\in[R, \beta _{\min}/2)$ we have $P\left( \hat{S} = S \right) \geq 1- 2 \varepsilon.$
\end{corollary}
These two results will be proven in the next section and here we give some comments on their meaning and significance. The above two results describe the properties of the~proposed estimator \eqref{minimizer} in recognizing the structure of the graph. Theorem \ref{thm:consistency} gives conditions under which the estimation error of \eqref{minimizer} can be controlled. Namely, let us for a moment ignore constants, $\Delta$ and parameters of  MJP such as $\nu,\pi,\rho_1, \zeta_0$, etc.~in the~assumptions. By condition \eqref{lambda_form}, if
\begin{equation}
\label{assu_d}
T \geq \frac{\log ^2(d/\varepsilon) |S|^2 }{F^2(\xi)},
\end{equation}
then the estimation error is small. This forms some restrictions on the number of vertices in the graph, sparsity of the graph (i.e.~the number of edges has to be small enough) and the expression \eqref{ass:cif}, which is discussed in Lemma \ref{cif_bound} (below). The condition \eqref{assu_d} is similar to standard results for LASSO estimators in \cite{YeZhang10, geerbuhl11, HuangGLM12, Cox13}.
The only difference is that  the right-hand side of \eqref{assu_d} usually depends linearly on $\log (d/\varepsilon)$, but here we have $\log ^2 (d/\varepsilon).$ 
The square in the logarithm could be omitted, if we imposed some additional restrictions on observation time~$T$ in the crucial auxiliary result (Lemma \ref{lem:lambda}),  where we use the Bernstein-type inequality for the Poisson random variable.
Obviously, it would reduce the applicability of the main result. In our opinion, the gain (having $\log (d/\varepsilon)$ instead of $\log ^2 (d/\varepsilon)$) is ``smaller'' than the price (additional assumptions), so we do not focus on it.

The next assumption in Theorem \ref{thm:consistency} that $T \Delta \geq 2 $ is quite natural since observation time has to increase when the maximal 
intensity of transitions decreases. Moreover, the conditions \eqref{Tform2} and \eqref{lambda_form} depend also on parameters of MJP. More precisely, they depend on
  the stationary distribution $\pi$ and the spectral gap $\rho_1,$ which in general decrease exponentially with $d.$ 
However, in some specific cases, it can be proved that they decrease polynomially.

Corollary \ref{thm:consistency2}  states that the LASSO estimator after thresholding is 
able to recognize the structure of a graph with probability close to 1,
if the nonzero coefficients of $\beta$ are not too close to zero and the threshold $\delta$ is appropriately chosen. 
However, Corollary~\ref{thm:consistency2} does not give a way of choosing the threshold $\delta$, because both endpoints of the interval $[R, \beta _{\min}/2] $ are unknown. 
It is not a surprising fact and has been already observed, for instance, in linear models \citep[Theorem 8]{YeZhang10}. In the experimental subsection of this chapter we propose a method of choosing a threshold that relies on information criteria. A similar procedure can be found in \cite{pokmiel:15, Rejchel18}.

Now we state a lower bound for \eqref{ass:cif} which has an intuitive interpretation.
\begin{lemma}
\label{cif_bound}
Define $A_\beta= \sum\limits_{w\in\V} \sum\limits_{s\p\not=s}\sum\limits_{u:\beta_{s,s\p}^{w} (u) \neq 0 }
\exp\left(-\beta_{s,s\p}^{w} (u) \right).$ Then for every $\xi >1$ we have
$F(\xi) \geq (\xi A_\beta)^{-1}.$
\end{lemma}

Notice that the term $A_\beta$ is larger, and in turn $F(\xi)$ is smaller,  when negative coefficients of~$\beta$ ,,dominate'' in the absolute value the positive ones. Note that, the more these negative coefficients dominate the more our process ,,gets stuck'', i.e.~tends to stay in the same state because intensities in this case tend to be close to zero (recall \eqref{def: beta}). Such behaviour in the~context of MJPs is natural, because 
multiplying the intensity matrix $Q$ by a constant~$\kappa$ is equivalent to considering time $T/\kappa$ instead of~$T$. While  $F(\xi)$ appears in the  lower bound~\eqref{assu_d} on  
$T$, such dependence on $\beta$ is expected.

\section{Proofs of the main results}
\label{subsec:proofsCTBN}
This subsection contains the proofs of all the statements made in the previous subsection. The proofs of the theorem and the corollary are based on a number of auxiliary results. Some of these results are well-known facts for LASSO estimators and some of them are new (Lemmas \ref{lem:lambda} and \ref{lem:cif}). The main novelty and difficulty of the considered model is the~continuous time nature of the observed phenomena which we investigate. 
 In~Lemma~\ref{lem:lambda}  we derive the new concentration inequality for MJPs based on the martingale theory. In Lemma \ref{lem:cif} we give new upper bounds on the occupation time for MJPs.
 
 In the proofs of subsequent results we use the first and second derivatives of $\ell$ given by \eqref{def: loglik_beta}, which can be also expressed in the following form
\begin{equation}
\label{l_decomp}
\ell (\theta)  =\frac{1}{T}\sum_{w\in\V} \sum_{ s\not=s'} \lssw (\tssw) ,
\end{equation} 
where 
$$
\lssw (\tssw)  = \sum_{c\in\X_{-w}} \left[-n_w(c; s,s\p){\theta_{s,s\p}^{w}}^{\top} Z_w(c)+t_w(c; s)\exp({\theta_{s,s\p}^{w}}^{\top} Z_w(c))\right].
$$
Therefore, we can calculate partial derivatives
\begin{equation}
\label{der1}
\nabla \lssw (\tssw) =\sum_{c\in\X_{-w}} [-n_w(c; s,s\p)+t_w(c; s)\exp({\theta_{s,s\p}^{w}}^{\top} Z_w(c))] Z_w(c).
\end{equation}
By Remark \ref{remark_intercept} the matrix $\theta$ of all parameters has $2d$ rows and $(d-1)$ columns. It can be also considered as a 
$2d(d-1)$-dimensional vector $$\theta = \left(\theta_{0,1}^{w_1 \top}, \theta_{1,0}^{w_1 \top},
\theta_{0,1}^{w_2\top}, \theta_{1,0}^{w_2 \top},
\ldots, \theta_{0,1}^{w_d \top }, \theta_{1,0}^{w_d \top}\right)^ \top\!,$$ where $(w_1, w_2,\ldots, w_d)$ is a fixed order of the nodes of the graph. Using this order we obtain the following representation of the gradient of $\ell$.
\begin{equation}\label{def: gradient}
\nabla \ell (\theta) = \frac{1}{T} \left[ \nabla \lssw (\tssw)
\right]_{w \in \V,s\neq s'}.
\end{equation}
Analogously we calculate second derivatives
$$
\nabla ^2 \lssw (\tssw) =\sum_{c\in\X_{-w}} t_w(c; s)\exp({\theta_{s,s\p}^{w}}^{\top} Z_w(c)) Z_w(c) Z_w(c)^ \top.
$$
The second derivative of $\ell (\theta)$ consists of matrices $\frac{1}{T} \nabla ^2 \lssw (\tssw)
$ along its diagonal and zeroes elsewhere. Moreover, for any vector $\theta\in\R^{2d(d-1)}$ and the true parameter vector $\beta$ we have
\begin{equation}
\label{eq:draddiff2}
\theta^{\top} \nabla ^2 \ell(\beta) \theta =
\frac{1}{T} \sum_{w\in\V}\sum_{c\in\X_{-w}} \sum_{ s\p\not=s} t_w(c; s) \left({\theta_{s,s\p}^{w}}^{\top} {Z}_w(c)\right)^2 \exp({\beta_{s,s\p}^{w}}^{\top} Z_w(c)).
\end{equation}
Next we provide an auxiliary proposition needed to prove an important concentration inequality for MJPs in Lemma \ref{lem:lambda}.
\begin{proposition}
 \label{prop:martingale}
 Let $X({\tau})$ be a Markov jump process with a bounded intensity matrix~$Q$. Let $$n^{\tau}_{s,s\p} =\sum_{c\in\X_{-w}\colon Z_w(c)[k]=1} n^{\tau}_w(c;s,s\p)$$ be a number of jumps from $s$ to $s\p$ on the interval $[0,{\tau}]$ and $t^{\tau}_s$ be an occupation time at state $s$ on the interval $[0,{\tau}]$. Then
 \[
  M_\nu({\tau})=n^{\tau}_{s,s\p}-t^{\tau}_s Q(s,s\p)
  \]
is a martingale with respect to the natural filtration $\mathcal{F}_{\tau}$. The notation
$M_\nu({\tau})$ means that the distribution at time $0$ is $\nu$.
\end{proposition}
\begin{proof}
 For any $u<{\tau}$ we have
 \begin{equation*}
     \begin{split}
         \Ex(M_\nu({\tau})\mid\mathcal{F}_u) & = M_\nu(u)+\Ex (M_\nu(\tau) - M_\nu(u)\mid \ccF_u) = \\ & = M_\nu(u)+\Ex (M_\nu(\tau) - M_\nu(u)\mid X(u)) \\
         & = M_\nu(u)+\Ex (M_{X(u)}({\tau}-u)\mid X(u)),
     \end{split}
 \end{equation*}
where the last equality is the consequence of Proposition 20.3 from \cite{bass_2011}. Now it is enough to show that for all ${\tau}>0$ and all initial measures $\nu$ we have $\Ex M_\nu({\tau})=0$, since it implies that $\Ex\left(\Ex(M_\nu(\tau)\mid \ccF_u)\right) = 0$ for $u < \tau$.

For any $n\in\N$ define the sequence $k_i = k_i(n) = \frac{\tau i}{n}$ for all $i=0,\dots, n$. 
Since the~trajectory of the process is \textit{c\`adl\`ag}, we have
 \begin{equation*}
  \Ex M_\nu({\tau}) = \Ex \lim_{n\to\infty}\sum_{i=1}^n\left[\Ind\left(X(k_{i-1})=s,X(k_i\right)=s\p) - \frac{{\tau}}{n}Q(s,s\p)\Ind\left(X(k_{i-1} )= s\right)\right].
 \end{equation*}
 We observe that for all $n\in\mathbb{N}$
 \begin{equation}
\label{mart1}
\left|\sum_{i=1}^n\left[\Ind\left(X(k_{i-1})=s, X(k_i)=s\p\right) - \frac{{\tau}}{n}Q(s,s\p)\Ind\left(X(k_{i-1} )= s\right)\right]\right|\leq N({\tau})+{\tau},
   \end{equation}
 where $N({\tau})$ is the total number of jumps. Since $N({\tau})$ is a Poisson process with a  bounded intensity, 
 the right-hand side of \eqref{mart1} is integrable and by the dominated convergence theorem and the definition of $Q$ we get
 \begin{equation*}
    \begin{split}
          \Ex M_\nu({\tau}) &= \lim_{n\to\infty}\Ex \sum_{i=1}^n\left[\Ind\left(X(k_{i-1})=s,X(k_i)=s\p\right) - \frac{{\tau}}{n}Q(s,s\p)\Ind\left(X(k_{i-1} )= s\right)\right]\\
         &=\lim_{n\to\infty}\Ex \sum_{i=1}^n\left[\Ex\left(\Ind\left(X(k_{i-1})=s,X(k_i)=s\p\right)\mid X(k_{i-1})\right) - \frac{{\tau}}{n}Q(s,s\p)\Ind\left(X(k_{i-1} )= s\right)\right].
    \end{split}
\end{equation*}
Next for $s\neq s'$ we have
\begin{equation*}
\begin{split}
        \Pr&\left(X(k_{i-1})=s,X(k_i)=s\p\mid X(k_{i-1})=s\right) = \\
        & \qquad= \Pr\left(X(k_i)=s\p\mid  X(k_{i-1})=s\right) = \frac{Q(s,s')}{n} + o\left(1/n\right)
\end{split}
\end{equation*}
and for $\sigma\neq s$
\begin{equation*}
    \Pr\left(X(k_{i-1})=s,X(k_i)=s\p\mid  X(k_{i-1}) = \sigma\right) = 0.
\end{equation*}
Hence,
\begin{equation*}
    \Ex\left[\Ind\left(X(k_{i-1})=s,X(k_i)=s\p\right)\mid  X(k_{i-1})\right] = \left(\frac{{\tau}}{n}Q(s,s\p)+o(1/n)\right)\Ind\left(X(k_{i-1} )= s\right).
\end{equation*}
Therefore, we further obtain
\begin{align*}
  \Ex M_\nu({\tau}) &=  \lim_{n\to\infty}\Ex \sum_{i=1}^n\left[\left(\frac{{\tau}}{n}Q(s,s\p)+o(1/n)\right)\Ind\left(X(k_{i-1} )= s\right) - \frac{{\tau}}{n}Q(s,s\p)\Ind(X(k_{i-1} )= s)\right]\\
  &=\lim_{n\to\infty}\Ex \sum_{i=1}^no(1/n)\Ind(X(k_{i-1} )= s)=0,
 \end{align*}
 where $o(1/n)$ does not depend on $i$.
 \end{proof}
 
\begin{lemma}\label{lem:lambda} Let $\varepsilon\in(0,1)$ and $\xi>1$ be arbitrary. Assume that $T\Delta \geq 2$ and \[\lambda \geq 2\frac{\xi+1}{\xi-1}\log(K/\varepsilon)\sqrt{\frac{\Delta}{{T}}},\] where
$K=2(2+e^2)d(d-1)$.
 Then  we have
\[
 \Pr \left(\| \nabla\ell(\beta)\| _{\infty}\leq \frac{\xi-1}{\xi+1}\lambda\right)\geq 1-\varepsilon.
\]
\end{lemma}

\begin{proof}
Note that by \eqref{def: beta}, \eqref{der1} and \eqref{def: gradient} we have the following inequality
\begin{equation*}\label{eq: bound_ell_infty}
\| \nabla\ell(\beta) \|_\infty \leq \frac{1}{T} \max_{w \in \V, s\not= s\p, 1\leq k\leq d-1} \left \vert \sum_{c\in\X_{-w}\colon Z_w(c)[k]=1}\left[ n_w(c; s,s\p)-t_w(c;s)Q_w(c; s,s\p)\right]\right\vert,
 \end{equation*}
where $Z_w(c)[k]$ is  the $k$-th coordinate of $Z_w(c)$ for each $w \in \V,$ $c \in \X _{-w}.$
The core step of the proof is to show that for fixed $w \in \V,$ $s\not= s\p,$ $1\leq k\leq d-1$ and $\eta=2\log\left(K/\varepsilon\right)$
\begin{equation}
\label{core}
  \Pr\left(\left \vert \sum_{c\in\X_{-w}\colon Z_w(c)[k]=1}\left[ n_w(c;s,s\p)-t_w(c;s)Q_w(c;s,s\p)\right]\right\vert>\eta\sqrt{T\Delta} \right)\leq 
  (2+e^2)\exp\left(-\frac{\eta}{2}\right).
 \end{equation}
Having \eqref{core} we finish the proof of Lemma \ref{lem:lambda} using union bounds. More precisely,
\begin{equation*}
    \begin{split}
        &\quad\; \Pr \left(\| \nabla\ell(\beta)\| _{\infty}> \frac{\xi-1}{\xi+1}\lambda\right) \leq \\
       &  \leq \Pr\left(\frac{1}{T} \max_{w \in \V, s\not= s\p, 1\leq k\leq d-1} \left \vert \sum_{c\in\X_{-w}\colon Z_w(c)[k]=1}\left[ n_w(c; s,s\p)-t_w(c; s)Q_w(c;s,s\p)\right]\right\vert > \eta\sqrt{\frac{\Delta}{T}}\right) \leq \\
         & \leq 2d(d-1) \Pr\left(\left \vert \sum_{c\in\X_{-w}\colon Z_w(c)[k]=1}\left[ n_w(c;s,s\p)-t_w(c; s)Q_w(c;s,s\p)\right]\right\vert> \eta\sqrt{T\Delta} \right)\leq \\
         & \leq 2d(d-1)(2+e^2)\exp(-\log(K/\varepsilon)) = \varepsilon.
    \end{split}
\end{equation*}
Therefore, we focus on proving \eqref{core}. The proof of this inequality is based on the~martingale arguments, so we make the dependence on the time explicit in \eqref{core}, that is $n_w(c; s,s\p)$ and $t_w(c;s)$ become $n_w^T(c;s,s\p)$ and $t_w^T(c;s),$ respectively. 

 For $\tau \in [0,T]$ we define a process
\begin{equation}
\label{mart}
 M(\tau) = \sum_{c\in\X_{-w}\colon Z_w(c)[k]=1}\left[ n_w^{\tau}(c;s,s\p)-t^{\tau}_w(c; s)Q_w(c; s,s\p)\right].
\end{equation}
We use the upper index $\tau$ in $n_w^{\tau}(c; s,s\p)$ and $t^{\tau}_w(c;s)$ to indicate that these quantities correspond to the time interval $[0,{\tau}].$
Using Proposition~\ref{prop:martingale} to each summand in \eqref{mart} we obtain that the process $\{M({\tau}): {\tau} \in [0,T]\}$ is a martingale. Let us define its jumps by $$\Delta M({\tau})=M({\tau})-M({\tau}_-)=\sum_{c\in\X_{-w}\colon Z_w(c)[k]=1}\Ind \left[X({\tau}_-)=(s,c),X({\tau})=(s\p,c)\right],$$ where $M({\tau}_-)$ is the left limit at ${\tau}$. 
By Theorem II.37 of \cite{Protter2005} and Theorem I.4.61 of \cite{Jacod2003} for any $x>-1$ the process
\begin{eqnarray*}
\mathcal{E}_x({\tau})&=&\exp\left(xM({\tau})\right)\prod_{u\leq {\tau}}(1+x\Delta M(u))\exp(-x\Delta M(u))\\&=&\exp\left\{ xM({\tau})-(x-\log(1+x))n^{\tau}_{s,s\p}\right\}
\end{eqnarray*}
is a local martingale, where $n^{\tau}_{s,s\p}=\sum_{c\in\X_{-w}\colon Z_w(c)[k]=1} n^{\tau}_w(c;s,s\p)$ is computed for a trajectory at the time interval 
$[0,{\tau}]$. Therefore, by Markov inequality together with the triangle inequality we get for any $x\in(0,1]$ 
\begin{eqnarray*}
 \label{eq:decomp_prob}
 \Pr(|M(T)|>L) &\leq & \Pr(|xM(T)-(x-\log(1+x))n^T_{s,s\p}|>xL/2)+ \nonumber \\ 
&&+\quad \Pr( (x-\log(1+x))n^T_{s,s\p}>xL/2) \leq \nonumber \\ &\leq & 2\exp\left(\frac{-xL}{2}\right)+\Pr( (x-\log(1+x))n^T_{s,s\p}>xL/2)\;.
\end{eqnarray*}
We observe that $n^T_{s,s\p}$ is bounded from above by the total number of jumps up to time~$T$, which in turn is bounded by a Poisson random variable $N(T)$ with the intensity $T\Delta$. 
Hence, again by Markov inequality we have
\[
 \Pr( (x-\log(1+x))n^T_{s,s\p}>xL/2)\leq \exp\left[\frac{-xL}{2}+T\Delta\left(\frac{e^x}{1+x}-1\right) \right].
\]
Applying an inequality $e^x\leq 1/(1-x)$ for $x<1$ and setting $x=1/\sqrt{T\Delta}$ we get
\[
 \Pr( (x-\log(1+x))n^T_{s,s\p}>xL/2)\leq \exp\left(\frac{-L}{2\sqrt{T\Delta}}+\frac{T\Delta}{T\Delta-1} \right).
\]
We use $T\Delta\geq 2$ and we plug in $L=\eta\sqrt{T\Delta}$ to conclude the proof.
\end{proof}

 The next lemma is a direct application of Theorem 3.4 of \cite{Lezaud1998} and it will be used in the second crucial auxiliary Lemma \ref{lem:cif}.

\begin{lemma}\label{lem:Lezaud}
For any $w \in \V$, $s \in\X_w$, $c_{\I _w}\in\X_{\I _w}$ we have
\[\Pr\left(\frac{1}{T}t_w(c_{S_w},0;s)\leq\pi(c_{\I _w},0;s)/2\right)\leq \Vert\nu \Vert_2\exp\left(-\frac{\pi^2(c_{\I _w},0;s)\rho_1 T}{16+20\pi(c_{\I _w},0;s)}\right).
 \]
\end{lemma}
\begin{proof}
Fix $w \in \V$, $s \in\X_w$, $c_{\I _w}\in\X_{\I _w}.$
By the definition we have
 \begin{align*}
  t_w(c_{S_w},0;s)&=\int_0^T  \Ind\left[X(t) =(c_{\I_w},0; s) \right]dt.
 \end{align*}
Let us define $f(X(t)) = \pi(c_{\I_w},0;s)-\Ind\left( X(t) =(c_{\I_w},0,s) \right)$.
Taking $\gamma = \pi(c_{\I_w},0;s)/2$ in Theorem 3.4 of  \cite{Lezaud1998}, we conclude the proof. 
\end{proof}
 
 \begin{lemma}
 \label{lem:cif}
Let $\varepsilon \in (0,1)$, $\xi >1$ be arbitrary.  Suppose that $F(\xi)$ defined in   \eqref{ass:cif} is positive and 
  \begin{equation}
\label{Tform} 
 T>
\frac{36 \left[ \left(\max\limits_{w \in \V} |\I _w| +1\right) \log 2 + \log\left(
d  \|\nu\|_2 /\varepsilon
\right) 
\right]}{\rho_1\min\limits_{ \substack{w \in \V,\,s \in \X _w\\ c_{\I _w}\in\X_{\I _w}}}
\pi^2(c_{\I _w},0;s)},
\end{equation}
 then
 \[
 \Pr\left(\bar F(\xi) \geq \zeta_0 F(\xi) \right)\geq 1-\varepsilon,
 \]
where  $ \zeta_0=\min\limits_{ \substack{w \in \V,\,s \in \X _w\\ c_{\I _w}\in\X_{\I _w}}} \pi(c_{\I _w},0;s)/2.
    $ 
\end{lemma}

\begin{proof}
 By the definition of $\bar F(\xi)$, the equation \eqref{ass:cif} and the formula for Hessian of $\ell$ (see \eqref{eq:draddiff2}) we have
 \begin{equation}
\label{Fprop}
  \frac{\bar F(\xi)}{F(\xi)}\geq \frac{1}{T}\min_{w \in \V, s,c_{\I _w}\in\X_{\I _w}}t_w(c_{S_w},0 ;s).
 \end{equation}
We complete the proof by bounding the right-hand side of 
\eqref{Fprop} from below.
First, we can calculate that
 \begin{equation}
\label{form1}
\begin{split}
     \Pr&\left(\min_{w \in \V, s\in \X _w,c_{\I _w}\in\X_{\I _w}} \quad \frac{1}{T}t_w(c_{S_w},0;s)\geq \zeta_0\right) \geq \\
&\geq \Pr\left(\forall_{w \in \V, s\in \X _w,c_{\I _w}\in\X_{\I _w}} \quad \quad \frac{1}{T}t_w(c_{S_w},0;s)\geq \pi(c_{\I _w},0;s)/2
\right) \geq\\ 
&\geq 1- 2d  \max_{w \in \V, s \in \X _w,c_{\I _w}\in\X_{\I _w}}
    2 ^{|\I _w|} \Pr\left(\frac{1}{T}t_w(c_{S_w},0;s)< \pi(c_{\I _w},0;s)/2 \right).
\end{split}
 \end{equation}
Using  Lemma~\ref{lem:Lezaud}
we bound \eqref{form1} from below by 
$$
1- 2d  \max_{w \in \V, s \in \X _w,c_{\I _w}\in\X_{\I _w}}
    2 ^{|\I _w|} \Vert\nu \Vert_2\exp\left(-\frac{\pi^2(c_{\I _w},0;s)\rho_1 T}{16+20\pi(c_{\I _w},0;s)}\right).
$$
Applying \eqref{Tform} we conclude the proof.
\end{proof}

Next we state and prove three lemmas, where Lemmas \ref{basiclem} and \ref{estim} will be used in the proof of the main result Theorem \ref{thm:consistency} and Lemma \ref{lem:second_deriv} is needed to prove Lemma~\ref{estim}.
\begin{lemma}
\label{basiclem}
Let $\tilde{\beta} = \hat{\beta} - \beta$, $z^* = \|\nabla \ell(\beta)\|_\infty.$ Then 
\begin{equation}
\label{basic}
(\lambda - z^*) \|\tilde{\beta}_{S^c}\|_1 \leq \tilde{\beta} ^\top \left[ \nabla \ell (\hat{\beta})
- \nabla \ell (\beta)\right] + (\lambda - z^*) \|\tilde{\beta}_{S^c}\|_1 
\leq (\lambda + z^*) \|\tilde{\beta}_{S}\|_1\, .
\end{equation}
Besides, for arbitrary $\xi >1$ on the event 
\begin{equation*}
\label{omega1}
\Omega_1=\left\{ \|\nabla \ell (\beta)\|_\infty \leq \frac{\xi -1}{\xi +1} \lambda \right\}
\end{equation*} 
 the random vector $\tilde{\beta}$ belongs to the cone $\cone (\xi,S).$ 
\end{lemma}
The proof of Lemma \ref{basiclem} is similar to the proof of Lemma 3.1 of \cite{Cox13} and is based on convexity of $\ell (\theta)$ and properties of the LASSO penalty. For convenience of the reader we will provide it here using our notation. 
\begin{proof}
Since $\ell(\theta)$ is a convex function, then $$\tilde{\beta} ^\top \left[ \nabla \ell (\hat{\beta})
- \nabla \ell (\beta)\right] = \tilde{\beta} ^\top \left[ \nabla \ell (\beta + \tilde{\beta})
- \nabla \ell (\beta)\right] \geq 0,$$ which instantly proves the left-hand side inequality in \eqref{basic}. As we have already mentioned, the minimized target function, given in \eqref{minimizer}, is also convex because of the convexity of the negative log-likelihood $\ell (\theta)$ and $\ell_1$-penalty functions. Hence $\hat{\beta}$ will be a~minimizer in \eqref{minimizer} if and only if the following conditions are met
\begin{equation}\label{kunn-conditions}
    \begin{cases}
        \dfrac{\partial \ell(\hat{\beta})}{\partial\beta_j} = - \lambda\text{ sgn}(\beta_j), & \quad\text{if } \hat{\beta}_j\neq 0,\\
         \left|\dfrac{\partial \ell(\hat{\beta})}{\partial\beta_j}\right| \leq \lambda, & \quad\text{if }  \hat{\beta}_j= 0,
    \end{cases}
\end{equation}
where $j\in\{1,\dots, 2d(d-1)\}$. First, we can write
\begin{equation*}
    \tilde{\beta} ^\top \left[ \nabla \ell (\beta + \tilde{\beta}) - \nabla \ell (\beta)\right] = \sum\limits_{j\in S^c}\tilde{\beta}_j\dfrac{\partial \ell(\beta + \tilde{\beta})}{\partial\beta_j} + \sum\limits_{j\in S}\tilde{\beta}_j\dfrac{\partial \ell(\beta + \tilde{\beta})}{\partial\beta_j} + \tilde{\beta}^{\top}(-\nabla\ell(\beta)).
\end{equation*}
Since $\tilde{\beta}_j = \hat{\beta}_j$ for $j\in S^c$, then applying the conditions \eqref{kunn-conditions} we can bound the last expression from above by
\begin{equation*}
    \sum\limits_{j\in S^c}\hat{\beta}_j(- \lambda\text{ sgn}(\hat{\beta}_j)) + \sum\limits_{j\in S}|\tilde{\beta}_j|\lambda + \|\tilde{\beta} \|_1z^* = \sum\limits_{j\in S^c} -\lambda|\tilde{\beta}_j| + \|\tilde{\beta}_S\|_1\lambda + z^*\|\tilde{\beta}_S\|_1 + z^*\|\tilde{\beta}_{S^c}\|_1.
\end{equation*}
This in turn equals to $(z^* - \lambda)\|\tilde{\beta}_{S^c}\|_1 + (\lambda+z^*)\|\tilde{\beta}_{S}\|_1$ meaning that the right-hand side inequality holds as well. Note that we used the fact that $\dfrac{\partial \ell(\hat{\beta})}{\partial\beta_j} = - \lambda\text{ sgn}(\beta_j)$ only on the set $S^c\cap\{j:\hat{\beta}_j\neq 0\}$, because $\tilde{\beta}_j = \hat{\beta}_j - \beta_j = 0,$ when $j\in S^c$ and $\hat{\beta}_j = 0.$ 
Finally, by~\eqref{basic} and the definition of $\Omega_1$ we obtain
\[
\|\tilde{\beta}_{S^c}\|_1\leq \frac{\lambda+z^*}{\lambda - z^*} \|\tilde{\beta}_S\|_1 \leq \|\tilde{\beta}_S\|_1
\]
proving the last claim of the lemma.
\end{proof}
\begin{lemma}\label{lem:second_deriv}
 For any $b \in \mathbb{R}^{2d(d-1)}$ we define  $c_b = \max\limits_{\substack{ w \in \V,\; s \neq s',\; c \in \X _{-w}} }\exp\left(\left|
{b^{w}_{s,s'}}^{\top} Z_w(c)
\right|\right)$. Then
we have
 \begin{equation}\label{inlem}
  c_b^{-1} b^{\top}\nabla^2\ell(\beta) b \leq b^{\top}[\nabla\ell(\beta+b)-\nabla\ell(\beta)]\leq c_b b^{\top}\nabla^2\ell(\beta) b
 \end{equation}
and
\begin{equation}\label{inlem2}
 c_b^{-1} \nabla^2\ell(\beta) \leq \nabla^2\ell(\beta+b) \leq c_b \nabla^2\ell(\beta),
\end{equation}
where for two symmetric matrices $A,B$ the expression $A \leq B$ means that $B-A$ is a~non-negative definite matrix.
\end{lemma}
\begin{proof}
First we prove the inequality \eqref{inlem}. 
By \eqref{def: gradient} we have
	\begin{multline}\label{eq:draddiff}
		b^\top[\nabla\ell(\beta + b)-\nabla\ell(\beta)] =\\ =\frac{1}{T} \sum_{w\in\V}\sum_{c\in\X_{-w}} \sum_{ s\p\not=s} t_w(c; s) \left({b_{s,s\p}^{w}}^{\top} {Z}_w(c)\right) \exp({\beta_{s,s\p}^{w}}^{\top} Z_w(c))
\left[\exp({b_{s,s\p}^{w}}^{\top} Z_w(c))-1\right].
	\end{multline}
Moreover, as in \eqref{eq:draddiff2}, we have
\begin{equation}
\label{eq:dr}
b^{\top} \nabla ^2 \ell(\beta) b =
\frac{1}{T} \sum_{w\in\V}\sum_{c\in\X_{-w}} \sum_{ s\p\not=s} t_w(c; s) \left({b_{s,s\p}^{w}}^{\top} {Z}_w(c)\right)^2 \exp({\beta_{s,s\p}^{w}}^{\top} Z_w(c)).
\end{equation}
	Let us consider an arbitrary summand in \eqref{eq:draddiff} and the corresponding one in \eqref{eq:dr}. We can focus only on cases where $t_w(c; s) >0$ and  ${b_{s,s\p}^{w}}^{\top} {Z}_w(c) \neq 0.$ 
From the mean value theorem we obtain for all non-zero $x\in\mathbb{R}$
	\begin{equation}
\label{expon}
	\e^{-|x|}\leq \frac{\e^x-1}{x}\leq \e^{|x|}.
\end{equation}
	So using \eqref{expon} we can write
\begin{equation*}\label{dineq}
	\exp({-|{b_{s,s\p}^{w}}^{\top}Z_w(c)|})\leq \frac{\exp({{b_{s,s\p}^{w}}^{\top}Z_w(c)})-1}{{b_{s,s\p}^{w}}^{\top}Z_w(c)}\leq \exp({|{b_{s,s\p}^{w}}^{\top}Z_w(c)|}).
	\end{equation*}
Finally, we multiply each side  by the expression $t_w(c; s)({b_{s,s\p}^{w}}^{\top}Z_w(c))^2
\exp({{\beta_{s,s\p}^{w}}^{\top}Z_w(c)})$ to conclude the~proof of \eqref{inlem}. Similarly we can prove \eqref{inlem2}. Finally, for any vector $x \in \mathbb{R}^{2d(d-1)}$ we have
\begin{equation*}
    x^{\top} \nabla ^2 \ell(\beta) x =
\frac{1}{T} \sum_{w\in\V}\sum_{c\in\X_{-w}} \sum_{ s\p\not=s} t_w(c; s) \left({x_{s,s\p}^{w}}^{\top} {Z}_w(c)\right)^2 \exp({\beta_{s,s\p}^{w}}^{\top} Z_w(c))
\end{equation*}
and
\begin{equation*}
    x^{\top} \nabla ^2 \ell(\beta+b) x =
\frac{1}{T} \sum_{w\in\V}\sum_{c\in\X_{-w}} \sum_{ s\p\not=s} t_w(c; s) \left({x_{s,s\p}^{w}}^{\top} {Z}_w(c)\right)^2 \exp({\beta_{s,s\p}^{w}}^{\top} Z_w(c))\exp({b_{s,s\p}^{w}}^{\top} Z_w(c)).
\end{equation*}
Comparing each summand separately and taking into account the definition of $c_b$ and the~inequality \eqref{expon} we finish the proof.
\end{proof}

\begin{lemma}
\label{estim}
Let $\xi >1 $ be arbitrary and assume that $\bar{F}(\xi)>0$.  Moreover,
let us denote $\tau = \dfrac{(\xi +1) |S| \lambda }{ 2\bar{F}(\xi) }$ and the event
$\Omega_2=\left\{\tau < e^{-1} \right\}.$
Then $\Omega_1 \cap \Omega_2 \subset A,$ where
\begin{equation*}
\label{estim1}
A= \left\{\|\hat{\beta} - \beta\| _\infty \leq \frac{2 \xi e^\eta \lambda}{(\xi+1) 
  \bar{F}(\xi)} \right\} 
\end{equation*}
and $\eta < 1 $ is the smaller solution of the equation $\eta e ^{- \eta} = \tau.$
\end{lemma}
\begin{proof}
The proof is similar to Theorem 3.1 of \cite{Cox13} or Lemma 6 of \cite{Rejchel18}. 
Suppose we are on the event $\Omega_1 \cap \Omega_2.$ Denote again $\tilde{\beta} = \hat{\beta} - \beta$, so by the previous lemma we have $\theta = \dfrac{\tilde{\beta}}{\|\tilde{\beta}\|_1}\in \ccC(\xi, S).$ Let us consider the function
\[
g(t) = \theta^{\top}\nabla\ell(\beta + t\theta) - \theta^{\top}\nabla\ell(\beta), \quad t\geq 0.
\]
This function is non-decreasing, because the negative log-likelihood function is convex. Hence, for every $t\in(0,\|\tilde{\beta}\|_1)$ we have $g(t)\leq g(\|\tilde{\beta}\|_1).$ By Lemma \ref{basiclem} on the event $\Omega_1$ we obtain
\begin{equation}\label{eq:bound1}
    \theta^{\top}[\nabla\ell(\beta + t\theta) - \nabla\ell(\beta)] + \dfrac{2\lambda}{\xi+1}\|\theta_{S^c}\|_1 \leq \dfrac{2\lambda\xi}{\xi+1}\|\theta_{S}\|_1.
\end{equation}
Using Lemma \ref{lem:second_deriv} with $b = t\theta$ and \eqref{inlem} we obtain
\begin{equation}\label{eq:bound2}
    t\theta^{\top}[\nabla\ell(\beta + t\theta) - \nabla\ell(\beta)] \geq t^2\exp(-t)\theta^{\top}\nabla^2\ell(\beta)\theta,
\end{equation}
in this case $c_b = c_{t\theta}\leq \exp(t).$ Now using the definition of CIF $\bar{F}(\xi),$ the fact that $\theta$ belongs to the cone $\ccC(\xi,S)$ and applying the bounds \eqref{eq:bound1}, \eqref{eq:bound2} we get
\begin{equation*}
\begin{split}
        t\exp(-t)\dfrac{\bar{F}(\xi)\|\theta_{S}\|_1^2}{|S|} & \leq t\exp(-t)\theta^{\top}\nabla^2\ell(\beta)\theta \leq \theta^{\top}[\nabla\ell(\beta + t\theta) - \nabla\ell(\beta)]\leq \\ 
        & \leq \dfrac{2\lambda\xi}{\xi+1}\|\theta_{S}\|_1 - \dfrac{2\lambda}{\xi+1}\|\theta_{S^c}\|_1 = \\ 
        & = 2\lambda \|\theta_S\|_1 - \dfrac{2\lambda}{\xi + 1} \leq \lambda(\xi + 1)\|\theta_S\|_1^2/2.
\end{split}
\end{equation*}
This means that for any $t$ satisfying \eqref{eq:bound1} we have
\begin{equation}\label{eq:tau}
    t\exp(-t) \leq\dfrac{(\xi+1)|S|\lambda}{2\bar{F}(\xi)} = \tau.
\end{equation}
Since, as we mentioned, the function $g(t)$ is non-decreasing, the set of all non-negative~$t$ satisfying \eqref{eq:bound1} is a closed interval $[0,\tilde{t}]$ for some $\tilde{t}>0$. Hence, \eqref{eq:tau} implies $\tilde{t}\leq \eta$, where~$\eta$ is the smallest solution of the equation $\eta e^{\eta} = \tau$. Now from \eqref{eq:bound1} and \eqref{eq:bound2} we obtain
\begin{equation*}
    \begin{split}
         \|\tilde{\beta}\|_1 e^{-\eta} & \leq \tilde{t}e^{-\tilde{t}}\leq 
          \dfrac{\tilde{t}\exp(-\tilde{t})\theta^{\top}\nabla^2\ell(\beta)\theta}{\bar{F}(\xi)\|\theta_T\|_1\|\theta\|_{\infty}} 
           \leq \dfrac{\theta^{\top}[\nabla\ell(\beta + \tilde{t}\theta) - \nabla\ell(\beta)]}{\bar{F}(\xi)\|\theta_T\|_1\|\theta\|_{\infty}} \leq \\
           & \leq \dfrac{2\lambda\xi}{(\xi+1)\bar{F}(\xi)\|\theta\|_{\infty}} = \dfrac{2\lambda\xi\|\tilde{\beta}\|_1}{(\xi+1)\bar{F}(\xi)\|\tilde{\beta}\|_{\infty}},
    \end{split}
\end{equation*}
which finishes the proof.
\end{proof}

\begin{proof} [Proof of Theorem~\ref{thm:consistency}]
Fix arbitrary $\varepsilon >0$ and $\xi>1.$ Then $F(\xi)$ is positive by Lemma~\ref{cif_bound}. Thus from Lemma \ref{lem:cif} we know that $
 \Pr\left(\bar F(\xi) \geq \zeta_0 F(\xi) \right)\geq 1-\varepsilon.$ Using it with the~right-hand side of \eqref{lambda_form} we obtain that $\Pr(\Omega_2) \geq 1- \varepsilon. $ Moreover, from Lemma \ref{lem:lambda} we have that $\Pr(\Omega_1) \geq 1- \varepsilon. $ Therefore, Lemmas~\ref{basiclem} and \ref{estim} (with $\eta=1$ for simplicity) imply the~inequality
$$
\Pr\left(\|\hat{\beta} - \beta\|_\infty \leq \frac{2 \xi e \lambda}{(\xi+1) 
  \bar{F}(\xi)} \right) \geq 1-2 \varepsilon.
 $$ 
Finally, we bound $\bar F(\xi)$ from below by $\zeta_0  F(\xi).$
\end{proof}

\begin{proof} [Proof of Corollary \ref{thm:consistency2}]
The proof is a simple consequence of the uniform bound \eqref{estim_formula}  obtained in Theorem \ref{thm:consistency}. Indeed, for arbitrary 
$w \in \V$, $ s \neq s'$ and $j$-th coordinate of the~vector $\beta_{s,s'}^w$ such that $\beta_{s,s\p}^{w} (j) =0$
 we obtain 
\[
| {\hat {\beta}_{s,s\p}}^{w} (j)| =|\hat \beta_{s,s\p}^{w} (j) -  \beta_{s,s\p}^{w} (j)| \leq \|\hat \beta - \beta\|_\infty \leq \delta.
\] 
Analogously, for each 
$w \in \V$, $s \neq s'$ and $j$-th coordinate such that $\beta_{s,s\p}^{w} (j)  \neq 0$
 we have 
\[
|\hat  \beta_{s,s\p}^{w} (j)  | \geq  | \beta_{s,s\p}^{w} (j)| -|\hat  \beta_{s,s\p}^{w} (j)  -  \beta_{s,s\p}^{w} (j) | \geq \beta_{\min} - \|\hat \beta - \beta\|_\infty > 2 \delta  - R \geq \delta,
\]
which concludes the proof.
\end{proof}

\begin{proof}[Proof of Lemma~\ref{cif_bound}]
Fix $\xi>1.$  For each $w$ and $c_{\I _w} $
we have $Z_w(c_{S_w},0)= (c_{S_w},0),$ so 
$$
F(\xi) =\inf_{0\not = \theta \in C(\xi,S)}\sum_{w\in\V}\sum_{s\p\not=s}\sum_{c_{\I _w}\in\X_{\I _w}}
\frac{\exp\left((\beta_{s,s\p}^{w})_{S_w}^\top c_{S_w}\right)\left[(\theta_{s,s\p}^{w})^\top_{S_w}  c_{S_w}\right]^2}{\|\theta_S\|_1 \|\theta\|_\infty},
$$
 where $\left(\beta_{s,s\p}^{w} \right)_{S_w}$ and $\left(\theta_{s,s\p}^{w} \right)_{S_w}$ are  restrictions of $\beta_{s,s\p}^{w} $ and $\theta_{s,s\p}^{w} $  to coordinates from 
$S_w,$ respectively. Therefore,  we need to bound from below  the expression
\begin{equation}
\label{cif2}
\frac{\sum\limits_{w\in\V} \sum\limits_{s\p\not=s}\sum\limits_{c_{\I _w}\in\X_{\I _w}}
\exp\left((\beta_{s,s\p}^{w})_{S_w}^\top c_{S_w}\right)\left[(\theta_{s,s\p}^{w})^\top_{S_w}  c_{S_w}\right]^2}{\|\theta_S\|_1 \|\theta\|_\infty}
\end{equation}
for each $\theta \in C(\xi,S) $ and 
$\theta \neq 0$. First, we restrict the third sum in the numerator of \eqref{cif2} to the  summands corresponding  only to vectors $e_u \in 
\X_{\I _w}$ having 1 on the coordinate corresponding to the node $u\in S_w$ and 0 elsewhere. Then we reduce the numerator of~\eqref{cif2} to the following form
\begin{equation}
\label{cif3}
\sum\limits_{w\in\V} \sum\limits_{s\p\not=s}\sum\limits_{u \in S_w}
\exp\left(\beta_{s,s\p}^{w} (u) \right)\left[\theta_{s,s\p}^{w} (u)  \right]^2 .
\end{equation}
Recall that  $S_w=\left\{u \in -w: \beta^w_{0,1} (u) \neq 0\ \text{ or }\ \beta^w_{1,0} (u) \neq 0\right\}.$ Therefore, if $\beta_{s,s\p}^{w} (u) \neq 0,$ then
$u \in S_w, $ so the sum \eqref{cif3} can be bounded from below by
\begin{equation}
\label{cif4}
\sum\limits_{w\in\V} \sum\limits_{s\p\not=s}\sum\limits_{u:\beta_{s,s\p}^{w} (u) \neq 0 }
\exp\left(\beta_{s,s\p}^{w} (u) \right)\left[\theta_{s,s\p}^{w} (u)  \right]^2\, ,
\end{equation}
because \eqref{cif3} has more summands and the summands are nonnegative. Using reverse H\"{o}lder's inequality we replace \eqref{cif4} by \begin{equation}
\label{cif5}
A_\beta^{-1}
\left[
\sum\limits_{w\in\V} \sum\limits_{s\p\not=s}\sum\limits_{u:\beta_{s,s\p}^{w} (u) \neq 0 } 
|\theta_{s,s\p}^{w} (u)|  \right]^2,
\end{equation}
 where $A_\beta$ is
$$
A_\beta= \sum\limits_{w\in\V} \sum\limits_{s\p\not=s}\sum\limits_{u:\beta_{s,s\p}^{w} (u) \neq 0 }
\exp\left(-\beta_{s,s\p}^{w} (u) \right).
$$ 
Next, recall that $S$ is the set of nonzero coordinates of $\beta,$ so \eqref{cif5} is just $ \|\theta_S\|_1^2 / A_\beta.$ Summarizing, the sum \eqref{cif2} is bounded from below by 
\begin{equation}
\label{cif6} \frac{\|\theta_S\|_1}{  A_\beta \|\theta\|_\infty}
\end{equation}
for each $\theta \in C(\xi,S) $ and 
$\theta \neq 0.$
The vector $\theta$ belongs to the cone  $C(\xi,S) ,$ which implies that 
$\|\theta_{S^c}\|_\infty \leq \|\theta_{S^c}\|_1  \leq \xi \|\theta_{S}\|_1 $ and
$$\|\theta\|_\infty = \max(\|\theta_{S}\|_\infty, \|\theta_{S^c}\|_\infty) \leq 
\max( \|\theta_S\|_1, \xi \|\theta_{S}\|_1),$$ 
which gives us $\|\theta\|_\infty \leq \xi \|\theta_{S}\|_1 .$ Applying it in \eqref{cif6}, we finish the proof.
\end{proof}

\section{Numerical examples}
\label{sec:numerical}

In this section we describe the details of algorithm implementation as well as the results of experimental studies.

\subsection{Details of implementation}\label{implementation}
We provide in details practical implementation of the proposed algorithm.
The solution of~\eqref{minimizer} depends on the choice of $\lambda$. 
Finding the ,,optimal'' parameter $\lambda$ and the threshold~$\delta$  is difficult in practice. Here we solve it using the information criteria \citep{Xueetal12, pokmiel:15, Rejchel18}. 

 First, using \eqref{l_decomp} we write the minimized function in \eqref{minimizer} as the sum
 \[
 \ell(\theta) - \lambda\|\theta\|_1 = \sum_{w\in\V} \sum_{ s\not=s'}\left( \frac{1}{T}\lssw (\tssw) - \lambda\sum\limits_{u\in -w}|\tssw| \right),
 \]
 where $s,s'\in\{0, 1\}$. Therefore, for fixed $w\in\V$ and $s,s\p\in\{0,1\}$ with $s\not=s\p$, the~cor\-res\-ponding summand is a function which depends on the vector $\theta$ restricted only to its coordinate vector $\tssw$ (see notation \eqref{beta}). So, for each triple $w$ and $s\neq s\p$ we can solve the problem separately. In our implementation we use the following scheme. 
We start with computing a sequence of minimizers on the grid, i.e.~for any triple $w\in\V$, $s \neq s\p$ we create a finite sequence $\{\lambda_i\}_{i=1}^{N}$ uniformly spaced on the log scale, starting from the~largest $\lambda_i$, which 
corresponds to the empty model. Next, for each value $\lambda_i$ we compute the esti\-mator~$\hat{\beta}_{s,s\p}^{w}[i]$ of the vector ${\beta}_{s,s\p}^{w}$
 \begin{equation}
\label{lassosw}
\hat{\beta}_{s,s\p}^{w}[i] = \argmin_{\theta_{s,s\p}^{w}} \left\{\ell_{s,s\p}^w(\theta_{s,s\p}^{w})+\lambda_i\|\theta_{s,s\p}^{w}\|_1\right\},
\end{equation}
 where as in \eqref{l_decomp}
 \begin{eqnarray*}
 \ell_{s,s\p}^w(\theta_{s,s\p}^{w})=\frac{1}{T} \sum_{c\in\X_{-w}}\left[-n_w(c; s,s\p){\theta_{s,s\p}^{w}}^{\top} Z_w(c) 
+t_w(c; s)\exp\left({\theta_{s,s\p}^{w}}^{\top} 
 Z_w(c)\right)\right].
\end{eqnarray*}
The notation $\hat{\beta}_{s,s\p}^{w}[i]$ should not be confused with ${\beta}_{s,s\p}^{w}(u)$ introduced before. Namely, $\hat{\beta}_{s,s\p}^{w}[i]$~is the $i$-th approximation of ${\beta}_{s,s\p}^{w}$, while ${\beta}_{s,s\p}^{w}(u)$ is the coordinate of ${\beta}_{s,s\p}^{w}$ cor\-res\-ponding to the node $u$. To solve \eqref{lassosw} numerically for a given  $\lambda_i$ we use the FISTA algorithm with backtracking from \cite{FISTA}.
 The final LASSO estimator $\hat{\beta}_{s,s\p}^{w}:=\hat{\beta}_{s,s\p}^{w} [{i^*}]$ is chosen using  the Bayesian Information Criterion (BIC), which is a~po\-pu\-lar method of choosing the value of $\lambda$ in the literature  \citep{Xueetal12, Rejchel18}, i.e.
\[
 i^*=\argmin_{1 \leq i \leq 100} \left \{n \ell_{s,s\p}^w(\hat{\beta}_{s,s\p}^{w}[i])+\log(n)\Vert \hat{\beta}_{s,s\p}^{w}[i] \Vert_0\right\}.
\]
Here $\Vert \hat{\beta}_{s,s\p}^{w}[i]\Vert_0$ denotes the  number of non-zero elements of  $\hat{\beta}_{s,s\p}^{w}[i]$ and $n$ is the number of observed jumps of the process. In our simulations we use $N=100$.

 Finally, the threshold $\delta$ is obtained using the Generalized Information Criterion (GIC). The similar way of choosing a threshold was used previously in \cite{pokmiel:15, Rejchel18}.
For a prespecified sequence of thresholds $\mathscr{D}$ we calculate  
\[
\delta^* =\argmin_{\delta \in \mathscr{D}} \left \{n\ell_{s,s\p}^w( \hat{\beta}_{s,s\p}^{w,\delta})+\log(2d(d-1))\Vert \hat{\beta}_{s,s\p}^{w,\delta} \Vert_0\right\},
\]
where $\hat{\beta}_{s,s\p}^{w,\delta}$ is the LASSO estimator $\hat{\beta}_{s,s\p}^{w}$ after thresholding with the level $\delta.$

\subsection{Simulated data}\label{simulated}
We consider three models defined as follows.
For shortness we denote these models later on as $M1$, $M2$ and $M3$, respectively.

\begin{model}
All vertices have the ``chain structure'', i.e.~for any node, except for the first one, its set of parents contains only a previous node. Namely, we put  $\V = \{1,\dots,d\}$ and
    $\pa(k)= \{k-1\}$, if $k>1$ and $\pa(1)=\emptyset$. We construct  CIM in the following way. For the~first node the intensities of leaving both states are equal to $5$. 
    For the rest of the nodes $k = 2,\ldots,d$,  we choose randomly $a\in\{0,1\}$ and we define
$Q_k(c,s,s\p)=9,$ if $s\not=|c-a|$ and $1$ otherwise. 
In other words, we choose randomly whether the  node  prefers to be at the same state as its parent ($a=0$) or not ($a=1$). Say, the node $k$ prefers to be at the~same state as the node  $k-1$. Then
if these two states coincide, the intensity of leaving the current state is $1$, otherwise it is $9$.  The intensity is  defined  analogously, when the node $k$ does not prefer to be at the same state as the node $k-1$.
\end{model}

\begin{model}
The first $5$ vertices are correlated, while the remaining vertices are independent. 
 We sample $10$ arrows between first $5$ nodes by choosing randomly $2$ parents for each node.
 In order to define the intensity matrix, consider the node $w\in\V$ with $\pa(w)\neq \emptyset$ and a~configuration $c\in\ccX_{\pa(w)}$ of the parents states. We denote $|c|=1$ if all the parents of $w$ are in the state 1, and $|c|=0$ otherwise. Next we define intensities as follows
 \begin{equation*}\label{int_example}
  Q_w(c,s,s\p) =\begin{cases}
                 5 &\text{if $\pa(w)=\emptyset$,}\\
                 9 &\text{if $\pa(w)\not=\emptyset$, $s$ is preferred state and }
  |c|=1, \\
                 1 &\text{if $\pa(w)\not=\emptyset$, $s$ is preferred state and }
|c|=0,\\
                 9 &\text{if $\pa(w)\not=\emptyset$, $s$ is not preferred state and }
 |c|=0,\\
                 1 &\text{if $\pa(w)\not=\emptyset$, $s$ is not preferred state and }
 |c|=1.\\
                \end{cases}
 \end{equation*}
As in the previous model, the preferred state is chosen randomly from $\{0,1\}$. In words, for every node $w$ with $\pa(w)\not = \emptyset$ we choose randomly one state, say $0$. 
  In this case, if all parents are $1$ the process prefers
  to be in $1$ and if some of the parents are $0$ the process prefers to be in $0$.
\end{model}

\begin{model}
All vertices have a  ,,binary tree'' structure with arrows from leaves to the root. So, leaves have no parents, while the inner nodes have two parents, with the exception that one node  has only a single parent, if $d$ is even. If the node has no parents or its parents have  different states, then the intensity of leaving a state is 5. Otherwise, if a~node has only one parent or both parents are in the same state, then the intensity of leaving a~state are computed as in Model 1.
\end{model}

The model $M1$ has a simple structure, which involves all vertices and satisfies our assumption~\eqref{def: beta}. 
The model $M2$ corresponds to a dense structure on a small subset of vertices and does not satisfy assumption~\eqref{def: beta}. Another potential difficulty is related to possible feedback loops, which are usually hard to recognize. 
Therefore, we also consider model $M2+$, which looks like $M2$, but contains  the interaction terms and fulfills~\eqref{def: beta}.
The model $M3$ has  slightly more complex structure than $M1$, but it also satisfies our assumption ~\eqref{def: beta}.

We consider two cases: $d=20$ and $d=50$ for all four models.   So, the considered number of possible parameters of the model (the size of $\beta$) is $ 2d^2 = 800$ or $5000$, respectively. For model with interactions, number of possible parameters is $d^2(d+1) = 8400$ or $127500.$ 
We use $T=10$ and $50$ for all models and we replicate simulations $100$ times for each scenario. In  Table~\ref{tab:results} we present averaged results of the simulations in terms of three quality measures
\begin{itemize}
    \item {\bf power}, which is a proportion of correctly selected edges;
\item {\bf false discovery rate (FDR)}, which is a fraction of incorrectly  selected  edges among all
selected edges;
\item {\bf model  dimension  (MD)}, which is a number of selected edges.
\end{itemize}

\begin{table}[ht]
\caption{Results for simulated data.  In the model $M1$ and $M3$ the true
dimension is~$19$ for $d=20$ and $49$ for $d=50$. In the model $M2$ the true model dimension is $10$.  }
\centering
\begin{tabular}{lll rrr }
  \toprule
 Model  & $d$ & $T$ & Power & FDR & MD\\
  \midrule
$M1$ & 20   & 10 &0.86& 0.03 & 16.8\\         
   &      & 50 &1& 0.02 &19.3 \\ 
   & 50   & 10 &0.61 &0.01 &30.3\\
   &      & 50 &1& 0.01 & 49.3\\
\midrule
$M2$ & 20   & 10 &0.16& 0.2&  2\\ 
   &      & 50 &0.78&0.04  &8.1\\
      & 50   & 10 &0.10 &0.15 &1.28\\
   &      & 50 &0.62& 0.02 & 6.4\\
\midrule
$M2 +$ & 20   & 10 &0.35& 0.1&  3.9\\ 
   &      & 50 &0.9&0.02  &9.2\\
      & 50   & 10 &0.17 &0.08 &2\\
   &      & 50 &0.68& 0.01 & 6.9\\

\midrule
$M3$ & 20   & 10 &0.17& 0.1&  3.7\\ 
   &      & 50 &0.97&0.01  &18.7\\
      & 50   & 10 &0.6 &0.09 &3.2\\
   &      & 50 &0.88& 0.003 & 43\\

 \bottomrule
\end{tabular}
\label{tab:results}
\end{table}

We observe that in the  models $M1$ and $M3$ the results of experiments confirm that the proposed method works in a satisfactory way. For 
$T=10$ the algorithm has high power and its FDR is not large. The final model selected by our procedure is slightly too small (it does not discover a few existing edges). When we increase observation time ($T=50$), then our estimator behaves almost perfectly. 

The  model $M2$ is much more difficult and this fact has a direct impact on simulation results. Namely, for $T=10$  the power of the algorithm is relatively low with FDR also being rather small. The procedure performs slightly better when we take $T=50.$ However, for both observation times the estimator cannot find the true edges in the~graph. 
 One of the reasons of such behaviour   is that in $M2$ the dependence structure in CIM is not additive in parents. This fact
combined with possible feedback loops  leads to recovering existing edges, but having the opposite to the true ones directions. Looking deeper into the results for a few examples chosen from our experiments we confirm this claim, i.e.~the~edges between nodes are correctly selected, but their directions are wrong. Therefore, we can conclude that in the complex model $M2$ our estimator seems at least to be able to recognize interactions between nodes, which is important in many practical problems on its own. The results for $M2+$ confirm that the performance of our method increases, when we consider more complex parametrization with interaction terms.   

\section{Extension of the results}
\label{sec:discussion}
In this chapter we proposed the method for structure learning of CTBNs. We confirmed the good performance of our method  both theoretically and experimentally.
To simplify the notation and help the reader to follow our reasoning we restricted ourselves to graphs with only two possible states for each node. However, our results can be generalized in a straightforward way to any finite graphs by extending $\beta$ to other possible jumps and possible values of parents. In terms of the explanatory variable, it is equivalent to the standard encoding of qualitative variables in 
linear or generalized linear models. To demonstrate the generalization more clearly we present an example similar to Example~\ref{example_ctbn} presented in the very beginning of this chapter.

\begin{example}
\label{example_ctbn2}
We consider CTBN with three nodes $A,B$ and $C.$ Let their state spaces be $\ccX_A = \{0,1,2\}$, $\ccX_B = \{0,1,2,3\}$ and $\ccX_C = \{0,1,2\}$, respectively. Then for the node~$A$ we define the function $Z_A$ as 
$$
Z_A(b,c)=[1, \Ind(b=1), \Ind(b=2), \Ind(b=3), \Ind(c=1), \Ind(c=2)]^{\top}
$$
for each $b \in \{0,1,2,3 \}$ and $c \in \{0,1,2 \}$. Analogously, we can define representations for the remaining nodes:
$$
Z_B(a,c)=[1, \Ind(a=1), \Ind(a=2), \Ind(c=1), \Ind(c=2)]^{\top}
$$
for each $a \in \{0,1,2 \}$ and $c \in \{0,1,2 \}$ and 
$$
Z_C(a,b)=[1, \Ind(a=1), \Ind(a=2),  \Ind(b=1), \Ind(b=2), \Ind(b=3)]^{\top}
$$
for each $a \in \{0,1,2 \}$ and $b \in \{0,1,2,3 \}$.

Therefore, for each node $w$ and for each configuration of parents' states (e.g.~for the node $A$ and values in the nodes $B$ and $C$) 
the value of the function $Z_w(\cdot,\cdot)$ is still a~binary vector with the dimension equal to the sum of the numbers of states in all parents nodes with subtracted number of nodes and added 2.   In this example the parameter vector~\eqref{beta} is defined as
$$\beta=\left(\beta_{0,1}^{A}, \beta_{1,0}^{A}, \beta_{0,2}^{A}, \beta_{2,0}^{A},\beta_{1,2}^{A}, \beta_{2,1}^{A},
\beta_{0,1}^{B}, \beta_{1,0}^{B}, \beta_{0,2}^{B}, \dots, \beta_{3,1}^{B}, \beta_{2,3}^{B}, \beta_{3,2}^{B},
\beta_{0,1}^{C}, \dots, \beta_{2,1}^{C}\right)^\top\!.$$
With a slight abuse of notation, the vector $\beta ^A_{0,1}$ is given as 
$$\beta^A_{0,1}=\left[\beta^A_{0,1} (1), \beta^A_{0,1} (B = 1), \beta^A_{0,1} (B = 2), \beta^A_{0,1} (B = 3), \beta^A_{0,1} (C = 1), \beta^A_{0,1} (C = 2) \right]^ \top,$$
and we interpret it as follows: if all $\beta^A_{0,1} (B=1)$, $\beta^A_{0,1} (B = 2)$ and $\beta^A_{0,1} (B = 3)$ are equal to 0, then the intensity of the change from the state $0$ to $1$ at the node $A$ does not depend on the state at the node $B.$ Similarly, the coordinates $\beta^A_{0,1} (C = 1)$ and $\beta^A_{0,1} (C = 2)$ describe the dependence between the above intensity and the state at the node $C,$ and  $\beta^A_{0,1} (1)$ corresponds to the intercept. For the node $B$ the coordinates of the vector 
$\beta^B_{1,3}=\left[\beta^B_{1,3} (1), \beta^B_{1,3} (A = 1), \beta^B_{1,3} (A = 2), \beta^B_{1,3} (C=1), \beta^B_{1,3} (C=2) \right]$ describe the relation between the intensity of the jump from the state $1$ to $3$ at the node $B$ to the intercept, states at the nodes $A$ and $C,$ respectively.
\end{example}

Our results can be also easily generalized to the case, where we consider not only additive effect in \eqref{def: beta}, but also interactions between parents. Let us again use an example.
\begin{example}
As previously we consider CTBN with three nodes $A,$ $B$ and $C$ with corresponding state spaces $\ccX_A = \{0,1\}$, $\ccX_B = \{0,1,2\}$ and $\ccX_C = \{0,1\}$.
In a linear model we have the following $Z_w$ functions:
	\begin{equation*}
	\begin{split}
	    Z_A(b,c)&=[1,\Ind(b=1),\Ind(b=2),\Ind(c=1)]^{\top},\\ Z_B(a,c)&=[1,\Ind(a=1),\Ind(c=1)]^{\top},\\  Z_C(a,b)&= [1,\Ind(a=1),\Ind(b=1)]^{\top}
	\end{split}
	\end{equation*}
for each $a,c\in\{0,1\}$ and $b\in \{0,1,2\}$. Then after we add pairwise interactions to the~linear model the functions above take the form
		\begin{align*}
		Z_A(b,c)&=[1,\Ind(b=1),\Ind(b=2),\Ind(c=1),{\color{red}\Ind(b=1, c=1),\Ind(b=2,c=1)}]^{\top},\\ 
		Z_B(a,c)&=[1,\Ind(a=1),\Ind(c=1),{\color{red}\Ind(a=1,c=1)}]^{\top},\\  Z_C(a,b)&= [1,\Ind(a=1),\Ind(b=1),{\color{red}\Ind(a=1,b=1)}]^{\top}.
		\end{align*}
For models with more nodes we can also take into account more complex interactions.
\end{example}


  \chapter{Structure learning for CTBNs for incomplete data}\label{chapter:CTBNpartial}
In the previous chapter we considered the case when we observe CTBN at each moment of time. Under this assumption we introduced a novel method of structure learning.
In this chapter we show that our method can be adapted to partially observed and noisy data. In the case of partial observations we need to introduce the observation and the~likelihood of the observed data given a hidden trajectory of a process. We can again parametrize CIM by \eqref{def: beta}.
However, in this case the problem  \eqref{minimizer} becomes more challenging and leads to the following two problems. First, the theoretical analysis becomes more challenging because the loss function is not convex. Second, the likelihood function can not be calculated explicitly, hence, it is difficult to obtain from the computational perspective.

In our solution we formulate the EM algorithm for this case, where the expectation step is standard and concerns the calculation of the expected log-likelihood. The maximization step is performed in the same way as for the complete data. Since the density belongs to the exponential family, the E-step requires to compute the expected values of sufficient
statistics, which is done with the MCMC algorithm developed in \cite{RaoTeh2013a}. In addition, the results from \cite{Majewski2018} combined with \cite{Miasojedow2017} are used in the analysis of the Monte Carlo scheme. 

\section{Introduction and notation}
Let $\boldsymbol{t} = (t_0,t_1,...,t_n)$ with $0=t_0<t_1<...<t_n$ and $\boldsymbol{S} = (\boldsymbol{S}_0,\boldsymbol{S}_1,...,\boldsymbol{S}_n)$ describe the~full trajectory $X$ of the process on the interval $[0,T]$ ($\boldsymbol{t}$ denotes times of jumps, $\boldsymbol{S}$ is a~skeleton, where $\boldsymbol{S}_i\in\ccX$ is a state at the moment ${t_i}$). Let $\mathbb{X}$ denote the set of all possible trajectories of the process. Let us consider the case when instead of observing the full trajectory $X$ we have access only to the partial and noisy data $Y$ with the conditional density $p(Y\mid X)$. More precisely, we assume that $Y$ is represented by the observation of~$X$ at times $t_1^{obs},\ldots, t_k^{obs}$ with the likelihoods $g_j(\boldsymbol{S}_{j^{\boldsymbol{t}}})$ for $j = 1,\ldots,k,$ where
\[
	j^{\boldsymbol{t}} = \max\{i:t_i\leq t_j^{obs} \}.
\]
We assume that $0<C<g_j<\tilde{C}$ for $1\leq j\leq k$. In this case the full density is given by
\[
 p_\beta(X,Y)=p(Y\mid X)p_\beta(X),
\]
where $p_\beta(X)$ is given by \eqref{eq:densCTBN}. Observe that the dependence of $p_\beta(X)$ on $\beta$ is mediated through our assumption \eqref{def: beta}, which can be inserted into \eqref{cbi}. For the clarity of presentation we assume that $p(Y\mid X)$ is known, however, the adaptation of our method to the~case where $p(Y\mid X)$ depends also on some unknown parameters is straightforward. The~negative of the log-likelihood function in this case is given by
\[
 \ell(\beta) =- \log\left(\int_{\mathbb{X}} p_\beta(X,Y) dX \right),
\]
where symbol $dX$ means the summation first over all possible numbers of jumps of the~trajectory $X$, then over all possible jumps, and the integration with respect to times of jumps. More precisely,
\begin{equation*}
    \int f(X)dX = \sum\limits_{n=0}^{\infty} \sum\limits_{\SS_1\in\ccX} \dots \sum\limits_{\SS_n\in\ccX} \int\limits_{0}^{t_2} \int\limits_{t_1}^{t_3} \dots \int\limits_{t_{n-1}}^{T} f(n, t_1, \dots, t_n, \SS_1, \dots, \SS_n) dt_1\dots dt_n.
\end{equation*}
Again we can define the estimator of the parameter vector $\beta$, as previously,
\begin{equation}\label{def: hat_beta}
 \hat\beta =\argmin_{\theta\in\mathbb{R}^{2d(d-1)}}\left\{\ell(\theta)+\lambda\Vert\theta\Vert_1\right\}.
\end{equation}
Since we are not able to compute $\ell(\theta)$ analytically, we need to propose an efficient algorithm for finding $\hat\beta$. One of the efficient algorithms of solving complex optimization problems of the form \eqref{def: hat_beta} is the projected Proximal Gradient Descent (p-PGD) algorithm (see for example \cite{FISTA} and \cite{Majewski2018}). 
For a~closed compact convex set $\ccK$ by ${\prod}_{\ccK}(a)$ we denote the projection of $a$ onto $\ccK$. Then p-PGD is defined iteratively by
\begin{equation}
 \label{def:prox_grad}
 \beta_{k+1}= {\prod}_{\ccK}\left(\prox_{\gamma_k,\lambda\Vert\cdot\Vert_1 }(\beta_k -\gamma_k\nabla\ell(\beta_k))\right),
\end{equation}
where $\{\gamma_k\}_{k\geq 0}$ is a sequence of  step-sizes, and $``\prox"$ denotes the proximal operator defined for any convex function $g$ by
\[
 \prox_{\gamma, g}(x)=\argmin_y \left( g(y)+\frac{1}{2\gamma}\Vert y-x \Vert^2\right).
\]
In the case of $\ell_1$ penalty, i.e.~$g=\lambda\Vert\cdot\Vert_1$, the proximal operator is just a soft-thresholding operator. Element-wise soft-thresholding operator $S_{\lambda} : \R^n\to\R^n$ is defined as 
$$S_{\lambda} (x_i) = [|x_i|-\lambda]_+ \text{ sgn} (x_i),
$$ where $[\cdot]_+$ denotes the positive part.

In our case we are not able to evaluate the gradient $\nabla\ell$ explicitly and we will use stochastic version of the projected proximal gradient algorithm (p-SPGD), where $\nabla\ell$ is replaced by its Monte Carlo approximation. Under the regularity conditions given for example in Assumption $AD.1$ in \cite{Douc} we derive that the gradient of the~negative log-likelihood is given by
\begin{equation}
 \label{eq:Fisher}
 \begin{split}
      \nabla\ell(\beta) & = - \nabla \log\left(\int p_\beta(X,Y) dX \right) =  -\dfrac{\nabla \int p_\beta(X,Y) dX}{\int p_\beta(X,Y)dX} =\\
       & = -\dfrac{\int p_\beta(X,Y) \nabla \log(p_\beta(X,Y)) dX}{\int p_\beta(X,Y)dX} =\\
        & = -\int\dfrac{\nabla \log(p_\beta(X,Y)) p_\beta(X,Y)}{\int p_\beta(X,Y)dX} dX=\\
        & = -\int\nabla \log(p_\beta(X,Y))p_\beta(X\mid Y)dX = \\
        & = -\Ex(\nabla\log(p_\beta(X,Y))\mid Y).
 \end{split}
\end{equation}
This equation is sometimes referred as \textit{Fisher's identity}.
Now based on \eqref{eq:Fisher} we can approximate $\nabla\ell$ by
\begin{equation}
 \label{eq:MCMC_approx}
 \Phi(\beta, X^1,\ldots,X^m)=-\frac{1}{m}\sum_{i=1}^m \nabla\log(p_\beta(X^i,Y)),
\end{equation}
where $X^1,\dots,X^m$ is a set of subsequent states of the Markov chain with the stationary distribution $\pi_\beta= p_\beta(X\mid Y)\propto p_\beta(X,Y)$, where in particular each of $X^1,\dots,X^m$ is a~trajectory of the process $X$. To generate
this Markov chain we will use the procedure described below.

\section{Sampling the Markov chain with Rao and Teh's algorithm}\label{RaoTeh}
Consider the set $\ccM$ of all possible intensity matrices of Markov jump processes with the~state space $\ccX$ equipped with some matrix distance.
Therefore, for any $Q\in\ccM$ and $\mathbf{s}\in\ccX$ each element $Q(\mathbf{s},\s)$ on the diagonal of $Q$ is nonpositive, and otherwise it is non-negative. Let $\ccL\subset\ccM$ be a compact set and choose $\eta > \sup\limits_{Q\in\mathcal{L}}\max\limits_{\mathbf{s}\in\ccX} Q(\mathbf{s})$, where we used the notation $Q(\s) = -Q(\s,\s)$ introduced in Section \ref{sec:CTBN}. To generate the Markov chain parametrized by $Q\in\ccL$ we will use Rao and Teh's algorithm, which uses the idea of \textit{uniformization} and the notion of \textit{virtual jumps} (\cite{RaoTeh2013a}). A virtual jump in a trajectory described by a pair of time points $\tt$ and a corresponding skeleton $\SS$ means that for two subsequent time points $t_{i}$ and $t_{i+1}$ we have $\SS_i = \SS_{i+1}$. Simply put, it means that the process can jump from a certain state back to the same state. Here we provide a comprehensive description of a single iteration of the algorithm. Given an arbitrary trajectory $(\boldsymbol{t},\boldsymbol{S})$ of $X$, such that $Y(t_i^{obs}) = X(t_i^{obs})$ for $1\leq i \leq k,$ we generate another trajectory
$( \boldsymbol{\bar t}, \boldsymbol{\bar S})$ with this property using the following procedure:
\begin{enumerate}
\item  We generate times of virtual jumps $\boldsymbol{v}$ from piecewise homogeneous Poisson process with the intensity $\eta - Q(X(t))$, which means that for every interval $[t_{i},t_{i+1})$ we sample a number~$k_i$ of virtual jumps
 from the Poisson distribution with the pa\-ra\-meter equal to $(\eta-Q(\SS_{i}))(t_{i+1}-t_{i})$; then the times of virtual jumps are uniformly distributed on $[t_{i-1},t_i)$.
 \item We add virtual jumps to the trajectory in a correct order and we obtain new times of jumps $\boldsymbol{t'}=\boldsymbol{t}\cup\boldsymbol{v}$ and a new corresponding skeleton $\boldsymbol{S}' = (\boldsymbol{S}'_0,...,\boldsymbol{S}'_{n'})$, where $n'$ is the~number of elements of $\boldsymbol{t'}$. Therefore $\SS' = (\SS_0,\ldots,\SS_0,\SS_1,\ldots,\SS_1,\ldots, \SS_n,\ldots,\SS_n)$ with $k_0$ instances of $\SS_0$, $k_1$ instances of $\SS_1$, etc. In other words, this skeleton contains the same jumps at the same time points as $\SS$ and for every interval $[t_{i},t_{i+1})$ the states in~$\boldsymbol{S'}$ are equal to $\boldsymbol{S}_{i-1}$.
 \item We sample a new skeleton $ \boldsymbol{\bar S}$ of the size $n'$ using the standard Forward-Filtering Backward-Sampling algorithm (FFBS), which for completeness is provided at the~end of this chapter in Section \ref{sec:ffbs}. Thus, the resulting distribution of the skeleton will~be
 \[
 \dfrac{\nu(\boldsymbol{S}_0)\prod_{i=1}^{n'}P(\boldsymbol{S}'_{i-1},\boldsymbol{S}'_i)\prod_{j=1}^k g_j(\boldsymbol{S}'_{j^{\boldsymbol{t'}}})}{\sum_{\boldsymbol{S}}\nu(\boldsymbol{S}_0)\prod_{i=1}^{n'}P(\boldsymbol{S}_{i-1},\boldsymbol{S}_i)\prod_{j=1}^k g_j(\boldsymbol{S}_{j^{\boldsymbol{t'}}})},
  \]
where $\nu$ is the initial distribution (which does not depend on $Q$), and
\begin{equation}\label{def:P}
  P(\s,\s')=\begin{cases}
            \dfrac{Q(\s,\s')}{\eta}, &\text{ if } \s\neq \s',\\
            1-\dfrac{Q(\s)}{\eta}, & \text{ if } \s = \s',
           \end{cases}
\end{equation}
where $\s,\s'\in\ccX$. The summation in the denominator is taken with respect to all possible skeletons of size $n'$ containing virtual jumps.
\item From the trajectory $(\boldsymbol{t'},\boldsymbol{\bar S})$ we discard newly acquired virtual jumps (i.e.~we remove~$\bar{\boldsymbol{S}_i}$ such that $\bar{\boldsymbol{S}_i}=\bar{\boldsymbol{S}}_{i-1}$) and we obtain a new set $\bar{\tt}$ of times of jumps and the~resulting trajectory $(\boldsymbol{\bar t}, \boldsymbol{\bar S}),$ which describes the desired Markov chain.
\end{enumerate}
The procedure above describes one step of the algorithm and \cite{RaoTeh2013a} showed that both trajectories $ (\boldsymbol{t},\boldsymbol{S})$ and $(\boldsymbol{\bar t},\boldsymbol{\bar S})$ describe the same MJP. Simply put, the Poisson rate $\eta$ dominates the leaving rates of all states of the MJP and the new skeleton will contain more events than there are jumps in the MJP path. The corresponding trajectory is regarded as a redundant representation of a pure-jump process that always jumps to a~new state. Note, that our new stochastic matrix $P$ defined in \eqref{def:P} allows self-transitions (we refer to them as virtual jumps), and as $\eta$ increases their number grows as well. These self-transitions will be discarded in the final step of the algorithm, which compensates for an increased number of events.

The step of the algorithm is the composition of two Markov kernels. First we add virtual jumps according to the kernel $M^J_Q$ defined by
	\begin{multline}\label{defKerJ}
	M^J_Q((\boldsymbol{t}, \boldsymbol{ S}),  (\boldsymbol{\tilde t}, \boldsymbol{ \tilde S})) =  \Ind(\boldsymbol{\bar t}= \boldsymbol{t}\cup\boldsymbol{v})\\\prod_{i=0}^{n-1}
	\left\{ \left[(\eta - Q(s_i))(t_{i+1}-t_i)\right]^{k_i}e^{-(\eta - Q(s_i))(t_{i+1}-t_i)}\Ind(t_i<v_{i,1}<\cdots<v_{i,k_i}<t_{i+1}) \prod_{l=j_i}^{j_{i+1}} \Ind({\tilde \SS}_{l}=\SS_i) \right\}.
	\end{multline}
Next we draw a skeleton according to the kernel $M^S_Q$ given by
\begin{equation}\label{defKerS}
   M^S_Q((\boldsymbol{\tilde t}, \boldsymbol{\tilde S}),  (\boldsymbol{\bar t}, \boldsymbol{ \bar S}))= \Ind(\boldsymbol{\bar t}=\boldsymbol{\tilde t}) \dfrac{\nu(\bar \SS_0)\prod_{i=1}^{n}P(\bar{\SS}_{i-1},\bar{\SS}_i)\prod_{j=1}^k g_j(\bar{\SS}_{j^{\boldsymbol{t}}})}{\sum_{\boldsymbol{\tilde S}}\left\{\nu(\tilde{\SS}_0)\prod_{i=1}^{n}P(\tilde{\SS}_{i-1}, \tilde{\SS}_i)\prod_{j=1}^k g_j(\tilde{\SS}_{j^{\boldsymbol{\tilde t}}})\right\}},
\end{equation}
where $n$ denotes the length of the vector $\boldsymbol{\tilde t}$. The dependence of $M^S_Q$ on $Q$ is hidden in the~definition of~$P$.

Note that  adding virtual jump times $ \boldsymbol{v}$ defines the new skeleton uniquely and we denote it by $\SS^{\vv}$. Therefore, since sampling a new skeleton does not change times of jumps, the~full kernel $M_Q= M_Q^S M_Q^J$ is given by
\begin{equation}\label{kerQ}
    M_Q((\boldsymbol{t}, \boldsymbol{S}),  (\boldsymbol{\bar t}, \boldsymbol{ \bar S}))= M^J_Q((\boldsymbol{t}, \boldsymbol{ S}),  (\boldsymbol{\bar t}, \boldsymbol{S^{v}})) M^S_Q((\boldsymbol{\bar t}, \boldsymbol{ S^{v}}),  (\boldsymbol{\bar t}, \boldsymbol{ \bar S})).
\end{equation}
The kernel $M_Q$ acts on a function $f(\boldsymbol{t}, \boldsymbol{ S})$ as follows
\begin{multline*}
M_Q f(\boldsymbol{ t}, \boldsymbol{S}) = \Ex \left[f(\boldsymbol{\bar t}, \boldsymbol{\bar S})\mid (\boldsymbol{t}, \boldsymbol{ S})\right] = \\  = \sum_{k_0=0}^\infty\cdots\sum_{k_{n-1}=0}^\infty  \int \prod_{i=0}^{n-1}\Ind\{t_{i}<v_{i,1}<\cdots<v_{i,k_{i}}<t_{i+1}\} \sum_{\boldsymbol{\bar S}} f(\boldsymbol{\bar t},\boldsymbol{\bar S})M_Q((\boldsymbol{t}, \boldsymbol{S}),  (\boldsymbol{\bar t}, \boldsymbol{ \bar S})) d\boldsymbol {v}.
\end{multline*}
where the inside sum is taken with respect to all possible skeletons of the same length as~${\bar\tt}$. Further, when it does not cause any confusion, for convenience we often denote any single trajectory $(\boldsymbol{t}, \boldsymbol{S})$ as $X$, and as $X_i$ $-$ the trajectory obtained after $i$-th iteration of the algorithm with the starting trajectory $X_0$. Then for any two adjacent trajectories~ $X_{i}$ and~$X_{i+1}$ the function $M_Q f(X_i)$ stands for $\Ex \left[f(X_{i+1})\mid X_{i}\right]$. Moreover, for any trajectory~$X$ let $V(X) = V(\boldsymbol{t},\boldsymbol{S})=n+1$ (recall that $n$ denotes the number of jumps on the trajectory~$X$ described by $\boldsymbol{t}$ and $\boldsymbol{S}$).

\section{Structure learning via penalized maximum likelihood function}
Recall the assumption \eqref{def: beta} introduced in the previous chapter
\begin{equation*}
\log(Q_w(c,s,s\p))= {\beta_{s,s\p}^{w}}^{\top} Z_w(c)\;.
\end{equation*}
For a given parameter vector $\beta\in\R^{2d(d-1)}$ this defines the intensity matrix $Q$ and this mapping can be regarded as an isometry. So, if $\beta$ belongs to a compact set~$\ccK\in\R^{2d(d-1)}$, then $Q$ belongs to some compact set $\ccL \in\ccM$ and the construction above is still valid. In this case, we will frequently write $M_{\beta}$ instead of $M_Q$. Now we introduce the~main theoretical result regarding the convergence of our algorithm.
\begin{theorem}
 \label{thm: convergence}
 Let $\ccK\in\R^{2d(d-1)}$ be some compact convex set. Denote as $N_{\ccK}(\beta)$ the normal cone to the set $\ccK$ at the point $\beta$
 \[
    N_{\ccK}(\beta) = \{a\in\R^{2d(d-1)} : \langle a, z-\beta \rangle \text{ for all } z\in\ccK\}.
 \]
 Moreover, denote
 \[
 \mathcal{S}=\{\beta\in\mathcal{K} \colon 0\in \nabla\ell(\beta)+\lambda \partial\Vert\beta\Vert_1 - N_{\ccK}\},
 \]
 where ${\partial}\Vert\beta\Vert_1$ denotes the subgradient. Suppose that $(\ell + \lambda\|\cdot\|_1)(\ccS)$ has non-empty interior. Assume also $\Ex V^2(X_0)<\infty$.  Let the sequence $\{\gamma_k,k\in\N\}$ satisfy $\gamma_k>0$, $\lim\limits_{k\to\infty}\gamma_k = 0$, and
 \[\sum_{k=1}^{\infty}\gamma_k = \infty,\quad
 \sum_{k=1}^{\infty}\gamma_k^2 < \infty,\quad \sum_{k=1}^{\infty}|\gamma_k-\gamma_{k+1}| < \infty. \]
 Let $\{\beta_k\}$ be a sequence generated by the projected stochastic proximal gradient descent as in (\ref{def:prox_grad}). Then 
 \[
  \dist(\beta_k, \mathcal{S}\cap\mathcal{K})\stackrel{k\to\infty}{\to} 0\quad a.s.
 \]
 \end{theorem}
 
\begin{remark}
\ 
\renewcommand\labelenumi{(\theenumi)}
\begin{enumerate}
    \item Obviously, $\lim\limits_{k\to\infty} \gamma_k= 0$ follows from the convergence $\sum_{k=1}^{\infty}\gamma_k^2 < \infty$.
    \item The symbol $(\ell + \lambda\|\cdot\|_1)(\ccS)$ should be understood as the image of the set $\ccS$ under the~function $h(\theta) = \ell(\theta) + \lambda\|\theta\|_1$.
\end{enumerate}
\end{remark}

The theorem is a consequence of Theorem 5.4 of \cite{Majewski2018}. It states that the sequence of parameter vectors $\beta_k$ generated by the projected SPGD algorithm converges almost surely to a stationary point of the function being minimized, where instead of the gradient of the negative log-likelihood we used its Markov chain approximation. 

Before proving the theorem we need a few auxiliary results and some additional notation. For any function $f$ of the trajectory~$X$ and any signed-measure $\mu$, we define its $V$-variation by
\begin{equation}\label{Vmeasure}
     \Vert \mu \Vert_V =\sup_{|f|\leq V}|\mu(f)|,
\end{equation}
where
\begin{equation*}
    \mu(f) = \int_{\mathbb{X}} f d\mu,
\end{equation*}
and the integration is over all possible trajectories of the process. Also we denote
\begin{equation}\label{Vfunction}
 |f|_V=\sup_{X\in\mathbb{X}}\frac{|f(X)|}{V(X)},
\end{equation}
where supremum is taken with respect to all possible trajectories.

First we prove three auxiliary lemmas concerning the kernels $M_Q$ defined by \eqref{kerQ}.
\begin{lemma}\label{drift}
Fix a compact set $\ccL\in\ccM$. Then there exist constants $\rho_1, \rho_2\in(0,1)$ and $b_1, b_2<\infty$ such that for any trajectory~$X$ we have
 \begin{equation*}
     \sup_{Q\in \mathcal{L}} M_Q V(X)\leq \rho_1 V(X) + b_1
 \end{equation*}
 and
  \[
  \sup_{Q\in \mathcal{L}} M_Q V^2(X)\leq \rho_2 V^2(X) + b_2.
 \]
\end{lemma}
\begin{proof}
The proof is a simple extension of proofs of Lemma 5 and Proposition 6 in \cite{Miasojedow2017}. First we note that in the first step of Rao and Teh's algorithm we do not add any new jumps. This implies that in order to obtain the desired bounds on~$M_\beta V(X)$ we simply need to bound the expectation $\Ex V(X')$ of the jumps for the trajectory $X'$ obtained on the Step 3. Analogously, instead of bounding $M_\beta V^2(X)$ we bound $\Ex V^2(X')$. Indeed we have
\[
    M_{\beta}V^2(X) = \Ex[V^2(X')\mid X] = \Ex[\Ex(V^2(X')\bigm| |X'|=n'+1)\mid X]
\]
and $\Ex(|X'|=n'+1\mid X)\leq n + 1+ \eta T$ with $V(X) = n+1$, where $|X'|$ denotes the number of states in the trajectory $X'$.
For the trajectory $X' = (\bar{\tt},\SS')$ we have $$V(X') = 1 + \sum_{i=0}^{n'-1}\Ind(\boldsymbol{S}'_i\neq \boldsymbol{S}'_{i+1}).$$ Therefore, we get
\begin{equation}\label{ind1}
	 \Ex V^2(X')=1 + 2\Ex\left[\sum_{i=0}^{n'-1}\Ind(\boldsymbol{S}'_i\neq \boldsymbol{S}'_{i+1})\right]+ \Ex\left[\sum_{i\neq j}\Ind(\boldsymbol{S}'_i\neq \boldsymbol{S}'_{i+1})\Ind(\boldsymbol{S}'_j\neq \boldsymbol{S}'_{j+1})\right].
\end{equation}

Applying Lemma 2 of \cite{Miasojedow2017} together with the definition~\eqref{def:P}, the definition of $\eta$ and assumptions on likelihoods $g_j(\SS_{j^{\tt}})$, for each $i = 0,\ldots, n'-1$ we obtain
 \begin{equation*}
 \Pr\left(\boldsymbol{S}'_i=\s \mid \boldsymbol{S}'_{i+1}=\s\right) \geq \delta_i > 0.
\end{equation*}
This is a lower bound for the backward transition probability used by the FFBS algorithm. An analogous inequality for the forward transition probability is also true. Hence,
\begin{equation}\label{ind2}
\begin{split}
     \Ex\left(\Ind(\boldsymbol{S}'_i\neq \boldsymbol{S}'_{i+1})\right) & = \Ex\left[\Ex(\Ind(\boldsymbol{S}'_i\neq \boldsymbol{S}'_{i+1}) \mid \boldsymbol{S}'_{i+1}) \right] = \\
     & = \Pr\left(\boldsymbol{S}'_i\neq \s \mid \boldsymbol{S}'_{i+1}=\s\right)\leq 1-\delta_i,
\end{split}
\end{equation}
which means that we can bound  the second term on the RHS of (\ref{ind1}) from above by~$2\sum_{i=0}^{n'-1}(1-\delta_i)$. Also, for each $i\neq j$ we have 
\begin{equation}\label{ind3}
	\Ind(\boldsymbol{S}'_i\neq \boldsymbol{S}'_{i+1})\Ind(\boldsymbol{S}'_j\neq \boldsymbol{S}'_{j+1}) \leq \Ind(\boldsymbol{S}'_i\neq \boldsymbol{S}'_{i+1}).
\end{equation}
Thus, using (\ref{ind2}) and \eqref{ind3}, the third term on the RHS of (\ref{ind1}) can be bounded by
\begin{equation*}
\begin{split}
    \Ex\left[\sum_{j=0}^{n'-1}\sum_{\substack{i=0\\ i\neq j}}^{n'-1}\Ind(\boldsymbol{S}'_i\neq \boldsymbol{S}'_{i+1}) \right] & = 
	\sum_{j=0}^{n'-1}\sum_{\substack{i=0,\\ i\neq j}}^{n'-1} \Ex\left[\Ex(\Ind(\boldsymbol{S}'_i\neq \boldsymbol{S}'_{i+1})\mid\boldsymbol{S}'_{i+1}) \right] = \\ 
	&= \sum_{j=0}^{n'-1}\sum_{\substack{i=0,\\ i\neq j}}^{n'-1} \Pr(\boldsymbol{S}'_i\neq \s\mid\boldsymbol{S}'_{i+1}=\s)\leq \\ 
	&\leq \sum_{j=0}^{n'-1}\sum_{\substack{i=0,\\ i\neq j}}^{n'-1} (1-\delta_i).
\end{split}
\end{equation*}
Combining both bounds we obtain
\begin{equation*}
    \begin{split}
        \Ex V^2(X') & \leq 1 + 2n'- 2\sum_{i=0}^{n'-1}\delta_i + n'(n'-1) - \sum_{j=0}^{n'-1}\sum_{\substack{i=0,\\ i\neq j}}^{n'-1}\delta_i \leq \\
        &  \leq (1-\delta)(n'+1)^2 + 1 \leq (1-\delta)(n+1) + b,
    \end{split}
\end{equation*}
where $\delta = \min\limits_i\{\delta_i\}\in(0,1)$ and $b$ is some finite constant. This finishes the proof of the~second part of the lemma.

The first inequality can be shown either analogously to the second one or by applying Jensen's inequality. Indeed,
\[
(M_{Q} V(X))^2 \leq M_{Q} V^2(X) \leq \rho_2 V^2(X) + b_2,
\]
which in turn implies that
\[
M_{Q} V(X) \leq \sqrt{\rho_2 V^2(X) + b_2}\leq \sqrt{\rho_2} V(X) + \sqrt{b_2},
\]
which concludes the proof.
\end{proof}

\begin{lemma}\label{lemma:integr}
If $\Ex[V(X_0)]<\infty$, then $ \sup\limits_{n\geq 1}M_{Q}V(X_n)<\infty.$ 
 If in addition $\Ex[V^2(X_0)]<\infty$, then $\sup\limits_{n\geq 1}M_{Q}V^2(X_n)<\infty.$
\end{lemma}
\begin{proof}
Recall that $X_n$ is the trajectory obtained after $n$-th iteration of Rao and Teh's algorithm starting from the trajectory $X_0$. As previously we can consider $\Ex V(X_{n+1})$ instead of $M_{Q} V(X_n)$. Hence, by the previous lemma we have
\begin{equation*}
    \Ex[V(X_{n+1})] = \Ex[\Ex[V(X_{n+1})\mid X_{n}]] = \Ex[M_{Q}V(X_{n})] \leq \rho\Ex V(X_{n}) + b,
\end{equation*}
where $\rho \in (0,1)$ and $b<\infty$. Then, by iterating this majorization procedure recursively we get
\begin{equation*}
    \Ex[V(X_{n+1})] \leq \rho^{n+1}\Ex V(X_{0}) + b\sum\limits_{i=1}^{n+1}\rho^i \leq \rho\Ex V(X_{0}) + \dfrac{b}{1-\rho}\;.
\end{equation*}
Since the RHS of this inequality does not depend on $n$, then $\Ex[V(X_{n+1})]$ is bounded by a~finite constant. This concludes the proof for the first bound. The second inequality can be shown analogously using the bound for $\sup_{Q\in\ccL}M_{Q}V^2(X)$ in Lemma \ref{drift}.
\end{proof}

\begin{lemma}\label{lem:Mdiff}
For any compact set $\ccL\subset\ccM$ there exists $C\in(0,\infty)$ such that for any~$Q,\tilde{Q}\in\ccL$ and all trajectories $X$ we have
 \[
  \Vert M_Q(X, \cdot) - M_{\tilde{Q}} (X, \cdot)\Vert_V \leq C V(X)\Vert Q - \tilde{Q}\Vert.
 \]

\end{lemma}
\begin{proof}

For any $Q,\tilde{Q}\in\mathcal{L}$ the expression of interest (see the definitions \eqref{defKerJ}-\eqref{Vmeasure}) can be bounded by the sum of two terms as follows
\begin{equation*}
    \begin{split}
&\sup_{|f|\leq V }\left| M_Q f(\boldsymbol{t}, \boldsymbol{ S}) - M_{\tilde Q} f(\boldsymbol{t}, \boldsymbol{S})\right|= \sup_{|f|\leq V }\left|  M_Q^SM_Q^J f(\boldsymbol{t}, \boldsymbol{S}) -  M_{\tilde Q}^SM_{\tilde Q}^J f(\boldsymbol{t}, \boldsymbol{ S})\right| \\
& \leq \sup_{|f|\leq V }\left|  M_Q^SM_Q^J f(\boldsymbol{t}, \boldsymbol{ S}) -  M_{ Q}^SM_{\tilde Q}^J f(\boldsymbol{t}, \boldsymbol{ S})\right| +\sup_{|f|\leq V }\left|  M_Q^SM_{\tilde Q}^J f(\boldsymbol{t}, \boldsymbol{S}) -  M_{\tilde Q}^SM_{\tilde Q}^J f(\boldsymbol{t}, \boldsymbol{S})\right|\\
&:=\II_1+\II_2.
\end{split}
\end{equation*}
We can bound $\II_1$ by
\begin{multline*}
\sum_{k_0=0}^\infty\cdots\sum_{k_{n-1}=0}^\infty  \int \prod_{i=0}^{n-1}\Ind\{t_{i}<v_{i,1}<\cdots<v_{i,k_{i}}<t_{i+1}\} \sum_{\boldsymbol{\bar S}} |f(\boldsymbol{\bar t},\boldsymbol{\bar S})| \times \\
\times \left|M^J_Q((\boldsymbol{t}, \boldsymbol{ S}),  (\boldsymbol{\bar t}, \boldsymbol{S^{v}})) M^S_Q((\boldsymbol{\bar t}, \boldsymbol{ S^{v}}),  (\boldsymbol{\bar t}, \boldsymbol{ \bar S}))- M^J_{\tilde Q}((\boldsymbol{t}, \boldsymbol{ S}),  (\boldsymbol{\bar t}, \boldsymbol{S^{v}})) M^S_Q((\boldsymbol{\bar t}, \boldsymbol{ S^{v}}),  (\boldsymbol{\bar t}, \boldsymbol{ \bar S}))\right|d\boldsymbol {v}.
\end{multline*}
Recall that $k_i$ denotes the number of virtual jumps on the interval $[t_{i}, t_{i+1})$. Since $|f|\leq V$ and for any possible $\boldsymbol{\bar S}$ we have $V(\boldsymbol{\bar t}, \boldsymbol{\bar S})\leq 1+n+\sum_{i=0}^{n-1}k_i$, and $\sum\limits_{\bar\SS}M_Q^S\left((\bar\tt,\SS^{\vv}),(\bar\tt,\bar\SS) \right) = 1$ (see \eqref{defKerS}), then we can further bound $\II_1$ by
\begin{equation}\label{boundKerJ}
\begin{split}
    \sum_{k_0=0}^\infty\cdots&\sum_{k_{n-1}=0}^\infty  \int \prod_{i=0}^{n-1}\Ind\{t_{i}<v_{i,1}<\cdots<v_{i,k_{i}}<t_{i+1}\}\times \\
&\times\left(1+n+\sum_{i=0}^{n-1}k_i\right)\left|M^J_Q\left((\boldsymbol{t}, \boldsymbol{ S}),  (\boldsymbol{\bar t}, \boldsymbol{S^{v}})\right) - M^J_{\tilde Q}\left((\boldsymbol{t}, \boldsymbol{ S}),  (\boldsymbol{\bar t}, \boldsymbol{S^{v}})\right) \right|d\boldsymbol {v},
\end{split}
\end{equation}
Next, let $U$ denote the set of indices $u$ such that only $u$-th rows of $Q$ and $\tilde{Q}$ differ. Let $Q_1=Q$, $Q_{|U|}={\tilde{Q}}$ and for $u\in U$ we define the matrix $Q_{u+1}$ as $Q_u$ with the $u$-th row replaced by the corresponding row of $\tilde{Q}$. In particular, for $u\in U$ we have $Q_u({\tilde{\s}})\neq Q_{u+1}({\tilde{\s}})$ for a certain state ${\tilde{\s}}$ and for all states $\s \neq {\tilde{\s}}$ we have $Q_u(\s)=Q_{u+1}(\s)$. Then the expression under the integral can be bounded by
\begin{multline*}         
\left|M^J_Q\left((\boldsymbol{t}, \boldsymbol{ S}),  (\boldsymbol{\bar t}, \boldsymbol{S^{v}})\right) - M^J_{\tilde Q}\left((\boldsymbol{t}, \boldsymbol{ S}),  (\boldsymbol{\bar t}, \boldsymbol{S^{v}})\right) \right|
\leq \\ \leq \sum_{u}\left|M^J_{Q_u}\left((\boldsymbol{t}, \boldsymbol{ S}),  (\boldsymbol{\bar t}, \boldsymbol{S^{v}})\right) - M^J_{Q_{u+1}}\left((\boldsymbol{t}, \boldsymbol{ S}),  (\boldsymbol{\bar t}, \boldsymbol{S^{v}})\right) \right|.
\end{multline*}
Each term in the last sum can be expressed in the form

\begin{equation*}
\begin{split} 
&\left|M^J_{Q_u}((\boldsymbol{t}, \boldsymbol{ S}),  (\boldsymbol{\bar t}, \boldsymbol{S^{v}})) - M^J_{Q_{u+1}}((\boldsymbol{t}, \boldsymbol{S}),  (\boldsymbol{\bar t}, \boldsymbol{S^{v}})) \right| =\\
&\quad=\left|\prod_{i=0}^{n-1}\left[\left(\eta-Q_u(\boldsymbol{S}_i)\right) \left(t_{i+1}-t_i\right)\right]^{k_i}e^{-\left(\eta-Q_u(\boldsymbol{S}_i)\right) \left(t_{i+1}-t_i\right)}\right. - \\
&\quad\quad- \left. \prod_{i=0}^{n-1}\left[\left(\eta-Q_{u+1}(\boldsymbol{S}_i)\right) \left(t_{i+1}-t_i\right)\right]^{k_i}e^{-\left(\eta-Q_{u+1}(\boldsymbol{S}_i)\right) \left(t_{i+1}-t_i\right)}\right| =
\\
&\quad=\prod_{i=0}^{n-1}\left(t_{i+1}-t_i\right)^{k_i}\prod_{\substack{i=0\\\boldsymbol{S}_i\neq {\tilde{\s}}}}^{n-1}\left(\eta-Q_u(\boldsymbol{S}_i)\right)^{k_i}e^{-\left(\eta-Q_{u}(\boldsymbol{S}_i)\right) \left(t_{i+1}-t_i\right)}\times \\ &\quad\times\left|\left(\eta-Q_u({\tilde{\s}})\right)^{\sum\limits_{\boldsymbol{S}_i= {\tilde{\s}}}k_i}e^{-\left(\eta-Q_{u}({\tilde{\s}})\right)\sum\limits_{\boldsymbol{S}_i= {\tilde{\s}}} \left(t_{i+1}-t_i\right)} - \left(\eta-Q_{u+1}({\tilde{\s}})\right)^{\sum\limits_{\boldsymbol{S}_i= {\tilde{\s}}}k_i}e^{-\left(\eta-Q_{u+1}({\tilde{\s}})\right)\sum\limits_{\boldsymbol{S}_i= {\tilde{\s}}} \left(t_{i+1}-t_i\right)}\right|.
\end{split}
\end{equation*}         

Now let $x = \eta-Q_{u}({\tilde{\s}})$, $ y = \eta-Q_{u+1}({\tilde{\s}})$, $ a = \sum\limits_{\SS_i = {\tilde{\s}}}k_i$ and $ b = \sum\limits_{\SS_i = {\tilde{\s}}}(t_{i+1}-t_i)$. Let also $r = \max(x,y)$. Hence, we need to bound the expression $|x^ae^{-bx} - y^ae^{-by}|$ for some $x,y,a,b > 0$. From Lagrange's mean value theorem we have
\begin{equation*}\label{sup1}         
\left|x^ae^{-bx} - y^ae^{-by}\right| \leq \sup_{z\in (x,y)}\left|\frac{d}{d z}(z^ae^{-bz})\right| |x-y| = \sup_{z\in (x,y)}\left|az^{a-1}e^{-bz}-bz^ae^{-bz}\right| |x-y|.
\end{equation*}
Next we can bound the above supremum by
\begin{equation}\label{sup2}
\sup_{z\in (x,y)}\left|az^{a-1}e^{-bz}-bz^ae^{-bz}\right| \leq \sup_{z\in (x,y)} \max(az^{a-1}e^{-bz}, bz^ae^{-bz}).
\end{equation}
Let us assume first that the first expression of (\ref{sup2}) is the maximum, then we obtain
\begin{equation}\label{ineqa}         
az^{a-1}e^{-bz} \leq a{\tilde{\eta}}^{a-1}e^{-b{\tilde{\eta}}},
\end{equation}
where ${\tilde{\eta}} = \min(r,\frac{a-1}{b})$. In the case where the second expression is the maximum we obtain analogously $bz^{a}e^{-bz} \leq b{\tilde{\eta}}^{a}e^{-b{\tilde{\eta}}}$,
where ${\tilde{\eta}} = \min(r,\frac{a}{b})$. Using the first assumption with the corresponding inequality (\ref{ineqa}), the fact that $a\leq \sum\limits_{i=0}^{n-1}k_i$ and the fact that the~set of indices $U$ is finite, we can bound $\II_1$
by
\begin{equation*}
\begin{split}
    \sum_{u}\sum_{k_0=0}^\infty\cdots\sum_{k_{n-1}=0}^\infty & \int \prod_{i=0}^{n-1}\Ind\{t_{i}<v_{i,1}<\cdots<v_{i,k_{i}}<t_{i+1}\} \prod_{\substack{i=0\\\boldsymbol{S}_i\neq {\tilde{\s}}}}^{n-1}\left(\eta-Q_u(\boldsymbol{S}_i)\right)^{k_i}e^{-\left(\eta-Q_{u}(\boldsymbol{S}_i)\right) \left(t_{i+1}-t_i\right)}\\ &\times\prod_{\substack{i=0\\\boldsymbol{S}_i= {\tilde{\s}}}}^{n-1}{\tilde{\eta}}^{k_i}e^{-{\tilde{\eta}} \left(t_{i+1}-t_i\right)}\cdot \sum_{i=0}^{n-1}k_i
\left(1+n+\sum_{i=0}^{n-1}k_i\right)\left|Q_{u+1}({\tilde{\s}})-Q_{u}({\tilde{\s}})\right|d\boldsymbol {v} \leq\\
& \leq\sum_{u}\left|Q_{u+1}({\tilde{\s}})-Q_{u}({\tilde{\s}})\right|\cdot \Ex\left[\sum_{i=0}^{n-1}k_i
\left(1+n+\sum_{i=0}^{n-1}k_i\right)\right] \leq\\
&\leq |U| \lVert Q - {\tilde{Q}} \rVert (\eta T+ n\eta T + \eta T + (\eta T)^2) = \\ &= |U| \lVert Q - {\tilde{Q}} \rVert (\eta T(n+2) + (\eta T)^2) = C_1(n) \lVert Q - {\tilde{Q}} \rVert,
\end{split}
\end{equation*}
where $C_1(n)$ is a certain linear function of $n.$

Using the same technique and the fact that $b \leq T$ we can bound $\II_1$ by
\begin{align*}     
\sum_{u}\left|Q_{u+1}({\tilde{\s}})-Q_{u}({\tilde{\s}})\right|\cdot \Ex\left[T
\left(1+n+\sum_{i=0}^{n-1}k_i\right)\right] &\leq T|U| \lVert Q - {\tilde{Q}} \rVert\cdot \Ex\left(1+n+\sum_{i=0}^{n-1}k_i\right) = \\ = T|U| \lVert Q - {\tilde{Q}} \rVert (1+n+\eta T) & = C_2(n) \lVert Q - {\tilde{Q}} \rVert,
\end{align*}
where $C_2(n)$ is a certain linear function of $n.$

Now we can bound $\II_2$ in a similar way as we did for~$\II_1$ 
by an expression similar to~\eqref{boundKerJ}, namely by
\begin{multline*}
 \sum_{k_0=0}^\infty\cdots\sum_{k_{n-1}=0}^\infty  \int \prod_{i=0}^{n-1}\Ind\{t_{i}<v_{i,1}<\cdots<v_{i,k_{i}}<t_{i+1}\} \left(1+n+\sum_{i=0}^{n-1}k_i\right)\times\\  \times M^J_{\tilde Q}((\boldsymbol{t}, \boldsymbol{ S}),  (\boldsymbol{\bar t}, \boldsymbol{S^{v}}))\sum_{\boldsymbol{\bar S}} \left|M_Q^S((\boldsymbol{\bar t}, \boldsymbol{ S^v}),  (\boldsymbol{\bar t}, \boldsymbol{\bar S})) - M_{\tilde Q}^S((\boldsymbol{\bar t}, \boldsymbol{ S^v}),  (\boldsymbol{\bar t}, \boldsymbol{\bar S})) \right| d\boldsymbol {v}.
\end{multline*}
Before we continue, for any possible skeleton $\SS$ let us denote a few auxiliary functions
\begin{align*}
L_Q(\SS)&=\nu (\SS_0)\prod_{i=1}^{n}P(\SS_{i-1},\SS_i)\prod_{j=1}^{k}g_j(\SS_{k_j}),\\
R_Q &= \sum_{\SS}L_Q(\SS),\ \ H_Q(\SS)=\frac{L_Q(\SS)}{R_Q}.
\end{align*}
Therefore, we need to obtain the bound for
\begin{equation}\label{boundHQ}
\begin{split}
    \sum_{\SS}\left|H_Q(\SS)-H_{\tilde{Q}}(\SS)\right| & = \sum_{\SS}\left|\frac{L_Q(\SS)}{R_Q}-\frac{L_{\tilde{Q}}(\SS)}{R_{\tilde{Q}}}\right|\leq \\
    & \leq \sum_{\SS}\frac{\left|L_Q(\SS)-L_{\tilde{Q}}(\SS)\right|}{R_Q} + \sum_{\SS}L_{\tilde{Q}}(\SS)\left|\frac{1}{R_Q}-\frac{1}{R_{\tilde{Q}}}\right|.
\end{split}
\end{equation}
The initial distribution $\nu$ and likelihoods $g_j$ for all $j$ are the same for different intensity matrices $Q$. Using this fact and ``triangle'' inequality for two products of positive numbers
\begin{equation*}
\left| \prod_{j=1}^{n}x_j -  \prod_{j=1}^{n}y_j\right| \leq \sum_{j=1}^{n} \left|x_j-y_j\right|\prod_{i=1}^{j-1}x_i\prod_{i=j+1}^{n}y_i,
\end{equation*}
with $x_i = P(\SS_{i-1},\SS_i)$ and $y_i = \tilde{P}(\SS_{i-1},\SS_i)$, where $\tilde{P}$ is defined by $\tilde{Q}$ in the same way as~$P$ defined by $Q$ (see \eqref{def:P}), we obtain the inequality
\begin{multline*}
\sum_{\SS}\left|L_Q(\SS)-L_{\tilde{Q}}(\SS)\right| \leq \sum_{\SS}\frac{{\tilde{C}}^k}{\eta}\lVert Q - {\tilde{Q}} \rVert \sum_{j=1}^{n}\prod_{i=1}^{j-1}P(\SS_{i-1}, \SS_i)\prod_{i=j+1}^{n}{\tilde{P}}(\SS_{i-1}, \SS_i) \leq \\ 
\leq \frac{{\tilde{C}}^k}{\eta}\lVert Q - {\tilde{Q}} \rVert \sum_{j=1}^{n}\sum_{\SS_j}\sum_{\SS_{<j}}\sum_{\SS_{>j}}\prod_{i=1}^{j-1}P(\SS_{i-1}, \SS_i)\prod_{i=j+1}^{n}{\tilde{P}}(\SS_{i-1}, \SS_i) \leq \frac{{\tilde{C}}^k}{\eta}\lVert Q - {\tilde{Q}} \rVert\cdot n \cdot |S|.
\end{multline*}
Here we used the assumption that all likelihoods are bounded $C\leq g_j\leq {\tilde{C}}$ for $ 1\leq j\leq k$, and we locally denoted the number of all possible states of the process  as $|S|$. Note, that the sum over all possible skeletons $\SS$ was divided into three sums: the first one is over all possible states of $\SS_j$, the second - over all possible states of $\SS_1,\ldots,\SS_{j-1}$ and the last one is over all possible states $\SS_{j+1},\ldots,\SS_n$. Applying again bounds on the likelihoods we easily obtain that for any $Q$
\begin{equation*}
    \frac{1}{\tilde{C}^k} \leq \frac{1}{R_Q} \leq \frac{1}{C^k},
\end{equation*}
which leads to the fact that the first expression in \eqref{boundHQ} is bounded from above by $C_3(n)\|Q-\tilde{Q}\|$, where $C_3(n)$ is some linear function of $n$. The second expression in \eqref{boundHQ} is bounded by
\[
\tilde{C}^k \left|\frac{1}{R_Q}-\frac{1}{R_{\tilde{Q}}}\right| \leq \tilde{C}^k\frac{\sum_{\SS}\left|L_Q(\SS)-L_{\tilde{Q}}(\SS)\right|}{R_Q R_{\tilde{Q}}}.
\]
Applying two previously obtained inequalities we derive the bound $C_4(n)\|Q-\tilde{Q}\|$, where $C_4(n)$ is some linear function of $n$. Now combining all obtained bounds for $\II_1$ and $\II_2$ we conclude the proof.
\end{proof}

\begin{lemma}\label{lemma: D_V}
 For the measurable function $V:\mathbb{X}\to[1,+\infty)$ let us denote by
 \[
  D_V(\beta,\beta\p)= \sup_X\frac{  \Vert M_\beta(X,\cdot) - M_{\beta\p} (X,\cdot)\Vert_V }{V(X)}
 \]
the $V$-variation of the kernels $M_{\beta}$ and $M_{\beta'}$ and let $F_{\beta}:\mathbb{X}\to\R^+$ be the function such that $\sup_{\beta\in \ccK} |F_\beta|_V<\infty$. Moreover, define
\[
  \hat{F}_{\beta} = \sum_{n\geq 0} M_{\beta}^n(F_{\beta} - \pi_{\beta}(F_{\beta})).
\]
Then
\[
 |M_\beta\hat F_\beta - M_{\beta\p}\hat F_{\beta\p}|_{V}\leq C \{ D_V(\beta,\beta\p)+ |F_\beta-F_{\beta\p}|_{V}\}.
\]

\end{lemma}
\begin{proof}
 
 The proof follows the same arguments as the proof of the Lemma 4.2 in \cite{fort:moulines:priouret:2012} in the supplement materials to the paper. In addition, some references to the first papers using similar argumentation can be found there.
 
 First, we use the following decomposition of $ M^{k}_{\beta}f - M^{k}_{\beta'}f$ for any $k\geq 1$
\[
M_\beta^k f - M_{\beta'}^k f = \sum_{j=0}^{k-1} M_\beta^j \left( M_\beta - M_{\beta'} \right)
\left( M_{\beta'}^{k-j-1} f - \pi_{\beta'}(f) \right).
\]
By the Proposition 7 in \cite{Miasojedow2017} the sets $\{X: \left| V(X)\right|  < h\}$ are the small sets for any $h\in\mathbb{R}$.  Therefore, combining it with Lemma (\ref{drift}) we have by Theorem~9 in \cite{roberts2004general} that there exist constants $C_\beta$ and $\rho_\beta \in (0,1)$ such that 
$$\Vert{M_\beta^k(X,\cdot) - \pi_{\beta}} \Vert_V\leq C_\beta \rho_\beta^k V(X).$$
This property is called \textit{geometric ergodicity} of the kernel $M_\beta$ with invariant distribution~$\pi_\beta$. Hence, for any $k \geq 1$ and any trajectory~$X_\star$
 \begin{align}
 &\Vert{\pi_\beta - \pi_{\beta'}}\Vert_V\nonumber \\
 & \leq \Vert{\pi_\beta - M_\beta^k(X_\star,\cdot)}\Vert_V+ \Vert{M_\beta^k(X_\star,\cdot) - M_{\beta'}^k(X_\star,\cdot)}\Vert_V +
 \Vert{M_{\beta'}^k(X_\star,\cdot) - \pi_{\beta'}}\Vert_V\nonumber  \\
 & \leq \left( C_\beta \rho_\beta^k + C_{\beta'} \rho_{\beta'}^k \right) \ V(X_\star) \nonumber\\
 & \quad +
 \sup_{|{f}|_V\leq 1} \left| \sum_{j=0}^{k-1} M_\beta^j \left( M_\beta -
 M_{\beta'} \right) \left( M_{\beta'}^{k-j-1} f - \pi_{\beta'}(f) \right)(X_\star)
 \right|.\label{sup}
 \end{align}
   We can bound each summand from the sum on the RHS by
 \[
 \sup_{|f|_V\leq 1 }M_{\beta}^j\left|(M_{\beta} - M_{\beta'})(M_{\beta'}^{k-j-1}f-\pi_{\beta'}(f))(X_{\star})\right|.
 \]
Now let us denote $H = [M_{\beta'}^{k-j-1}f-\pi_{\beta'}(f)]$. Then the expression within the absolute value operator is bounded by
\begin{equation*}
    \begin{split}
        \sup_{|f|_V\leq 1 }&\sup_{|g|\leq|H|} \left|(M_{\beta} - M_{\beta'})g(X_{\star})\right| \leq  \sup_{|f|_V\leq 1 }|H|_V\sup_{|g|\leq V} |(M_{\beta} - M_{\beta'})g(X_{\star})| = \\
 & = \sup_{|f|_V\leq 1 } \sup_X\frac{|M_{\beta'}^{k-j-1}f(X)-\pi_{\beta'}(f)(X)|}{V(X)}\cdot \Vert{M_{\beta}(X_{\star},\cdot) - M_{\beta'}(X_{\star},\cdot)} \Vert_V \leq \\
 & \leq C_{\beta'} \rho_{\beta'}^{k-j-1}\cdot D_V(\beta,\beta\p)V(X_{\star}).
    \end{split}
\end{equation*}
Thus the last term in the (\ref{sup}) is bounded by
\begin{align*}
& C_{\beta'} \ D_{V}(\beta,\beta') \ \sum_{j=0}^{k-1} \rho^{k-j-1}_{\beta'} \ M_\beta^j V(X_\star) \leq \\
& \quad \leq C_{\beta'} \ D_{V}(\beta,\beta') \ \sum_{j=0}^{k-1} \rho^{k-j-1}_{\beta'} \ \left\{\pi_\beta(V) + C_\beta \rho_\beta^j V(X_\star) \right\} \leq \\
& \quad \leq \frac{C_{\beta'} }{1-\rho_{\beta'}} \ D_{V}(\beta,\beta') \left(\pi_\beta(V) + C_\beta V(X_\star) \right).
\end{align*}
Taking the limit as $k\rightarrow+\infty$ in the first term in (\ref{sup}) we obtain 
\begin{equation}\label{pibound}
	\Vert{\pi_\beta - \pi_{\beta'}}\Vert_V \leq \frac{C_{\beta'} }{1-\rho_{\beta'}} \ D_{V}(\beta,\beta') \left(\pi_\beta(V) + C_\beta V(X_\star) \right).
\end{equation}
Now from (7) in \cite{fort:moulines:priouret:2012} we have
\begin{multline}\label{dec}
M_\beta \hat{F}_\beta - M_{\beta'} \hat{F}_{\beta'} = \sum_{n \geq 1} \sum_{j=0}^{n-1}
\left(M_\beta^j - \pi_\beta \right)\left( M_\beta - M_{\beta'}\right) \left(
M_{\beta'}^{n-j-1} F_\beta - \pi_{\beta'}(F_\beta) \right) \\
- \sum_{n \geq 1} \{M_{\beta'}^n (F_{\beta'} -F_\beta) - \pi_{\beta'}(F_{\beta'} -F_\beta) \}
-\sum_{n \geq 1} \pi_\beta \{ M_{\beta'}^n F_\beta-\pi_{\beta'}(F_\beta)\}.
\end{multline}
 Let us consider the first term. Similarly to the previous step by $G$ we denote the operator $G=[M_{\beta'}^{n-j-1}F_\beta-\pi_{\beta'}(F_\beta)]$. Then we can bound
 \begin{equation*}
 \begin{split}
     	\left|\left(M_\beta^j - \pi_\beta \right)\right.&\left.\left( M_\beta - M_{\beta'}\right) \left(
 	M_{\beta'}^{n-j-1} F_\beta - \pi_{\beta'}(F_\beta) \right)(X)\right| \leq \\
 	& \leq \sup_{|g|\leq|G|}\left|\left(M_\beta^j - \pi_\beta \right)\left( M_\beta - M_{\beta'}\right)g(X)\right|\leq \\
 	& \leq |G|_V\sup_{|g|\leq V}\left|\left(M_\beta^j - \pi_\beta \right)\left( M_\beta - M_{\beta'}\right)g(X)\right| \leq \\
 	& \leq |G|_V\sup_{|h|\leq\lVert M_\beta-M_{\beta'}\rVert_V} \left|\left(M_\beta^j - \pi_\beta \right)h(X)\right|\leq \\
 	& \leq|G|_V D_{V}(\beta,\beta')\sup_{|h|\leq V} \left|\left(M_\beta^j - \pi_\beta \right)h(X)\right|\leq\\
 	& \leq C_{\beta'}\rho_{\beta'}^{n-j-1} |F_{\beta}|_V D_{V}(\beta,\beta')\cdot C_\beta\rho_\beta^jV(X).
 \end{split}
 \end{equation*}
 For the second and third terms in (\ref{dec}) we obtain the bounds
 \begin{equation*}
 	\left|M_{\beta'}^n (F_{\beta'} -F_\beta)(X) - \pi_{\beta'}(F_{\beta'} -F_\beta)\right|\leq C_{\beta'}\rho_{\beta'}^nV(X)\lvert F_{\beta'}-F_\beta\rvert_V
 \end{equation*}
 and
 \begin{equation}\label{dec3}
 \begin{split}
      	\lvert\pi_\beta \{ M_{\beta'}^n F_\beta-\pi_{\beta'}(F_\beta)\}(X)\rvert & = \lvert(\pi_\beta - \pi_{\beta'}) \{ M_{\beta'}^n F_\beta-\pi_{\beta'}(F_\beta)\}(X)\rvert\leq \\ 
      	&\leq \lVert\pi_\beta - \pi_{\beta'}\rVert_V\lvert M_{\beta'}^n F_\beta(X)-\pi_{\beta'}(F_\beta)\rvert_V \leq \\
      	&\leq \lVert\pi_\beta - \pi_{\beta'}\rVert_VC_{\beta'}\rho_{\beta'}^n\lvert F_\beta \rvert_V.
 \end{split}
 \end{equation}
 Therefore, combining the inequalities (\ref{pibound}) -- (\ref{dec3}) we get
 \begin{equation*}
 \begin{split}
      	\lvert M_\beta \hat{F}_\beta(X) - M_{\beta'} \hat{F}_{\beta'}(X)\rvert &\leq \frac{C_{\beta'}C_\beta}{(1-\rho_{\beta'})(1-\rho_\beta)}|F_{\beta}|_V D_{V}(\beta,\beta')V(X) +\\
      	& + \frac{C_{\beta'}}{1-\rho_{\beta'}}V(X)\lvert F_{\beta'}-F_\beta\rvert_V +\\
      	& +\frac{C_{\beta'}}{(1-\rho_{\beta'})}\lvert F_\beta \rvert_V D_{V}(\beta,\beta') \left(\pi_\beta(V) + C_\beta V(X) \right).
 \end{split}
 	 \end{equation*}
 Thus, since $\sup_{\beta\in K} |F_\beta|_V<\infty$, there exists a positive constant $L_{\beta,\beta'}$ for which we have
 \[
 	\lvert M_\beta \hat{F}_\beta(X) - M_{\beta'} \hat{F}_{\beta'}(X)\rvert \leq L_{\beta,\beta'}V(X)(D_{V}(\beta,\beta') + \lvert F_{\beta'}-F_\beta\rvert_V ).
 \]
 This concludes the proof.
\end{proof}
The proof of the main Theorem \ref{thm: convergence} is based on Theorem \ref{thmajewski} which is obtained by combining Theorem 5.4 and Proposition 5.5 of \cite{Majewski2018} with a slight adjustment of the notation to our context. For any compact convex set $\ccK$ by $$N_{\ccK}(x) = \{a\in\R^d : \langle a, z-x \rangle \text{ for all } z\in\ccK\}$$ we denote the normal cone to $\ccK$ at the point $x$. We consider an open set $\mathcal{B}\in\R^d$ and functions $f,g:\mathcal{B}\to\R$. We assume that $f$ is a continuously
differentiable function and also
for all $\beta\in\ccK$ its gradient satisfies
\begin{equation*}
    \nabla f(\beta) = \int_{\mathbb{X}}\Phi(\beta,X)\pi_{\beta}(dX)
\end{equation*}
for some probability measure $\pi_{\beta}$ and an integrable function $\Phi(\beta, X)$.

By $\{\beta_k, k\in\N\}$ we denote the sequence generated by the projected SPGD:
\begin{equation}\label{iterx}
    \beta_k \in {\prod}_{\ccK}\left(\prox_{\gamma_k,g}(\beta_{k-1}-\gamma_k \Phi(\beta_{k-1},\xi_k))\right),
\end{equation}
where $\xi_k$ is a random variable with $\pi_{\beta_{k-1}}$ distribution. Moreover, by $\{\delta_k, n\in\mathbb{N}\}$ we denote the gradient perturbation sequence defined by
\begin{equation*}
    \delta_k = \Phi(\beta_{k-1}, \xi_k) - \nabla\ell(\beta_{k-1}).
\end{equation*}

Moreover, for any measurable function $W:\mathbb{X}\to[1,+\infty)$ recall the definitions of $\|\mu\|_W$ and $|f|_W$ given in \eqref{Vmeasure} and \eqref{Vfunction}. Then we define $W$-variation of the kernels $M_{\beta}$ and~$M_{\beta'}$~by
\begin{equation*}
      D_W(\beta,\beta\p)= \sup_X\frac{  \Vert M_\beta(X,\cdot) - M_{\beta\p} (X,\cdot)\Vert_W }{W(X)}.
\end{equation*}

\begin{theorem}\label{thmajewski}
Denote
\[
\mathcal{S} = \{\beta\in\ccK :0\in \nabla f(\beta) + {\partial}g(\beta) - N_{\ccK}(\beta)\},
\]
where ${\partial}g$ is a subgradient of $g:\mathbb{B}\to\R$ (see e.g.~\cite{rockafellar1970}). Suppose that the set $(f + g)(\mathcal{S})$ has empty interior and $\sup\limits_{k\in\mathbb{N}}\|\delta_k\|\leq\infty$. We also make the following assumptions.
\renewcommand\labelenumi{(\theenumi)}
 \begin{enumerate}
     \item \label{as1} The function $g$ is convex, Lipschitz and bounded from below.
     \item \label{as2} The sequence of step sizes $\{\gamma_k\}$ satisfies $\gamma_k>0$ and $\lim_{k\to\infty}\gamma_k=0$ and
     $$\sum_{k=1}^{\infty}\gamma_k = \infty, \quad \sum_{k=1}^{\infty}|\gamma_k-\gamma_{k-1}|<\infty, \quad \sum_{k=1}^{\infty}\gamma_k^2<\infty.$$
     \item \label{as3} There exist constants $\rho\in[0,1)$ and $b<\infty$ and a measurable function $W:\mathbb{X}\to[1,+\infty)$ such that
     \[
     \sup\limits_{\beta\in\ccK}|\Phi(\beta,\cdot)|_{W^{1/2}}<\infty, \quad  \sup\limits_{\beta\in\ccK}M_{\beta}W\leq\lambda W +b. 
     \]
     In addition, for any $l\in(0,1]$ there exists $C<\infty$ and $\rho\in(0,1)$ such that for any~$X\in\mathbb{X}$
     \[
      \sup\limits_{\beta\in\ccK}\|M_{\beta}^n(X,\cdot) - \pi_{\beta}\|_{W^l}\leq C\rho^nW^l(X).
     \]
     \item \label{as4} The kernels $M_{\beta}$ and the stationary distributions $\pi_{\beta}$ are locally Lipschitz with
respect to $\beta$, i.e.~for any compact set $\ccK$ and any $\beta,\beta'\in\ccK$ there exists $C <\infty$ such that
\[
\sup\limits_{\beta\in\ccK}\|\Phi(\beta,\cdot) - \Phi(\beta',\cdot)\|_{W^{1/2}} +  D_{W^{1/2}}(\beta,\beta') \leq C\|\beta-\beta'\|.
\]
\item \label{as5} $\Ex[W(\xi_1)]<\infty$.
 \end{enumerate}
 Then the sequence $\{x_k, k\in\N\}$ generated by
iterations \eqref{iterx} converges to $\ccS$.
\end{theorem}

\begin{proof}[Proof of Theorem \ref{thm: convergence}.]
For better transparency of the proof we will use $m=1$ in \eqref{eq:MCMC_approx}, the generalization of the reasoning to the case of $m>1$ is straightforward. In our case the role of the function $f$ plays the negative log-likelihood $\ell(\beta)$ and the function $g$ is the $\ell_1$-penalty. Both functions satisfy the assumptions of Theorem \ref{thmajewski}.

Then, by the formula (\ref{def: gradient}) for the gradient of the negative log-likelihood function and the Fisher identity in \eqref{eq:Fisher} the function $\Phi(\beta_k, X)$  in our case takes the form
\begin{eqnarray}\label{Phi}
\Phi(\beta_k, X)= \sum_{w\in\V}\sum_{c\in\X_{-w}} \sum_{ s\not=s'}\left[-n_w(c;\; s,s\p)+t_w(c;\; s)\exp({\beta_{k,s,s\p}^{w}}^{\top} Z_w(c))\right] Z_w(c),
\end{eqnarray}
where we take $\beta_k$ from the $k$-th iteration of the p-SPGD algorithm as the parameter vector. Other components such as $n_w(c;\; s,s\p)$,\; $t_w(c;\; s)$ and $Z(c)$ correspond to a single trajectory $X$ of the Markov jump process. Integrating this function over all possible trajectories with respect to $\pi_{\beta} = p_{\beta}(Y\mid X)$ gives us the desired gradient $\nabla\ell(\beta)$ of the~negative log-likelihood.


In place of the function $W$ in the assumptions of Theorem \ref{thmajewski} we take the function~$V^2$. Note that the original Theorem 5.4 of \cite{Majewski2018} on the convergence of the~algorithm does not use the function $W$, instead it has an additional assumption on the~gradient perturbation sequence 
\begin{equation*}
    \delta_k = \Phi(\beta_{k-1}, X_k) - \nabla\ell(\beta_{k-1}), \quad k\in\N.
\end{equation*}
That assumption states that the sequence $\{\delta_k,k\in\N\}$ can be decomposed as $\delta_k = e_k^{\delta} + r_k^{\delta}$, where   $\{e_k^{\delta},k\in\N\}$ and  $\{r_k^{\delta},k\in\N\}$ are two sequences satisfying $ \lim_{k\to\infty}\|r_k^{\delta}\| =0 $ and the series $\sum_{k=1}^{\infty}\gamma_k e_k^{\delta} $ converges. However, Proposition 5.5 in \cite{Majewski2018} implies that by introducing Assumptions \eqref{as3}--\eqref{as5} we obtain the required decomposition of $\delta_k$. Therefore let us check the rest of the assumptions of Theorem \ref{thmajewski}.

Assumption \eqref{as2} on step-sizes is automatically satisfied. First we review Assumption \eqref{as3} with $W=V^2$, which consists of three conditions. The first condition that $\sup\limits_{\beta_k\in\ccK}|\Phi(\beta_k,\cdot)|_V$ is bounded is easy to check because for any trajectory $X$ the sum of the terms $n_w(c;\; s,s\p)$ is bounded by the total number of jumps $V(X)$, the sum of the terms $t_w(c;\; s)$ is bounded by the total observation time $T$ and vectors $\beta_k$ come from the compact set $\ccK$, which means that exponent is bounded by some constant. The second condition follows directly from Lemma \ref{drift}.
The last condition representing geometric ergodicity was shown in Lemma~\ref{lemma: D_V}.

In our setting Assumption \eqref{as4} takes the form
\[
D_V(\beta,\beta') + |\hat{\Phi}(\beta,\cdot)-\hat{\Phi}(\beta',\cdot)|_{V}\leq C\| \beta - \beta'\|
\]
for some constant $C$. We obtain it by combining Lemma \ref{lem:Mdiff} and the trivial fact that $|\Phi(\beta,\cdot) - \Phi(\beta',\cdot)|_V\leq C\| \beta - \beta'\|$ for some constant $C$.

In the course of the proof of the mentioned above decomposition \cite{Majewski2018} used the following property of the function $W$, which needs to be checked as well. For any trajectory $\xi_k$  under the assumption $\Ex W(\xi_0)<\infty$ there holds $\sup\limits_{k\geq 1}\Ex[W(\xi_k)]<\infty$. In~our case we can obtain the same property. Assuming $\Ex V^2(X_0)<\infty$ by Lemma \ref{lemma:integr} we have that $\sup\limits_{k\geq 1}\Ex[V^2(X_k)]<\infty.$ This concludes the proof of the theorem.
\end{proof}

\section{Numerical results}

In this section we describe the details of implementation of the proposed algorithm as well as the results of experimental studies.

\subsection{Details of implementation}\label{implementationP}
We provide in details implementation of the proposed algorithm in practice. Recall that the optimization problem \eqref{def: hat_beta} is solved by the iterative algorithm called projected stochastic proximal gradient descent given in \eqref{def:prox_grad}. Instead of the gradient of the ne\-ga\-tive log-likelihood $\nabla\ell(\beta)$ we use its MCMC approximation $\Phi(\beta, X^1,\dots, X^m)$, where $X^1,\dots,X^m$ is a set of trajectories generated by Rao and Teh's scheme given in Section~\ref{RaoTeh}. The solution of~\eqref{def: hat_beta} depends on the choice of $\lambda$. As we mentioned in previous chapters, finding the ,,optimal'' parameter $\lambda$ and the threshold~$\delta$ is difficult in practice. In this case we also solve it using the same information criteria as in Chapter \ref{chapter: CTBN structure}, where again instead of the gradient of the negative log-likelihood we use its MCMC approximation.

The function $\Phi(\theta, X^1,\dots, X^m)$ is an average of the functions $\Phi(\theta, X^i)$ introduced in~\eqref{Phi} (recall that we use the symbol $\beta$ only for the true parameter vector and $\theta$ otherwise). Now, in the analogous way as we divided the optimization problem \eqref{minimizer} in Subsection \ref{implementation} we can divide the current one. Namely, for  fixed $w\in\V$ and $s,s\p\in\{0,1\}$ with $s\not=s\p$, the corresponding summand in $\Phi(\theta, X^1,\dots, X^m)$ is a function which depends on the vector $\theta$ restricted only to its coordinate vector $\tssw$ (see notation \eqref{beta}). So, for each triple $w$ and $s\neq s\p$ we can solve the problem separately. Let us denote these summands of $\Phi(\theta, X^1,\dots, X^m)$ as $\Phi_{s,s'}^w(\tssw)$. 
 
Hence, in the current implementation we can use the scheme from Subsection \ref{implementation}. Namely, we start with computing a sequence of minimizers on the grid, i.e.~for any triple $w\in\V$, $s \neq s\p$ we create a finite sequence $\{\lambda_i\}_{i=1}^{N}$ uniformly spaced on the log scale, starting from the largest $\lambda_i$, which 
corresponds to the empty model. Next, for each value~$\lambda_i$ we compute the estimator $\hat{\beta}_{s,s\p}^{w}[i]$ of the vector ${\beta}_{s,s\p}^{w}$
 \begin{equation}
\label{lassoswP}
\hat{\beta}_{s,s\p}^{w}[i] = \argmin_{\theta_{s,s\p}^{w}} \left\{\Phi_{s,s'}^w(\tssw) +\lambda_i\|\theta_{s,s\p}^{w}\|_1\right\}.
\end{equation}

The notation $\hat{\beta}_{s,s\p}^{w}[i]$ means the $i$-th approximation of ${\beta}_{s,s\p}^{w}$. To solve \eqref{lassoswP} numerically for a given  $\lambda_i$ we use the SPGD algorithm without the projection onto the compact set. In~practice, the algorithm still converges well so we did not use the projection. The~final LASSO estimator $\hat{\beta}_{s,s\p}^{w}:=\hat{\beta}_{s,s\p}^{w} [{i^*}]$ is chosen using  the Bayesian Information Criterion (BIC) applied to the MCMC approximation of the gradient of the negative log-likelihood,~i.e.
\[
 i^*=\argmin_{1 \leq i \leq N} \left \{n \Phi_{s,s'}^w(\tssw)(\hat{\beta}_{s,s\p}^{w}[i])+\log(n)\Vert \hat{\beta}_{s,s\p}^{w}[i] \Vert_0\right\}.
\]
Here $\Vert \hat{\beta}_{s,s\p}^{w}[i]\Vert_0$ denotes the  number of non-zero elements of  $\hat{\beta}_{s,s\p}^{w}[i]$ and $n$ is the number of jumps in the trajectory generated by Rao and Teh's algorithm. In our simulations we use $N=100$.

 Finally, the threshold $\delta$ is obtained using the Generalized Information Criterion (GIC) as in Subsection \ref{implementation}, also applied to the MCMC approximation of the gradient of the~ne\-ga\-tive log-likelihood. For a prespecified sequence of thresholds $\mathscr{D}$ we calculate  
\[
\delta^* =\argmin_{\delta \in \mathscr{D}} \left \{n\Phi_{s,s'}^w( \hat{\beta}_{s,s\p}^{w,\delta})+\log(2d(d-1))\Vert \hat{\beta}_{s,s\p}^{w,\delta} \Vert_0\right\}\;,
\]
where $\hat{\beta}_{s,s\p}^{w,\delta}$ is the LASSO estimator $\hat{\beta}_{s,s\p}^{w}$ after thresholding with the level $\delta.$

\subsection{Simulated data}
We consider the chain model analogous to the model $M1$ in Subsection \ref{simulated}. All vertices have the ``chain structure'', i.e.~for any node, except for the first one, its set of parents contains only a previous node. Namely, we put  $\V = \{1,\dots,d\}$ and
    $\pa(k)= \{k-1\}$, if $k>1$ and $\pa(1)=\emptyset$. We construct  CIM in the same way as in Subsection \ref{simulated}. Namely, for the first node the intensities of leaving both states are equal to $5$. 
    For the rest of the nodes $k = 2,\ldots,d$,  we choose randomly $a\in\{0,1\}$ and we define
$Q_k(c,s,s\p)=9,$ if $s\not=|c-a|$ and $1$ otherwise. 
In other words, we choose randomly whether the  node  prefers to be at the same state as its parent ($a=0$) or not ($a=1$).

We consider two cases with the number of nodes equal to $d=5$ and $d=10$. So, the considered number of possible parameters of the model (the size of $\beta$) is $ 2d^2 = 50$ or $200$, respectively. We use $T=10$ for 5 nodes and $T=20$ for 10 nodes.  We replicate simulations $100$ times for each scenario. As the partial observation we take 100, 200 and 400 equally spaced points for 5 nodes and 200, 400 and 800 for 10 nodes. In Figure \ref{fig:partialres} we present averaged results of the simulations in terms of three quality measures
\begin{itemize}
    \item {\bf power}, which is a proportion of correctly selected edges;
    \item {\bf false discovery rate (FDR)}, which is a fraction of incorrectly  selected  edges among all selected edges;
    \item {\bf true model  (TM)}, which is an indicator whether the algorithm selected the true model without any errors.
\end{itemize}

In Figure \ref{fig:fullres} we provide the results of simulations for the same models in case of complete trajectories. We observe that the results of experiments confirm that the proposed method works in a satisfactory way. We observe that with increasing number of observation points results are close to the ones in case of complete data. The larger the~number of points the higher the power of the algorithm and tends to 1. The FDR is quite low in all cases. For the half simulations in case of 10 nodes and the time $T=20$ the algorithm discovers the true model when we choose a big enough number of observation points.
\begin{figure}[ht!]
\begin{center}
 \includegraphics[width=0.85\textwidth]{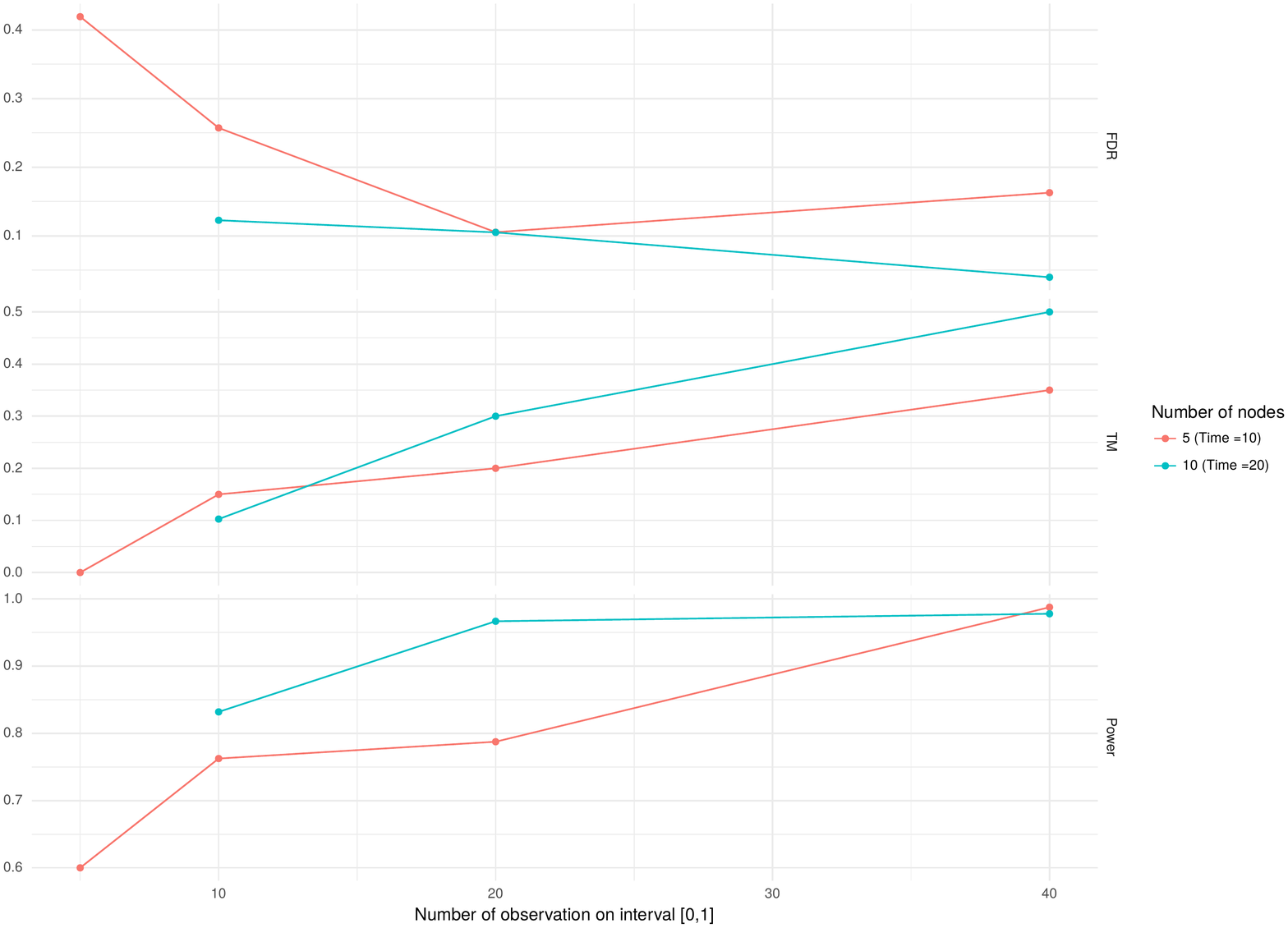}
\end{center}
\caption{Results of simulations for partially observed data.}
\label{fig:partialres}
\end{figure}
\begin{figure}[hb!]
\begin{center}
 \includegraphics[width=0.85\textwidth]{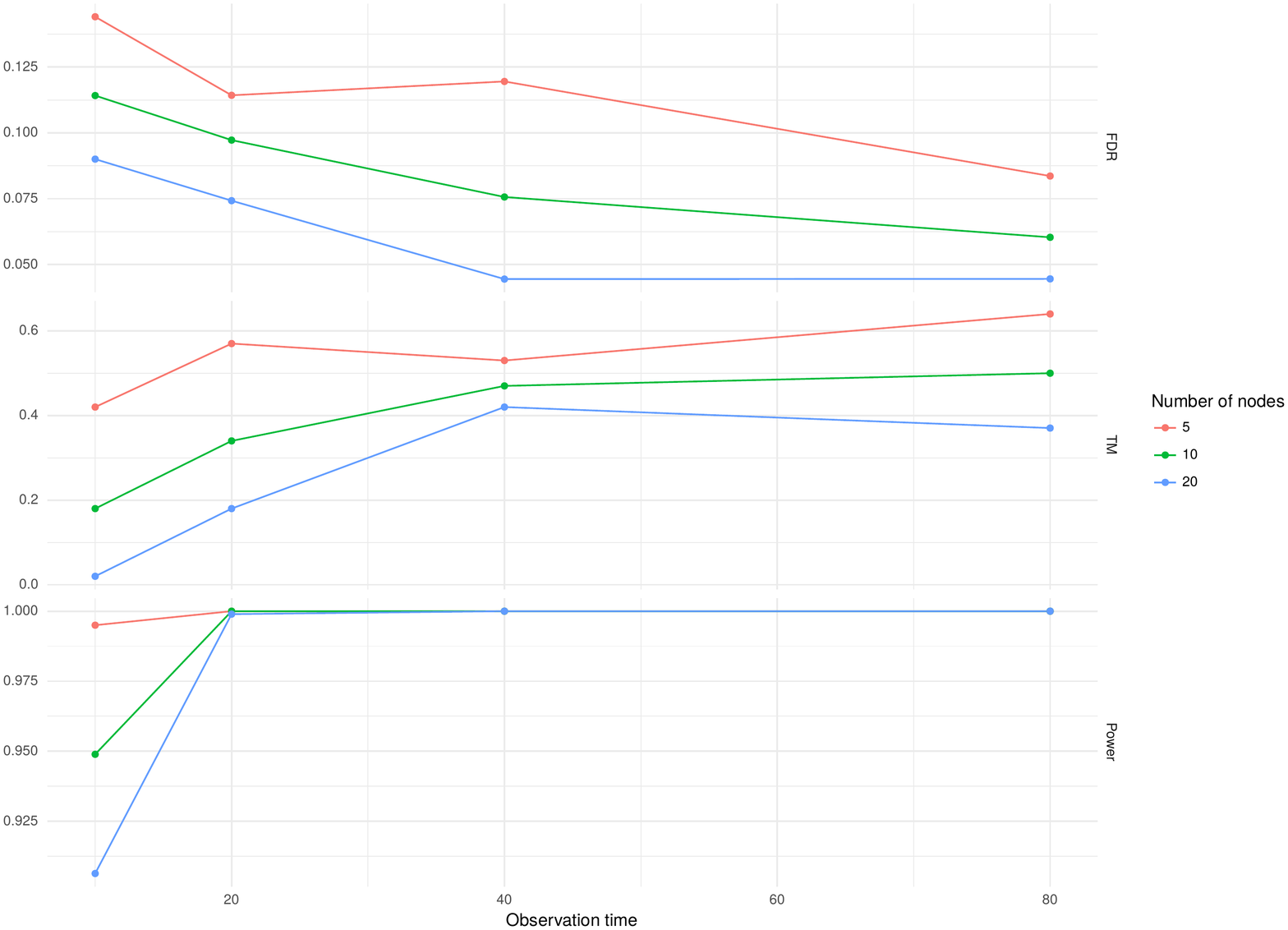}
\end{center}
\caption{Results of simulations for fully observed data.}
\label{fig:fullres}
\end{figure}

\section{FFBS Algorithm}\label{sec:ffbs}
For completeness of the proposed scheme we provide the description of the forward-filtering backward-sampling algorithm for discrete-time Markov chains taken from \cite{RaoTeh2013a} with a slightly changed notation. Earlier references for the FFBS algorithm can be found there as well.

Let $(\SS_0,\ldots,\SS_n)$ be a discrete-time Markov chain with a discrete state space $\ccX = \{1,\ldots,N\}$. Let $P$ be a transition matrix $P(\s,\s') = p(\SS_{j+1} = \s'\mid\SS_j = \s)$. Let $\nu$ be an~initial distribution over states at time point $0$ and let $Y = (Y_0,\ldots, Y_n)$ be a sequence of noisy observations with likelihoods $g_j(\s)= p(Y_j\mid\SS_{j^{\tt}}=\s)$. Given a set of
observations $Y = (Y_0,\ldots, Y_n)$, FFBS returns an independent posterior sample of the state vector.

Define $a_j(\s) = p(Y_0,\ldots,Y_{j-1},\SS_j = \s)$. From the Markov property, we have the~following recursion:
\[
a_{j+1}(\s') = \sum\limits_{\s}a_j(\s)g_j(\s)P(\s,\s').
\]
We calculate this for all possible states $\s'\in\ccX$ performing a forward pass. At the end of the forward pass we obtain the distribution
\[
    b_n(\s) = g_n(\s)a_n(\s) = p(Y, \SS_n = \s) \propto p(\SS_n = \s\mid Y)
\]
and sample $\SS_n$ from it. Next, note that
\begin{equation*}
    \begin{split}
        p(\SS_j = \s\mid \SS_{j+1} = \s', Y) & \propto  p(\SS_j = \s, \SS_{j+1}=\s', Y) = \\
        & = a_j(\s)g_j(\s)P(\s,\s')p(Y_{j+1},\ldots,Y_{n}\mid \SS_{j+1} = \s') \propto \\
        & \propto a_j(\s)g_j(\s)P(\s,\s'),
    \end{split}
\end{equation*}
where the second equality follows from the Markov property. This is also an easy distribution to sample from, and the backward pass of FFBS successively samples new elements of Markov chain from $\SS_{n-1}$ to $\SS_0$. The pseudocode for the algorithm is given below.
\newpage
\begin{algorithm}[H]
\SetAlgoLongEnd
\SetKwFor{For}{for}{}{}%
\SetAlgoBlockMarkers{\{}{\}}%
\AlgoDisplayBlockMarkers
\newcommand{\forcond}{$j=0$ \KwTo $n-1$}
\newcommand{\forcondd}{$j=n-1$ \KwTo $0$}
\SetStartEndCondition{ }{}{}%
\SetAlgoSkip{smallskip}
\bigskip
\KwIn{An initial distribution over states $\nu$, a transition matrix $P$, a sequence of noisy observations $Y = (Y_0,\ldots, Y_n)$ with likelihoods $g_j(\s)= p(Y_j\mid\SS_{j^{\tt}}=\s)$.}
\KwOut{A realization of the Markov chain $(\SS_0,\ldots,\SS_n) $} 

\hrulefill

 Initialize $a_0(\s) = \nu(\s)$. \;
 
 \For{\forcond}{
  $a_{j+1}(\s') = \sum\limits_{\s}a_j(\s)g_j(\s)P(\s,\s')$ \quad \text{ for } $\s'\in\ccX.$\;
 
 }
 \bigskip
 Sample $\SS_n \sim b_n(\cdot)$, \text{ where } $b_n = g_n(\s)a_n(\s)$.
 
 \For{\forcondd}{
Define $b_j(\s) = a_j(\s)g_j(\s)P(\s,\SS_{j+1})$;\;
 
Sample $\SS_{j}\sim b_j(\cdot)$.\;
 
 }
 \bigskip
 \KwRet{$(\SS_0,\ldots,\SS_n)$}
 
 \caption{The forward-filtering backward-sampling algorithm}
\end{algorithm}
   \chapter{Conclusions and discussion}\label{chapter:conclusions}

In this thesis we explored two types of probabilistic graphical models (PGM): Bayesian networks (BN) and continuous time Bayesian networks (CTBN).
First, we explained the~concept of PGMs and the motivation to study them with a few examples of successful applications.
Then, we discussed more thoroughly PGMs of interest describing the problems within both frameworks and provided necessary preliminaries.
In terms of contributions we were focused on structure learning, which is one of the most challenging tasks in the~process of exploring PGMs and is interesting in itself.
We also discussed other types of problems and reviewed some previously known results concerning these problems to provide some context.

The problem of structure learning for BNs is difficult due to the superexponential growth of the space of directed acyclic graphs (DAG) with the number of variables and also because the underlying graph needs to be acyclic. We solve this problem by dividing it into two tasks. First, we use a known method called partition MCMC to slice the~set of variables into layers where any variable in any layer can have parents only from the~previous layers and has at least one parent from the previous adjacent layer. Second, we find the arrows using the knowledge about the layers. In the case of continuous data we use the assumption that our network is a Gaussian Bayesian network and hence each variable is a~linear combination of its parents.
Thus, we solve the problem of finding arrows by finding the non-zero coefficients in the linear combination of all the variables from previous layers using Thresholded LASSO estimator. In the case of discrete and binary data we use the assumption that probability of each variable being equal to 1 is a~sigmoid function of a linear combination of its parents. Hence, again we solve the~problem of finding arrows by finding the non-zero coefficients in the linear combination of all the variables from previous layers using Thresholded LASSO estimator for logistic regression. Finally, for the discrete data where each variable has a finite state space we use a softmax function instead of the sigmoid function. We demonstrated theoretical consistency of LASSO and Thresholded LASSO estimators for the continuous model and showed their effectiveness on the benchmark Bayesian networks of different sizes and structure comparing the~proposed method to several existing methods for structure learning.

The problem of structure learning for CTBNs in the case of complete data is also reduced to solving the optimizational problem for the penalized with $\ell_1$-penalty maximum likelihood function. We assumed that a conditional intensity of a variable is a linear function of the states of its parents, which can be easily extended to a polynomial dependence.
Starting from the full graph we remove irrelevant edges and estimate parameters for exis\-ting ones simultaneously in case of LASSO estimator. In case of thresholded version of this LASSO estimator we only learn the structure. We proved the consistency of the proposed estimators and demonstrated coherence of theoretical results with numerical results from simulated data.

The last problem considered in the thesis was structure learning for CTBNs in the~case of incomplete data. The optimizational problem takes the same form as for complete data but we cannot write the likelihood function explicitly anymore. Instead of the negative log-likelihood function we used its Markov chain Monte Carlo approximation, where Markov chain was generated 
using Rao and Teh's algorithm. The optimizational problem itself was solved by projected stochastic proximal gradient descent algorithm. We proved the~convergence of this algorithm to the set of stationary points of the minimized function. We used the same assumption on conditional intensities as in the case of complete data. In practice to discover the arrows we used the thresholded version of the obtained estimator. We showed on a small simulated example that the quality of the proposed method is similar to the case of complete data and increases with the number of observed points per interval.

As for the future research we want to obtain similar theoretical results for Bayesian networks in the case of discrete data as we have obtained for the continuous data. In future we intend to perform more experiments and comparisons with existing approaches for all proposed methods. For some methods there are no open implementations or there are implementations in different programming languages, which makes it difficult to perform the comparison. The main goal was to show theoretical value of the proposed methods and show that the results of experiments are consistent with theory, which in our opinion was achieved.


    \bibliography{refs}
\end{document}